\documentclass[11pt]{article} %
\usepackage{newtxtext}
\usepackage[left=1.25in, top=1in, bottom=1in, right=1.25in]{geometry}

\usepackage{amsmath,amsfonts,bm}

\def\1{\bm{1}}

\DeclareMathAlphabet{\mathsfit}{\encodingdefault}{\sfdefault}{m}{sl}
\SetMathAlphabet{\mathsfit}{bold}{\encodingdefault}{\sfdefault}{bx}{n}

\def\sI{{\mathbb{I}}}

\newcommand{\R}{\mathbb{R}}

\DeclareMathOperator*{\argmin}{arg\,min}

\DeclareMathOperator{\sign}{sign}
\DeclareMathOperator{\Tr}{Tr}

\newcommand*{\one}{{\bm 1}}

\newcommand{\rank}{\mathrm{rank}}

\newcommand{\norm}[1]{\left\|#1\right\|}
\def\abs#1{\left| #1 \right|}

\newcommand{\diff}{\mathrm{d}}

\newcommand*{\E}{\mathbb{E}}
\newcommand*{\prob}{\mathbb{P}}
\newcommand{\eps}{\epsilon}
\newcommand{\dist}{\mathrm{dist}}

\newcommand{\xinit}{x_{\text{init}}}

\newcommand{\projn}[1][x]{P_{#1, \Gamma}}
\newcommand{\projt}[1][x]{P_{#1, \Gamma}^\perp}

\newcommand{\maxloss}[1][\rho]{L^{\textup{Max}}_{#1}}
\newcommand{\avgloss}[1][\rho]{L^{\textup{Avg}}_{#1}}
\newcommand{\ascloss}[1][\rho]{L^{\textup{Asc}}_{#1}}
\newcommand{\maxlimit}{S^{\textup{Max}}}
\newcommand{\avglimit}{S^{\textup{Avg}}}
\newcommand{\asclimit}{S^{\textup{Asc}}}
\newcommand{\stomaxlimit}{\widetilde S^{\textup{Max}}}
\newcommand{\stoavglimit}{\widetilde S^{\textup{Avg}}}
\newcommand{\stoasclimit}{\widetilde S^{\textup{Asc}}}

\newcommand{\stmaxloss}[1][\rho]{\E_k[L^{\textup{Max}}_{k,#1}]}
\newcommand{\maxsharpness}[1][\rho]{R^{\textup{Max}}_{#1}}
\newcommand{\avgsharpness}[1][\rho]{R^{\textup{Avg}}_{#1}}
\newcommand{\ascsharpness}[1][\rho]{R^{\textup{Asc}}_{#1}}
\newcommand{\stascsharpness}[1][\rho]{\E_k[R^{\textup{Asc}}_{k,#1}]}
\newcommand{\stmaxsharpness}[1][\rho]{\E_k[R^{\textup{Max}}_{k,#1}]}
\newcommand{\stavgsharpness}[1][\rho]{\E_k[R^{\textup{Avg}}_{k,#1}]}

\newcommand{\continuous}[1]{\mathcal{C}^{#1}}

\newcommand{\lspectraltwo}{\zeta}
\newcommand{\lspectralthree}{\nu}
\newcommand{\lspectralfour}{\Upsilon}
\newcommand{\phispectraltwo}{\xi}
\newcommand{\phispectralthree}{\chi}

\newcommand{\sIj}{\cap \sI_j}
\newcommand{\tx}{\tilde x}

\newcommand{\hx}{\hat x}
\newcommand{\ths}{\frac{\eta \lambda_1^2 }{2 - \eta \lambda_1}}
\newcommand{\best}{\eta \lambda_i \sqrt{\frac{ \frac{1}{2}\lambda_i^2 + \lambda_1^2}{1 - \eta \lambda_1}}}
\newcommand{\Pjd}{P^{(j:D)}}
\newcommand{\Pd}[1]{P^{( #1 :D)}}
\newcommand{\Ptd}{P^{(2:D)}}

\newcommand{\dl}[1]{\nabla L\left( #1 \right)}
\newcommand{\dpo}[1]{\partial \Phi \left( #1 \right)}
\newcommand{\dpt}[1]{\partial^2 \Phi \left( #1 \right)}
\newcommand{\dlt}[1]{\nabla^2 L\left( #1 \right)}
\newcommand{\proj}{P_{X,\Gamma}^{\perp}}
\newcommand{\ndl}[1]{\frac{\nabla L\left( #1 \right)}{\| \nabla L\left( #1 \right)\|}}
\newcommand{\ndli}[1]{\frac{\nabla L_k\left( #1 \right)}{\| \nabla L_k\left( #1 \right)\|}}
\newcommand{\dli}[1]{\nabla L_k\left( #1 \right)}
\newcommand{\dlti}[1]{\nabla^2 L_k\left( #1 \right)}

\newcommand{\nexts}{\mathrm{next}}

\newcommand{\tgf}{t_{\mathrm{GF}}}
\newcommand{\xgf}{x(t_{\mathrm{GF}})}
\newcommand{\tdec}{t_{\mathrm{DEC}}}
\newcommand{\xdec}{x(t_{\mathrm{DEC}})}
\newcommand{\xdecs}{x(t_{\mathrm{DEC2}})}
\newcommand{\tdecs}{t_{\mathrm{DEC2}}}
\newcommand{\xloc}{x'}
\newcommand{\tinv}{t_{\mathrm{INV}}}

\newcommand{\sIq}{\sI^{\mathrm{quad}}}
\newcommand{\sx}{x_{\mathrm{sur}}}
\newcommand{\sxq}{x'_{\mathrm{sur}}}
\newcommand{\talign}{t_{\mathrm{ALIGN}}}
\newcommand{\talignmid}{t_{\mathrm{ALIGNMID}}}
\newcommand{\tlocal}{t_{\mathrm{LOCAL}}}
\newcommand{\tphase}{0}

\newcommand{\indicatordec}[1]{\mathcal{A}(#1)}
\newcommand{\tphaseone}{t_{\mathrm{PHASE}}}

\usepackage[colorlinks,
            linkcolor=red,
            anchorcolor=blue,
            citecolor=blue,
            backref=page
            ]{hyperref}       %
\usepackage{url}            %
\usepackage{nicefrac}       %
\usepackage{xcolor}  
\usepackage{amsfonts, amsmath, amssymb, amsthm, bm, mathtools}
\usepackage{cancel}
\usepackage{cleveref}
\usepackage{enumitem}
\usepackage{natbib}
\usepackage{thmtools} 
\usepackage{thm-restate}
\usepackage{caption}
\usepackage{subcaption}
\usepackage{nicefrac}
\usepackage{xfrac}
\usepackage[disable]{todonotes}
\newtheorem{theorem}{Theorem}[section]
\newtheorem{definition}[theorem]{Definition}
\newtheorem{lemma}[theorem]{Lemma}
\newtheorem{corollary}[theorem]{Corollary}

\newtheorem{assumption}[theorem]{Assumption}

\newtheorem{condition}[theorem]{Condition}

\newtheorem{example}[theorem]{Example}
\crefdefaultlabelformat{#2#1#3}
\creflabelformat{equation}{#2#1#3}

\title{How Does Sharpness-Aware Minimization\\ Minimize Sharpness?}

\date{}
\author{Kaiyue Wen \\
Tsinghua University\\
\texttt{wenky20@mails.tsinghua.edu.cn} 
\and
Tengyu Ma  \\
Stanford University\\
\texttt{tengyuma@stanford.edu} \\
\and
Zhiyuan Li \\
Stanford University\\
\texttt{zhiyuanli@stanford.edu} \\
}

\begin{document}

\maketitle

\begin{abstract}

Sharpness-Aware Minimization (SAM) is a highly effective regularization technique for improving the generalization of deep neural networks for various settings. However, the underlying working of SAM remains elusive because of various intriguing approximations in the theoretical characterizations. SAM intends to penalize a notion of sharpness of the model but implements a computationally efficient variant; moreover, a third notion of sharpness was used for proving generalization guarantees. The subtle differences in these notions of sharpness can indeed lead to significantly different empirical results. This paper rigorously nails down the exact sharpness notion that SAM regularizes and clarifies the underlying mechanism. We also show that the two steps of approximations in the original motivation of SAM individually lead to inaccurate local conclusions, but their combination accidentally reveals the correct effect, when full-batch gradients are applied. Furthermore, we also prove that the stochastic version of SAM in fact regularizes the third notion of sharpness mentioned above, which is most likely to be the preferred notion for practical performance. The key mechanism behind this intriguing phenomenon is  the alignment between the gradient and the top eigenvector of Hessian when SAM is applied.

\end{abstract}

\section{Introduction}\label{sec:intro}
Modern deep nets are often overparametrized and have the capacity to fit even randomly labeled data~\citep{zhang2016understanding}. Thus, a small training loss does not necessarily imply good generalization.
Yet, standard gradient-based training algorithms such as SGD are able to find generalizable models.
Recent empirical and theoretical studies suggest that generalization is well-correlated with the sharpness of the loss landscape at the learned parameter~\citep{keskar2016large,dinh2017sharp,dziugaite2017computing,neyshabur2017exploring,jiang2019fantastic}. 
Partly motivated by these studies, \citet{foret2021sharpnessaware,wu2020adversarial,zheng2021regularizing,norton2021diametrical} propose to penalize the sharpness of the landscape to improve the generalization. %
We refer this method to \emph{Sharpness-Aware Minimization}~(SAM) and focus on the version of \citet{foret2021sharpnessaware} in this paper. 

Despite its empirical success, the underlying working of SAM remains elusive because of the various intriguing approximations made in its derivation and analysis. There are three different notions of sharpness involved --- SAM intends to optimize the first notion, the sharpness along the worst direction, but actually implements a computationally efficient notion, the sharpness along the direction of the gradient. But in the analysis  of generalization, a third notion of sharpness is actually used to prove generalization guarantees, which admits the first notion as an upper bound. The subtle difference between the three notions can lead to very different  biases (see \Cref{fig:demo} for demonstration).

More concretely, let $L$ be the training loss, $x$ be the parameter and $\rho$ be the \emph{perturbation radius}, a hyperparameter requiring tuning. The first notion corresponds to the following optimization problem~\eqref{eq:max_sharpness}, where we call $\maxsharpness(x) =\maxloss(x)-L(x)$ the \emph{worst-direction sharpness} at $x$.  SAM intends to minimize the original training loss plus the worst-direction sharpness at $x$. 
\begin{align}\label{eq:max_sharpness}
\min_{x}  \maxloss(x),\quad  \textrm{where}\quad \maxloss(x)	= \max_{\norm{v}_2\le 1} L(x+\rho v)\,.
\end{align}

However, even evaluating $\maxloss(x)$ is computationally expensive, not to mention optimization. Thus \citet{foret2021sharpnessaware,zheng2021regularizing} have introduced a second notion of sharpness, which approximates the worst-case direction in \eqref{eq:max_sharpness} by the direction of gradient, as defined below in \eqref{eq:asc_sharpness}.  We call $\ascsharpness(x)=\ascloss(x)-L(x)$ the \emph{ascent-direction sharpness} at $x$.
\begin{align}\label{eq:asc_sharpness}
\min_{x}  \ascloss(x),\quad  \textrm{where}\quad \ascloss(x)	= L\left(x+\rho \frac{\nabla L(x)}{\norm{\nabla L(x)}_2}\right)\,.
\end{align}
 
For further acceleration, \citet{foret2021sharpnessaware,zheng2021regularizing} omit the gradient through other occurrence of $x$ and approximate the gradient of ascent-direction sharpness by gradient taken after one-step ascent, \emph{i.e.}, $\nabla \ascloss(x) \approx \nabla L\left(x+\rho \frac{\nabla L(x)}{\norm{\nabla L(x)}_2}\right)$ and derive the update rule of SAM, where $\eta$ is the learning rate. 
\begin{align}\label{eq:sam}
\textrm{Sharpness-Aware Minimization (SAM):} \quad x(t+1) = x(t) -\eta  \nabla L\left(x+\rho \frac{\nabla L(x)}{\norm{\nabla L(x)}_2}\right)\,.
\end{align}

Intriguingly, the generalization bound of SAM upperbounds the generalization error by the third notion of sharpness, called \emph{average-direction sharpness}, $\avgsharpness(x)$ and defined formally below.
\begin{align}
\label{eq:avg_sharpness}
\avgsharpness(x)=\avgloss(x)-L(x), 	\textrm{ where }\avgloss(x)	= \E_{g \sim N(0,I)}L\left(x+\rho g/\|g\|_2 \right).
\end{align}

The worst-case sharpness is an upper bound of the average case sharpness and thus it is a looser bound for generalization error. In other words, according to the generalization theory in \citet{foret2021sharpnessaware,wu2020adversarial} in fact motivates us to directly minimize the average case sharpness (as opposed to the worst-case sharpness that SAM intends to optimize).
\begin{table}[t]
  \vspace{-0.5cm}
\centering\setlength{\tabcolsep}{4pt}
 \begin{tabular}{c| c c c} 
 \hline
\!\!Type of Sharpness-Aware Loss & \!\! Notation \!\!& \!\!Definition & \!\!\!\!  Biases (among minimizers) \\ 
 \hline
 Worst-direction  & $\maxloss$ & $\max_{\norm{v}_2\le 1} L(x+\rho v)$ & $\min_x \lambda_1(\nabla^ 2L(x))$~(Thm~\ref{thm:maxsharp}) \\ 
 Ascent-direction  & $\ascloss$ & $L\left(x+\rho \frac{\nabla L(x)}{\norm{\nabla L(x)}_2}\right)$ & $\min_x \lambda_{\textrm{min}}(\nabla^ 2L(x))$~(Thm~\ref{thm:ascsharp}) \\ 
  Average-direction  & $\avgloss$ & $\E_{g\sim N(0,I)} L(x+\rho \frac{g}{\norm{g}_2})$ & $\min_x \mathrm{Tr}(\nabla^ 2L(x))$~(Thm~\ref{thm:avgsharp}) \\
   \hline
 \end{tabular}
 \caption{\textbf{Definitions and biases of different notions of sharpness-aware loss}. The corresponding sharpness is defined as the difference between sharpness-aware loss and the original loss. Here $\lambda_1$ denotes the largest eigenvalue and $\lambda_{\textrm{min}}$ denotes the smallest \emph{non-zero} eigenvalue.
 }\label{tb:main} 
\end{table}

In this paper, we analyze the  biases introduced by penalizing these various notions of sharpness as well as the bias of SAM (\Cref{eq:sam}). Our analysis for SAM is performed for small perturbation radius $\rho$ and learning rate $\eta$ under the setting where the minimizers of loss form a manifold following the setup of~\citet{fehrman2020convergence,li2021happens}.  In particular, we make the following theoretical contributions.
\begin{enumerate}
	\item We prove that full-batch  SAM indeed minimizes worst-direction sharpness. (\Cref{thm:nsam})
	\item Surprisingly, when batch size is 1, SAM minimizes  average-direction sharpness. (\Cref{thm:1sam})\!
	\item We provide a characterization (\Cref{thm:explicit_bias_deterministic_main,thm:explicit_bias_stochastic_main}) of what a few sharpness regularizers bias towards among the minimizers (including all the three notions of the sharpness in \Cref{tb:main}), when the perturbation radius $\rho$ goes to zero. Surprisingly, both heuristic approximations made for SAM lead to inaccurate conclusions: (1) Minimizing worst-direction sharpness and ascent-direction sharpness induce different biases among minimizers, and (2) SAM doesn't minimize ascent-direction sharpness.

\end{enumerate}

The key mechanism behind this bias of SAM is the alignment between  gradient and the top eigenspace of Hessian of the original loss in the latter phase of training---the angle between them decreases gradually to the level of $O(\rho)$. 
It turns out that the worst-direction sharpness starts to decrease once such alignment is established (see \Cref{sec:proof_full}). 
Interestingly, such an alignment is not implied by the minimization problem \eqref{eq:asc_sharpness}, but rather, it is an implicit property of the specific update rule of SAM. Interestingly, such an alignment property holds for SAM with full batch and SAM with batch size one, but does not necessarily hold for the mini-batch case.

\section{Related Works}\label{sec:related_works}
 
\begin{figure}[t]  
 \vspace{-1cm}
\centering
\includegraphics[width=1\textwidth]{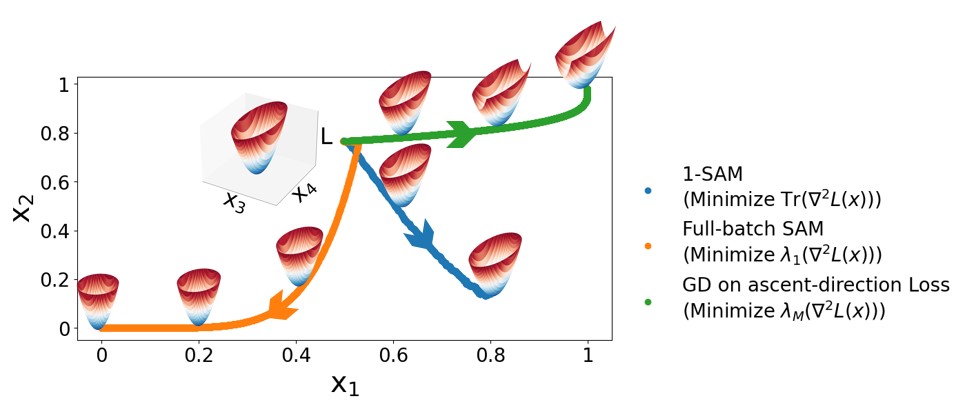}
\caption{\textbf{Visualization of the different biases of different sharpness notions on a 4D-toy example.} Let $F_1,F_2:\R^2\to \R^+$ be two positive functions satisfying that $F_1>F_2$ on $[0,1]^2$.  For $x \in\R^4$, consider loss $L(x) = F_1(x_1,x_2)x_3^2 + F_2(x_1,x_2)x_4^2$. The loss $L$ has a zero loss  manifold $\{x_3 = x_4 = 0\}$ of codimension $M=2$ and the two non-zero eigenvalues of $\nabla ^2 L$ of any point $x$ on the manifold are $\lambda_1(\nabla^2 L(x))= F_1(x_1,x_2)$ and $\lambda_2(\nabla^2 L(x)) = F_2(x_1,x_2)$. We test three optimization algorithms on this 4D-toy model with small learning rates. 
They all quickly converge to zero loss,  \emph{i.e.}, $x_3(t),x_4(t)\approx 0$, and after that  $x_1(t),x_2(t)$ still change slowly, \emph{i.e.}, moving along the zero loss manifold.
We visualize the loss restricted to $(x_3,x_4)$ as the 3D shape at various $(x_1,x_2)$'s where $x_1=x_1(t), x_2=x_2(t)$ follows the trajectories of the three algorithms. In other words, each of the 3D surface visualize the function $g(x_3,x_4) = L(x_1(t), x_2(t), x_3, x_4)$. 
As our theory predicts, (1) \textbf{Full-batch SAM} (\Cref{eq:sam}) finds the minimizer with \textbf{the smallest top eigenvalue}, $F_1(x_1,x_2)$;  (2) \textbf{GD on ascent-direction loss} $\ascloss$ (\Cref{eq:asc_sharpness}) finds the minimizer with \textbf{the smallest bottom eigenvalue}, $F_2(x_1,x_2)$; (3) \textbf{1-SAM} (\Cref{eq:1sam}) (with $L_0(x) = F_1(x_1,x_2)x_3^2$ and $L_1(x) = F_2(x_1,x_2)x_4^2$) finds the minimizer with \textbf{the smallest trace of Hessian}, $F_1(x_1,x_2)+F_2(x_1,x_2)$. See more details in \Cref{app:detail_caption}. 
}
\label{fig:demo}
\end{figure}

\paragraph{Sharpness and Generalization.} The study on the connection between sharpness and  generalization can be traced back to~\citet{hochreiter1997flat}. \citet{keskar2016large} observe a positive correlation between the batch size, the generalization error, and the sharpness of the loss landscape when changing the batch size. \citet{jastrzkebski2017three} extend this by finding a correlation between the sharpness
and the ratio between learning rate to batch size. \citet{dinh2017sharp} show that one can easily construct networks with good generalization but with arbitrary large sharpness by reparametrization. \citet{dziugaite2017computing,neyshabur2017exploring,wei2019data,wei2019improved} give theoretical guarantees on the generalization error using sharpness-related measures. \citet{jiang2019fantastic} perform a large-scale empirical study on various generalization measures and show that sharpness-based measures have the highest correlation with generalization.

\paragraph{Background on Sharpness-Aware Minimization.}~\citet{foret2021sharpnessaware,zheng2021regularizing} concurrently propose to minimize the loss at the perturbed from current parameter towards the worst direction to improve generalization. \citet{wu2020adversarial} propose an almost identical method for a different purpose, robust generalization of adversarial training. \citet{kwon2021asam} propose a different metric for SAM to fix the rescaling problem pointed out by \citet{dinh2017sharp}. \citet{liu2022towards} propose a more computationally efficient version of SAM. \citet{zhuang2022surrogate} proposes a variant of SAM, which improves generalization by simultaneously optimizing the surrogate gap and the sharpness-aware loss. \citet{zhao2022penalizing} propose to improve generalization by penalizing gradient norm. Their proposed algorithm can be viewed as a generalization of SAM. \citet{andriushchenko2022towards} study a variant of SAM where the step size of ascent step is $\rho$ instead of $\frac{\rho}{\norm{\nabla L(x)}_2}$. They show that for a simple model this variant of SAM has a stronger regularization effect when batch size is 1 compared to the full-batch case and argue that this might be the explanation that SAM generalizes better with small batch sizes.

In a concurrent work, \citet{bartlett2022dynamics} prove that on quadratic loss, the iterate of SAM (\Cref{eq:sam_quadratic}) and its gradient converges to the top eigenvector of Hessian, which is almost the same as our \Cref{thm:nsamquad}. Assuming such alignment for a general loss, the work of \citet{bartlett2022dynamics} shows that the largest eigenvalue of Hessian decreases in the next step. This paper also proves such a Hessian-gradient alignment for general loss functions (\Cref{lem:final_alignment}) and an end-to-end theorem showing that the largest eigenvalue of Hessian and worst-direction sharpness decrease along the trajectory of SAM (\Cref{thm:nsam}), which are not shown in \citet{bartlett2022dynamics}. Moreover, this paper also characterize implicit bias of stochastic SAM with batch size $1$, which is minimizing the average-direction sharpness, while \citet{bartlett2022dynamics} only considers the deterministic case.

\paragraph{Implicit Bias of Sharpness Minimization.} Recent theoretical works \citep{blanc2019implicit,damian2021label,li2021happens} show that SGD with label noise implicitly biased toward local minimizers with a smaller trace of Hessian under the assumption that the minimizers locally connect as a manifold. \citet{arora2022understanding} show that normalized GD implicitly penalizes the largest eigenvalue of the Hessian. \citet{ma2022multiscale} argues that such flatness driven phenomenon can also be caused by a multi-scale loss landscape. \citet{lyu2022understanding} show that GD with weight decay  on a scale invariant loss function implicitly decreases penalize the spherical sharpness, \emph{i.e.}, the largest eigenvalue of the Hessian evaluated at the normalized parameter. 

Another line of works study the sharpness minimization effect of  large learning rate assuming the (stochastic) gradient descent converges in the end of training, where the analysis is mainly based on linear stability~\citep{wu2018sgd,cohen2021gradient,ma2021linear,cohen2022adaptive}. Recent theoretical analysis~\citep{damian2022self,li2022analyzing} show that the sharpness minimization effect of large learning rate in gradient descent do not necessarily rely on the convergence assumption and linear stability via a four-phase characterization of the dynamics at the so-called Edge of Stability regime~\citep{cohen2021gradient}.

\paragraph{Comparison with~\citet{arora2022understanding}.} Our proof uses a similar framework as~\citet{arora2022understanding}. However, our analysis has its own difficulty for the following reasons.
First, \citet{arora2022understanding} only deal with the deterministic case, while our analysis extends to stochastic SAM as well (\Cref{sec:1sam}). Second, our analysis for the deterministic case is different from that of \citet{arora2022understanding} in the following two aspects. First, the alignment analysis is more complicated because we have two hyperparameters,learning rate $\eta$ and perturbation radius $\rho$, while \citet{arora2022understanding} only needs to deal with one hyperparameter, learning rate $\eta$.
Second, the mechanism of penalizing worst-direction sharpness is different, which can be seen from the dependency of the sharpness-reduction rate over learning rate $\eta$.  In \citet{arora2022understanding}, normalized GD reduces the sharpness via a second-order effect of GD and thus the sharpness is reduced by $O(\eta^2)$ per step. In our analysis, for fixed small perturbation radius $\rho$, the sharpness is reduced by $O(\rho^2 \eta)$ per step, which is linear in $\eta$.

\paragraph{Analyzing Discrete-time Dynamics via Continuous-time Approaches. } There is a long line of research that shows the trajectory of stochastic discrete iterations with \textit{decaying step size} eventually tracks the solution of some ODE  (see~\citet{kushner2003stochastic,borkar2009new,duchi2018stochastic} and the reference therein). However, those results mainly focus on the convergence property of the stochastic iterates (e.g., convergence to stationary points), while we are interested in characterizing the trajectory especially when the process is running for \textit{a long time} even after the iterate reaches the neighborhood of the manifold of stationary points. 

Recently there has been an effort of modeling the discrete-time trajectory of (stochastic) gradient methods by continuous-time approximations~\citep{su2014differential,mandt2017stochastic,li2017stochastic,li2019stochastic}. Notably, \citet{li2019stochastic} presents a general and rigorous mathematical framework to prove such continuous-time approximation. More specifically, \citet{li2019stochastic} proves for various stochastic gradient-based methods, the discrete-time weakly converges to the continuous-time one when LR $\eta\to 0$ in $\Theta(1/\eta)$ steps. The main difference between our results with these results (e.g., Theorem 9 in \cite{li2019stochastic}) is that we focus on a much longer training regime, \emph{i.e.}, $T= \Theta(\eta^{-1}\rho^{-2})$ steps where the previous continuous-time approximation results no longer holds throughout the entire training. As a result, their continuous approximation is only equivalent to the Phase I dynamics in our \Cref{thm:nsam,thm:1sam} and cannot capture the dynamics of SAM in Phase II, when the sharpness-reduction implicit bias happens. The latter requires a more fine-grained analysis to capture the effects of higher-order terms in $\eta$ and $\rho$ in SAM~\Cref{eq:sam}.

\section{Notations and Assumptions}
\label{notation}

For any natural number $k$, we say a function is  $\continuous{k}$ if it is $k$-times continuously differentiable and is $\overline{\mathcal{C}}^k$ if its $k$th order derivatives are locally lipschitz.  We say a subset of $\R^D$ is compact if  each of its open covers has a finite subcover. It is well known that a subset of $\R^D$ is compact if and only if it is closed and bounded.
For any positive definite symmetric matrix $A \in \R^{D\times D}$, define $\{\lambda_i(A), v_i(A)\}_{i \in [D]}$ as all its eigenvalues and eigenvectors satisfying $\lambda_1(A) \ge \lambda_2(A) ... \ge \lambda_D(A)$ and $\| v_i(A)\|_2 = 1$. 
 For any mapping $F$, we define $\partial F(x)$ as the Jacobian where $[\partial F(x)]_{ij} = \partial_j F_i(x)$. Thus the directional derivative of $F$ along the vector $u$ at $x$ can be written as $\partial F(x) u$. We further define the second order directional derivative of $F$ along the vectors $u$ and $v$ at $x$, $\partial^2 F(x)[u,v]$, $\partial (\partial F \cdot u)(x) v$, that is, the directional derivative of $\partial F \cdot u$ along the vector $v$ at $x$.

\begin{definition}[Differentiable Submanifold of $\R^D$]\label{defi:manifold} We call a subset $\Gamma\subset \R^D$ a $\continuous{k}$ submanifold of $\R^D$ if and only if for every $x\in\Gamma$, there exists a open neighborhood $U$ of $x$ and an invertible $\continuous{k}$ map $\psi:U\to \R^D$, such that $\psi(\Gamma\cap U) = (\R^n\times \{0\})\cap \psi(U)$.
	
\end{definition}

Given a $\continuous{1}$ submanifold $\Gamma$ of $\R^D$ and a point $x \in \Gamma$, define $\projn$ as the projection operator onto the manifold of the normal space of $\Gamma$ at $x$ and $\projt = I_D - \projn$. We fix our initialization as $\xinit$ and our loss function as $L: \R^D \to \R$. Given the loss function, its gradient flow is denoted  by mapping $\phi: \R^D \times [0, \infty) \to \R^D$. Here, $\phi(x, \tau)$ denotes the iterate at time $\tau$ of a gradient flow starting at $x$ and is defined as the unique solution of  $\phi(x,\tau) = x - \int_0^{\tau} \nabla L(\phi(x,t))dt$, $\forall x\in\R^D$.
We further define the limiting map $\Phi$ as $\Phi(x) = \lim_{\tau \to \infty} \phi(x,\tau)$, that is, $\Phi(x)$ denotes the convergent point of the gradient flow starting from $x$. When $L(x)$ is small, $\Phi(x)$ and $x$ are near. Hence in our analysis, we regularly use $\Phi(x(t))$ as a surrogate to analyze the dynamics of $x(t)$. \Cref{lem:property_of_Phi_away_manifold} is an important property of $\Phi$ from \citet{li2021happens} (Lemma C.2), which is repeatedly used in our analysis. For completeness, we attach its proof below.

\begin{lemma}
\label{lem:property_of_Phi_away_manifold}
For any $x$ at which $\Phi$ is defined and differentiable, we have that $\partial \Phi(x) \nabla L(x) = 0$. 
\end{lemma}
\begin{proof}[Proof of \Cref{lem:property_of_Phi_away_manifold}]
Since $\Phi$ is defined the limit map of gradient flow, it holds that for any $t \ge 0$, $\Phi(\phi(x,t))=\Phi(x)$.  Differentiating both sides at $t = 0$, we have $\partial \Phi(\phi(x,0))\frac{\partial \phi(x,t)}{\partial t} = 0$. The proof is completed by noting that $\frac{\partial \phi(x,t)}{\partial t} = -\nabla L (\phi(x,t))$ by definition of $\phi$.
 \end{proof}

Recent empirical studies have shown that  there are essentially no barriers in loss landscape between different  minimizers, that is, the set of minimizers are path-connected~\citep{draxler2018essentially,garipov2018loss}. Motivated by this empirical discovery, we make the assumption below following \citet{fehrman2020convergence,li2021happens,arora2022understanding}, which is theoretically justified by \citet{cooper2018loss} under a generic setting.

\begin{restatable}[]{assumption}{assumsmooth}
\label{assump:smoothness}
Assume loss $L:\R^D \to \R$ is  $\continuous{4}$, and there exists a $\continuous{2}$ submanifold $\Gamma$ of $\R^D$ that is a $(D-M)$-dimensional  for some integer $1 \le M \le D$, where for all $x \in \Gamma$, $x$ is a local minimizer of $L$ and $\mathrm{rank}(\nabla^2 L(x)) = M$.

\end{restatable}

The connectivity of the set of local minimizers implied by the manifold assumption above allows us to take limits of perturbation radius $\rho\to 0$ while still yield interesting and insightful implicit bias results in the end-to-end analysis. 
So far almost all analysis of implicit bias for general model parameterizations relies on Taylor expansion, \emph{e.g.}~\cite{blanc2019implicit,damian2021label,li2021happens,arora2022understanding}, so does the derivation of the SAM algorithm~\citet{foret2020sharpness,wu2020adversarial}. Thus it's crucial to consider small perturbation size $\rho$. On the contrary, if the set of global minimizers are a set of discrete points, then with small perturbation radius $\rho$, implicit bias of optimizers is not sufficient to drive the iterate from global minimum to the other one.

It can be shown that for a minimum loss manifold, the rank of Hessian plus the dimension of the manifold is at most the environmental dimension $D$, and thus our assumption about Hessian rank essentially says the the rank is maximal. This assumption is necessary for the analysis to guarantee the differentiability of $\Phi$. 

Though our analysis for the full-batch setting are performed under the general and abstract setting, \Cref{assump:smoothness}, our analysis for stochastic setting uses a more concrete one, \Cref{setting:1sam}, where we can prove that \Cref{assump:smoothness} holds. (see \Cref{thm:derive_manifold_assumption})

\begin{restatable}[Attraction Set]{definition}{attractionset}\label{defi:attraction_set}
Let $U$ be the attraction set of $\Gamma$ under gradient flow, that is, a neighborhood of $\Gamma$ containing all points starting from which gradient flow w.r.t. loss $L$ converges to some point in $\Gamma$, or mathematically, $U\triangleq  \{x \in \R^D | \Phi(x) \text{ exists and } \Phi(x) \in  \Gamma \}$. 
\end{restatable}

\Cref{assump:smoothness} implies that $U$ is open and $\Phi$ is $\overline{\mathcal{C}}^2$ on $U$~\citep[Lemma B.15]{arora2022understanding}. 	

By definition, $\Phi(x)=x$ for any $x\in\Gamma$. Differentiating this equality yields the following important lemma about the property of $\partial \Phi$ on manifold $\Gamma$.
\begin{lemma}[\citet{li2021happens}, Lemma 4.3]\label{lem:property_of_Phi_on_manifold}
For $x \in \Gamma$, $\partial \Phi(x) = \projt$, the orthogonal projection matrix onto the tangent space of $\Gamma$ at $x$. Since $d$   $\partial \Phi (x)\nabla ^2 L(x) =0$.
\end{lemma}

\paragraph{Implicit versus Explicit Bias.} 
If an algorithm or optimizer has a bias towards certain type of global/local minima of the loss over other  minima of the loss, and this bias is not encoded in the loss function, then we call such bias an \textit{implicit bias}. 
On the other hand, a  bias emerges as solely a consequence of successfully minimizing certain regularized loss regardless of the optimizers (as long as the optimzers minimize the loss), we say such bias is an \emph{explicit bias} of the regularized loss (or the regularizer).

As a concrete example, we will prove that full-batch SAM~(\Cref{eq:sam}) prefers local minima with certain sharpness property. The bias stems from the particular update rule of full-batch SAM~(\Cref{eq:sam}), and \textit{not} all optimizers for the intended target loss function $\ascloss$ (\Cref{eq:asc_sharpness}) has this bias. Therefore, it's considered as an implicit bias. As an example for explicit bias, all optimizers minimizing a loss combined with $\ell_2$ regularization will prefer model with smaller parameter norm and this is considered as an explicit bias of $\ell_2$ regularization.

\paragraph{Usage of $O(\cdot)$ Notation:} Our analysis assumes  small $\eta$ and $\rho$ while treating all other problem-dependent parameters as constants, such as the dimension of parameter space and the maximum possible value of derivatives (of different orders) of loss function $L$ and the limit map $\Phi$. In $O(\cdot),\Omega(\cdot),o(\cdot),\omega(\cdot),\Theta(\cdot)$, we hide all the dependency related to the problem, e.g., the (unique) initialization $\xinit$, the manifold $\Gamma$, compact set $\overline{U'}$ in \Cref{thm:explicit_bias_deterministic_main}, and the continuous time $T_3$ in \Cref{thm:nsam,thm:1sam}, and only keep the dependency on $\rho$ and $\eta$. For example, $O(f(\rho))$ is a placeholder for some function $g(\rho)$ such that there exists problem-dependent constant $C>0$, $\forall \rho>0, |g(\rho)|\le C|f(\rho)|$.  In informal equations such as \Cref{eq:taylor_expansion_of_Phi} in the proof sketch section, we are a bit more sloppy and hide dependency on $x(t)$ in $O(\cdot)$ notation as well. But these will be formally dealt with in the proofs. 

\paragraph{Ill-definedness of SAM with Zero Gradient.}  The update rule of SAM (\Cref{eq:sam,eq:1sam}) is  ill-defined when the gradient is zero. However, our analysis in \Cref{appsec:well_definedness} shows that when the stationary point of loss $L$, $\{x\mid\nabla L(x)=0\}$, is a zero-measure set, for any perturbation radius $\rho$, except for countably many learning rates, full-batch SAM is well-defined for almost all initialization and all steps~(\Cref{thm:nsam_well_definedness}). A similar result is shown for stochastic SAM if the stationary points of each stochastic loss form a zero-measure set~(\Cref{thm:1sam_well_definedness}). Thus SAM is generically well-defined. For the sake of rigorousness, when SAM encountering zero gradients,  we modify the algorithm via replacing the ill-defined normalized gradient by an arbitrary vector with unit norm and our analysis for implicit bias of SAM still holds.

\section{Explicit and Implicit Bias in the Full-Batch Setting}
\label{sec:nsam}

In this section, we present our main results in the full-batch setting. \Cref{sec:explicit_full} provides characterization of explicit bias of \emph{worst-direction}, \emph{ascent-dircetion}, and average-direction sharpness. In particular, we show that ascent-direction sharpness and worst-direction sharpness have different explicit biases. 
However, it turns out the explicit bias of ascent-direction sharpness is not the effective bias of SAM (that approximately optimizes the ascent-direction sharpness), because the particular implementation of SAM imposes additional, different biases, which is the main focus of \Cref{sec:implicit_full}. We provide our main theorem in the full-batch setting, that SAM implicitly minimizes the worst-direction sharpness, via characterizing its limiting dynamics as learning rate $\rho$ and $\eta$ goes to $0$ with a Riemmanian gradient flow with respect to the top eigenvalue of the Hessian of the loss on the manifold of local minimizers. In \Cref{sec:proof_full}, we sketch the proof of the implicit bias of SAM and identify a key property behind the implicit bias, which we call the \emph{implicit alignment} between the gradient and the top eigenvector of the Hessian.

\subsection{Worst- and Ascent-direction Sharpness Have Different Explicit Biases}
\label{sec:explicit_full}

In this subsection, we show that  the explicit biases of three notions of sharpness are all different under \Cref{assump:smoothness}. We first recap  the heuristic derivation of ascent-direction sharpness $\ascsharpness$.

The intuition of approximating $\maxsharpness$ by $\ascsharpness$ comes from the following Taylor expansions~\citep{foret2021sharpnessaware,wu2020adversarial}. Consider any compact set, for sufficiently small $\rho$, the following holds uniformly for all $x$ in the compact set:
\begin{align}\thickmuskip=0\thickmuskip
\maxsharpness(x) = \sup_{\norm{v}_2\le 1} L(x+\rho v)-L(x) = \sup_{\norm{v}_2\le 1} \bigl( \rho v^\top \nabla L(x) + \frac{\rho^2}{2}v^\top\nabla ^2 L(x)v + O(\rho^3)\bigr)\,, \label{eq:maxsharp}\\
\!\!\ascsharpness(x) \!= \! L\bigl(x+\rho \frac{\nabla L(x)}{\norm{\nabla L(x)}_2}\bigr)\!\!-\!L(x) 
= \!\! \rho \norm{\nabla L(x)}_2 \!+\! \frac{\rho^2}{2}  \frac{\nabla L(x)^\top \nabla ^2 L(x) \nabla L(x)}{\norm{\nabla L(x)}_2^2} \!+  \!O(\rho^3)\,. \label{eq:ascsharp}
\end{align}

Here, the preference among the local or global minima is what we are mainly concerned with.
Since $\sup_{\norm{v}_2\le 1} v^\top \nabla L(x) = \norm{\nabla L(x)}_2$ when $\norm{\nabla L(x)}_2> 0$, the leading terms in \Cref{eq:maxsharp,eq:ascsharp} are both the first order term, $\rho \norm{\nabla L(x)}_2$, and are the same. However, it is erroneous to think that the first order term decides the explicit bias, as the first order term $\norm{\nabla L(x)}_2$ vanishes at the local  minimizers of the loss $L$ and thus the second order term becomes the leading term.  Any global minimizer $x$ of the original loss $L$ is an $O(\rho^2)$-approximate minimizer of the sharpness-aware loss because $\nabla L(x) =0$. Therefore, the sharpness-aware loss needs to be of order $\rho^2$ so that we can guarantee the second-order terms in \Cref{eq:maxsharp} and/or \Cref{eq:ascsharp} to be non-trivially small. Our main result in this subsection (\Cref{thm:explicit_bias_deterministic_main}) gives an explicit characterization for this phenomenon. The corresponding explicit biases for each type of sharpness is given below in \Cref{defi:example_limiting_regularizer}. As we will see later, they can be derived from a general notion of \emph{limiting regularizer} (\Cref{defi:limiting_regularizer}).

\begin{definition}\label{defi:example_limiting_regularizer}
	For $x\in\R^D$, we define $\maxlimit(x) = \lambda_1(\nabla^2 L(x))/2$, $\asclimit(x) = \lambda_M(\nabla^2 L(x))/2$ and $\avglimit(x) = \Tr(\nabla^2 L(x))/(2D)$.
\end{definition}

\begin{theorem}\label{thm:explicit_bias_deterministic_main}
Under \Cref{assump:smoothness}, let $U'$ be any bounded open set such that its closure $\overline{U'}\subseteq U$ and  $\overline {U'}\cap \Gamma \subseteq\overline{ U'\cap \Gamma}$. For any $\mathrm{type}\in\{\mathrm{Max},\mathrm{Asc},\mathrm{Avg}\}$ and any optimality gap $\Delta>0$,  there is a function $\eps:\R^+\to \R^+$ with $\lim_{\rho\to 0}\eps(\rho)=0$, such that for all sufficiently small $\rho>0$ and all $u\in U'$ satisfying that 
$$L(u) +  R^\mathrm{type}_\rho(u) -\inf\limits_{x\in U'}\bigl({L(x) + R^\mathrm{type}_\rho(x)}\bigr)  \le \Delta \rho^2,\footnote{We note that $\ascsharpness(x)$ is undefined when $\norm{\nabla L(x)}_2=0$. In such cases, we set $\ascsharpness(x)=\infty$.}$$ 
it holds that $ L(u) - \inf_{x\in U'} L(x) \le (\Delta+\eps(\rho))\rho^2$ and  that 
$$ S^\mathrm{type}(u)-\inf_{x\in U'\cap \Gamma}S^\mathrm{type}(x) \in[-\eps(\rho), \Delta+\eps(\rho)].$$ 
\end{theorem}

\Cref{thm:explicit_bias_deterministic_main} suggests a sharp phase transition of the property of the solution of $\min_{x} L(x)+R_\rho(x)$ when the optimization error drops from $\omega(\rho^2)$ to $O(\rho^2)$. When the optimization error is larger than $\omega(\rho^2)$, no regularization effect happens and any minimizer satisfies the requirement. When the error becomes $O(\rho^2)$, there is a non-trivial restriction on the coefficients in the second-order term.

Next we give a heuristic derivation for the above defined $S^\mathrm{type}$. First, for worst- and average-direction sharpness, the calculations are fairly straightforward and well-known in literature~\citep{keskar2016large,kaur2022maximum,zhuang2022surrogate,orvieto2022explicit}, and we sketch them here. In the limit of perturbation radius $\rho\to0$, we know that the minimizer of the sharpness-aware loss will also converges to $\Gamma$, the manifold of minimizers of the original loss $L$. Thus to decide to which  $x\in\Gamma$ the minimizers will converge to as $\rho\to 0$, it suffices to take Taylor expansion of $\ascloss$ or $\avgloss$ at each $x\in \Gamma$ and compare the second-order coefficients, \emph{e.g.}, we have that $\avgsharpness(x)=  \frac{\rho^2}{2D}\mathrm{Tr}(\nabla ^2 L(x))+ O(\rho^3)$ and $\maxsharpness(x)=  \frac{\rho^2}{2}\mathrm{\lambda_1}(\nabla ^2 L(x)) + O(\rho^3)$ by \Cref{eq:maxsharp}.

However, the analysis for ascent-direction sharpness is more tricky because $\ascsharpness(x)=\infty$ for any $x\in\Gamma$ and thus is not continuous around such $x$. Thus we have to aggregate information from neighborhood to capture  the explicit bias of $R_\rho$ around manifold $\Gamma$. This motivates the following definition of \emph{limiting regularizer} which allows us to compare the regularization strength of $R_\rho$ around each point on manifold $\Gamma$ as $\rho\to 0$.

\begin{restatable}[Limiting Regularizer]{definition}{limitingregularizer}
\label{defi:limiting_regularizer}
 We define the \emph{limiting regularizer} of $\{R_\rho\}$ as the function\footnote{Here we implicitly assume the zeroth and first order term varnishes, which holds for all three sharpness notions. If not, then the notion of limiting regularizer is undefined.} 
\begin{align}
    S:\Gamma \to \R,\quad S(x) =\lim_{ \rho \to 0} \lim_{r\to 0}\inf_{\norm{x'-x}_2\le r} R_\rho(x')/\rho^2. \nonumber
\end{align}
\end{restatable} 
To minimize $\ascsharpness$ around $x$, we can pick $x'\to x$ satisfying that $\norm{\nabla L(x')}_2\to 0$ yet strictly being non-zero. 
By \Cref{eq:ascsharp},  we have $\ascsharpness(x') \approx  \frac{\rho^2}2\frac{\cdot \nabla L(x')^\top \nabla ^2 L(x) \nabla L(x')}{\|\nabla L(x') \|_2^2}$. Here the crucial step of the proof is that because of \Cref{assump:smoothness}, $\nabla  L(x)/\norm{\nabla  L(x)}_2$ must almost lie in the column span of $\nabla ^2 L(x)$, which implies that $\inf_{x'} \nabla L(x')^\top \nabla ^2 L(x) \nabla L(x') /\|\nabla L(x') \|_2^2 \overset{\rho\to 0}{\to} \lambda_M(\nabla ^2 L(x))$, where $\rank(\nabla^2 L(x))=M$ by \Cref{assump:smoothness}. The above alignment property between the gradient and the column space of Hessian can be checked directly for any non-negative quadratic function. The maximal Hessian rank assumption in \Cref{assump:smoothness} ensures that this property extends to general losses.

We defer the proof of \Cref{thm:explicit_bias_deterministic_main} into \Cref{sec:explicit_full_appendix}, where we develop a sufficient condition where the notion of limiting regularizer characterizes the explicit bias of $R_\rho$ as $\rho \to 0$.

\subsection{SAM Provably Decreases Worst-direction Sharpness}
\label{sec:implicit_full}
Though ascent-direction sharpness has different explicit bias from worst-direction sharpness, in this subsection we will show that surprisingly, SAM (\Cref{eq:sam}), a heuristic method designed to minimize ascent-direction sharpness,  provably decreases worst-direction sharpness. The main result here is an exact characterization of the trajectory of SAM (\Cref{eq:sam}) via the following ordinary differential equation (ODE) (\Cref{lambda1ode}), when learning rate $\eta$ and perturbation radius $\rho$ are small and the initialization $x(0)=\xinit$ is in $U$, the attraction set of manifold $\Gamma$. 
\begin{align}
\label{lambda1ode}
    X(\tau) = X(0) - \frac{1}{2}\int_{s=0}^{\tau} \projt[X(s)] \nabla \lambda_1(\nabla^2 L(X(s))) ds,\quad  X(0) = \Phi(\xinit).
\end{align}

We assume ODE (\Cref{lambda1ode}) has a solution till time $T_3$, that is, \Cref{lambda1ode} holds for all $t\le T_3$. We call the solution of \Cref{lambda1ode} the \emph{limiting flow} of SAM, which is exactly the Riemannian Gradient Flow on the manifold $\Gamma$ with respect to the loss $\lambda_1(\nabla^2 L(\cdot))$. In other words, the ODE (\Cref{lambda1ode}) is essentially a projected gradient descent algorithm with loss $\lambda_1(\nabla^2 L(\cdot))$ on the constraint set $\Gamma$ and an infinitesimal learning rate. Note $\lambda_1(\nabla^2 L(x))$ may not be differentiable at $x$ if $\lambda_1(\nabla^2 L(x)) = \lambda_2(\nabla^2 L(x))$, thus to ensure \Cref{lambda1ode} is well-defined, we assume there is a positive eigengap for $L$ on $\Gamma$.\footnote{In fact we only need to assume the positive eigengap along the solution of the ODE. If $\Gamma$ doesn't satisfy \Cref{assum_eigengap}, we can simply perform the same analysis on its submanifold $\{x\in\Gamma \mid \textrm{eigengap is positive at }x\}$.}

\begin{restatable}[]{assumption}{assumeigen}
\label{assum_eigengap}
For all $x \in \Gamma$, there exists a positive eigengap, i.e., $\lambda_1(\nabla^2 L(x)) > \lambda_2(\nabla^2 L(x))$.
\end{restatable}

\Cref{thm:nsam} is the main result of this section, which is a direct combination of \Cref{thm:nsamphase1,thm:nsamphase2}. The proof is deferred to~\Cref{sec:nsam_proof}.

\begin{restatable}[Main]{theorem}{thmnsam}
\label{thm:nsam}
Let $\{x(t)\}$ be the iterates of full-batch SAM (\Cref{eq:sam}) with $x(0) =\xinit \in U$.  Under Assumptions~\ref{assump:smoothness} and~\ref{assum_eigengap}, for all $\eta,\rho$ such that $\eta \ln(1/\rho)$ and $\rho/\eta$ are sufficiently small, the dynamics of SAM can be characterized in the following two phases: 
\begin{itemize}
    \item Phase I: (\Cref{thm:nsamphase1}) Full-batch SAM (\Cref{eq:sam}) follows Gradient Flow with respect to $L$ until entering an $O(\eta\rho)$ neighborhood of the manifold $\Gamma$ in $O(\ln(1/\rho)/\eta)$ steps;
    \item Phase II: (\Cref{thm:nsamphase2})  Under a mild non-degeneracy assumption (\Cref{assum:reg}) on the initial point of phase II, full-batch SAM (\Cref{eq:sam}) tracks the solution $X$ of \Cref{lambda1ode}, the Riemannian Gradient Flow with respect to the loss $\lambda_1(\nabla^2 L(\cdot))$ in an $O(\eta\rho)$ neighborhood of manifold $\Gamma$. Quantitatively, the approximation error between the iterates $x$ and the corresponding limiting flow $X$ is $O(\eta \ln (1/\rho))$, that is,
    \begin{align*}
        \|x\bigl(\lceil \frac{T_3} {\eta\rho^2} \rceil\bigr) - X(T_3) \|_2 = O(\eta \ln (1/\rho))\,.
    \end{align*}
 Moreover, the angle between $\nabla L\bigl(x(\lceil \frac{T_3} {\eta\rho^2}\rceil\bigr)$ and the top eigenspace of $\nabla^2 L(x(\lceil\frac{T_3} {\eta\rho^2}\rceil))$  is at most $O(\rho)$.
\end{itemize} 
\end{restatable}

Theorem~\ref{thm:nsam} shows that SAM decreases the largest eigenvalue of Hessian of loss locally around the  manifold of local minimizers. Phase I uses standard approximation analysis as in~\citet{hairer2008solving}. In Phase II, as $T_3$ is arbitrary, the approximation and alignment properties hold simultaneously for all $X(t)$ along the trajectory, provided that $\eta \ln (1/\rho)$ and $\rho/\eta$ are sufficiently small. The subtlety here is that the threshold of being "sufficiently small" on $\eta \ln (1/\rho)$ and $\rho/\eta$ actually depends on $T_3$, which decreases when $T_3 \to 0$ or $\to \infty$.
 We defer the proof of \Cref{thm:nsam} to \Cref{app:nsam}.

As a corollary of \Cref{thm:nsam}, we can also show that the largest eigenvalue of the limiting flow closely tracks the worst-direction sharpness.   

\begin{corollary}
\label{corr:loss_ode_n}
 In the setting of \Cref{thm:nsam}, the difference between the worst-direction sharpness of the iterates and the corresponding scaled largest eigenvalues along the limiting flow is at most $ O(\eta\rho^2\ln(1/\rho))$. That is, 
 \begin{align}
\left| \maxsharpness(x(\lceil T_3/\eta\rho^2 \rceil)) - \rho^2 \lambda_1(\nabla^2 L(X(T_3))/2\right| =  O(\eta\rho^2\ln(1/\rho))\,.
 \end{align}
 
\end{corollary}
Since $\eta\ln (1/\rho)$ is assumed to be sufficiently small, the error $O(\eta\ln (1/\rho)\cdot \rho^2)$ is only $o(\rho^2)$, meaning that  penalizing the top eigenvalue on the manifold does lead to non-trivial reduction of worst-direction sharpness, in the sense of \Cref{sec:explicit_full}.

 Hence we can show that full-batch SAM (\Cref{eq:sam}) provably minimizes \emph{worst-direction sharpness} around the manifold if we additionally assume the limiting flow converges to a minimizer of the top eigenvalue of Hessian in the following \Cref{corr:loss_opt_n}.

\begin{corollary}
\label{corr:loss_opt_n}
Under Assumptions~\ref{assump:smoothness} and~\ref{assum_eigengap}, define $U'$ as in Theorem~\ref{thm:explicit_bias_deterministic_main} and  suppose $X(\infty) = \lim \limits_{t \to \infty} X(t)$ exists and is a minimizer of $\lambda_1(\nabla^2 L(x))$ in $U'\cap \Gamma$. Then for all $\epsilon > 0$, there exists  $T_{\eps} > 0$, such that for all $\rho,\eta$ such that $\eta \ln(1/\rho)$ and $\rho/\eta$ are sufficiently small, we have that
\begin{align}
    \maxloss(x(\lceil T_{\eps} / (\eta\rho^2)\rceil)) \le \epsilon \rho^2 + \inf_{x \in U'} \maxloss(x) \,.\nonumber
\end{align}
\end{corollary}

We defer the proof of Corollaries~\ref{corr:loss_ode_n} and~\ref{corr:loss_opt_n} to \Cref{app:nsamcorr}.

\subsection{Analysis Overview For Sharpness Reduction in Phase II of Theorem~\ref{thm:nsam}}
\label{sec:proof_full}

Now we give an overview of the analysis for the trajectory of full-batch SAM (\Cref{eq:sam}) in Phase II (in Theorem~\ref{thm:nsam}). The framework of the analysis is similar to \citet{arora2022understanding,lyu2022understanding,damian2021label}, where the high-level idea is to use $\Phi(x(t))$ as a proxy for $x(t)$ and study the dynamics of $\Phi(x(t))$ via Taylor expansion. 
We will first closely follow the machinery developed in~\citet{arora2022understanding} to arrive at \Cref{eq:Taylor_expansion_of_Phi_simplified}, starting from which we will discuss the key innovation in this paper regarding implicit Hessian-gradient alignment.

\paragraph{Dynamics of $\Phi(x(t))$ via Taylor expansion.}
In Phase II, $x(t)$ is $O(\eta\rho)$-close to the manifold $\Gamma$ and therefore it can be shown that $\|x(t) - \Phi(x(t))\|_2 = O(\eta\rho)$ holds for every step in Phase II. This also implies that $\norm{x(t+1)-x(t)}_2 =  O(\eta\rho)$ (See \Cref{lem:boundedphi}). 
Using Taylor expansion around $x(t)$, we  have that
\begin{align}\label{eq:taylor_expansion_of_Phi}
	\Phi(x(t + 1)) - \Phi(x(t)) 
	=&  \partial \Phi(x(t)) (x(t+1) - x(t)) + O(\|x(t+1) - x(t)\|_2^2)\notag\\
	= & -\eta \partial \Phi(x(t))\nabla L\bigl(x(t) + \rho\frac{\nabla L(x(t))}{\norm{\nabla L(x(t))}_2}\bigr)  + O(\eta^2\rho^2)\,.
\end{align}
For any $x\in\R^D$, applying Taylor expansion on $\nabla L\bigl(x + \rho\frac{\nabla L(x)}{\norm{\nabla L(x)}_2}\bigr)$ around $x$, we have that
\begin{align}\label{eq:taylor_expansion_of_ascent_grad}
	&\nabla L\bigl(x + \rho\frac{\nabla L(x)}{\norm{\nabla L(x)}_2}\bigr)\notag\\
	 = &\nabla L(x) + \rho \nabla ^2 L(x) \frac{\nabla L(x)}{\norm{\nabla L(x)}_2} 
	+  \frac{\rho^2}{2}\partial^2 (\nabla L) (x)\bigl[\frac{\nabla L(x)}{\norm{\nabla L(x)}_2},\frac{\nabla L(x)}{\norm{\nabla L(x)}_2}\bigr] + O(\rho^3).
\end{align}
Using \Cref{eq:taylor_expansion_of_ascent_grad} with $x = x(t)$, plugging in \Cref{eq:taylor_expansion_of_Phi} and then rearranging, we have that 
\begin{align}
	&\Phi(x(t + 1)) - \Phi(x(t)) + \frac{\eta\rho^2}{2}\partial \Phi(x(t))\partial^2  (\nabla L)(x(t))\bigl[\frac{\nabla L(x(t))}{\norm{\nabla L(x(t))}_2},\frac{\nabla L(x(t))}{\norm{\nabla L(x(t))}_2}\bigr] \notag  \\
	= & -\eta \partial \Phi(x(t)) \nabla L(x(t)) - \eta\rho  \partial \Phi(x(t)) \nabla ^2L(x(t)) \frac{\nabla L(x(t))}{\norm{\nabla L(x(t))}_2} + O(\eta^2\rho^2+\eta\rho^3)\,.\notag \notag 
\end{align}

By \Cref{lem:property_of_Phi_away_manifold}, we have that $\partial \Phi(x(t)) \nabla L(x(t))=0$. Furthermore, by \Cref{lem:property_of_Phi_on_manifold}, 
we have that  $\partial \Phi(\Phi(x(t))) \nabla ^2L(\Phi(x(t)))=0$. This implies that 
\begin{align*}
    \partial \Phi(x(t)) \nabla ^2L(x(t)) = \partial \Phi(\Phi(x(t))) \nabla ^2L(\Phi(x(t)))+O(\norm{x(t)-\Phi(x(t))}_2) = O(\eta\rho)\,.
\end{align*}
Thus we conclude that 
\begin{align}
	\Phi(x(t + 1)) - \Phi(x(t))
	= - &\frac{\eta\rho^2}{2}\partial \Phi(x(t))\partial^2  (\nabla L)(x(t))\bigl[\frac{\nabla L(x(t))}{\norm{\nabla L(x(t))}_2},\frac{\nabla L(x(t))}{\norm{\nabla L(x(t))}_2}\bigr] \notag\\
	+ & O(\eta^2\rho^2+\eta\rho^3)\,.
	\label{eq:Taylor_expansion_of_Phi_simplified}
\end{align}

Now, to understand how $\Phi(x(t))$ moves over time, we need to understand what the direction of the RHS of \Cref{eq:Taylor_expansion_of_Phi_simplified} corresponds to---we will prove that it corresponds to the Riemannian gradient of the loss function $\nabla \lambda_1(\nabla^2 L(x))$ at $x = \Phi(x(t))$. To achieve this, the key is to understand the direction $\frac{\nabla L(x(t))}{\norm{\nabla L(x(t))}_2}$. It turns out that we will prove $\frac{\nabla L(x(t))}{\norm{\nabla L(x(t))}_2}$ is close to the top eigenvector of the Hessian up to sign flip, that is $\|\frac{\nabla L(x(t))}{\norm{\nabla L(x(t))}_2} - s\cdot v_1(\nabla^2 L(x))\|_2 \le O(\rho)$ for some $s\in \{-1,1\}$. We call this phenomenon Hessian-gradient alignment and will discuss it in more detail at the end of this subsection. 

Using this property, we can proceed with the derivation: 
\begin{align}\label{eq:Taylor_expansion_of_Phi_pre_final}
 &\Phi(x(t + 1)) - \Phi(x(t)) \notag \\
=& -\frac{\eta\rho^2}{2}\partial \Phi(x(t))\partial^2  (\nabla L)(x(t))\bigl[\frac{\nabla L(x(t))}{\norm{\nabla L(x(t))}_2},\frac{\nabla L(x(t))}{\norm{\nabla L(x(t))}_2}\bigr] + O(\eta^2\rho^2+\eta\rho^3)\notag \\
=& -\frac{\eta\rho^2}{2}\partial \Phi(x(t))\partial^2  (\nabla L)(x(t))\bigl[v_1(\nabla^2 L(x(t))),v_1(\nabla^2 L(x(t)))\bigr] + O(\eta^2\rho^2+\eta\rho^3)\notag \\
=& - \frac{\eta\rho^2}{2}\partial \Phi(x(t)) \nabla \lambda_1(\nabla ^2 L(x(t))) + O(\eta^2\rho^2+\eta\rho^3)\notag\\
= & -  \frac{\eta\rho^2}{2}\partial \Phi(\Phi(x(t))) \nabla \lambda_1(\nabla ^2 L(\Phi(x(t)))) + O(\eta^2\rho^2+\eta\rho^3),
\end{align}
where the second to last step we use the property of the derivative of eigenvalue (\Cref{lem:top_eigenvector}) and the last step is due to Taylor expansion of $\partial \Phi (\cdot)\nabla \lambda_1(\nabla^2 L (\cdot))$ at $\Phi(x(t))$ and the fact that $\norm{\Phi(x(t))-x(t)}=O(\eta\rho)$. 

\paragraph{Implicit Hessian-gradient Alignment.} It remains to explain why the gradient  implicitly aligns to the top eigenvector of the Hessian, which is the key component of the analysis in Phase II. The proof strategy here is to first show alignment for a quadratic loss function, and then generalize its proof to general loss functions satisfying \Cref{assump:smoothness}. Below we first give the formal statement of the implicit alignment on quadratic loss, \Cref{thm:nsamquad} and defer the result for general case (\Cref{lem:final_alignment}) to appendix. Note this alignment property is an implicit property of the SAM algorithm as it is not explicitly enforced by the objective that SAM is intended to minimize, $\ascloss$. Indeed optimizing $\ascloss$ would rather explicitly align gradient  to the smallest non-zero eigenvector (See proofs of \Cref{thm:avgsharp})! 
\begin{restatable}[]{theorem}{thmnsamquad}
\label{thm:nsamquad}
Suppose $A$ is a positive definite symmetric matrix with unique top eigenvalue. Consider running full-batch SAM (\Cref{eq:sam}) on loss $L(x) \coloneqq \frac{1}{2}x^TAx$ as in \Cref{eq:sam_quadratic} below. 
\begin{align}\label{eq:sam_quadratic}
x(t+1) &= x(t) - \eta A \bigl(x(t) + \rho \frac{Ax(t)}{\| Ax(t)\|_2}\bigr)\,.    
\end{align}
Then, for almost every $x(0)$, we have $x(t)$ converges in direction to $v_1(A)$ up to a sign flip and $\lim_{t \to \infty} \|x(t)\|_2 = \frac{\eta \rho \lambda_1(A)}{2-\eta \lambda_1(A)}$ with $\eta \lambda_1(A) < 1$.
\end{restatable}

The proof of Theorem~\ref{thm:nsamquad} relies on a two-phase analysis of the behavior of \Cref{eq:sam_quadratic}, where we first show that $x(t)$ enters an invariant set from any initialization and in the second phase, we construct a potential function to show alignment. The proof is deferred to \Cref{app:nsamquad}.

Below we briefly discuss why the case with general loss is closely related to the quadratic loss case. We claim that, in the general loss function case, the analog of \Cref{eq:sam_quadratic} is the update rule for the gradient: 
\begin{align}\label{eq:sam_gradient_taylor}
\! \!\!\!\!\! \nabla L(x(t+1)) \!\!
    = &  \nabla L(x(t)) \!-\! \eta \nabla^2 L(x(t)) \bigl(\nabla L(x(t)) + \rho \nabla^2 L(x(t)) \frac{\nabla L(x(t))}{\| \nabla L(x(t)))\|_2}\bigr) + O(\eta\rho^2) \,.\!\!
\end{align}
We first note that indeed in the quadratic case where $\nabla L(x) = Ax$ and $\nabla^2 L(x) = A$, \Cref{eq:sam_gradient_taylor} is equivalent to \Cref{eq:sam_quadratic} because they only differ by a multiplicative factor $A$ on both sides. 

Hence, in the general case, the update of the gradient (\Cref{eq:sam_gradient_taylor}) can be viewed as an $O(\eta\rho^2)$-perturbed version of the update of the iterate in the quadratic case. Note $O(\eta\rho^2)$ is a higher order term comparing to the other two terms, which are  on the order of $\Theta(\eta^2\rho)$ and $\Theta(\eta\rho)$ respectively. By controlling the error terms, the mechanism and analysis of the implicit alignment between Hessian and gradient still apply to the general case. We can also show that once this alignment happens, it will be kept until the end of our analysis, which is $\Theta(\eta^{-1}\rho^{-2})$ steps.

Finally, we derive \Cref{eq:sam_gradient_taylor} by Taylor expansion. We first apply Taylor expansion~(\Cref{eq:taylor_expansion_of_ascent_grad}) on the update rule of the iterate of SAM (\Cref{eq:sam}):
\begin{align}\label{eq:sam_taylor}
    x(t+1) = x(t) - \eta \nabla L(x(t)) -\eta \rho\nabla^2L(x(t))\frac{\nabla L(x(t))}{\norm{\nabla L(x(t))}_2} +O(\eta\rho^2).
\end{align}

Since  phase II happens in an $O(\eta\rho)$-neighborhood of manifold $\Gamma$, we have $\norm{x(t+1) - x(t)}_2 = O(\eta\rho)$. 
Then by \Cref{eq:sam_taylor} and Taylor expansion on $\nabla L(x(t+1))$ at $x(t)$, we have that
\begin{align}
 \!\!\!\!  \!\!\! \nabla L(x(t+1)) \!\!
    =& \nabla L(x(t))\! - \!\nabla ^2 L(x(t)) \bigl(x(t+1) - x(t) \bigr)+ O(\eta^2\rho^2)\notag\\
    = &  \nabla L(x(t)) \!-\! \eta \nabla^2 L(x(t)) \bigl(\nabla L(x(t)) + \rho \nabla^2 L(x(t)) \frac{\nabla L(x(t))}{\| \nabla L(x(t)))\|_2}\bigr) + O(\eta\rho^2) \,.\!\!
\end{align}
 
\section{Explicit and Implicit Biases in the Stochastic Setting}\label{sec:1sam}

In practice, people usually use  SAM in the stochastic mini-batch setting, and the test accuracy improves as the batch size decreases~\citep{foret2021sharpnessaware}. Towards explaining this phenomenon, \citet{foret2021sharpnessaware} argue intuitively that stochastic SAM minimizes stochastic worst-direction sharpness. Given our results in Section~\ref{sec:nsam}, it is natural to ask if we can justify the above intuition by showing the Hessian-gradient alignment in the stochastic setting. Unfortunately, such alignment is not possible in the most general setting.  
Yet when the batch size is 1,  we can prove rigorously in \Cref{sec:implicitstoc} that stochastic SAM minimizes \emph{stochastic worst-direction sharpness}, which is the expectation of the worst-direction sharpness of loss over each data (defined in \Cref{sec:explicitstoc}), which is the main result in this section. We stress that the stochastic worst-direction sharpness has a different explicit bias to the worst-direction sharpness, which full-batch SAM implicitly penalizes. When perturbation radius $\rho\to 0$, the former corresponds to $\Tr(\nabla^2 L(\cdot))$, the same as average-direction sharpness, and the latter corresponds to $\lambda_1(\nabla^2 L(\cdot))$.

Below we start by introducing our setting for SAM with batch size $1$, or $1$\emph{-SAM}. We still need \Cref{assump:smoothness} in this section. We first analyze the explicit bias of the stochastic ascent- and worst-direction sharpness in Section~\ref{sec:explicitstoc} via the tools developed in \Cref{sec:explicit_full}. It turns out they are all proportional to the trace of hessian as $\rho \to 0$. In \Cref{sec:implicitstoc}, we show that 1-SAM penalizes the trace of Hessian. Below we formally state our setting for stochastic loss of batch size one~(\Cref{setting:1sam}).

\begin{restatable}{setting}{stochasticsetting}\label{setting:1sam}
Let the total number of data be $M$. Let $f_k(x)$ be the model output on the $k$-th data where $f_k$ is a $\mathcal{C}^4$-smooth function and $y_k$ be the $k$-th label, for $k=1,\ldots,M$. We define the loss on the $k$-th data as $L_k(x) = \ell(f_k(x),y_k)$ and the total loss $L= \sum_{k=1}^M L_k/M$, where function $\ell(y',y)$ is $\mathcal{C}^4$-smooth in $y'$. 
We also assume for any $y\in\R$, it holds that $\argmin_{y'\in\R}\ell(y',y) = y$ and that $\frac{\partial ^2\ell(y',y)}{(\partial y')^2}\vert_{y'=y} >0$.
Finally, we denote the set of global minimizers of $L$ with full-rank Jacobian by $\Gamma$ and assume that it is non-empty, that is, 
\begin{align}
	\Gamma\triangleq \left\{ x\in\R^D \mid f_k(x) = y_k, \forall k\in[M] \textup{ and }\{\nabla f_k(x)\}_{k=1}^M \textup{ are linearly independent}\right\} \neq \emptyset. \notag
\end{align}
 \end{restatable}
 
 We remark that given training data (\emph{i.e.}, $\{f_k\}_{k=1}^M$),   $\Gamma$ defined above is just equal to the set of global minimizers, $\left\{ x\in \R^D \mid f_k(x) = y_k, \forall k\in[M]\right\}$, except for a zero measure set of labels $(y_k)_{k=1}^M$ when $f_k$ are $\continuous{\infty}$ smooth, by Sard's Theorem. Thus \citet{cooper2018loss} argued that the global minimizers form a differentiable manifold generically if we allow perturbation on the labels. In this work we do not make such an assumption for labels. Instead, we consider the subset of the global minimizers with full-rank Jacobian, $\Gamma$. A standard application of implicit function theorem implies that $\Gamma$ defined in \Cref{setting:1sam} is indeed a manifold. (See \Cref{thm:derive_manifold_assumption}, whose proof is deferred into \Cref{sec:proof_rank_1})
 
\begin{restatable}{theorem}{derivemanifoldassumption}\label{thm:derive_manifold_assumption}
Loss $L$, set $\Gamma$ and integer $M$ defined in \Cref{setting:1sam} satisfy \Cref{assump:smoothness}.
\end{restatable}

\paragraph{1-SAM:} We use $1$\emph{-SAM} as a shorthand for SAM on a stochastic loss with batch size $1$ as below \Cref{eq:1sam}, where $k_t$ is sampled i.i.d from uniform distribution on $[M]$.
\begin{align}
\label{eq:1sam}
\textrm{1-SAM}: \qquad\qquad  x(t+1) = x(t) -\eta   \nabla L_{k_t}\bigl(x+\rho \frac{\nabla L_{k_t}(x)}{\norm{\nabla L_{k_t}(x)}_2}\bigr)\,.\qquad\qquad
\end{align}

\subsection{Stochastic Worst-, Ascent- and Average- direction Sharpness Have the Same Explicit Biases as Average Direction Sharpness}
\label{sec:explicitstoc}

Similar to the full-batch case, we use $\maxloss[k,\rho],\ascloss[k,\rho],\avgloss[k,\rho]$ to denote the corresponding sharpness-aware loss for $L_k$ and $\maxsharpness[k,\rho],\ascsharpness[k,\rho],\avgsharpness[k,\rho]$ to denote corresponding sharpness for $L_k$ respectively (defined as \Cref{eq:max_sharpness,eq:asc_sharpness,eq:avg_sharpness} with $L$ replaced by $L_k$).
We further use \emph{stochastic worst-, ascent- and average-direction sharpness} to denote $\stmaxsharpness, \stascsharpness $ and $\stavgsharpness$. Unlike the full-batch setting, these three sharpness notions have the same explicit biases, or more precisely, they have the same limiting regularizers (up to some scaling factor).

\begin{theorem}\label{thm:explicit_bias_stochastic_main}
The limiting regularizers of three notions of stochastic sharpness, denoted by $\stomaxlimit, \stoasclimit,\stoavglimit $, satisfy that 
\begin{align}
    \stomaxlimit(x) = \stoasclimit(x) = D \cdot \stoavglimit(x) = \Tr(\nabla^2 L(x))/2. \notag
\end{align}
Furthermore, define $U'$ in the same way as in \Cref{thm:explicit_bias_deterministic_main} . For any $\mathrm{type}\in\{\mathrm{Max},\mathrm{Asc},\mathrm{Avg}\}$, it holds that  if for some $u\in U'$, $ L(u) +  \E_k[R^\mathrm{type}_{k,\rho}(u)] \le \inf\limits_{x\in U'}\bigl({L(x) + \E_k[R^\mathrm{type}_{k,\rho}(x)]}\bigr) + \epsilon\rho^2 $,\footnote{We note that $\ascsharpness(x)$ is undefined when $\norm{\nabla L(x)}_2=0$. In such cases, we set $\ascsharpness(x)=\infty$.} then we have that $ L(u) - \inf_{x\in U'} L(x) \le \eps\rho^2 +o(\rho^2)$ and that $\bigl|\widetilde S^\mathrm{type}(u)-\inf_{x\in U'\cap \Gamma}\widetilde S^\mathrm{type}(x)\bigr|\le  \epsilon + o(1)$. 
\end{theorem}

We defer the proof of \Cref{thm:explicit_bias_stochastic_main} to \Cref{app:explicitcorr}.
Unlike in the full-batch setting where the implicit regularizer of ascent-direction sharpness and  worst-direction sharpness have different explicit bias, here they are the same because there is no difference between the maximum and minimum of its non-zero eigenvalue for rank-1 Hessian of each individual loss $L_k$, and that the average of limiting regularizers is equal to the limiting regularizer of the average regularizer by definition.

\subsection{Stochastic SAM Minimizes Average-direction Sharpness } \label{sec:implicitstoc}

This subsection aims to show that the implicit bias of 1-SAM (\Cref{eq:1sam}) is minimizing the average-direction sharpness for small perturbation radius $\rho$ and learning rate $\eta$, which has the same implicit bias as all three notions of stochastic sharpness do (\Cref{thm:explicit_bias_stochastic_main}). As an analog of the analysis in \Cref{sec:proof_full}, which shows full-batch SAM minimizes worst-direction sharpness, analysis in this section conceptually shows that 1-SAM  minimizes the stochastic worst-direction sharpness.

Mathematically, we prove that the trajectory of 1-SAM tracks the following Riemannian gradient flow (\Cref{eq:traceode}) with respect to their limiting regularize $\mathrm{Tr}(\nabla^2 L(\cdot))$ on the manifold for sufficiently small $\eta$ and $\rho$ and thus penalizes stochastic worst-direction sharpness (of batch size  $1$). We assume the  ODE (\Cref{eq:traceode}) has a solution till time $T_3$.
\begin{align}
\label{eq:traceode}
    X(\tau) = X(0) - \frac{1}{2}\int_{s=0}^{\tau} \projt[X(s)] \nabla \mathrm{Tr}(\nabla^2 L(X(s))) ds,\ \  X(0) = \Phi(\xinit).
\end{align}

\begin{restatable}{theorem}{thmonesam}
\label{thm:1sam}
Let $\{x(t)\}$ be the iterates of 1-SAM (\Cref{eq:1sam}) and $x(0) = \xinit \in U$, then under  \Cref{setting:1sam}, for almost every $\xinit$, for all $\eta$ and $\rho$ such that $(\eta + \rho)\ln(1/\eta\rho)$ is sufficiently small, with probability at least $1 - O(\rho)$ over the randomness of the algorithm, the dynamics of 1-SAM (\Cref{eq:1sam})  can be split into two phases: 
\begin{itemize}
    \item Phase I (\Cref{thm:1samphase1}): 1-SAM follows Gradient Flow with respect to  $L$ until entering an $\tilde O(\eta\rho)$ neighborhood of the manifold $\Gamma$ in $O(\ln(1/\rho\eta)/\eta)$ steps;    \item Phase II (\Cref{thm:1samphase2}): 1-SAM tracks the solution of \Cref{eq:traceode}, $X$, the Riemannian gradient flow with respect to $\mathrm{Tr}(\nabla^2 L(\cdot))$ in an $\tilde O(\eta\rho)$ neighborhood of manifold $\Gamma$. Quantitatively, the approximation error between the iterates $x$ and the corresponding limiting flow $X$ is $\tilde O(\eta^{1/2} + \rho)$, that is, 
\begin{align}
\|x(\lceil T_3 / (\eta\rho^2)\rceil) - X(T_3) \|_2 = \tilde O(\eta^{1/2} + \rho).\notag
\end{align}
\end{itemize} 
\end{restatable}

The high-level intuition for the Phase II result of \Cref{thm:1sam} is that Hessian-gradient alignment holds true for every stochastic loss $L_k$ along the trajectory of 1-SAM  and therefore by Taylor expansion (the same argument in \Cref{sec:proof_full}), at each step $\Phi(x(t))$ moves towards the negative (Riemannian) gradient of $\lambda_1(\nabla^2 L_{k_t})$ where $k_t$ is the index of randomly sampled data, or the limiting regularizer of the worst-direction sharpness of $L_{k_t}$.  Averaging over a long time, the moving direction  becomes the negative (Riemmanian) gradient of $\E_{k_t}[\lambda_1(\nabla^2 L_{k_t})]$, which is  the  limiting regularizer of stochastic worst-direction sharpness and equals to $\mathrm{Tr}(\nabla^2 L)$ by \Cref{thm:explicit_bias_stochastic_main}.

The reason that Hessian-gradient alignment holds under \Cref{setting:1sam} is that the Hessian of each stochastic loss $L_k$ at minimizers $p\in\Gamma$, $\nabla^2 L_k(p) = \frac{\partial ^2\ell(y',y_k)}{(\partial y')^2}\vert_{y'=f_k(p)} \nabla f_k(p)(\nabla f_k(p))^\top$(\Cref{lem:rank1hessian}), is exactly rank-1, which enforces the gradient $\nabla L_k(x) \approx \nabla ^2 L_k(\Phi(x))(x-\Phi(x))$ to (almost) lie in the top (which is also the unique) eigenspace of $\nabla^2 L_k(\Phi(x))$. \Cref{lem:informal_direction_stocastic_dl} formally states this property.

\begin{lemma}
 \label{lem:informal_direction_stocastic_dl}
Under \Cref{setting:1sam}, for any $p \in \Gamma$ and $k\in [M]$, it holds that $\nabla f_k(p) \neq  0$ and that there is an open set $V$ containing $p$, satisfying that 
\begin{align}
 \forall x\in V, \dli{x} \neq 0 \implies  \exists s \in \{-1,1 \},\ndli{x} = s \frac{\nabla f_k(p)}{\norm{\nabla f_k(p)}_2} + O(\|x - p \|_2).	\notag
\end{align}
 \end{lemma}

 \Cref{corr:loss_ode_1,corr:loss_opt_1} below are stochastic counterparts of  \Cref{corr:loss_ode_n,corr:loss_opt_n}, saying that the trace of Hessian are close to the stochastic worst-direction sharpness along the limiting flow~\eqref{eq:traceode}, and therefore when the limiting flow converges to a local minimizer of trace of Hessian, 1-SAM (\Cref{eq:1sam}) minimizes the average-direction sharpness.  We defer the proofs of \Cref{corr:loss_ode_1,corr:loss_opt_1} to \Cref{app:1samcorr}.
\begin{corollary}
\label{corr:loss_ode_1}
 Under the condition of \Cref{thm:1sam}, we have that with probability $1 - O(\sqrt{\eta} + \sqrt{\rho})$, the difference between the stochastic worst-direction sharpness of the iterates and the corresponding scaled trace of Hessian along the limiting flow is at most $O\bigl((\eta^{1/4} + \rho^{1/4})\rho^2\bigr)$, that is,
 \begin{align*}
      \left| \E_k[R_{k,\rho}^{\textup{Max}}(x(\lceil T_3 / (\eta\rho^2)\rceil))] - \rho^2 \mathrm{Tr}(\nabla ^2 L(X(T_3)))/2\right| =  O\bigl((\eta^{1/4} + \rho^{1/4})\rho^2\bigr) \,.
 \end{align*}
\end{corollary}

\begin{corollary}
\label{corr:loss_opt_1}
Define $U'$ as in Theorem~\ref{thm:explicit_bias_deterministic_main}, suppose $X(\infty) = \lim \limits_{t \to \infty} X(t)$ exists and is a minimizer of $\mathrm{Tr}(\nabla^2 L(x)))$ in $U'\cap \Gamma$. Then for all $\epsilon > 0$, there exists a constant $T_\eps > 0$, such that for all $\rho,\eta$ such that $(\eta + \rho)\ln(1/\eta\rho)$ are sufficiently small, we have that with probability $1 - O(\sqrt{\eta} + \sqrt{\rho})$,
\begin{align}
    \E_k[L_{k,\rho}^{\textup{Max}}(x(\lceil T_\eps / (\eta\rho^2)\rceil))]   \le \epsilon \rho^2 + \inf_{x \in U'}\E_k[L_{k,\rho}^{\textup{Max}}(x)] \,.\nonumber
\end{align}
\end{corollary}

Finally we give a concrete counter example to demonstrate why the condition of batch size equal to one is crucial to this alignment property.
\begin{example}
 Take a simple quadratic loss $L(x) = \frac{L_1(x)+L_2(x)}{2}$, where $L_k(x) = 0.5 x^\top A_k x$ and $A_k$ is a positive definite matrix for $k=1,2$. If $A_1$ and $A_2$ have different top eigenspaces, then no $x$ can simultaneously satisfy that $\nabla L_k (x) = A_k x$ aligns to the top eigenvector of $\nabla^2 L_k(x)=A_k$, because this implies $x$ is both an eigenvector of $A_1$ and $A_2$. 	
\end{example}

\section{Conclusion}\label{sec:conclusion}
 
In this work, we have performed a rigorous mathematical analysis of the explicit bias of various notions of sharpness when used as regularizers and the implicit bias of the SAM algorithm.  In particular, we show the explicit biases of worst-, ascent- and average-direction sharpness around the manifold of minimizers are minimizing the largest eigenvalue, the smallest nonzero eigenvalue, and the trace of Hessian of the loss function.  We show that in the full-batch setting, SAM provably decreases the largest eigenvalue of Hessian, while in the stochastic setting when batch size is 1, SAM provably decreases the trace of Hessian.

 The most interesting future work is to generalize the current analysis for stochastic SAM to arbitrary batch size. This is challenging because, without the alignment property which holds automatically with  batch size  1, such an analysis essentially requires understanding the stationary distribution of the gradient direction along the SAM trajectory. It is also interesting to incorporate other features of modern deep learning like normalization layers, momentum, and weight decay into the current analysis. 
 
 Another interesting open question is to further bridge the difference between generalization bounds and the implicit bias of the optimizers. Currently, the generalization bounds in~\cite{wu2020adversarial,foret2020sharpness} only work for the randomly perturbed model. Moreover, the bound depends on the average sharpness with finite $\rho$, whereas the analysis of this paper only works for infinitesimal $\rho$. It's an interesting open question whether the generalization error of the model (without perturbation) can be bounded from above by some function of the training loss, norm of the parameters, and the trace of the Hessian.  

\section*{ACKNOWLEDGEMENTS}

We thank Jingzhao Zhang for helpful discussions. 
The authors would like to thank the support from NSF IIS 2045685. 

\bibliography{all}
\bibliographystyle{iclr2023_conference}

\appendix
\newpage {
\hypersetup{linkcolor=black}
\tableofcontents
}
\section{Experimental Details for Figure~\ref{fig:demo}}
\label{app:detail_caption}

In~\Cref{fig:demo}, we choose $F_1(x) = x_1^2 + 6x_2^2 +8$ and $F_2(x) = 4(1- x_1)^2 + (1-x_2)^2 + 1$. The loss $L$ has a zero loss  manifold $\{x = 0\}$ and the eigenvalues of its Hessian on the manifold are $F_1(x)$ and $F_2(x)$ with $F_1(x) \ge 8>6\ge  F_2(x)$ on $[0,1]^2$. The loss $L$ has a zero loss  manifold $\{x_3 = x_4 = 0\}$ of codimension $M=2$ and the two non-zero eigenvalues of $\nabla ^2 L$ of any point $x$ on the manifold are $\lambda_1(\nabla^2 L(x))= F_1(x_1,x_2)$ and $\lambda_2(\nabla^2 L(x)) = F_2(x_1,x_2)$.

As our theory predicts, 
\begin{enumerate}
    \item Full-batch SAM (\Cref{eq:sam}) finds the minimizer with the smallest top eigenvalue $F_1(x)$, which is $x_1 = 0, x_2 = 0, x_3 = 0, x_4 = 0$;
    \item GD on ascent-direction loss $\ascloss$~\eqref{eq:asc_sharpness} finds the minimizer with the smallest bottom eigenvalue, $F_2(x)$, which is $x_1 = 1, x_2 = 1, x_3 = 0, x_4 = 0$; 
    \item Stochastic SAM (\Cref{eq:1sam}) (with $L_0(x,y) = F_1(x)y_0^2,\ L_1(x,y) = F_2(x)y_1^2$) finds the minimizer with smallest trace of Hessian, which is $x_1 = 4/5, x_2 = 1/7, x_3 = 0, x_4 = 0$.
\end{enumerate}

\section{Well-definedness of SAM}\label{appsec:well_definedness}

In this section, we discuss the well-definedness of SAM. When $\nabla L(x)=0$, SAM (\Cref{eq:sam}) is not well-defined, because the normalized gradient $\frac{\nabla L(x)}{\norm{\nabla L(x)}_2}$ is not well-defined. The main result of this section are \Cref{thm:nsam_well_definedness,thm:1sam_well_definedness}, which say that (stochastic) SAM starting from random initialization only has zero probability to reach points that SAM is undefined (\emph{i.e.}, points with zero gradient), for \emph{all except countably many learning rates}. These results follow from  \Cref{thm:general_well_definednesa}, which is a more general theorem also applicable to other discrete update rules as well, like SGD.  Note results in this section does not rely on the manifold assumption, \emph{i.e.}, \Cref{assump:smoothness}. We end this section with a concrete example where SAM is undefined with constant probability, suggesting that the exclusion of countably many learning rates are necessary in \Cref{thm:nsam_well_definedness,thm:1sam_well_definedness}.

\begin{theorem}\label{thm:nsam_well_definedness}
Consider any $\continuous{2}$	 loss $L$ with zero-measure stationary set $\{x\mid \nabla L(x)=0\}$. For every $\rho>0$, except countably many learning rates, for almost all initialization and all $t$, the iterate of full-batch SAM (\Cref{eq:sam}) $x(t)$ has non-zero gradient and is thus well-defined.
\end{theorem}
 
\begin{theorem}\label{thm:1sam_well_definedness}
Consider any $\continuous{2}$	 losses $\{L_k\}_{k=1}^M$ with zero-measure stationary set $\{x\mid \nabla L_k(x)=0\}$ for each $k\in [M]$. For every $\rho>0$, except countably many learning rates $\eta$, for almost all initialization and all $t$, with probability one of the randomness of the algorithm, the iterate of stochastic SAM (\Cref{eq:1sam}) $x(t)$ has non-zero gradient and is thus well-defined. \footnote{Though we call \Cref{eq:1sam} 1-SAM, but our result here applies to any batch size where $L_k$ can be regarded as the loss for $k$-th possible batch and $M$ is the number of the total number of batches.}
\end{theorem}

Before present our main theorem (\Cref{thm:general_well_definednesa}), we need to introduce some notations first. For a map $F$ mapping from $\R^D\setminus Z\to \R^D$, we define that $F_\eta:\R^D\setminus Z\to\R^D $ as $F_\eta(x) \triangleq x - \eta F(x)$ for any  $\eta\in \R^+$. Given a sequence of functions $\{F^n\}_{n=1}^\infty$, we define $F^n_\eta(x) \triangleq x-\eta F^n(x)$, for any $x\in\R^D$. We further define that $\overline F_\eta^n(x)\triangleq F_\eta^{n}(\overline F^{n-1}_\eta(x))$ for any $n\ge 1$ and that $\overline F_\eta^0(x) =x$. 

\begin{theorem}
\label{thm:general_well_definednesa}
Let $Z$ be a closed subset of $\R^D$ with zero Lebesgue measure and $\mu$ be any probability measure  on $\R^D$ that is absolutely continuous to the Lesbegue measure.  For any sequence of $\continuous{1}$ functions $F^n:\R^D\setminus Z\to \R^D, n\in\mathbb{N^+}$, the following claim holds for all except countably many $\eta\in\R^+$:
\begin{align}
\mu\left( \{x\in \R^D \mid \exists n \in \mathbb{N}, \overline F_\eta^n(x) \in Z \text{ and } \forall 0\le i\le n-1,\overline F^i_\eta(x)\notin Z	\}\right) =0. \notag
\end{align}

In other words, for almost all $\eta$ (except countably many positive numbers), iteration $x(t+1) = x(t) -\eta F(x(t)) = \overline F^t_{\eta}(x(0))$ will not enter $Z$ almost surely, provided that $x(0)$ is sampled from $\mu$.

\end{theorem}

\Cref{thm:nsam_well_definedness} and \Cref{thm:1sam_well_definedness} follows immediately from \Cref{thm:general_well_definednesa}. 
\begin{proof}[Proof of \Cref{thm:nsam_well_definedness}]
	Let $F(x) = \nabla L(x + \rho \frac{\nabla L(x)}{\norm{\nabla L(x)}_2})$ and $Z = \{x\in\R^D\mid \nabla L(x) = 0\}$. We can easily check $F$ is $\mathcal{C}^1$ on $\R^D\setminus Z$ and by assumption $Z$ is a zero-measure set. Applying \Cref{thm:general_well_definednesa} with $F^n\equiv F$ for all $n\in\mathbb{N}^+$, we get the desired results.
\end{proof}

\begin{proof}[Proof of \Cref{thm:1sam_well_definedness}]
	Let $G^k(x) = \nabla L(x + \rho \frac{\nabla L_k(x)}{\norm{\nabla L_k(x)}_2})$ and $Z = \cup_{k=1}^M\{x\in\R^D\mid \nabla L_k(x) = 0\}$. We can easily check $F_k$ is $\mathcal{C}^1$ on $\R^D\setminus Z$ and by assumption $Z$ is a zero-measure set. Applying \Cref{thm:general_well_definednesa} with $F^n= G^{k_n}$ for all $n\in\mathbb{N}^+$ where $k_n$ is the $n$th data/batch sampled by the algorithm, we get the desired results.
\end{proof}

Now we will turn to the proof of \Cref{thm:general_well_definednesa}, which is based on the following two lemmas.

\begin{lemma}\label{lem:almost_invertible}
	Let $Z$ be a closed subset of $\R^D$ with zero Lebesgue measure and $F:\R^D\setminus Z\to \R^D$ be a continuously differentiable function. Then  except countably many $\eta\in\R^+$, $\{x\in\R^d\setminus Z \mid \det(\partial F_\eta(x)) = 0\}$ is a zero-measure set under Lebesgue measure.
\end{lemma}

\begin{lemma}\label{lem:almost_zero_measure_preimage}
	Let $Z$ be a closed subset of $\R^D$ with zero Lebesgue measure and $H:\R^D\setminus Z\to \R^D$ be a continuously differentiable function. If $\{x\in\R^d\setminus Z \mid \det(\partial H(x)) = 0\}$ is a zero-measure set, then for any zero-measure set $Z'$, $H^{-1}(Z')$ is a zero-measure set.
\end{lemma}

\begin{proof}[Proof of \Cref{thm:general_well_definednesa}]
It suffices to prove that for every $N\in\mathbb{N}^+$, at most for countably many $\eta$:
\begin{align}\label{eq:well_definedness_inverse}
\mu\left( \{x\in \R^D \mid \overline F_\eta^N(x) \in Z \text{ and } \forall 0\le i\le N-1,\overline F^i_\eta(x)\notin Z	\}\right) =0. 
\end{align}

The desired results is immediately implied by the above claim because the countable union of countable set is still countable, and countable union of zero-measure set is still zero measure.

 To prove $\Cref{eq:well_definedness_inverse}$, we first introduce some notations. For any $\eta>0$, $0\le n\le N-1$, and  $x\in\R^D$, we define $\overline F_{\eta}^{-(n+1)}(x)\triangleq  (F_{\eta}^{N-n})^{-1}(\overline F_{\eta}^{-n}(x))$, where $\overline F^0_\eta(x) = x$. We extend the definition to set in a natural way, namely $\overline F_{\eta}^{-n}(S)\triangleq \cup_{x\in S}\overline F_{\eta}^{-n}(x)$ for any $S\subseteq \R^D$. Under this notation, we have that 
 \begin{align}
  F_{\eta}^{-N}(Z) = \mu\left( \{x\in \R^D \mid \overline F_\eta^N(x) \in Z \text{ and } \forall 0\le i\le N-1,\overline F^i_\eta(x)\notin Z	\}\right)  	\notag
 \end{align}
We will prove by induction. We claim that for each $0\le n\le N$ except for countably many $\eta\in\R^+$, $ \overline F_{\eta}^{-n}(Z)$ has zero Lebesgue measure. The base case $n=0$ is by trivial as $Z$ is assumed to be zero-measure. Suppose this holds for  $n$. By \Cref{lem:almost_invertible}, except countably many $\eta\in\R^+$, $\{x\in\R^d\setminus  \overline F_{\eta}^{-n}(Z) \mid \det(\partial F^{N-n-1}_\eta(x)) = 0\}$ is a zero-measure set. Next by \Cref{lem:almost_zero_measure_preimage} if for some $\eta\in\R^+$, $\{x\in\R^d\setminus  \overline F_{\eta}^{-n}(Z) \mid \det(\partial F^{N-n-1}_\eta(x)) = 0\}$ is a zero-measure set, then $\overline F^{-n-1}_\eta(Z) = (F_\eta^{N-n-1})^{-1}( \overline F_{\eta}^{-n}(Z))$ is a zero-measure set. Then by induction, we know that except countably many $\eta\in\R^+$, for all integer $0\le n\le N$, $\overline F^{-n}_\eta(Z)$ is zero-measure.  Since $\mu$ is absolutely continuous to Lebesgue measure, $\mu( F_{\eta}^{-N}(Z))=0$.
\end{proof}

We end this section with the proofs of \Cref{lem:almost_invertible,lem:almost_zero_measure_preimage}.

\begin{proof}[Proof of \Cref{lem:almost_invertible}]
	We use $\lambda_i(x)$ to denote that the real part of the $i$th eigenvalue of the matrix $\partial F(x)$ in the descending order. Since $\partial F(x)$ is continuous in $x$, $\lambda_i(x)$ is continuous in $x$ as well, for any $i\in[D]$, and thus $\{x\in \R^D\setminus Z\mid \lambda_i(x)=1/\eta\}$ is a measurable set. Note that for a fixed $i\in[D]$, for each positive integer $n$, let $I_n$ be the set of $\eta$ where $\mu(\{x\in \R^D\setminus Z\mid \lambda_i(x)=1/\eta\})>1/n$, then $|I_n|\le n$, because 
 \begin{align}
     \frac{|I_n|}{n} \le \sum_{\eta\in I_n}\mu((\{x\in \R^D\setminus Z\mid \lambda_i(x)=1/\eta\} )\le \mu((\{x\in \R^D\setminus Z\mid 1/\lambda_i(x) \in I_n\} )\le 1. \notag
 \end{align}

 Therefore, there are at most countably many $\eta\in\R^+$, such that $\mu(\{x\in \R^D\setminus Z\mid \lambda_i(x)=1/\eta\})>0$. Further note that $\det(\partial F_\eta(x)) = 0 \iff \exists i\in[D], \lambda_i(x)= 1/\eta$, we know that there are at most countably many $\eta\in\R^+$, such that $\mu(\{x\in \R^D\setminus Z\mid \det(\partial F_\eta(x))=0\})=0$. This completes the proof.
\end{proof}

\begin{proof}[Proof of \Cref{lem:almost_zero_measure_preimage}]
  Denote $\{x\in \R^D\setminus Z\mid \det(\partial H(x))= 0\}$ by $Z''$, since $\det(\partial H(x))$ is continuous in $x$ as $F$ is $\mathcal{C}^1$, $Z''$ is relatively closed in $\R^D\setminus Z$. Since $Z'$ is a closed set, $\R^D\setminus (Z'\cup Z'')$ is open.  Thus for all $x\in\R^D\setminus (Z'\cup Z'')$ with $\det(\partial H(x))\neq 0$, there exists a open neighborhood of $x$, $U$, where for all $x'\in U$, $\det(\partial H(x'))\neq 0$, since  thus $\det(\partial H(x))$ is continuous. This further implies $H$ is invertible on $U$ and its inverse $(H\vert_U)^{-1}$ is differentiable on $F(U)$.  Therefore, $(H\vert_U)^{-1}$ maps any zero-measure set to a zero-measure set. In particular, $(H\vert_U)^{-1}(Z'\cap H(U))$ is zero measure, so is $(H)^{-1}(Z')\cap U \subset (H\vert_U)^{-1}(Z'\cap H(U))$. Now  for every $x\in \R^D\setminus Z$ we take an open neighborhood $U_x \subseteq \R^D\setminus (Z'\cup Z'')$. Since $\R^D$ is a separable metric space, the open cover of $\R^D$, $\{U_x\}_{x\in\R^D\setminus (Z'\cup Z'')}$ has a countable subcover, $\{U_x\}_{x\in I}$, where $I$ is a countable set of $\R^D\setminus (Z'\cup Z'')$. Therefore we have that $H^{-1}(Z') \setminus (Z'\cup Z'') = H^{-1}(Z') \cap (\R^D\setminus (Z'\cup Z''))= \cup_{x\in I} H^{-1}(Z')\cap U_{x} $ is a zero-measure set. Thus $ H^{-1}(Z')$ is also zero-measure since $Z', Z''$ are both zero-measure. This completes the proof.
\end{proof}

We end this section with an example where SAM is undefined with constant probability.

\begin{theorem}\label{thm:sam_zero_gradient}
	For any $\eta,\rho>0$, there is a $\continuous{2}$ loss function $L:\R\to\R$ satisfying that (1) $L$ has a unique stationary point  and (2) the set of initialization that makes SAM with learning rate $\eta$ and perturbation radius $\rho$ to reach the unique stationary point has positive Lebesgue measure.
\end{theorem}

\begin{proof}[Proof of \Cref{thm:sam_zero_gradient}]
	We first consider the case with $\rho=\eta=1$ with
	\begin{align}
	L(x)=\begin{cases}
	x^2/2 +x +1/2 , &\text{for } x\in (-\infty,-2);\\
	x^4/64 + x^2/8, &\text{for } x\in[-2,2];\\
	x^2/2 -x + 1/2, &\text{for } x\in (2,\infty).
\end{cases}
	\end{align}

We first check $L$ is indeed $\continuous{{1}}$: $L(2)=L(-2)=1/2$, $L'(2) = -L'(-2) = 1$ and $L''(2)=L''(-2)=1$. Now we claim that for all $|x(0)|>2$, $x(1)=0$, which is a stationary point.
Note that $L$ is even and monotone increasing on $[0,\infty)$, we have $\nabla L(x)/|\nabla L(x)|=\sign(x)$. Thus for $|x(t)|>1$, it holds that $|x(t) + \sign(x(t)|>2$ and therefore
\begin{align}
x(t+1) 
=& x(t) - \eta  L'(x(t) + \rho \nabla L(x)/|\nabla L(x)|)\notag\\
=& x(t) -  	L'(x(t) + \sign(x(t)))\notag \\
= & x(t) - \left(x(t)+ \sign(x(t)) - \sign(x(t)+\sign(x(t))) \right)\notag\\
=  & x(t) -x(t) =0.
\end{align}
 
 Now we turn to the case with arbitrary positive $\eta,\rho$. It suffices to consider 
 $L_{\eta,\rho}(x)\triangleq \frac{\rho}{\eta}L(\frac{x}{\rho})$. We can use the calculation for $\rho=\eta =1$ to verify for any $|x|>2\rho$, 
 $$L_{\eta,\rho}'(x+\rho\sign(L_{\eta,\rho}'(x))) = L_{\eta,\rho}'(x+\rho \sign(x))  = \frac{1}{\eta} L(x/\rho +\sign(x)) = \frac{x}{\eta},$$
 namely $x-L_{\eta,\rho}'(x+\rho\sign(L_{\eta,\rho}'(x))) = 0$. This completes the proof.
\end{proof}

A common (but wrong) intuition here is that, for a continuously differentiable update rule, as long as the points where the update rule is ill-defined (here it means the points with zero gradient) has zero measure, then almost surely for all initialization, gradient-based optimization algorithms like SAM will not reach exactly at any stationary point. However the above example negate this intuition. The issue here is that though a differentiable map (like SAM $x\mapsto x-\eta\nabla L(x+\rho \frac{\nabla L(x)}{\norm{\nabla L(x)}}_2)$) always maps the zero-measure set to zero-measure set, the preimage of zero-measure set is not necessarily zero-measure, as the map $x\mapsto x-\eta\nabla L(x + \rho \frac{\nabla L(x)}{\norm{\nabla L(x)}}_2)$ is not necessarily invertible. The update rule of SAM is not invertible at $0$ is    exactly the reason of why  preimage of $0$ has a positive measure.

\section{Proof Setups}\label{appsec:proof_setup}

In this section we provide details of our proof setups, including notations and assumptions/settings.

We first introduce some additional notations that will be used in the proofs. For  any subset $S\in\R^D$, we define $\dist(x,S) \triangleq \inf_{y \in S} \| x - y \|_2$. For any $d>0$ and any subset $S\in\R^D$, we define $S^{d} \triangleq \{x \in \R^D\mid \dist(x,S) \le d \}$. Our convention is to use $K$ to denote a compact set and $U$ to denote an open set.

Below we restate our main assumption in the full-batch case and related notations in \Cref{notation}. Throughout the analysis,  we fix our initialization as $\xinit$, our loss function as $L: \R^D \to \R$.

\assumsmooth*

\paragraph{Notations for Full-Batch Setting:} 
Given any  point $x \in \Gamma$, define $\projn$ as the projection operator onto the manifold of the normal space of $\Gamma$ at $x$ and $\projt = I_D - \projn$. Given the loss function $L$, its gradient flow is denoted  by mapping $\phi: \R^D \times [0, \infty) \to \R^D$. Here, $\phi(x, \tau)$ denotes the iterate at time $\tau$ of a gradient flow starting at $x$ and is defined as the unique solution of  $\phi(x,\tau) = x - \int_0^{\tau} \nabla L(\phi(x,t))dt$, $\forall x\in\R^D$.
We further define the limiting map of $\phi(x,\cdot)$ as $\Phi(x) = \lim_{\tau \to \infty} \phi(x,\tau)$, that is, $\Phi(x)$ denotes the convergent point of the gradient flow starting from $x$.
For convenience, we define $\lambda_i(x), v_i(x)$ as $\lambda_i(\nabla^2 L(\Phi(x))),v_i(\nabla^2 L(\Phi(x)))$ whenever the latter is well defined. When $x(t)$ and $\Gamma$ is clear from context, we also use $\lambda_i(t) \coloneqq \lambda_i(x(t)), v_i(t) \coloneqq v_i(x(t)), \projt[t] \coloneqq \projt[\Phi(x(t))], \projn[t] \coloneqq \projn[\Phi(x(t))] $.

\attractionset*

Below we restate the setting for stochastic loss of batch size one in \Cref{sec:1sam}. 

\stochasticsetting*

\derivemanifoldassumption*

In our analysis, we prove our main theorems in the stochastic setting under a more general condition than \Cref{setting:1sam}, which is \Cref{cond:rank1} (on top of \Cref{assump:smoothness}). The only usage of \Cref{setting:1sam} in the proof is \Cref{thm:derive_manifold_assumption,thm:rank_1_manifold}. 
\begin{restatable}{condition}{generalcondition}
\label{cond:rank1}
Total loss $L=\frac{1}{M}\sum_{k=1}^M L_k$. For each $k\in[M]$, $L_k$ is $\continuous{4}$, and there exists a $(D-1)$-dimensional $\continuous{2}$-submanifold of $\R^D$, $\Gamma_k$, where for all $x \in \Gamma_k$, $x$ is a global minimizer of $L_k$, $L_k(x) = 0$ and $\text{rank}(\nabla^2 L_k(x)) = 1$. Moreover,  $\Gamma = \cap_{k=1}^M  \Gamma_k$ for $\Gamma$ defined in~\Cref{assump:smoothness}. 
\end{restatable}

\begin{restatable}{theorem}{derivestochasticassumption}\label{thm:rank_1_manifold}
 \Cref{setting:1sam} implies \Cref{cond:rank1}.
\end{restatable}

\paragraph{Notations for Stochastic Setting:}
Since $L_k$ is rank-$1$ on $\Gamma_k$ for each $k\in[M]$, we can write it as $L_k(x)= \Lambda_k(x)w_k(x)w^\top_k(x)$ for any $x\in\Gamma$, where $w_k$ is a continuous function on $\Gamma$ with pointwise unit norm.
Given the loss function $L_k$, its gradient flow is denoted  by mapping $\phi_k: \R^D \times [0, \infty) \to \R^D$. Here, $\phi_k(x, \tau)$ denotes the iterate at time $\tau$ of a gradient flow starting at $x$ and is defined as the unique solution of  $\phi_k(x,\tau) = x - \int_0^{\tau} \nabla L_k(\phi_k(x,t))dt$, $\forall x\in\R^D$.
We further define the limiting map $\Phi_k$ as $\Phi_k(x) = \lim_{\tau \to \infty} \phi_k(x,\tau)$, that is, $\Phi_k(x)$ denotes the convergent point of the gradient flow starting from $x$.
Similar to \Cref{defi:attraction_set}, we define $U_k = \{x \in \R^D | \Phi(x) \text{ exists and } \Phi_k(x) \in  \Gamma_k \}$ be the attraction set of $\Gamma_i$. We have that each $U_k$ is open and $\Phi_k$ is $\overline{\mathcal{C}}^2$ on $U_k$ by Lemma B.15 in \cite{arora2022understanding}.

\begin{definition}
A function $L$ is $\mu$-PL in a set $U$ iff $\forall x \in U$,
$\|\nabla L(x) \|_2^2 \ge 2\mu (L(x) - \inf_{x \in U} L(x))$.
\end{definition}

\begin{definition}
The spectral 2-norm of a $k$-order tensor $X_{i_1,...,i_k} \in R^{d_1 \times ...\times d_k}$ is defined as 
\begin{align*}
    \|X \|_2 = \max_{x_i \in R^{d_i}, \|x_i\|_2 = 1} X[x_1,...,x_k].
\end{align*}
\end{definition}

\begin{lemma}[\citet{arora2022understanding} Lemma B.2]
\label{lem:mu-pl}
Given any compact set $K\subseteq \Gamma$, there exist  $r(K), \mu(K),\Delta(K) \in \R^+$ such that
\begin{itemize}
    \item [1.] $K^{r(K)} \cap \Gamma$ is compact.
    \item [2.] $K^{r(K)} \subset U \cap  (\cap_{k \in [M] } U_k)$.
    \item [3.] $L$ is $\mu(K)$-PL on $K^{r(K)}$.
    \item [4.] $\inf_{x \in K^{r(K)}} (\lambda_1(\nabla^2 L(x)) - \lambda_2(\nabla^2 L(x))) \ge \Delta(K) > 0$.
    \item [5.] $\inf_{x \in K^{r(K)}} \lambda_M(\nabla^2 L(x)) \ge \mu(K) > 0$.
    \item [6.] $\inf_{x \in K^{r(K)}} \lambda_1(\nabla^2 L_k(x)) \ge \mu(K) > 0$.
\end{itemize}
\end{lemma}

Given compact set $K\subset \Gamma$, we further define 
\begin{align*}
    &\lspectraltwo(K) = \sup_{x \in K^{r(K)}} \| \nabla^2 L(x)\|_2,\ 
    \lspectralthree(K) =  \sup_{x \in K^{r(K)}} \| \nabla^3 L(x)\|_2,\   
    \lspectralfour(K) =  \sup_{x \in K^{r(K)}} \| \nabla^4 L(x)\|_2,\ \\
    &\phispectraltwo(K) =  \sup_{x \in K^{r(K)}} \| \nabla^2 \Phi(x)\|_2,\ 
    \phispectralthree(K) =  \sup_{x,y \in K^{r(K)}} \frac{\| \nabla^2 \Phi(x) - \nabla^2 \Phi(y)\|_2}{\|x-y\|_2}.
\end{align*} 
Similarly, we use notations like $\lspectraltwo_k(K),\lspectralthree_k(K),\lspectralfour_k(K),\phispectraltwo_k(K),\phispectralthree_k(K)$ to denote the counterpart of the above quantities defined for stochastic loss $L_k$ and its limiting map $\Phi_k$ for $k\in[M]$.

\begin{lemma}[\citet{arora2022understanding}, Lemma B.5 and B.7]
\label{lem:smallerzone}
Given any compact subset $K\subset \Gamma$, let $r(K)$ be defined in \Cref{lem:mu-pl}, there exist $0< h(K) < r(K)$ such that
\begin{itemize}
    \item [1.] $\sup \limits_{x \in K^{h(K)}} L(x) - \inf\limits_{x \in K^{h(K)}} L(x)\le \frac{\mu(K) \rho^2(K)}{8}$.
    \item [2.] $\forall x \in K^{h(K)}, \Phi(x) \in K^{r(K)/2}$.
    \item [3.]$\forall x \in K^{h(K)}, \|x - \Phi(x)\|_2\le \frac{8 \mu(K)^2}{\lspectraltwo(K)\lspectralthree(K)}$.
    \item [4.] The whole segment $\overline{x\Phi(x)}$ lies in $K^{r(K)}$, so does $\overline{x\Phi_k(x)}$, for any $k\in[D]$.
\end{itemize}
\end{lemma}

The proof of the lemmas above can be found in~\citet{arora2022understanding}. Readers should note that although~\citet{arora2022understanding} only prove these lemmas when $K$ is a special compact set (the trajectory of an ODE), all the proof does not use any property of $K$ other than it is a compact subset of $\Gamma$, and thus our \Cref{lem:mu-pl,lem:smallerzone} hold for general compact subsets of $\Gamma$.

In the rest part of the appendix, for convenience we will drop the dependency on $K$ in various constants when there is no ambiguity.

\subsection{Proofs of \Cref{thm:derive_manifold_assumption,thm:rank_1_manifold}}
\label{sec:proof_rank_1}

\begin{proof}[Proof of \Cref{thm:derive_manifold_assumption}]
	Define $F:\R^D\to \R^M$ as $[F(x)]_k= f_k(x),\forall k\in[M]$. 
Let $T_x\triangleq \mathrm{span}(\{\nabla f_k(x)\}_{k=1}^M) $ and  $T_x^\perp$ be the orthogonal complement of $T_x$ in $\R^D$. Now we apply implicit function theorem on $F$ at each $x\in \Gamma$. Without loss of generality (e.g. by rotating the coordinate system), we can assume that $x=0$, $T_x = \R^{D-M}\times\{0\}$, and that $T_x^\perp = \{0\}\times \R^{M}$. Implicit function theorem ensures that there are two open sets $0\in U\subset \R^{D-M}$ and $0\in V\subset \R^{M}$ and an invertible $\continuous{4}$ map $g:U\to V $ such that \begin{align}
F^{-1}(Y)\cap (U\times V) = \{ (u,g(u)) \mid u\in U\}, 	\notag
 \end{align}
 where $Y\triangleq [y_1,\ldots, y_M]\in\R^M$. Moreover, $\{\nabla f_k(x)\}_{k=1}^M$ is linearly independent for every $x'\in U\times V$. Thus by definition of $\Gamma$, it holds that $\Gamma\cap (U\times V) = F^{-1}(y)\cap (U\times V) = \{ (u,g(u)) \mid u\in U\}$. Now for $x=(u,v)\in U\times V$, we define $\psi:U\times V\to \R^D$ by $\psi(u,v)\triangleq (u, v- g(u))$. We can check that $\psi$ is $\continuous{4}$ and $\psi(\Gamma\cap (U\times V)) = \{ (u, v-g(u)) \mid v=g(u), u\in U)\} = \{(u,0) \mid u\in U)\} = U\times \{0\} = (\R^{D-M}\times \{0\})\cap \psi(U)$. This proves that $\Gamma$ is a $\continuous{4}$ submanifold of $\R^D$ of dimension $D-M$. (c.f. \Cref{defi:manifold}) Since $\argmin_{y'\in\R}\ell(y',y) = y$ for any $y\in\R$, it is clear that $\forall x\in \Gamma$, $x$ is a global minimizer of $L$. Finally we check the rank of Hessian of loss $L$. Note that for any $x\in\Gamma$, $\nabla^2 L_k(x) = \frac{\partial ^2\ell(y',y_k)}{(\partial y')^2}\vert_{y'=y_k} \nabla f_k(x)(\nabla f_k(x))^\top$ and that $\frac{\partial ^2\ell(y',y_k)}{(\partial y')^2}\vert_{y'=y_k}>0$, $\rank(\nabla^2 L(x)) = \rank(\partial F(x)) = M$. This completes the proof.
\end{proof}

\begin{proof}[Proof of \Cref{thm:rank_1_manifold}]
\
\begin{enumerate}
    \item $L = \frac 1 M \sum_{k = 1}^M L_k$ by definition.
    \item $\forall k \in [M]$, $L_k(x) = \ell (f_k(x), y_k)$ is $\mathcal{C}^4$ as $\ell$ and $f_k$ are both $\mathcal{C}^4$.
    \item For any $x \in \Gamma$, by~\Cref{lem:informal_direction_stocastic_dl}, we have $\nabla f_k(x)  \neq 0$. Then there exists an open neighborhood $V_k$ such that $\Gamma \subset V_k$ and $\nabla f_k(x) \neq 0$ for any $x \in V_k, k \in [M]$. Then applying implicit function theorem as in the proof of \Cref{thm:derive_manifold_assumption}, for any $k \in M$ there exists a $(D-1)$-dimensional $\continuous{4}$-manifold $\Gamma'_{k} \subset V_k$, such that for any $x' \in V$, $f_k(x') = y_k$ if and only if $x' \in \Gamma'_{k}$. As for any $x \in \Gamma \subset V_k$, $f_k(x') = y_k$, we can infer that $\Gamma \subset \Gamma'_k$. Then $\Gamma \subset \cup_{k=1}^M \Gamma_k$.
    \item For any $x \in \Gamma_k$, we have $f_k(x) = y_k$, which implies $L_k(x) = 0$. Also as $x \in V$,$\nabla f_k(x) \neq 0$. By~\Cref{lem:rank1hessian}, we have $\mathrm{rank}(\nabla^2 L(x)) = 1$.
\end{enumerate}
\end{proof}
\section{Properties of Limiting Map of Gradient Flow, $\Phi$}\label{appsec:property_phi}
\label{technicalaboutphi}

In our analysis, the property of $\Phi$ will be heavily used. In this section, we will recap some related lemmas from \citet{arora2022understanding}, and then introduce some new lemmas for the stochastic setting with batch size one. 

\begin{lemma}[\citet{arora2022understanding} Lemma B.6]
\label{lem:bounddl}
Given any compact set $K \subset \Gamma$, for any $x \in K^h$, 
\begin{align*}
    \|x - \Phi(x) \|_2 \le \int_0^{\infty} \| \frac{d \phi(x,t)}{dt} \|_2 \le \sqrt{\frac{2(L(x)-L(\Phi(x)))}{\mu}} \le \frac{\| \nabla L(x) \|_2}{\mu}\,.
\end{align*}
\end{lemma}

\begin{lemma}
\label{lem:revertbounddl}
Given any compact set $K \subset \Gamma$, for any $x \in K^h$,
\begin{align*}
    \| \nabla L(x) \|_2 \le \lspectraltwo \|x - \Phi(x) \|_2 \le \lspectraltwo \sqrt{\frac{2(L(x)-L(\Phi(x)))}{\mu}}\,.
\end{align*}
\end{lemma}

\begin{proof}[Proof of~\Cref{lem:revertbounddl}]
The first inequality is by~\Cref{lem:mu-pl} and Taylor Expansion. The second inequality is by~\Cref{lem:bounddl}.
\end{proof}

\begin{lemma}[\citet{arora2022understanding} Lemmas B.16 and B.22]
\label{lem:relatelphi}
\begin{align*}
    \partial \Phi(x) \dl{x}  &= 0,  x \in U; \\
    \dpo x  \dlt x \dl x &=  - \dpt x \left[\dl x, \dl x \right], x \in U; \\
    \dpo x \partial^2 (\nabla L)(x)[v_1,v_1] &= \projt[x] \nabla(\lambda_1(\nabla^2(L(x)))), x\in \Gamma.
\end{align*}
\end{lemma}

\begin{lemma}[\citet{arora2022understanding} Lemmas B.8 and B.9]
\label{lem:directiondl}
Given any compact set $K \subset \Gamma$, for any $x \in K^h$,
\begin{align*}
  \|\projt[\Phi(x)] (x - \Phi(x)) \|_2 &\le \frac{\lspectraltwo \lspectralthree}{4\mu^2} \|x - \Phi(x)\|_2^2; \\
    \|\dl x - \dlt{\Phi(x)}(x - \Phi(x)) \|_2 &\le \frac{\lspectralthree}{2} \|x - \Phi(x) \|_2^2;\\
    \left |\frac{\|\dl x\|_2}{\|\dlt{\Phi(x)}(x - \Phi(x))\|_2} - 1\right| &\le \frac{2\lspectralthree}{ \mu} \|x - \Phi(x) \|_2; \\
    \ndl{x} &= \frac{\dlt{\Phi(x)}(x - \Phi(x)) }{\|\dlt{\Phi(x)}(x - \Phi(x))\|_2} + O(\frac{\lspectralthree}{\mu} \| x-\Phi(x)\|_2).
\end{align*}
\end{lemma}

The proof of above lemmas can be found in~\cite{arora2022understanding}.

\begin{lemma}
\label{lem:stochastic_phi_lk}
Given any compact set $K \subset \Gamma$, for any $x \in K^h$,
\begin{align*}
    \|\partial \Phi(x) \nabla L_k(x)\|_2 &\le (\lspectralthree_k  + \lspectraltwo_k \phispectraltwo) \|x - \Phi(x)\|_2^2 \\
    \|\partial \Phi(x) \nabla^2 L_k(x) \ndli{x} \|_2 &\le (\lspectralthree_k + \lspectraltwo_k\phispectraltwo ) \|x - \Phi(x)\|_2
\end{align*}
\end{lemma}

\begin{proof}[Proof of~\Cref{lem:stochastic_phi_lk}]
By~\Cref{lem:smallerzone} and Taylor Expansion,
\begin{align*}
 \| \partial \Phi(x) \nabla L_k(x) \|_2 &\le \| \partial \Phi(x) \nabla^2 L_k(\Phi(x))(x - \Phi(x)) \|_2 + \lspectralthree_k \|x - \Phi(x) \|_2^2 \\
 &\le \| \partial \Phi(\Phi(x)) \nabla^2 L_k(\Phi(x))(x - \Phi(x)) \|_2 + \lspectralthree_k \|x - \Phi(x) \|_2^2  + \lspectraltwo_k \phispectraltwo  \| x - \Phi(x)\|_2^2 \\
 &= \| \projt \partial\Phi(\Phi(x)) \nabla^2 L_k(\Phi(x))(x - \Phi(x)) \|_2 + \lspectralthree_k \|x - \Phi(x) \|_2^2  + \lspectraltwo_k \phispectraltwo  \| x - \Phi(x)\|_2^2  \\
 &= (\lspectralthree_k  + \lspectraltwo_k \phispectraltwo) \|x - \Phi(x) \|_2^2,
\end{align*}
this proves the first claim.

Again by~\Cref{lem:mu-pl} and Taylor Expansion,
\begin{align*}
    \|\partial \Phi(x) \nabla^2 L_k(x) \ndli{x} \|_2 &\le \| \partial \Phi(x) \nabla^2 L_k(\Phi(x))\ndli{x} \|_2 + \lspectralthree_k\| x- \Phi(x)\|_2 \\
    &\le \| \partial \Phi(\Phi(x)) \nabla^2 L_k(\Phi(x))\ndli{x} \|_2 + (\lspectralthree_k + \lspectraltwo_k \phispectraltwo) \|x - \Phi(x) \|_2 \\
    &= (\lspectralthree_k + \lspectraltwo_k \phispectraltwo) \|x - \Phi(x) \|_2,
\end{align*}
this proves the second claim.
\end{proof}

\begin{lemma}
\label{lem:boundedphi}
Suppose $x \in K^h$ and $y = x -\eta \dl{x+ \rho \ndl{x}}$, 
\begin{align*}
    \| y - x \|_2 &\le \eta \|\dl{x}\|_2 + \eta \lspectraltwo \rho \\
    \|\Phi(x) - \Phi(y) \|_2 &\le \phispectraltwo \eta \rho \|\dl x \|_2 + \lspectralthree \eta \rho^2  + \phispectraltwo \eta^2 \|\dl x \|_2^2 + \phispectraltwo  \lspectraltwo^2 \eta^2 \rho^2 \\
    &\le  \lspectraltwo \phispectraltwo \eta \rho \|x - \Phi(x)\|_2 + \lspectraltwo^2 \phispectraltwo \eta^2\|x - \Phi(x)\|_2^2  + \lspectralthree \eta \rho^2  + \phispectraltwo  \lspectraltwo^2 \eta^2 \rho^2   
\end{align*}
\end{lemma}

\begin{proof}[Proof of~\Cref{lem:boundedphi}]

For sufficient small $\rho$, $x + \rho \ndl{x} \in K^r$. By Taylor Expansion,
\begin{align*}
    \|y - x\|_2 &= \eta\| \dl{x+ \rho \ndl{x}} \|_2 \le \eta \|\dl{x} \|_2 + \eta \lspectraltwo\rho
\end{align*}

This further implies that for sufficiently small $\eta$ and $\rho$, $\overline{xy} \in K^r$.

Again by Taylor Expansion,
\begin{align*}
    \|\partial \Phi(x) (y-x) \|_2 &\le \eta \|\partial \Phi(x) \dl x + \rho \partial \Phi(x) \nabla^2 L(x) \ndl{x} \|_2 + \eta \rho^2 \lspectralthree/2\,. \\
\end{align*}
By \Cref{lem:relatelphi}, $\partial \Phi(x) \dl x = 0$ and $\dpo x  \dlt x \dl x =  - \dpt x \left[\dl x, \dl x \right]$. Hence,
\begin{align*}
    \|\partial \Phi(x) (y - x)\|_2 &\le \eta\rho \| \dl x\|_2 \|\partial^2 \Phi(x) \left[\ndl x, \ndl x \right]  \|_2 +  \eta \rho^2 \lspectralthree/2 \\
    &\le \phispectraltwo \eta \rho \|\dl x \|_2 +  \eta \rho^2 \lspectralthree/2\,.
\end{align*}

As $\overline{xy} \in K^r$, by Taylor Expansion,
\begin{align*}
    \| \Phi(y) - \Phi(x) \|_2 \le \| \partial \Phi(x) (y-x) \|_2 + \phispectraltwo \| y-x \|_2^2/2
\end{align*}

Putting together we have 
\begin{align*}
    \|\Phi(x) - \Phi(y) \|_2 &\le \phispectraltwo \eta \rho \|\dl x \|_2 + \eta \rho^2 \lspectralthree + \phispectraltwo \eta^2 \|\dl x \|_2^2 + \phispectraltwo  \lspectraltwo^2 \eta^2 \rho^2\,.
\end{align*}

Finally, by~\Cref{lem:revertbounddl}, we have
\begin{align*}
    \|\Phi(x) - \Phi(y) \|_2 &\le \phispectraltwo \eta \rho \|\dl x \|_2 + \lspectralthree \eta \rho^2  + \phispectraltwo \eta^2 \|\dl x \|_2^2 + \phispectraltwo  \lspectraltwo^2 \eta^2 \rho^2 \\
    &\le  \lspectraltwo \phispectraltwo \eta \rho \|x - \Phi(x)\|_2 + \lspectraltwo^2 \phispectraltwo \eta^2\|x - \Phi(x)\|_2^2  + \lspectralthree \eta \rho^2  + \phispectraltwo  \lspectraltwo^2 \eta^2 \rho^2   \,.
\end{align*}
This completes the proof.
\end{proof}

\begin{lemma}
\label{lem:stochasticboundedphi}
Suppose $x \in K^h$ and $y = x -\eta \dli{x+ \rho \ndli{x}}$, 
\begin{align*}
    \| y - x \|_2 &\le \eta \|\dli{x}\|_2 + \eta \lspectraltwo \rho\,, \\
    \|\Phi(x) - \Phi(y) \|_2 &\le O(\eta \| \nabla L(x)\|_2^2 + \eta \rho \|\nabla L(x) \|_2 + \eta\rho^2)\,.
\end{align*}
\end{lemma}

\begin{proof}[Proof of~\Cref{lem:stochasticboundedphi}]

For sufficient small $\rho$, $x + \rho \ndli{x} \in K^r$. By Taylor Expansion,
\begin{align*}
    \|y - x\|_2 &= \eta\| \dli{x+ \rho \ndli{x}} \|_2 \le \eta \|\dli{x} \|_2 + \eta \lspectraltwo\rho\,.
\end{align*}

This further implies that for sufficiently small $\eta$ and $\rho$, $\overline{xy} \in K^r$.

Again by Taylor Expansion,
\begin{align*}
    \|\partial \Phi(x) (y-x) \|_2 &\le \eta \|\partial \Phi(x) \dli x + \rho \partial \Phi(x) \nabla^2 L_k(x) \ndli{x} \|_2 + \eta \rho^2 \lspectralthree/2\,. \\
\end{align*}
We further have by~\Cref{lem:bounddl},
\begin{align*}
    &\|\partial \Phi(x) \dli x \| \\
    \le& \| \partial \Phi(\Phi(x)) \dli{x}\| + \phispectraltwo \|\dli x \|_2 \| x - \Phi(x) \| \\
    \le& \| \partial \Phi(\Phi(x)) \nabla^2 L_k(\Phi(x)) (x - \Phi(x)) \| + \lspectralthree \|x - \Phi(x) \|_2^2 +  \lspectraltwo \phispectraltwo \| x - \Phi(x) \|_2^2 \\
    \le& \frac{\lspectralthree}{\mu} \| \nabla L(x) \|_2^2 +  \frac{\lspectraltwo \phispectraltwo }{\mu^2} \| \nabla L(x) \|_2^2 \,.
\end{align*}
Similarly,
\begin{align*}
    &\|\rho \partial \Phi(x) \nabla^2 L_k(x) \ndli{x} \|_2 \\\le& \|\rho \partial \Phi(\Phi(x)) \nabla^2 L_k(x) \ndli{x} \|_2 + \rho \lspectraltwo \phispectraltwo \| x - \Phi(x)\|_2 \\
    \le& \|\rho \partial \Phi(\Phi(x)) \nabla^2 L_k(\Phi(x)) \ndli{x} \|_2+ \rho \lspectraltwo \phispectraltwo \| x - \Phi(x)\|_2 + \rho \lspectralthree \|x - \Phi(x)\|_2 \\
    \le& \rho \frac{\lspectraltwo \phispectraltwo}{\mu^2} \| \nabla L(x) \|_2 + \rho \frac{\lspectralthree}{\mu}  \| \nabla L(x) \|_2\,.
\end{align*}
This completes the proof.
\end{proof}
\section{Analysis for Explicit Bias}
\label{app:explicitbias}

Throughout this section, we assume that \Cref{assump:smoothness} holds.
\subsection{A General Theorem for Explicit Bias in the Limit Case}
\label{sec:explicit_full_appendix}

In this subsection we provide the proof details for \cref{sec:explicit_full}, which shows that  the explicit biases of three notions of sharpness are all different, using our new mathematical tool, \Cref{thm:generalreg}. 

\paragraph{Notation for Regularizers.} Let $R_{\rho}:\R^D\to \R\cup\{\infty\}$ be a family of regularizers parameterized by $\rho$. 
If $R_\rho$ is not well-defined at some $x$, then we let $R_\rho(x)=\infty$. This convention will be useful when analyzing ascent-direction sharpness $\ascsharpness = \ascloss -L$ which is not defined when $\nabla L(x) =0$. This convention will not change the minimizers of the regularized loss. 
Intuitively, a regularizer should always be non-negative, but however, when far away from manifold, there are regularizers $R_\rho(x)$ of our interest that can actually be negative, \emph{e.g.}, $\avgsharpness(x)\approx \frac{\rho^2}{2D}\mathrm{Tr}(\nabla ^2 L(x))$. Therefore we make the following assumption to allow the regularizer to be mildly negative.
\begin{condition}
\label{cond:regular_R}
Suppose for any bounded closed set $B \subset U$, there exists $C > 0$, such that for sufficiently small $\rho$, $\forall x \in B, R_\rho(x) \ge -C \rho^2$.
\end{condition}

\limitingregularizer*
The high-level intuition is that we want to use the notion of limiting regularizer to capture  the explicit bias of $R_\rho$ among the manifold of minimizers $\Gamma$ as $\rho\to 0$, which is decided by the second order term in the Taylor expansion, \emph{e.g.}, \Cref{eq:maxsharp} and \Cref{eq:ascsharp}.
In other words, the hope is that whenever the regularized loss is optimized, the final solution should be in a neighborhood of minimizer $x$ with smallest value of limiting regularizer $S(x)$. However, such hope cannot be true without further assumptions, which motivates the following definition of \emph{good limiting regularizer}. 
 \begin{definition}[Good Limiting Regularizer]\label{defi:good_limiting_regularizer}
 We say the limiting regularizer $S$ of $\{R_\rho\}$ is \emph{good} around some $x^*\in \Gamma$, if $S$ is  non-negative and continuous at $x^*$ and that there is an open set $V_{x^*}$ containing $x^*$, such that   for any $C>0$, $ \inf_{x':\norm{x'-x}_2\le C \rho} R_\rho(x')/\rho^2 $ converges uniformly to $S(x)$ in for all  $x\in \Gamma\cap V_{x^*}$ as $\rho\to 0$.
 
 In other words, a good limiting regularizer satisfy that for any $C,\epsilon>0$, there is some $\rho_{x^*}>0$, 
\begin{align}
    \forall x \in \Gamma\cap V_{x^*} \textrm{ and } \rho \le \rho_{x^*},\quad  \bigl| S(x) - \inf_{\norm{x'-x}_2\le C \cdot\rho} R_\rho(x')/\rho^2\bigr| < \epsilon.\nonumber
\end{align}

We say the limiting regularizer $S$ is \emph{good} on $\Gamma$, if $S$ is good around every point $x\in \Gamma$. In such case we also say $R_\rho$ admits $S$ as a good limiting regularizer on $\Gamma$.
\end{definition}

The intuition of the concept of a good limiting regularizer is that, the value of the regularizer should not drop too fast when moving away from a minimizer $x$ in its $O(\rho)$ neighborhood. If so, the minimizer of the regularized loss may be $\Omega(\rho)$ away from any minimizer to reduce the regularizer at the cost of increasing the original loss, which makes the limiting regularizer unable to capture the explicit bias of the regularizer. (See \Cref{app:nogoodlimitreg} for a counter example)
 We emphasize that the conditions of good limiting regularizer is natural and covers a large family of regularizers, including worst-, ascent- and average-direction sharpness. See \Cref{thm:maxsharp,thm:ascsharp,thm:avgsharp} below.

\begin{theorem}
\label{thm:maxsharp}
Worst-direction sharpness $\maxsharpness$ admits $\lambda_1(\nabla^2 L(\cdot))/2$ as a good limiting regularizer on $\Gamma$ and satisfies \Cref{cond:regular_R}.
\end{theorem}

\begin{theorem}
\label{thm:ascsharp} 
Ascent-direction sharpness $\ascsharpness$ admits $\lambda_M(\nabla^2 L(\cdot))/2$ as a good limiting regularizer on $\Gamma$ and satisfies \Cref{cond:regular_R}.
\end{theorem}

\begin{theorem}
\label{thm:avgsharp}
Average-direction sharpness $\avgsharpness$ admits $\mathrm{Tr}(\nabla^2 L(\cdot))/(2D)$ as a good limiting regularizer on $\Gamma$ and satisfies \Cref{cond:regular_R}.
\end{theorem}

Next we present the main mathematical tool to analyze the explicit bias of regularizers admitting good limiting regularizers, \Cref{thm:generalreg}.

\begin{theorem}
\label{thm:generalreg}
Let $U'$ be any bounded open set such that its closure $\overline{U'}\subseteq U$ and  $\overline {U'}\cap \Gamma =\overline{ U'\cap \Gamma}$. Then for any family of parametrized regularizers  $\{R_\rho\}$ admitting a good limiting regularizer $S(x)$ on $\Gamma$ and satisfying \Cref{cond:regular_R}, for sufficiently small $\rho$, it holds that
\begin{align*}
    \Bigl |\inf_{x\in U'}\big({L(x) + R_\rho(x)}\big) -
    \inf_{x\in U'} L(x) - \rho^2\inf_{x \in U' \cap \Gamma} S(x) \Bigr| \le o(\rho^2).
\end{align*}

Moreover, for sufficiently small $\rho$, it holds  uniformly for all $u\in U'$ that 
\begin{align}
 L(u) +  R_\rho(u) \le \inf_{x\in U'} ({L(x) + R_\rho(x)}) +O(\rho^2) \implies 	 R_\rho(u)/\rho^2 - \inf_{x \in {U'} \cap \Gamma} S(x) \ge -o(1). \notag
\end{align}
\end{theorem}

\Cref{thm:generalreg} says that minimizing the regularized loss $L(u) +  R_\rho(u)$ is not very different from minimizing the original loss $L(u)$ and the regularizer $ R_\rho(u)$ respectively. To see this, we define the following optimality gaps 
\begin{align}
	A(u) & \triangleq      L(u) +  R_\rho(u) -\inf_{x\in U'} ({L(x) + R_\rho(x)}) \ge 0 \notag\\
	B(u) & \triangleq       L(u) - \inf_{x\in U'} L(x) \ge 0 \notag\\
	C(u) & \triangleq  R_\rho(u)/\rho^2 - \inf_{x \in {U'} \cap \Gamma} S(x), \notag
\end{align}
and \Cref{thm:generalreg} implies that $\left|A(u)-B(u)-\rho^2 C(u)\right| = o(\rho^2)$. Moreover, $A(u),B(u)$ are non-negative by definition, and $C(u)\ge -o(1)$ are almost non-negative, whenever $A(u)$ is $O(
\rho^2)$-approximately optimized.

For the applications we are interested in in this paper, the good limiting regularizer $S$ can be continuously extended to the entire space $\R^D$. 
In such a case, the third optimality gap has an approximate alternative form which doesn't involve $R_\rho$, namely $S(u) - \inf_{x \in \overline{U'} \cap \Gamma} S(x)$. 
 \Cref{corr:general_reg_extension} shows minimizing regularized loss $L(u) +  R_\rho(u)$  is equivalent to minimizing the limiting regularizer, $S(u)$ around the manifold of local minimizer, $\Gamma$.

\begin{corollary}\label{corr:general_reg_extension}
Under the setting of \Cref{thm:generalreg}, let $\overline{S}$ be an continuous extension of $S$ to $\R^d$. For any optimality gap $\Delta>0$,  there is a function $\eps:\R^+\to \R^+$ with $\lim_{\rho\to 0}\eps(\rho)=0$, such that for all sufficiently small $\rho>0$ and all $u\in U'$ satisfying that 
$$L(u) +  R_\rho(u) -\inf\limits_{x\in U'}\bigl({L(x) + R_\rho(x)}\bigr)  \le \Delta \rho^2,$$ 
it holds that $ L(u) - \inf_{x\in U'} L(x) \le (\Delta+\eps(\rho))\rho^2$ and  that 
$$ \overline{S}(u)-\inf_{x\in U'\cap \Gamma}\overline{S}(x) \in[-\eps(\rho), \Delta+\eps(\rho)].$$ 
\end{corollary}

\subsection{Bad Limiting Regularizers May Not Capture Explicit Bias}
\label{app:nogoodlimitreg}
In this subsection, we provide an example where a \emph{bad} limiting regularizer cannot capture the explicit bias of regularizer when $\rho\to 0$,  to  justify the necessity of~\Cref{defi:good_limiting_regularizer}. Here a bad limiting regularizer is a limiting regularizer which is not good.

Consider choosing $R_\rho(x) = L(x + \rho e) - L(x)$ with $\| e \|  = 1$ as a fixed unit vector. We will show minimizing the regularized loss $L(x) + R_\rho(x)$ does not imply minimizing the limiting regularizer of $R_\rho(x)$ on the manifold.

By~\Cref{defi:limiting_regularizer} and the continuity of $R_\rho$, the limiting regularizer $S$ of $R_\rho$ is
\begin{align*}
  \forall x\in\Gamma,\quad   S(x)= \lim_{\rho \to  0} \lim_{r \to 0} \inf_{\|x' -x\|_2 \le r} R_\rho(x') / \rho^2 = \lim_{\rho \to  0} R_\rho(x)/\rho^2 = \nabla^2 L(x)[e,e]\ge 0.
\end{align*}

However, for any $x \in \Gamma$, we can choose $x' = x - \rho e$, then
\begin{align*}
    L(x') + R_\rho(x') = L(x' + \rho e) = L(x) = 0.
\end{align*}
Therefore, no matter how small $\rho$ is, minimizing $L(x)+R_\rho(x)$ can return a solution which is $\rho$-close to any point point of $\Gamma$. In other words, the explicit bias of minimizing $L(x)+R_\rho(x)$ is trivial and thus is not equivalent to minimizing the limiting regularizer $S$ on the manifold $\Gamma$.

The reason behind the inefficacy of the limiting regularizer $S$ in explaining the explicit bias of $R_\rho$ is that 
$S(x)$ is not a good limiting regularizer for any $x \in \Gamma$ satisfying $S(x)>0$. To be more concrete, choose $C = 1$ and $\epsilon =  S(x)/2$ in~\Cref{defi:good_limiting_regularizer}. For any $x \in \Gamma$ and sufficiently small $\rho > 0$, considering $x' = x - \rho e_1$, by Taylor Expansion,
\begin{align*}
    R_\rho(x') &=  L(x' + \rho e) - L(x')\\
    &= \rho \langle \nabla L(x'), e \rangle + \rho^2 \nabla^2 L(x')[e,e] + o(\rho^2) \\
    &= \rho \langle \nabla^2 L(x) (x' - x), e \rangle +  \rho^2 \nabla^2 L(x')[e , e] + o(\rho^2) \\
    &= -\rho^2\nabla^2 L(x)[e,e] +   \rho^2 \nabla^2 L(x')[e , e] + o(\rho^2) \\
    &= \rho^2 e^T(\nabla^2 L(x') - \nabla^2 L(x))e +  o(\rho^2) = o(\rho^2)
\end{align*}

This implies $\inf_{\|x' -x\|_2 \le C\rho} R_\rho(x') \le R_\rho(x_1) =  o(\rho^2)$. Hence,
\begin{align*}
    S(x) - \inf_{\|x' -x\|_2 \le C\rho} R_\rho(x')/ \rho^2 \ge S(x) - o(1) > S(x)/2 = \epsilon.
\end{align*}

\subsection{Proof of \Cref{thm:generalreg}}
\label{app:explicitmain}
This subsection aims to prove \Cref{thm:generalreg}. We start with a few lemmas that will be used later.

\begin{lemma}\label{lem:Gamma_relative_close_in_U}
$\Gamma =U \cap \overline \Gamma$.	
\end{lemma}
\begin{proof}[Proof of \Cref{lem:Gamma_relative_close_in_U}]
	For any point $x\in U \cap \overline \Gamma$, there exists $\{x_k\}_{k=1}^\infty \in \Gamma$ such that $\lim_{k\to\infty}x_k=x$. Since $x\in U$ and $\Phi$ is continuous in $U$, it holds that $\Phi$ is continuous at $x$, thus $\lim_{k\to \infty}\Phi(x_k) = \Phi(x)\in \Gamma$. However $\Phi(x_k)=x_k$ because $x_k\in \Gamma, \forall k$. Thus we know $x = \Phi(x) \in \Gamma$. Hence $U \cap \overline \Gamma \subset \Gamma$. The other side is clear because $\Gamma \subset U$ and $\Gamma \subset \overline{\Gamma}$. 
\end{proof}

\begin{lemma}\label{lem:equivalence_of_U}
	Let $U'$ be any bounded open set such that its closure $\overline{U'}\subseteq U$.  If
 $	\overline {U'}\cap \Gamma \subseteq \overline{ U'\cap \Gamma}$, then $ \overline {U'}\cap \Gamma = \overline{ U'\cap \Gamma} $.
\end{lemma}

\begin{proof}[Proof of \Cref{lem:equivalence_of_U}]
	By \Cref{lem:Gamma_relative_close_in_U}, it holds that $\overline {U'}\cap \Gamma = \overline {U'}\cap  U \cap \overline{\Gamma} = \overline{U'}\cap \overline \Gamma$. Note that $\overline{ U'\cap \Gamma}\subseteq \overline {U'}, \overline{ U'\cap \Gamma} \subseteq \overline \Gamma$, we have that $\overline{ U'\cap \Gamma}\subseteq  \overline{U'}\cap \overline \Gamma = \overline {U'}\cap \Gamma$, which completes the proof.
\end{proof}

\begin{lemma}
\label{lem:gap_of_U_Gamma}
Let $U'$ be any bounded open set such that its closure $\overline{U'}\subseteq U$ and  $\overline {U'}\cap \Gamma \subseteq \overline{ U'\cap \Gamma}$. Then for all $h_2 > 0$,$\exists \rho_0 > 0$ if $x \in U', \dist(x, \Gamma) \le \rho_0 \Rightarrow \dist(x, \overline {U'\cap \Gamma}) \le h_2$.
\end{lemma}

\begin{proof}[Proof of \Cref{lem:gap_of_U_Gamma}]
We will prove by contradiction. Suppose  there exists $h_2>0$ and $\{x_k\}_{k=1}^\infty \in U'$, such that $\lim_{k\to\infty}\dist(x_k, \Gamma)=0$ but $\forall k>0, \dist(x_k, \overline {U'\cap \Gamma}) \ge h_2$. Since $U'$ is bounded, $\overline{U'}$ is compact and thus  $\{x_k\}_{k=1}^\infty$ has at least one accumulate point $x^*$ in $\overline{U'}\subseteq U$. Since $U$ is the attraction set of $\Gamma$ under gradient flow, we know that $\Phi(x^*)\in \Gamma$. 
Now we claim $x^*\in \Gamma$. This is because $\lim_{k\to \infty}\dist(x_k, \Gamma)=0$ and thus there exists a sequence of points on $\Gamma$, $\{y_k\}_{k=1}^\infty$, where $\lim_{k\to\infty} \norm{x_k-y_k}=0 $. Thus we have that $x^* = \lim_{k\to \infty} y_k = \lim_{k\to \infty} \Phi(y_k) = \Phi(x^*)$, where the last step we used that $x^*\in U$ and $\Phi$ is continuous on $U$. By the definition of $U$, $x^*\in U\iff \Phi(x^*)\in\Gamma$, thus $x^*\in \Gamma$.
Then we would have $x^* \in  \overline{U'} \cap \Gamma$, which is contradictory to  $ \dist(x_k, \overline {U'}\cap \Gamma) \ge \dist(x_k, \overline {U'\cap \Gamma}) \ge h_2, \forall k>0$. This completes the proof.
\end{proof}

\begin{lemma}
\label{lem:gap_of_U_Gamma_loss}
Let $U'$ be any bounded open set such that its closure $\overline{U'}\subseteq U$ and  $\overline {U'}\cap \Gamma \subseteq \overline{ U'\cap \Gamma}$. Then for all $h_2 > 0$,$\exists \rho_1 > 0$ if $x \in U', L(x) \le \inf_{x\in U'} L(x) + \rho_1 \Rightarrow \dist(x, \overline {U'\cap \Gamma}) \le h_2$.
\end{lemma}

\begin{proof}[Proof of \Cref{lem:gap_of_U_Gamma_loss}]
We will prove by contradiction. If there exists a list of $\rho_1,...,\rho_k,...$, such that $\rho_k \to 0$ and there exists $x_k \in U'$, such that $L(x_k) \le \inf_{x\in U'} L(x) +\rho_k$ and $\dist(x_k, \overline {U'\cap \Gamma}) \ge h_2$. Since $U'$ is bounded, $\overline{U'}$ is compact and thus  $\{x_k\}_{k=1}^\infty$ has at least one accumulate point $x^*$ in $\overline{U'}\subseteq U$. Since $L$ is continuous in $U$, $L(x^*) = \lim_{k\to \infty} L(x_k) = \inf_{x\in U'} L(x)$. Thus $x^*$ is a local minimizer of $L$ and thus has zero gradient, which further implies that $x^*= \Phi(x^*)$. Thus $x^* \in  \overline{U'} \cap \Gamma$, which is contradictory to  $ \dist(x_k, \overline {U'}\cap \Gamma) \ge \dist(x_k, \overline {U'\cap \Gamma}) \ge h_2, \forall k>0$. This completes the proof. 
\end{proof}

\begin{lemma}
\label{lem:boundregularizedloss}
Let $U'$ be any bounded open set such that its closure $\overline{U'}\subseteq U$ and  $\overline {U'}\cap \Gamma =\overline{ U'\cap \Gamma}$. Suppose regularizers $\{R_\rho\}$ admits a limiting regularizer $S$ on $\Gamma$, then 
\begin{align*}
    \inf_{x \in U'} (L(x) + R_\rho(x)) \le \rho^2\inf_{x \in U' \cap \Gamma} S(x) + \inf_{x\in U'  }L(x) + o(\rho^2).
\end{align*}
\end{lemma}

\begin{proof}[Proof of \Cref{lem:boundregularizedloss}]
First choose sufficiently small $\rho$, such that $\rho < h(\overline{U' \cap \Gamma})$.
Choose an approximate minimizer of $S(x)$, $x_0 \in U' \cap \Gamma$, such that $S(x_0) \le \inf_{x \in U' \cap \Gamma} S(x) + \rho^2$. Then by the definition of limiting regularizers~(\Cref{defi:limiting_regularizer}) and the assumption that $U'$ is open, there exists $x_1\in U'$ satisfying that $\|x_1 - x_0\|_2 \le r_\rho < \rho^2$ and $  R_\rho(x_1)/\rho^2 - S(x_0) \le \rho^2$. Thus, $R_\rho(x_1)\le \rho^2 S(x_0) + \rho^4$.

As $\|x_1 -x_0\|_2 \le \rho^2 < h$ and $x_0 \in \overline{U' \cap \Gamma}$. This further leads to $\overline{x_0x_1} \in \overline{U' \cap \Gamma}^{h}$. By Taylor expansion on $L$ at $x_0$, we would have $L(x_1) \le L(x_0)+ O( \|x_0  - x_1\|_2^2) =  \inf_{x\in U' \cap \Gamma}L(x) + O(\rho^4)  $. Thus it holds that  
\begin{align*}
    \inf_{x \in U'} (L(x) + R_\rho(x))\le L(x_1) + R_\rho(x_1) 
   \le  \rho^2\inf_{x \in U' \cap \Gamma} S(x) + \inf_{x\in U'}L(x) + O(\rho^4).
\end{align*}
This completes the proof.
\end{proof}

\begin{lemma}
\label{lem:lower_bound_R}
Let $U'$ be any bounded open set such that its closure $\overline{U'}\subseteq U$ and  $\overline {U'}\cap \Gamma =\overline{ U'\cap \Gamma}$. Suppose regularizers $\{R_\rho\}$ admits a good limiting regularizer $S$ on $\Gamma$, then for all $u \in U'$,
\begin{align*}
    \|u - \Phi(u)\|_2 = O(\rho) \implies  R_\rho(u) \ge \rho^2 \inf_{x \in U' \cap \Gamma} S(x) - o(\rho^2)\,.
\end{align*}
\end{lemma}

\begin{proof}[Proof of \Cref{lem:lower_bound_R}]

Define $r =r(K), h=h(K)$ as the constant in \Cref{lem:mu-pl} with $K = \overline{U' \cap \Gamma}$. Note  $K$ is compact and by \Cref{lem:equivalence_of_U}, $K = \overline{U'} \cap \Gamma   \subset \Gamma$. By \Cref{lem:mu-pl}, we have $K^r\cap \Gamma$ is a compact set, so is $K^h\cap \Gamma$. 
Since $S$ is a good limiting regularizer for $\{R_\rho\}$, by  \Cref{defi:good_limiting_regularizer},  for any $x^*\in K^h\cap \Gamma$, there exists open neighborhood of $x^*$, $V_{x^*}$ such that for any $C,\eps_1>0$, there is a $\rho_{x^*}$ such that
\begin{align}
    \forall x \in V_{x^*}\textrm{ and } \rho \le \rho_{x^*},\quad  \abs{ S(x) - \inf_{\norm{x'-x}_2\le C \cdot\rho} R_\rho(x')/\rho^2} < \epsilon_1.\notag
\end{align}

Note that $K^h\cap \Gamma$ is compact, there exists a finite subset of $K^h\cap \Gamma$, $\{x_k\}_k$, such that $K^h\cap \Gamma \subset \cup_k V_{x_k}$. Hence for any $C, \epsilon_1 > 0$, there is some $\rho_{K} = \min_k \rho_{x_k} > 0$, it holds that,
\begin{align}
\label{eq:control_K_h}
    \forall x \in K^h\cap \Gamma \textrm{ and } \rho \le \rho_{K},\quad  \abs{ S(x) - \inf_{\norm{x'-x}_2\le C \cdot\rho} R_\rho(x')/\rho^2} < \epsilon_1.
\end{align}

We can rewrite~\Cref{eq:control_K_h} as for any $C>0$,
\begin{align}
\label{eq:control_K_h_2}
    \quad  \sup_{x \in K^h\cap \Gamma}\abs{ S(x) - \inf_{\norm{x'-x}_2\le C \cdot\rho} R_\rho(x')/\rho^2} = o(1), \quad \textrm{as }\rho\to 0.
\end{align}

As $u \in U'\subseteq U$, we have that $\Phi(u) \in \Gamma$. If $\| u - \Phi(u) \|_2 = O(\rho)$, then $\dist(u,\Gamma) \le O(\rho)$. By~\Cref{lem:gap_of_U_Gamma}, we have that $\dist(u, K) = o(1)$. This further implies $\dist(\Phi(u),K) \le \dist(u,K)+ \dist(\Phi(u),u) = o(1)$. Hence we have that $\Phi(u) \in K^h \cap \Gamma$ for sufficiently small $\rho$. Thus we can pick $x=\Phi(u)$ in \Cref{eq:control_K_h_2} and $C$ sufficiently large, which yields that
    \begin{align}
    \label{eq:general_reg_side_1}
        \rho^2 S(\Phi(u)) &\le \inf_{\norm{u'-\Phi(u)}_2\le O(\rho)} R_\rho(u') + o(\rho^2) 
        \le R_\rho(u) + o(\rho^2),
    \end{align}
    where the last step is because $\|u - \Phi(u) \|_2 = O(\rho)$. 
 On the other hand, we have that 
    \begin{align}
    \label{eq:general_reg_side_2}
        S(\Phi(u)) \ge \inf_{x \in U' \cap \Gamma} S(x) - o(1)\,.
    \end{align} 
    as $S$ is continuous on $\Gamma$ and $\dist(\overline{U' \cap \Gamma}, \Phi(u)) = o(1)$.  Combining \Cref{eq:general_reg_side_1,eq:general_reg_side_2}, we have $R_\rho(u) \ge \rho^2 \inf_{x \in U' \cap \Gamma} S(x) - o(\rho^2)$.
\end{proof}

\begin{proof}[Proof of \Cref{thm:generalreg}]

We will first lower bound $L(x) + R_\rho(x)$ for $x \in U'$. Suppose $C_{U'}$ is the constant in \Cref{cond:regular_R}.
Define $C_1 = \sqrt{2\frac{C_{U'} + \inf_{x \in \overline{U'} \cap \Gamma} S(x) + 1}{\mu}}$. We discuss by cases. For sufficiently small $\rho$,
\begin{enumerate}
    \item If $x \not \in K^h$, then by~\Cref{lem:gap_of_U_Gamma_loss}, $L(x)$ is lower bounded by a positive constant.
    \item If $x \in K^h$ and $\|x - \Phi(x) \|_2 \ge C_1 \rho$, then by~\Cref{lem:bounddl},
    \begin{align*}
        L(x)\ge \frac{\mu \|x - \Phi(x) \|_2^2}{2} \ge (C_{U'} + \inf_{x \in {U'} \cap \Gamma} S(x) + 1) \rho^2\,.
    \end{align*} This implies $L(x) + R_\rho(x) \ge (\inf_{x \in {U'} \cap \Gamma} S(x) + 1)  \rho^2 + \inf_{x \in U'} L(x)$. 
    \item If $\| x - \Phi(x)\|_2 \le C_1 \rho$, by~\Cref{lem:lower_bound_R}, $R_\rho(x) \ge \rho^2 \inf_{x \in {U'} \cap \Gamma} S(x) - o(\rho^2)$, hence 
    \begin{align*}
        L(x) + R_\rho(x) + o(\rho^2) \ge \inf_{x \in {U'} \cap \Gamma} S(x)  \rho^2 + \inf_{x \in U'} L(x)\,.
    \end{align*}
\end{enumerate}

Concluding the three cases, we have 
\begin{align*}
    \inf_{x \in U'} (L(x) + R_\rho(x)) \ge  \inf_{x \in U' \cap \Gamma} L(x) +\inf_{x \in {U'} \cap \Gamma} S(x)  \rho^2 - o(\rho^2)\,.
\end{align*}

By~\Cref{lem:boundregularizedloss}, we have that 
\begin{align*}
    \inf_{x \in U'} (L(x) + R_\rho(x)) \le \rho^2\inf_{x \in U' \cap \Gamma} S(x) + \inf_{x\in {U'} \cap \Gamma}L(x) + o(\rho^2)\,.
\end{align*}

Combining the above two inequalities, we prove the main statement of~\Cref{thm:generalreg}. 

Furthermore, if $L(u) +  R_\rho(u) \le \inf_{x\in U'} ({L(x) + R_\rho(x)}) +O(\rho^2)$, then by the main statement and~\Cref{cond:regular_R}, we have that 
\begin{align*}
    L(u) - \inf_{x\in U'}L(x) \le& \inf_{x\in U'} ({L(x) + R_\rho(x)}) - R_\rho(u) - \inf_{x\in U'}L(x) +O(\rho^2) \\
    \le&  \rho^2\inf_{x \in {U'} \cap \Gamma} S(x) +  C\rho^2 + O(\rho^2) = O(\rho^2)\,.
\end{align*}

Then by~\Cref{lem:gap_of_U_Gamma_loss}, we have $u \in (\overline{U'} \cap \Gamma )^h$ for sufficiently small $\rho$. By~\Cref{lem:bounddl}, we have $\| u - \Phi(u) \|_2 = O(\rho)$. By~\Cref{lem:lower_bound_R}, we have $R_\rho(u) \ge \rho^2 \inf_{x \in {U'} \cap \Gamma} S(x) - o(\rho^2)$.
\end{proof}

\subsection{Proofs of \Cref{corr:general_reg_extension}}
\label{app:explicitcorr}

\begin{proof}[Proof of \Cref{corr:general_reg_extension}]

 Since $L(u) +  R_\rho(u) -\inf\limits_{x\in U'}\bigl({L(x) + R_\rho(x)}\bigr)  \le \Delta \rho^2= O(\rho^2)$, by \Cref{thm:generalreg}, we have that 
 $$  L(u) - \inf_{x\in U'} L(x) \le (\Delta+o(1))\rho^2, $$
 and
$$ R_\rho(x)-\inf_{x\in U'\cap \Gamma}S(x) \in[-o(1), \Delta+o(1)].$$
Thus it suffices to show $R_\rho(x) - \overline S(x) = o(\rho^2)$. Since $L(u) - \inf_{x\in U'} L(x) \le (\Delta+\eps(\rho))\rho^2 = o(1)$,  by \Cref{lem:gap_of_U_Gamma_loss}, we know $\dist(x, \overline {U'\cap \Gamma}) = o(1)$. Thus by \Cref{lem:bounddl}, $\norm{x-\Phi(x)}=o(1)$, which implies that $ \rho^2 S(\Phi(x)) - o(\rho^2) \le R_\rho(u)$. Since $\overline S$ is an continuous extension,  $\overline S(x)- \overline S(\Phi(x)) = \overline S(x)-  S(\Phi(x))= O(\norm{x-\Phi(x)}_2) = o(1)$.  Thus we conclude that $\overline S(x)\le \overline S(\Phi(x))\le \inf_{x\in U'\cap \Gamma}S(x)+\Delta +o(1)$. On the other hand, $\overline S(x) \ge  S(\Phi(x))-o(1)\ge \inf_{x\in U'\cap \Gamma}S(x) -o(1)$, where the last step we use the fact that  $\dist(x, \overline {U'\cap \Gamma}) = o(1)$. This completes the proof.
\end{proof}

\subsection{Limiting Regularizers For Different Notions of Sharpness}
\label{app:expliciteg}

\begin{proof}[Proof of Theorem~\ref{thm:maxsharp}]
\
\begin{itemize}
    \item [1.] We will first verify \Cref{cond:regular_R}. For fixed compact set $B \subset U$, as $\| \nabla^3 L (x)\|_2$ is continuous, there exists  constant $\lspectralthree$, such that $\forall x \in B^1$, $\| \nabla^3 L (x)\|_2 \le \nu$. Then by Taylor Expansion,
    \begin{align*}
        \maxsharpness(x) &= \max_{\norm{v}_2\le 1} L(x +\rho v) - L(x) \\
        &\ge \max_{\norm{v}_2\le 1} \left( \rho \langle \nabla L(x), v    \rangle + \rho^2 v^T \nabla^2 L(x) v/2 \right) - \lspectralthree \rho^3/6 \\
        &\ge - \lspectralthree \rho^3/6 \,.
    \end{align*}
    \item [2.]  Now we verify $\maxlimit(x) = \lambda_1(\nabla^2 L(\cdot))/2$ is the limiting regularizer of $\maxsharpness$. Let $x$ be any point in $\Gamma$, by continuity of $\maxsharpness$,
    \begin{align*}
        \lim_{\rho \to 0}\lim_{r \to 0} \inf_{\|x' -x \|_2 \le r}\frac{\maxsharpness(x')}{\rho^2} &=  \lim_{\rho \to 0}
        \frac{\maxsharpness(x)}{\rho^2} = \lambda_1(\nabla^2 L(x))/2 \,.
    \end{align*}
    \item[3.] Finally we verify  definition of good limiting regularizer, by \Cref{assump:smoothness}, $\maxlimit(x) = \lambda_1(x)/2$ is non-negative and continuous on $\Gamma$. For any $x^* \in \Gamma$, choose a sufficiently small open convex set $V$ containing $x^*$ such that  $\forall x \in V^1, \| \nabla^3 L(x)\|_2\le \lspectralthree$. For any $x \in V \cap \Gamma$, for any $x'$ satisfying that $\|x' - x\|_2 \le C \rho$, by \Cref{thm:Davis-Kahan},
    \begin{align*}
        \maxsharpness(x') &= \max_{\norm{v}_2\le 1} L(x'+\rho v) - L(x') \ge \max_{\norm{v}_2\le 1} \left( \rho \langle \nabla L(x'), v    \rangle + \rho^2 v^T \nabla^2 L(x') v/2 \right) - \lspectralthree \rho^3/6 \\
        &\ge \rho^2 \lambda_1(\nabla^2 L(x'))/2 - \lspectralthree \rho^3/6 \ge \rho^2 \lambda_1(\nabla^2 L(x))/2  - O(\rho^3)\,.
    \end{align*}
    This implies $\inf\limits_{\|x' - x\|_2 \le C \rho} \maxsharpness(x') \ge \rho^2 \lambda_1(\nabla^2 L(x))/2  - O(\rho^3)$.
    
    On the other hand, for any $x \in V \cap \Gamma$,
    \begin{align*}
        \maxsharpness(x) &= \max_{\norm{v}_2\le 1} L(x+\rho v) - L(x) \le \max_{\norm{v}_2\le 1} \left( \rho \langle \nabla L(x), v    \rangle + \rho^2 v^T \nabla^2 L(x) v/2 \right) +  \lspectralthree \rho^3 \\
        &= \max_{\norm{v}_2\le 1}  \rho^2 v^T \nabla^2 L(x) v/2 +  \lspectralthree \rho^3 = \rho^2 \lambda_1(\nabla^2 L(x'))/2 +  O(\rho^3)\,.
    \end{align*}
    This implies $\inf\limits_{\|x' - x\|_2 \le C \rho} \maxsharpness(x') \le \rho^2 \lambda_1(\nabla^2 L(x))/2  + O(\rho^3)$.
    
    Thus, we conclude that $\abs{\inf\limits_{\|x' - x\|_2 \le C \rho} \maxsharpness(x')/\rho^2 - \lambda_1(\nabla^2 L(x))/2} = O(\rho), \forall x \in V\cap \Gamma$, indicating $S^{\mathrm{Max}}$ is a good limiting regularizer of $\maxsharpness$ on $\Gamma$.
\end{itemize}
This completes the proof.
\end{proof}

\begin{proof}[Proof of Theorem~\ref{thm:ascsharp}]
\ 
\begin{itemize}
    \item [1.] We will first prove \Cref{cond:regular_R} holds. For any fixed compact set $B \subset U$, as  $\lambda_1(\nabla^2 L)$ and $\| \nabla^3 L\|$is continuous, there exists constant $C$, such that  $ \forall x \in B^2$, $\lambda_1(\nabla^2 L) > -\lspectraltwo$ and $\| \nabla^3 L(x) \| < \lspectralthree$. Then by Taylor Expansion,
    \begin{align*}
        \ascsharpness(x) &=  L(x+\rho \ndl{x}) - L(x)\\&
        \ge \left( \rho \|\dl{x} \|_2 + \rho^2 (\ndl{x})^T \nabla^2 L(x) \ndl{x}/2 \right) - \lspectralthree \rho^3/6 \\
        &\ge - (\lspectraltwo + \lspectralthree/6) \rho^2.
    \end{align*}
    \item [2.] Now we verify $\asclimit(x) = \mathrm{Tr}(\nabla^2 L(\cdot))/2$ is the limiting regularizer of $\ascsharpness$. Let $x$ be any point in $\Gamma$. Let $K = \{x\}$ and choose $h = h(K)$ as in \Cref{lem:mu-pl}. For any $x' \in K^h \cap U'$,
    \begin{align}
        \ascsharpness(x') &=  L(x'+\rho \ndl{x'}) - L(x') \nonumber \\&\ge \rho \|\dl{x'} \|_2 + \rho^2  (\ndl{x'})^T \nabla^2 L(x') \ndl{x'}/2 -  \lspectralthree \rho^3/6 \nonumber \\
        &\ge \rho^2 (\ndl{x'})^T \nabla^2 L(\Phi(x')) \ndl{x'}/2  - \lspectralthree \rho^3/6 \, \nonumber.
    \end{align}
    
    By \Cref{lem:directiondl}, we have $\ndl{x'} = \frac{\dlt{\Phi(x')}(x' - \Phi(x')) }{\|\dlt{\Phi(x')}(x' - \Phi(x'))\|_2} + O(\frac{\lspectralthree}{\mu} \| x'-\Phi(x')\|_2)$. Hence 
    \begin{align*}
    \ascsharpness(x') &\ge \rho^2 \lambda_M(\nabla^2 L(\Phi(x')))/2 -\lspectraltwo \rho^2  O(\|x'-\Phi(x')\|_2) - \lspectralthree \rho^3/6\,.
    \end{align*}
    
    This implies 
    $\lim_{\rho \to 0} \lim_{r \to 0} \inf_{\|x' -x \|_2 \le r}\frac{\ascsharpness(x')}{\rho^2} \ge \lambda_M(\nabla^2 L(\Phi(x')))/2.$
    
    We now show the above inequality is in fact equality.
    If we choose $x''_{r} = x + r v_M$, then by Taylor Expansion,
    \begin{align*}
        \nabla L(x''_{r}) &= \nabla L(x) + \nabla^2 L(x) (x''_{r} - x) + O(\|x''_{r} - x\|^2) \\
        &= rv_M + O(r^2)
    \end{align*}
    This implies $\lim_{r \to 0}\ndl{x''_{r}} = v_M$.
    We also have
    $\lim_{r \to 0}\nabla^2 L({x''_{r}}) = \nabla^2 L(x)$ and $\lim_{r \to 0}\nabla^ L({x''_{r}}) = 0$.
    Putting together, 
    \begin{align*}
        \lim_{r \to 0} \ascsharpness(x''_r) &= \lim_{r \to 0}  L(x''_r+\rho \ndl{x''_r}) - L(x''_r) \\
        &= \lim_{r \to 0} \Bigl(\rho \|\dl{x''_r} \|_2 + \rho^2  (\ndl{x''_r})^T \nabla^2 L(x''_r) \ndl{x''_r}/2 + O(  \lspectralthree \rho^3) \Bigr) \\
        &= \rho^2 \lambda_M (\nabla^2 L(x))/2 + O(\rho^3).
    \end{align*}
    This implies $\lim_{\rho \to 0} \lim_{r \to 0} \inf_{\|x' -x \|_2 \le r}\frac{\ascsharpness(x')}{\rho^2} \le \lim_{\rho \to 0} \lim_{r \to 0} \frac{\ascsharpness(x''_r)}{\rho^2} = \lambda_M(\nabla^2 L(x))/2$. 
    
    Hence the limiting regularizer $S$ is exactly $\lambda_M(\nabla^2 L(\cdot))/2$.
    
    \item [3.] Finally we verify  definition of good limiting regularizer, by \Cref{assump:smoothness}, $\maxlimit(x) = \lambda_M(x)/2$ is non-negative and continuous on $\Gamma$. For any $x^* \in \Gamma$, choose a sufficiently small open convex set $V$ containing $x^*$ such that  $\forall x \in V^1, \| \nabla^3 L(x)\|_2\le \lspectralthree$. For any $x \in V \cap \Gamma$, for any $x'$ satisfying that $\|x' - x\|_2 \le C \rho$, 
     \begin{align}
        \ascsharpness(x') &=  L(x'+\rho \ndl{x'}) - L(x') \nonumber \\&\ge \rho \|\dl{x'} \|_2 + \rho^2  (\ndl{x'})^T \nabla^2 L(x') \ndl{x'}/2 -  \lspectralthree \rho^3/6 \nonumber \\
        &\ge \rho^2 (\ndl{x'})^T \nabla^2 L(\Phi(x')) \ndl{x'}/2  - \lspectralthree \rho^3/6 \, \nonumber.
    \end{align}
    By \Cref{lem:directiondl}, we have $\ndl{x'} = \frac{\dlt{\Phi(x')}(x' - \Phi(x')) }{\|\dlt{\Phi(x')}(x' - \Phi(x'))\|_2} + O(\frac{\lspectralthree}{\mu} \| x'-\Phi(x')\|_2)$.
    This implies $\inf\limits_{\|x' - x\|_2 \le C \rho} \ascsharpness(x') \ge \rho^2 \lambda_M(\nabla^2 L(x))/2  - O(\rho^3)$.
    
    On the other hand, simillar to the proof in the second part, we have $\inf\limits_{\|x' - x\|_2 \le C \rho} \ascsharpness(x') \le \rho^2 \lambda_M(\nabla^2 L(x))/2  + O(\rho^3)$.
    
    Thus, we conclude that $\abs{\inf\limits_{\|x' - x\|_2 \le C \rho} \maxsharpness(x')/\rho^2 - \lambda_1(\nabla^2 L(x))/2} = O(\rho), \forall x \in V\cap \Gamma$, indicating $S^{\mathrm{Max}}$ is a good limiting regularizer of $\maxsharpness$ on $\Gamma$.
\end{itemize}
This completes the proof.
\end{proof}

\begin{proof}[Proof of \Cref{thm:avgsharp}]
\
\begin{itemize}
    \item [1.] We will first verify \Cref{cond:regular_R}. For fixed compact set $B \subset U$, as $\| \nabla^3 L (x)\|_2$ is continuous, there exists  constant $\lspectralthree$, such that $\forall x \in B^1$, $\| \nabla^3 L (x)\|_2 \le \nu$. Then by Taylor Expansion,
    \begin{align*}
        \avgsharpness(x) &= \E_{g \sim N(0,I)} L(x +\rho \frac{g}{\| g\|}) - L(x) \\
        &\ge \E_{g \sim N(0,I)} \left( \rho \langle \nabla L(x),  \frac{g}{\| g\|}   \rangle + \rho^2 (\frac{g}{\| g\|})^T \nabla^2 L(x) \frac{g}{2\| g\|} \right) - \lspectralthree \rho^3/6 \\
        &\ge - \lspectralthree \rho^3/6 \,.
    \end{align*}
    \item [2.]  Now we verify $\maxlimit(x) = \mathrm{Tr}(\nabla^2 L(\cdot))/2D$ is the limiting regularizer of $\avgsharpness$. Let $x$ be any point in $\Gamma$, by continuity of $\avgsharpness$,
    \begin{align*}
        \lim_{\rho \to 0}\lim_{r \to 0} \inf_{\|x' -x \|_2 \le r}\frac{\avgsharpness(x')}{\rho^2} &=  \lim_{\rho \to 0}
        \frac{\avgsharpness(x)}{\rho^2} = \mathrm{Tr}(\nabla^2 L(x))/2D \,.
    \end{align*}
    \item[3.] Finally we verify  definition of good limiting regularizer, by \Cref{assump:smoothness}, $\avglimit(x) = \mathrm{Tr}(x)/2D$ is non-negative and continuous on $\Gamma$. For any $x^* \in \Gamma$, choose a sufficiently small open convex set $V$ containing $x^*$ such that  $\forall x \in V^1, \| \nabla^3 L(x)\|_2\le \lspectralthree$. For any $x \in V \cap \Gamma$, for any $x'$ satisfying that $\|x' - x\|_2 \le C \rho$, by \Cref{thm:Davis-Kahan},
    \begin{align*}
        \avgsharpness(x') &= \E_{g \sim N(0,I)} L(x'+\rho\frac{g}{\| g\|}) - L(x') \\
        &\ge \E_{g \sim N(0,I)} \left( \rho \langle \nabla L(x'),\frac{g}{\| g\|}    \rangle + \rho^2\frac{g}{\| g\|}^T \nabla^2 L(x')\frac{g}{\| 2g\|} \right) - \lspectralthree \rho^3/6 \\
        &\ge \rho^2 \mathrm{Tr}(\nabla^2 L(x'))/2D - \lspectralthree \rho^3/6 \ge \rho^2 \mathrm{Tr}(\nabla^2 L(x))/2D  - O(\rho^3)\,.
    \end{align*}
    This implies $\inf\limits_{\|x' - x\|_2 \le C \rho} \avgsharpness(x') \ge \rho^2 \mathrm{Tr}(\nabla^2 L(x))/2D  - O(\rho^3)$.
    
    On the other hand, for any $x \in V \cap \Gamma$,
    \begin{align*}
        \avgsharpness(x) &= \E_{g \sim N(0,I)} L(x+\rho\frac{g}{\| g\|}) - L(x)\\
        &\le \E_{g \sim N(0,I)} \left( \rho \langle \nabla L(x),\frac{g}{\| g\|}    \rangle + \rho^2\frac{g}{\| g\|}^T \nabla^2 L(x)\frac{g}{2\| g\|} \right) +  \lspectralthree \rho^3 \\
        &= \E_{g \sim N(0,I)}  \rho^2\frac{g}{\| g\|}^T \nabla^2 L(x)\frac{g}{2\| g\|} +  \lspectralthree \rho^3 = \rho^2 \mathrm{Tr}(\nabla^2 L(x'))/2D +  O(\rho^3)\,.
    \end{align*}
    This implies $\inf\limits_{\|x' - x\|_2 \le C \rho} \avgsharpness(x') \le \rho^2 \mathrm{Tr}(\nabla^2 L(x))/2D  + O(\rho^3)$.
    
    Thus, we conclude that $\abs{\inf\limits_{\|x' - x\|_2 \le C \rho} \avgsharpness(x')/\rho^2 - \mathrm{Tr}(\nabla^2 L(x))/2D} = O(\rho), \forall x \in V\cap \Gamma$, indicating $S^{\mathrm{Avg}}$ is a good limiting regularizer of $\avgsharpness$ on $\Gamma$.
\end{itemize}
\end{proof}

\begin{theorem}
\label{thm:stmaxsharp}
Stochastic worst-direction sharpness $\stmaxsharpness$ admits $\mathrm{Tr}(\nabla^2 L(\cdot))/2$ as a good limiting regularizer on $\Gamma$ and satisfies \Cref{cond:regular_R}.
\end{theorem}

\begin{proof}[Proof of \Cref{thm:stmaxsharp}]

By  \Cref{thm:rank_1_manifold},  \Cref{cond:rank1} holds.

Easily deducted from \Cref{thm:maxsharp} $\Lambda_k(x)$ is a good limiting regularizer for $R^{\mathrm{max}}_{k,\rho}$ on $\Gamma_k$. Then as $\Gamma \subset \Gamma_k$, $\Lambda_k(x)$ is a good limiting regularizer for $R^{\mathrm{max}}_{k,\rho}$ on $\Gamma$. Hence $S(x) = \sum_k \Lambda_k(x) /2M = \mathrm{Tr(\nabla^2 L(x))}/2$ is a good limiting regularizer of $\stmaxsharpness(x)$ on $\Gamma$.
\end{proof}

\begin{theorem}
\label{thm:stascsharp}
Stochastic ascent-direction sharpness $\stascsharpness$ admits $\mathrm{Tr}(\nabla^2 L(\cdot))/2$ as a good limiting regularizer on $\Gamma$ and  satisfies \Cref{cond:regular_R}.
\end{theorem}

\begin{proof}[Proof of \Cref{thm:stascsharp}]

By  \Cref{thm:rank_1_manifold},  \Cref{cond:rank1} holds.

Easily deducted from \Cref{thm:ascsharp} $\Lambda_k(x)$ is a good limiting regularizer for $R^{\mathrm{asc}}_{k,\rho}$ on $\Gamma_k$ as the codimension of $\Gamma_k$ is 1. Then as $\Gamma \subset \Gamma_k$, $\Lambda_k(x)$ is a good limiting regularizer for $R^{\mathrm{max}}_{k,\rho}$ on $\Gamma$.Hence $S(x) = \sum_k \Lambda_k(x) /2M = \mathrm{Tr}(\nabla^2 L(x))/2$ is a good limiting regularizer of $\stascsharpness(x)$ on $\Gamma$.
\end{proof}

\begin{theorem}
\label{thm:stavgsharp}
Stochastic average-direction sharpness $\stavgsharpness$ admits $\mathrm{Tr}(\nabla^2 L(\cdot))/(2D)$ as a good limiting regularizer on $\Gamma$ and  satisfies \Cref{cond:regular_R}.
\end{theorem}
\begin{proof}[Proof of \Cref{thm:stavgsharp}]
	By definition, we know that $\stavgsharpness = \avgsharpness$. The rest follows from \Cref{thm:avgsharp}.
\end{proof}

\subsection{Proof of Theorems~\ref{thm:explicit_bias_deterministic_main} and~\ref{thm:explicit_bias_stochastic_main}}

To end this section, we prove the two theorems presented in the main text. The readers will find the proof straight forward after we established the framework of \emph{good limiting regularizers}.

\begin{proof}[Proof of \Cref{thm:explicit_bias_deterministic_main}]
    Apply \Cref{corr:general_reg_extension} on $R^{\mathrm{type}}$. The mapping from $R$ to good limiting regularizers $S^{\mathrm{type}}$ are characterized by \Cref{thm:maxsharp,thm:ascsharp,thm:avgsharp}.
\end{proof}

\begin{proof}[Proof of \Cref{thm:explicit_bias_stochastic_main}]
    Apply \Cref{corr:general_reg_extension} on $R^{\mathrm{type}}$. The mapping from $R$ to good limiting regularizers $\tilde S^{\mathrm{type}}$ are characterized by \Cref{thm:stmaxsharp,thm:stascsharp,thm:stavgsharp}.
\end{proof}

\section{Analysis Full-batch SAM on Quadratic Loss (Proof of \Cref{thm:nsamquad})}
\label{app:nsamquad}

The goal of this section is to prove \Cref{thm:nsamquad}. In this section, we use $A \prec B$ to indicate $B - A$ is positive semi-definite.

\thmnsamquad*

\begin{proof}[Proof of~\Cref{thm:nsamquad}]

We first rewrite the iterate as 
\begin{align*}
    x(t+1) = x(t) - \eta A x(t) -\eta \rho \frac{A^2 x(t)}{\| Ax(t)\|_2}\,.
\end{align*}

Define $\tx(t) \triangleq \frac{\nabla L(x(t))}{\rho}  = \frac{ A x(t)}{\rho}$, and  we have 
\begin{align}
\label{eq:quadupdate}
    \tx(t+1) = \tx(t) - \eta A \tx(t) -\eta  \frac{A^2 \tx(t)}{\| \tx(t)\|_2}\,.
\end{align}

We suppose $A \in R^{D\times D}$ and use $\lambda_i, v_i$ to denote $\lambda_i(A), v_i(A)$.

Further, we define that
\begin{align}
    \Pjd &\triangleq \sum_{i=j}^D v_i(A)v_i(A)^T\nonumber,\\
    \sI_j &\triangleq \{\tx \mid \| \Pjd \tx  \|_2 \le \eta  \lambda_j^2\}\,,\nonumber\\
    \tx_i(t) &\triangleq \langle \tx(t), v_i \rangle\,,\nonumber  \\
    S &\triangleq \{t\mid \| \tx(t) \|_2 \le \ths , t > T_1 \}\, \nonumber. 
\end{align}
By~\Cref{lem:invariant_set}, $\sI_j$ is an invariant set for update rule~\Cref{eq:quadupdate}.

Our proof consists of two steps.

\begin{itemize}
    \item [(1)] \textit{Entering Invariant Set.}  \Cref{lem:preparequad} implies that there exists constant $T_1 > 0$, such that $\forall t > T_1, \| \Pjd \tx (t) \|_2 \le \eta  \lambda_j^2$
    \item [(2)] \textit{Alignment to Top Eigenvector.} \Cref{lem:quadnormconverge,lem:quaddirconverge} show that $\| \tx(t)\|_2$ and $| \tx_1(t) |$ converge to $\ths$, which implies our final results.
\end{itemize}

\end{proof}

\subsection{Entering Invariant Set}
\label{sec:prepare}

In this subsection, we will prove the following three lemmas.
\begin{enumerate}
    \item \Cref{lem:invariant_set} shows $\sI_j$ is an invariant set for update rule~(\Cref{eq:quadupdate}).
    \item \Cref{lem:preparequad} shows that under the update rule~(\Cref{eq:quadupdate}), all iterates not in $\sI_j$ will shrink exponentially in $\ell_2$ norm.
    \item \Cref{lem:quad_prepare}  combines \Cref{lem:invariant_set,lem:preparequad} to show that for sufficiently large $t$, $x(t) \in \cap_j \sI_j$.
\end{enumerate}

\begin{lemma}
\label{lem:invariant_set}
For $t \ge 0$, if $\eta \lambda_1(A) < 1$ and $\tx(t) \in \sI_j $, then $\tx(t+1) \in \sI_j$. 
\end{lemma}

\begin{proof}[Proof of \Cref{lem:invariant_set}]
By~(\Cref{eq:quadupdate}), we have that 
\begin{align*}
    \Pjd \tx(t+1) = (I - \Pjd \eta A  -\eta  \frac{\Pjd A^2 }{\| \tx(t)\|_2} )\Pjd \tx(t)\,.
\end{align*}

Hence we have that
\begin{align*}
    \| \Pjd \tx(t+1) \|_2 &=\| (I - \Pjd \eta A  -\eta  \frac{\Pjd A^2 }{\| \tx(t)\|_2} )\Pjd \tx(t) \|_2 \\
    &\le \| I - \Pjd \eta A  -\eta  \frac{\Pjd A^2 }{\| \tx(t)\|_2} \|_2 \| \Pjd \tx(t)\|_2\,.
\end{align*}

Because $\tx(t) \in \sI_j$, $\| \tx(t) \|_2 \le \frac{\eta  \lambda_j^2}{ 1 - \eta \lambda_j}$. This implies,
\begin{align*}
   I(1 - \eta \lambda_j -  \eta \frac{\lambda_j^2}{\| \Pjd \tx(t)\|_2})\prec I(1 - \eta \lambda_j -  \eta \frac{\lambda_j^2}{\| \tx(t)\|_2}) \prec I - \Pjd \eta A  -\eta  \frac{\Pjd A^2 }{\| \tx(t)\|_2} \prec I\,.
\end{align*}
Hence, $\| I - \Pjd \eta A  -\eta  \frac{\Pjd A^2 }{\|_2 \tx(t)\|} \|_2 \le \max(1,  \eta \lambda_j +  \eta \frac{\lambda_j^2}{\| \Pjd \tx(t)\|_2} - 1)$\,. It holds that 
\begin{align*}
    \| \Pjd \tx(t+1) \|_2 
    \le \max(\| \Pjd \tx(t)\|_2, \eta  \lambda_j^2 - (1-\eta\lambda_j) \| \Pjd \tx(t)\|_2) 
    \le \eta  \lambda_j^2,
\end{align*}
where the last equality is because $1 - \eta \lambda_j \ge 0$. This above inequality is exactly the definition of $\tx(t+1) \in \sI_j$ and thus is proof is completed.
\end{proof}

\begin{lemma}
\label{lem:preparequad}
 For $t \ge 0$, if $\eta \lambda_1(A) < 1$ and $\tx(t) \not \in \sI_j$, then 
\begin{align}
    \| \Pjd \tx(t+1) \|_2 &\le \max\left (1 - \eta \lambda_D - \eta \frac{\lambda_D^2}{\|\tx(t)\|_2}, \eta \lambda_j \right)\| \Pjd \tx(t) \|_2 \\ &\le \max\left (1 - \eta \lambda_D, \eta \lambda_j \right)\| \Pjd \tx(t) \|_2\,. \nonumber
\end{align}
\end{lemma}

\begin{proof}[Proof of \Cref{lem:preparequad}] Note that
\begin{align*}
    \| \Pjd \tx(t+1) \|_2 &=\| (I - \Pjd \eta A  -\eta  \frac{\Pjd A^2 }{\| \tx(t)\|_2} )\Pjd \tx(t) \|_2 \\
    &\le \| \Pjd - \Pjd \eta A  -\eta  \frac{\Pjd A^2 }{\| \tx(t)\|_2} \|_2 \| \Pjd \tx(t)\|_2\,.
\end{align*}

As $\tx(t) \not \in \sI_j$, We have $\|\tx(t) \|_2 \ge \|\Pjd \tx(t) \|_2 > \eta  \lambda_j^2$, hence $\eta  \frac{\Pjd A^2 }{\| \tx(t)\|_2} \prec \eta  \frac{\Pjd A^2 }{\eta  \lambda_j^2} \prec \Pjd$.

This implies that 
\begin{align*}
      -\eta \lambda_j \Pjd\prec - \Pjd \eta A \prec \Pjd - \Pjd \eta A - \eta  \frac{\Pjd A^2 }{\| \tx(t)\|_2} \,,
\end{align*}
and
\begin{align*}
    \Pjd - \Pjd \eta A - \eta  \frac{\Pjd A^2 }{\| \tx(t)\|_2} \prec \Pjd ( 1- \eta \lambda_D) - \eta \frac{\lambda_D^2}{\|\tx(t)\|_2}\,.
\end{align*}

Hence we have that
\begin{align*}
    \| \Pjd \tx(t+1) \|_2 &\le \max\left (1 - \eta \lambda_D - \eta \frac{\lambda_D^2}{\|\tx(t)\|_2}, \eta \lambda_j \right)\| \Pjd \tx(t) \|_2. \\
    &\le \max\left (1 - \eta \lambda_D, \eta \lambda_j \right)\| \Pjd \tx(t) \|_2
\end{align*}
This completes the proof.
\end{proof}

\begin{lemma}
\label{lem:quad_prepare}
Choosing $T_1 = \max_j\left ( - \log_{\max\left (1 - \eta \lambda_D, \eta \lambda_j \right)}{\max(\frac{\|\tx(0)\|_2}{{\eta  \lambda_j^2}}, 1)}\right)$, then $\forall t \ge T_1, D > j \ge 1, \tx(t) \in \sI_j$
\end{lemma}

\begin{proof}[Proof of \Cref{lem:quad_prepare}]

We will prove by contradiction. Suppose $\exists j\in [D] $ and $T > T_1$, such that $ \tx(T) \not \in \sI_j$.
By Lemma ~\ref{lem:invariant_set}, it holds that $\forall t < T, \tx(t) \not \in \sI_j$. Then by Lemma ~\ref{lem:preparequad}, 
$$\| \Pjd \tx(T)\|_2 \le  \max\left (1 - \eta \lambda_D, \eta \lambda_j \right)^T \| \Pjd \tx(0)\|_2 \le \eta  \lambda_j^2,$$ which leads to a contradiction.
\end{proof}

\subsection{Alignment to Top Eigenvector}
\label{sec:quadalign}

In this subsection, we  prove the following lemmas towards showing that   $\tx(t)$ converges  in direction to $v_1(A)$ up to a proper sign flip.

\begin{enumerate}
    \item  \Cref{cor:almost_sure_x1}  show that  for almost every learning rate $\eta$ and initialization $\xinit$,  $\tx_1(t) \neq 0$, for every $t\ge 0$. This condition is important because if $\tx_1(t)=0$ at some step $t$, then for any $t'\ge t$, $\tx_1(t')$ will also be $0$ and thus alignment is impossible. 
    \item \Cref{lem:norminequal} shows that under update rule~(\Cref{eq:quadupdate}), $t \not \in S \Rightarrow t+1 \in S$ for sufficiently large $t$, where the definition of $S$ is $\{t| \| \tx(t) \|_2 \le \ths , t > T_1 \}$.
    \item \Cref{lem:monotoneproj}, a combination of~\Cref{lem:controlonehop,lem:short_increase}, shows that following update rule~(\Cref{eq:quadupdate}), $\tx_1(t)$ increases for $t \in S$.
    \item \Cref{lem:quadnormconverge} shows that $\|\tx(t)\|$ converges to $\ths$ under~\Cref{eq:quadupdate}.
    \item \Cref{lem:quaddirconverge} shows that
    $\|\tx_1(t)\|_2$ converges to $\ths$ under~\Cref{eq:quadupdate}.
\end{enumerate}

 We  will first prove that $\forall t, \tx_1(t) \neq 0$ happens for almost every learning rate $\eta$ and initialization $\xinit$~(\Cref{cor:almost_sure_x1}), using a much more general result~(\Cref{thm:general_well_definednesa}).

\begin{corollary}\label{cor:almost_sure_x1}
	Except for countably many $\eta\in \R^+$, for almost all initialization $\xinit= x(0)$, it holds that for all natural number $t$, $\tilde x_1(t)\neq 0$.
\end{corollary}
\begin{proof}[Proof of \Cref{cor:almost_sure_x1}]
Let $F_n(x) \equiv F(x)\triangleq A (x + \rho \frac{Ax}{\norm{Ax}_2}),\ \forall n\in\mathbb{N}^+, x\in\R^D$ and $Z = \{x\in\R^D\mid \langle x, v_1 \rangle = 0\}$. We can easily check $F$ is $\mathcal{C}^1$ on $\R^D\setminus Z$ and $Z$ is a zero-measure set. Applying \Cref{thm:general_well_definednesa}, we have the following corollary.	
\end{proof}

\begin{lemma}
\label{lem:norminequal}
 For $t \ge 0$, if $\|\tx(t)\|_2 > \ths, \tx(t) \in \sIj$, then 
\begin{align*}
    \|\tx(t+1)\|_2 \le \max(\ths - \eta \frac{\lambda_D^4}{2\lambda_1^2},\eta \lambda_1^2 -(1 - \eta \lambda_1)\| \tx(t) \|_2 ) 
\end{align*}
\end{lemma}

\begin{proof}[Proof of \Cref{lem:norminequal}]

Note that
\begin{align*}
    \tx(t+1) &= (I -  \eta A  -\eta  \frac{A^2 }{\| \tx(t)\|_2} )\tx(t) \\
    &= \frac{1}{\| \tx(t) \|_2} \sum_{j=1}^D \left((1 - \eta \lambda_j)\| \tx(t) \|_2 - \eta \lambda_j^2 \right) \tx_j(t) v_j
\end{align*}

Consider the following two cases.
\begin{itemize}
    \item [1] If for any  $i$, such that $\left|(1 - \eta \lambda_1)\| \tx(t) \|_2 - \eta \lambda_1^2 \right| \ge \left|(1 - \eta \lambda_i)\| \tx(t) \|_2 - \eta \lambda_i^2 \right| $, then we have 
    \begin{align*}
        \|\tx(t+1) \|_2 \le \left|(1 - \eta \lambda_1)\| \tx(t) \|_2 - \eta \lambda_1^2 \right| =  \eta \lambda_1^2 -(1 - \eta \lambda_1)\| \tx(t) \|_2\,.
    \end{align*}

    \item[2] If there exists $i$, such that $\left|(1 - \eta \lambda_1)\| \tx(t) \|_2 - \eta \lambda_1^2 \right| < \left|(1 - \eta \lambda_i)\| \tx(t) \|_2 - \eta \lambda_i^2 \right|$, then suppose WLOG, $i$ is the smallest among such index.
    
    As 
    \begin{align*}
        \eta \lambda_i^2 - (1 - \eta \lambda_i)\| \tx(t) \|_2 <  \eta \lambda_1^2 - (1 - \eta \lambda_1)\| \tx(t) \|_2  = \left|(1 - \eta \lambda_1)\| \tx(t) \|_2 - \eta \lambda_1^2 \right|
    \end{align*}
    We have $-\eta \lambda_i^2 + (1 - \eta \lambda_i)\| \tx(t) \|_2 >  \eta \lambda_1^2 - (1 - \eta \lambda_1)\| \tx(t) \|_2$. Equivalently, 
    \begin{align}
        \| \tx(t) \|_2 > \frac{\eta \lambda_1^2 + \eta \lambda_i^2}{2 - \eta \lambda_1 -\eta \lambda_i}
    \end{align}
    
    Combining with 
    $\tx(t) \in \sI_1 \Rightarrow \| \tx(t) \|_2 \le \eta \lambda_1^2  $, we have $\eta < \frac{\lambda_1 - \lambda_i} {\lambda_1^2}$.
    
    Now consider the following vertors,
    \begin{align*}
        v^{(1)}(t) &\triangleq ( \eta \lambda_1^2 - (1 - \eta \lambda_1)\| \tx(t) \|_2) \tx(t)\,, \\
        v^{(2)}(t) &\triangleq ((2 - \eta \lambda_1 -\eta \lambda_i) \| \tx(t)\|_2 - \eta \lambda_i^2 - \eta \lambda_1^2) \Pd i\tx(t)\,, \\
        v^{(2 + j)}(t) &\triangleq (( \eta \lambda_{i+j-1} - \eta \lambda_{i+j} ) \| \tx(t)\|_2 - \eta \lambda_{i+j}^2 + \eta \lambda_{i+j-1}^2) \Pd {i+j}\tx(t), 1 \le j \le D - i\,.
    \end{align*}

    Then we have 
    \begin{align*}
        \|\tx(t+1)\|_2 =& \|\frac{1}{\| \tx(t) \|_2}  \sum_{j=1}^D \left((1 - \eta \lambda_j)\| \tx(t) \|_2 - \eta \lambda_j^2 \right) \tx_j(t) v_j  \|_2\\ \le&  \| \frac{1}{\| \tx(t) \|_2}  \sum_{j=1}^{i-1} \left( \eta \lambda_1^2 -(1 - \eta \lambda_1)\| \tx(t) \|_2\right) \tx_j(t) v_j\| + \\
        & \|\frac{1}{\| \tx(t) \|_2} \sum_{j=i}^{D} \left((1 - \eta \lambda_j)\| \tx(t) \|_2 - \eta \lambda_j^2 \right) \tx_j(t) v_j  \|_2 \\
        \le& \frac{1}{\| \tx(t) \|_2}  \sum_{j = 1}^{D+1-i} \| v^{(j)} \|_2
    \end{align*}
    
    By assumption, we have $\tx(t) \in \sIj$, hence we have 
    \begin{align*}
        \| v^{(1)}(t) \|_2 &= ( \eta \lambda_1^2 - (1 - \eta \lambda_1)\| \tx(t) \|_2) \| \tx(t)\|_2\,,\\
        \| v^{(2)}(t) \|_2 &\le \eta ((2 - \eta \lambda_1 -\eta \lambda_i) \| \tx(t)\|_2 - \eta \lambda_i^2 - \eta \lambda_1^2) \lambda_i^2 \,, \\
        \|  v^{(2 + j)}(t) \|_2 &\le \eta(( \eta \lambda_{i+j-1} - \eta \lambda_{i+j} ) \| \tx(t)\|_2 - \eta \lambda_{i+j}^2 + \eta \lambda_{i+j-1}^2)  \lambda_{i+j}^2, 1 \le j \le D - i \,.
    \end{align*}
    Using AM-GM inequality, we have
    \begin{align*}
        \lambda_{i+j-1} \lambda_{i+j}^2 &\le \frac{\lambda_{i+j-1}^3 + 2\lambda_{i+j}^3 }{3} \,, \\
        \lambda_{i+j-1}^2\lambda_{i+j}^2 &\le \frac{\lambda_{i+j-1}^4 + \lambda_{i+j}^4} {2} \,.
    \end{align*}
    
    Hence
    \begin{align*}
        \|  v^{(2 + j)}(t) \|_2 &\le \eta(( \eta \lambda_{i+j-1} - \eta \lambda_{i+j} ) \| \tx(t)\|_2 - \eta \lambda_{i+j}^2 + \eta \lambda_{i+j-1}^2)  \lambda_{i+j}^2 \\
        &\le \eta^2 \| \tx(t)\|_2 \frac{\lambda_{i+j-1}^3 - \lambda_{i+j}^3 }{3}  + \eta^2 \frac{\lambda_{i+j-1}^4 - \lambda_{i+j}^4}{2}, 1 \le j \le D - i \\
        \sum_{j=1}^{D-i} \|  v^{(2 + j)}(t) \|_2 &\le \eta^2 \| \tx(t)\|_2 \frac{\lambda_i^3 - \lambda_D^3}{3} + \eta^2 \frac{\lambda_i^4-\lambda_D^4}{2} \,.
    \end{align*}
    Putting together,
    \begin{align*}
    &\|\tx(t+1)\|_2 \le  \frac{1}{\| \tx(t) \|_2}  \sum_{j = 1}^{D+1-i} \| v^{(i)} \|_2 \\
    \le& \eta \lambda_1^2 + \eta \lambda_i^2 (2 - \eta\lambda_1 - \eta \lambda_i) + \eta^2 \frac{\lambda_i^3 - \lambda_D^3}{3} - (1 - \eta \lambda_1) \| \tx(t) \|_2 \\
    &- \eta^2\lambda_i^2(\lambda_i^2 + \lambda_1^2)\frac{1}{\| \tx(t) \|_2} +  \eta^2 \frac{\lambda_i^4 - \lambda_D^4 }{2} \frac{1}{\| \tx(t) \|_2} \\
    \le& \eta \lambda_1^2 + \eta \lambda_i^2 (2 - \eta\lambda_1 - \frac23 \eta \lambda_i) - (1 - \eta \lambda_1) \| \tx(t) \|_2 - \eta^2\lambda_i^2(\frac{1}{2} \lambda_i^2 + \lambda_1^2)  \frac{1}{\| \tx(t) \|_2} - \eta^2 \frac{\lambda_D^4}{2\| \tx(t) \|_2} \\
    \le& \eta \lambda_1^2 + \eta \lambda_i^2 (2 - \eta\lambda_1 - \frac23 \eta \lambda_i) - (1 - \eta \lambda_1) \| \tx(t) \|_2 - \eta^2\lambda_i^2(\frac{1}{2} \lambda_i^2 + \lambda_1^2)  \frac{1}{\| \tx(t) \|_2} - \eta \frac{\lambda_D^4}{2 \lambda_1^2}\,.
    \end{align*}
    
    We further discuss three cases
    \begin{itemize}
        \item [1.] If $\best < \frac{\eta \lambda_1^2 + \eta \lambda_i^2}{2 - \eta \lambda_1 -\eta \lambda_i}$, we have $\| \tx(t) \|_2> \frac{\eta \lambda_1^2 + \eta \lambda_i^2}{2 - \eta \lambda_1 -\eta \lambda_i} > \best$,then
        \begin{align*}
            &\|\tx(t+1)\|_2 \\
            \le& \eta \lambda_1^2 + \eta \lambda_i^2 (2 - \eta\lambda_1 - \frac23 \eta \lambda_i) - (1 - \eta \lambda_1) \| \tx(t) \|_2 - \eta^2\lambda_i^2(\frac{1}{2} \lambda_i^2 + \lambda_1^2)  \frac{1}{\| \tx(t) \|_2} - \eta \frac{\lambda_D^4}{2 \lambda_1^2} \\
            \le& \eta \lambda_1^2 + \eta \lambda_i^2 (2 - \eta\lambda_1 - \frac23 \eta \lambda_i) - (1 - \eta \lambda_1) \frac{\eta \lambda_1^2 + \eta \lambda_i^2}{2 - \eta \lambda_1 -\eta \lambda_i} \\
            &- \eta^2\lambda_i^2(\frac{1}{2} \lambda_i^2 + \lambda_1^2)  \frac{2 - \eta \lambda_1 -\eta \lambda_i}{\eta \lambda_1^2 + \eta \lambda_i^2} - \eta \frac{\lambda_D^4}{2 \lambda_1^2}\\
            \le& \ths - \eta \frac{\lambda_D^4}{2 \lambda_1^2}\,.
        \end{align*}
        The second line is because $(1 - \eta \lambda_1) \| \tx(t) \|_2 + \eta^2\lambda_i^2(\frac{1}{2} \lambda_i^2 + \lambda_1^2)  \frac{1}{\| \tx(t) \|_2}$ monotonously increase w.r.t $\|\tx(t)\|_2$ when $\|\tx(t)\|_2 > \best$.
        The last line is due to  \Cref{lem:technorminequal}.

    \item[2.] If $\eta \lambda_1^2 \ge \best \ge \frac{\eta \lambda_1^2 + \eta \lambda_i^2}{2 - \eta \lambda_1 -\eta \lambda_i}$, then
    \begin{align*}
        &\|\tx(t+1)\|_2  \\ \le& \eta \lambda_1^2 + \eta \lambda_i^2 (2 - \eta\lambda_1 - \frac23 \eta \lambda_i) - (1 - \eta \lambda_1) \| \tx(t) \|_2 - \eta^2\lambda_i^2(\frac{1}{2} \lambda_i^2 + \lambda_1^2)  \frac{1}{\| \tx(t) \|_2} - \eta \frac{\lambda_D^4}{2 \lambda_1^2} \\
        \le& \eta \lambda_1^2 + \eta \lambda_i^2 (2 - \eta\lambda_1 - \frac23 \eta \lambda_i) - 2 \eta \lambda_i \sqrt{( \lambda_1^2 + \frac{1}{2} \lambda_i^2)(1 - \eta \lambda_1)} - \eta \frac{\lambda_D^4}{2 \lambda_1^2}\\
        \le& \ths - \eta \frac{\lambda_D^4}{2 \lambda_1^2}\,.
    \end{align*}
        The second line is because of AM-GM inequality.
        The last line is due to \Cref{lem:norminequal3}.
        
    \item[3.] If $\eta \lambda_1^2 < \best$,  we have $\| \tx(t) \|_2 < \eta \lambda_1^2 < \best$, then
    \begin{align*}
    &\|\tx(t+1)\|_2 \\ \le& \eta \lambda_1^2 + \eta \lambda_i^2 (2 - \eta\lambda_1 - \frac23 \eta \lambda_i) - (1 - \eta \lambda_1) \| \tx(t) \|_2 - \eta^2\lambda_i^2(\frac{1}{2} \lambda_i^2 + \lambda_1^2)  \frac{1}{\| \tx(t) \|_2} - \eta \frac{\lambda_D^4}{2 \lambda_1^2} \\
    \le& \eta \lambda_1^2 + \eta \lambda_i^2 (2 - \eta\lambda_1 - \frac23 \eta \lambda_i) - (1 - \eta \lambda_1) \eta \lambda_1^2 - \eta\lambda_i^2(\frac{1}{2} \lambda_i^2 + \lambda_1^2) \frac1{\lambda_1^2} - \eta \frac{\lambda_D^4}{2 \lambda_1^2} \\
    \le& \ths - \eta \frac{\lambda_D^4}{2 \lambda_1^2}\,.
    \end{align*}
    
    The second line is because $(1 - \eta \lambda_1) \| \tx(t) \|_2 + \eta^2\lambda_i^2(\frac{1}{2} \lambda_i^2 + \lambda_1^2)  \frac{1}{\| \tx(t) \|_2}$ monotonously decrease w.r.t $\|\tx(t)\|_2$ when $\|\tx(t)\|_2 < \best$.
        The last line is due to \Cref{lem:norminequal2}.
    \end{itemize}

\end{itemize}
\end{proof}

\begin{lemma}
\label{lem:short_increase}
 if $\| \tx(t) \|_2 \le \ths$, it holds that $|\tx_1(t+1)| \ge |\tx_1(t)|\,$.
\end{lemma}

\begin{proof}[Proof of~\Cref{lem:short_increase}]
Nota that  $|\tx_1(t+1)| = |1 - \eta \lambda_1 - \eta \frac{\lambda_1^2}{\|\tx(t)\|_2} | |\tx_1(t)|$ and that $ \eta \frac{\lambda_1^2}{\|\tx(t)\|_2} > 2 - \eta \lambda_1^2$. It follows that $ 1 - \eta \lambda_1 - \eta \frac{\lambda_1^2}{\|\tx(t)\|_2} < -1$.
Hence we have that $|\tx_1(t+1)| > |\tx_1(t)|$.
\end{proof} 

\begin{lemma}
\label{lem:controlonehop}
For any $t \ge 0$, if $\| \tx(t) \|_2 \le \ths, \tx(t) \in \sIj$, it holds that 
\begin{align*}
    \| \tx(t+1) \|_2\le \eta \lambda_1^2 - (1-\eta \lambda_1) \| \tx(t)\|_2\,.
\end{align*}
\end{lemma}
\begin{proof}[Proof of~\Cref{lem:controlonehop}]
Note that
\begin{align*}
    \|  I -  \eta A  -\eta  \frac{ A^2 }{\| \tx(t)\|_2}  \|_2 \le \max_{ 1\le j \le D} \{| 1 - \eta \lambda_j - \eta\frac{\lambda_j^2}{\|\tx(t)\|} |\} = \eta\frac{\lambda_1^2}{\|\tx(t)\|} - (1 - \eta \lambda_j)\,.
\end{align*}
The proof is completed by noting that $\norm{\tilde x(t+1)}\le \|  I -  \eta A  -\eta  \frac{ A^2 }{\| \tx(t)\|_2}  \|_2 \norm{\tilde x(t)}_2$.
\end{proof}

\begin{lemma}
\label{lem:subtlecontrolonehop}
For any $t \ge 0$, if $\|\tx(t)\|_2 \le \frac{\eta \lambda_1^2}{1- \eta \lambda_1}$, it holds that 
\begin{align*}
    &\|\tx(t + 1) \|_2  \\ \le& 
    (\eta \lambda_1^2 - (1 + \eta \lambda_1) \| \tx(t) \|_2) \times \\&
    \sqrt{\frac{|\tx_1(t)|^2}{\|\tx(t)\|^2} + \Bigl(\max_{j \in [2:M]}\left(\frac{| (1 - \eta\lambda_j) \|\tx(t)\|_2 - \eta \lambda_j^2|}{\eta \lambda_1^2 - (1 - \eta \lambda_1) \| \tx(t) \|_2}\right) \Bigr)^2 \bigl(1 - \frac{|\tx_1(t)|^2}{\|\tx(t)\|^2} \bigr)}.
\end{align*}
\end{lemma}

\begin{proof}[Proof of~\Cref{lem:subtlecontrolonehop}]
We will discuss the movement along $v_1$ and orthogonal to $v_1$. First,
\begin{align*}
    \| \Ptd \tx(t+1) \|_2 &=\| (I - \Ptd \eta A  -\eta  \frac{\Ptd A^2 }{\| \tx(t)\|_2} )\Ptd \tx(t) \|_2 \\
    &\le \| \Ptd - \Ptd \eta A  -\eta  \frac{\Ptd A^2 }{\| \tx(t)\|_2} \|_2 \| \Ptd \tx(t)\|_2 \\
    &\le \max_{j \in [2:M]} \{|1 - \eta \lambda_j - \frac{\eta \lambda_j^2}{\|\tx(t)\|_2}| \} \|\Ptd \tx(t) \|_2 \,.
\end{align*}
Second, $|\tx_1(t + 1)| = (\frac{\eta \lambda_1^2}{\| \tx(t) \|_2} - 1 + \eta \lambda_1) |\tx_1(t)|$.
Hence we have that
\begin{align*}
    &\|\tx(t + 1) \|_2  \\ \le& 
    (\eta \lambda_1^2 - (1 + \eta \lambda_1) \| \tx(t) \|_2) \times \\&
    \sqrt{\frac{|\tx_1(t)|^2}{\|\tx(t)\|^2} + \Bigl(\max_{j \in [2:M]}\left(\frac{| (1 - \eta\lambda_j) \|\tx(t)\|_2 - \eta \lambda_j^2|}{\eta \lambda_1^2 - (1 - \eta \lambda_1) \| \tx(t) \|_2}\right)  \} \Bigr)^2 \bigl(1 - \frac{|\tx_1(t)|^2}{\|\tx(t)\|^2} \bigr)}.
\end{align*}

\end{proof}

\begin{lemma}
\label{lem:monotoneproj}
For $t,t'\in S,0 \le t \le t'$, then $ |\tx_1(t)| \le |\tx_1(t')|$.
\end{lemma}

\begin{proof}[Proof of~\Cref{lem:monotoneproj}]
For $t \in S$, by~\Cref{lem:norminequal}, $t+1 \in S$ or $t+1 \not \in S, t+2 \in S$. We will discuss by case.

\begin{enumerate}

    \item If $t+1 \in S$, we can use \Cref{lem:short_increase} to show $|\tx_1(t)| \le |\tx_1(t+1)|$.
    \item If $t+1 \not \in S, t+2 \in S$, then
    \begin{align*}
        |\tx_1(t+2)| = \frac{(\eta \lambda_1^2 - (1-\eta\lambda_1)\|\tx(t)\|_2)(\eta \lambda_1^2 - (1-\eta\lambda_1)\|\tx(t+1)\|_2)}{\|\tx(t)\|_2 \| \tx(t+1)\|_2}  |\tx_1(t)|\,.
    \end{align*}

    As
    \begin{align*}
        &(\eta \lambda_1^2 - (1-\eta\lambda_1)\|\tx(t)\|_2)(\eta \lambda_1^2 - (1-\eta\lambda_1)\|\tx(t+1)\|_2) \ge \|\tx(t)\|_2 \| \tx(t+1)\|_2 \\
        \iff & \eta^2 \lambda_1^4 -\eta\lambda_1^2 (1 -\eta \lambda_1) (\|\tx(t)\|_2+ \|\tx(t+1) \|_2)   \\ & \ge (2\eta\lambda_1 - \eta^2\lambda_1^2)\|\tx(t)\|_2 \| \tx(t+1)\|_2   \\
        \iff &  \eta^2 \lambda_1^4 -\eta\lambda_1^2 (1 -\eta \lambda_1) \|\tx(t)\|_2 \\ &\ge \left((2\eta\lambda_1 - \eta^2\lambda_1^2)\|\tx(t)\|_2 +\eta\lambda_1^2 (1 -\eta \lambda_1) \right) \| \tx(t+1)\|_2 \,,
    \end{align*}
    combining with~\Cref{lem:controlonehop}, we only need to prove,
    \begin{align*}
        &\eta^2 \lambda_1^4 -\eta\lambda_1^2 (1 -\eta \lambda_1) \|\tx(t)\|_2 \\
        \ge& \left((2\eta\lambda_1 - \eta^2\lambda_1^2)\|\tx(t)\|_2 +\eta\lambda_1^2 (1 -\eta \lambda_1) \right)\left( \eta \lambda_1^2 - (1-\eta \lambda_1) \| \tx(t)\|_2 \right) \,.
    \end{align*}
    
    Through some calculation, this is equivalent to 
    \begin{align*}
        ((2 - \eta \lambda_1)\|\tx(t) \|_2  - \eta \lambda_1^2) ((1 - \eta \lambda_1)\|\tx(t) \|_2  -  \eta \lambda_1^2) \ge 0 \,.
    \end{align*}
    which holds for $\|\tx(t) \|_2 \le \ths$.
\end{enumerate}

Combining the two cases and using induction, we can get the desired result.
\end{proof}

\begin{lemma}
\label{lem:quadnormconverge}
 $\|\tx(t)\|$ converges to $\ths$ when $t \to \infty$.
\end{lemma}
\begin{proof}[Proof of~\Cref{lem:quadnormconverge}]
By \Cref{lem:monotoneproj}, $|\tx_1(t)|$ increases monotonously for $t \in S$. By \Cref{lem:norminequal}, $S$ is infinite. By \Cref{lem:preparequad}, for sufficiently large $t$, $|\tx_1(t)|$ is bounded. Combining the three facts, we know $\tx_1(t)$ for $t \in S$ converges.

Formally $\forall \epsilon > 0$, there  exists $T_\epsilon > 0$ such that $\forall t,t' \in S, t' > t > T_\epsilon, \frac{\|\tx_1(t')\|_2}{\|\tx_1(t) \|_2} < 1 + \epsilon$.

Then by~\Cref{lem:norminequal}, $\forall t \in S, t + 1 \in S$ or $t + 2 \in S$, we will discuss by case.
For $t \ge T_{\epsilon}$,
\begin{enumerate}
    \item If $t+1 \in S$, then \begin{align*}
    1 + \epsilon \ge \frac{\| \tx_1(t+1) \|_2}{\|\tx_1(t) \|_2} = \frac{\eta \lambda_1^2 - (1 -\eta \lambda_1) \| \tx(t)\|_2} {\|\tx(t) \|_2}\,.
\end{align*}
\item If $t+1 \not \in S$ and $t + 2 \in S$, then \begin{align*}
    1 + \epsilon &\ge \frac{\| \tx_1(t+2) \|_2}{\|\tx_1(t) \|_2} \\
    &= \frac{(\eta \lambda_1^2 - (1-\eta\lambda_1)\|\tx(t)\|_2)(\eta \lambda_1^2 - (1-\eta\lambda_1)\|\tx(t+1)\|_2)}{\|\tx(t)\|_2 \| \tx(t+1)\|_2}   \\
    &\ge \frac{(\eta \lambda_1^2 - (1-\eta\lambda_1)\|\tx(t)\|_2)\left(\eta \lambda_1^2 - (1-\eta\lambda_1) \left(\eta \lambda_1^2 - (1-\eta \lambda_1) \| \tx(t)\|_2\right) \right)}{\|\tx(t)\|_2 \left(\eta \lambda_1^2 - (1-\eta \lambda_1) \| \tx(t)\|_2\right) } \\
    &= \frac{\eta \lambda_1^2 - (1-\eta\lambda_1) \left(\eta \lambda_1^2 - (1-\eta \lambda_1) \| \tx(t)\|_2\right) }{\|\tx(t)\|_2 }\,.
\end{align*}
Here in the last inequality, we apply~\Cref{lem:controlonehop}.
\end{enumerate}

Concluding, $\| \tx(t) \|_2 \ge \min \left(\frac{\eta\lambda_1^2}{2 - \eta \lambda_1^2 + \epsilon}, \frac{\eta^2 \lambda_1^3 }{(2-\lambda_1\eta)\lambda_1 \eta + \epsilon}\right), \forall t > T_{\epsilon}, t \in S$. As $\forall t \not \in S, t > T_{\epsilon}$, we have $\| \tx(t) \|_2 \ge \frac{\eta\lambda_1^2}{2 - \eta \lambda_1^2 }$. Hence we have $\forall t > T_{\epsilon}, \| \tx(t) \|_2 \ge \min \left(\frac{\eta\lambda_1^2}{2 - \eta \lambda_1^2 + \epsilon}, \frac{\eta^2 \lambda_1^3 }{(2-\lambda_1\eta)\lambda_1 \eta + \epsilon}\right)$.

Further by \Cref{lem:controlonehop}, $\forall t > T_{\epsilon} + 1, \| \tx(t) \|_2 \le \eta \lambda_1^2 - (1- \eta \lambda_1) \min \left(\frac{\eta\lambda_1^2}{2 - \eta \lambda_1^2 + \epsilon}, \frac{\eta^2 \lambda_1^3 }{(2-\lambda_1\eta)\lambda_1 \eta + \epsilon}\right)$.

Combining both bound, we have $\lim \limits_{t \to \infty} \| \tx(t) \|_2 = \ths$.
\end{proof}

\begin{lemma}
\label{lem:quaddirconverge}
 $\|\tx_1(t)\|_2$ converges to $\ths$, when $t \to \infty$.
\end{lemma}
\begin{proof}[Proof of~\Cref{lem:quaddirconverge}]
Notice that 
\begin{align*}
    \| \Ptd \tx(t + 1) \|_2 \le \max\left(|1 - \eta \lambda_2 - \eta \frac{\lambda_2^2}{\|\tx(t)\|_2}|,|1 - \eta \lambda_D - \eta \frac{\lambda_D^2}{\|\tx(t)\|_2}|\right) \|\Ptd \tx(t) \|_2\,.
\end{align*}

When $\| \tx(t) \|_2 >  \frac{\eta \lambda_2^2}{2 - \eta\lambda_2 - \delta}$,
\begin{align*}
     -1 + \delta \le 1 - \eta \lambda_2 - \eta \frac{\lambda_2^2}{\|\tx(t)\|_2} &\le 1 - \eta \lambda_D - \eta \frac{\lambda_D^2}{\|\tx(t)\|_2} \le 1 - \eta \lambda_D \\
     \| \Ptd \tx(t + 1) \|_2  &\le \max(1 - \eta \lambda_D, 1 - \delta)\| \Ptd \tx(t) \|_2  
\end{align*}

Hence for sufficiently large $t$, $\| \Ptd \tx(t) \|_2$ shrinks exponentially, showing that $\lim \limits_{t \to \infty} \| \tx_1(t) \|_2= \ths $. 
\end{proof}

\section{Analysis for Full-batch SAM on General Loss (Proof of Theorem~\ref{thm:nsam})}
\label{app:nsam}

The goal of this section is to prove the following theorem.

\thmnsam*

To prove the theorem, we will separate the dynamic of SAM on general loss $L$ to two phases. 

Define 
\begin{align*}
    R_j(x) = \sqrt{\sum_{i=j}^M \lambda_i^2(x) \langle v_i(x),x - \Phi(x) \rangle^2} - \eta \rho \lambda_j^2(x),  \forall j \in [M], x \in U,
\end{align*}
which is the length projection of $x - \Phi(x)$ on button$-k$ non-zero eigenspace of $\nabla^2 L(\Phi(x))$. We will provide a fine-grained convergence bound on $R_j(x)$.

\begin{theorem}[Phase I]
\label{thm:nsamphase1}
Let $\{x(t)\}$ be the iterates defined by SAM ( \Cref{eq:sam}) and $x(t) = \xinit \in U$, then under Assumption ~\ref{assump:smoothness} there exists a positive number $T_1$ independent of $\eta$ and $\rho$, such that for any $T_1' > T_1$, it holds for all $\eta,\rho$ such that $(\eta + \rho)\ln(1/\eta\rho)$ is sufficiently small, we have \begin{align*}
    &\max\limits_{ T_1\ln(1/ \eta \rho)\le \eta t \le T_1'\ln(1/ \eta \rho) } \max_{j \in [M]} R_j(x(t)) \le O(\eta\rho^2) \\
    &\max\limits_{ T_1\ln(1/ \eta \rho)\le \eta t \le T_1'\ln(1/ \eta \rho) } \|\Phi(x(t)) - \Phi(\xinit) \| \le O((\eta + \rho)\ln (1/\eta\rho))  
\end{align*} 
\end{theorem}

Theorem~\ref{thm:nsamphase1} implies SAM will converge to an $O(\eta \rho)$ neighbor of $\Gamma$. Notice in the time frame defined by Theorem~\ref{thm:nsamphase1}, $x(t)$ effectively operates at a local regime around $\Phi(\lceil T_1\ln(1/\eta\rho)/\eta \rceil)$, this allows us to approximate $L$ with the quadratic Taylor expansion of $L$ at $\Phi(\lceil T_1\ln(1/\eta\rho)/\eta \rceil)$ and prove the following theorem~\Cref{thm:nsamphase2}. 

Towards proving~\Cref{thm:nsamphase2}, we  need to make one assumption about the trajectory of SAM, \Cref{assum:reg}.

\begin{assumption}
\label{assum:reg}
There exists step $t$, satisfying that $ T_1\ln(1/ \eta \rho)/\eta \le t \le   O(\ln(1/ \eta \rho/\eta))$,
$|\langle x(t) - \Phi(x(t)), v_1(x(t))\rangle| \ge \Omega(\rho^2)$ and that $\|x(t) - \Phi(x(t))\|_2 \le  \lambda_1(t) \eta\rho - \Omega(\rho^2)$, where $T_1$ is the constant defined in~\Cref{thm:nsamphase1}.
\end{assumption}

We remark that the above assumption is very mild as we only need the above two conditions in \Cref{assum:reg} to hold for some step in $\tilde{\Theta}(1/\eta)$ steps after Phase I ends, and since then our analysis for Phase II shows that these two conditions will hold until Phase II ends.

\begin{theorem}[Phase II]
\label{thm:nsamphase2}
  Let $\{x(t)\}$ be the iterates defined by SAM (\Cref{eq:sam}) under Assumptions~\ref{assump:smoothness} and~\ref{assum_eigengap}, for all $\eta,\rho$ such that $\eta \ln(1/\rho)$ and $\rho/\eta$ is sufficiently small, further assuming that (1) $\max_j R_j(x(\tphase)) = O(\eta\rho^2)$, (2) $\|\Phi(x(\tphase)) - \Phi(\xinit)\| = O((\eta + \rho)\ln (1/\eta\rho))$, (3) $|\langle x(\tphase) - \Phi(x(\tphase)), v_1(x(\tphase))\rangle| \ge \Omega(\rho^2)$ and (4) $\|x(\tphase) - \Phi(x(t))\|_2  \le  \lambda_1(0) \eta\rho - \Omega(\rho^2)$, the iterates $x(t)$ tracks the solution $X$ of \Cref{lambda1ode}. Quantitatively for $t = \lceil T_3/\eta\rho^2 \rceil$, we have that
  \begin{align*}
      \|\Phi(x(t)) - X(\eta\rho^2 t) \| &= O(\eta  \ln(1/\rho))\,.
\end{align*} 
Moreover, the angle between $\nabla L(x(t))$ and the top eigenspace of $\nabla^2 L( \Phi(x(t)))$ is at most $O(\rho)$. Quantitatively, 
\begin{align*}
     |\langle x(t) - \Phi(x(t)), v_1(x(t)) \rangle | &= \Theta(\eta \rho)  \,.\\
     \max_{j \in [2:M]} |\langle x(t) - \Phi(x(t)), v_j(x(t)) \rangle | &= O(\eta\rho^2)\,.
\end{align*}
\end{theorem}

In this section we will define $K$ as $\{ X(t) \mid 0 \le t \le T_3 \}$ where $X$ is the solution of \Cref{lambda1ode}. To simplify our proof, we assume WLOG $L(x) = 0$ for $x \in \Gamma$.

\subsection{Phase I (Proof of Theorem~\ref{thm:nsamphase1})}

\begin{proof}[Proof of~\Cref{thm:nsamphase1}]
The proof consists of three major parts.
\begin{enumerate}
    \item \textit{Tracking Gradient Flow.} \Cref{lem:nsam_gf} shows the existence of step $\tgf = O(1/\eta)$ such that $\xgf$ is in a subset of $K^h$ and $\Phi(\xgf)$ is $O(\eta + \rho)$ close to $\Phi(\xinit)$.
    \item \textit{Decreasing Loss.}  \Cref{lem:nsam_decrease} shows the existence of step $\tdec = O(\ln(1/\rho)/\eta)$ such that $\xdec$ is in $O(\rho)$ neighbor of $\Gamma$ and $\Phi(\xdec)$ is $O((\eta + \rho)\ln(1/\rho))$ close to $\Phi(\xinit)$.
    \item \textit{Entering Invariant Set.} \Cref{lem:nsam_inv,lem:nsam_decrease_2} shows the existence of step $\tinv = O(\ln(1/\rho\eta)/\eta)$ such that for any $t$ satisfying $\tinv \le t \le \tinv + \Theta(\ln(1/\eta)/\eta)$, we have that $x(t) \in \cap_{k \in [M]} \sI_k$ and $\Phi(x(t))$ is $O((\eta + \rho)\ln(1/\eta\rho))$ close to $\Phi(\xinit)$.
\end{enumerate}
\end{proof}

\subsubsection{Tracking Gradient Flow}
\label{sub:1.A}

\Cref{lem:nsam_gf} shows that the iterates $x(t)$ tracks gradient flow to an $O(1)$ neighbor of $\Gamma$.

\begin{lemma}
\label{lem:nsam_gf}
Under condition of~\Cref{thm:nsamphase1}, there exists $\tgf = O(1/\eta)$, such that the iterate $x(\tgf)$ is $O(1)$ close to the manifold $\Gamma$ and $\Phi(x(\tgf))$ is $O(\eta + \rho)$ is close to $\Phi(\xinit)$. Quantitatively, 
\begin{align*}
    L(x(\tgf)) &\le \frac{\mu h^2}{32}\\
    \|x(\tgf) - \Phi(x(\tgf)) \| &\le h/4\,,\\
    \|\Phi(x(\tgf)) - \Phi(\xinit) \| &= O(\eta + \rho)\,.
\end{align*}
\end{lemma}

\begin{proof}[Proof of~\Cref{lem:nsam_gf}]
Choose $C =  \frac{1}{4} \sqrt{\frac{\mu}{\lspectraltwo}}$. Since $\Phi(\xinit) = \lim_{T\to \infty}\phi(\xinit, T) $, there exists $T > 0$, such that
 $   \| \phi(\xinit, T) - \Phi(\xinit) \|_2 \le Ch/2\,$. Note that
\begin{align*}
    x(t+1) &= x(t) - \eta \nabla L(x(t) + \rho \ndl{x(t)}) = x(t) - \eta \nabla L(x(t)) + O(\eta\rho)\,.
\end{align*}

By~\Cref{thm:deterministic_ode_approximation}, let $b(x) = -\nabla L(x)$, $p = \eta$ and $\eps = O(\rho)$, we have that the iterates $x(t)$ tracks gradient flow $\phi(\xinit,T)$ in $O(1/\eta)$ steps. Quantitatively for $\tgf = \lceil \frac{T}{\eta} \rceil$, we have that
\begin{align*}
    \| x(\tgf) - \phi(\xinit,T) \|_2 &= O(\eps + p) =  O(\eta + \rho)\,.
\end{align*}

This implies $x(\tgf) \in K^h$, hence by Taylor Expansion on $\Phi$,
\begin{align*}
    \| \Phi(x(\tgf)) - \Phi(\xinit)\|_2&=\| \Phi(x(\tgf)) - \Phi(\phi(\xinit,T))\|_2 \\
    &\le O(\| x(\tgf) - \phi(\xinit,T) \|_2) \\
    &\le O(\eta + \rho)\,.
\end{align*}
This implies
\begin{align*}
    &\|  x(\tgf) - \Phi(x(\tgf))\|_2 \\
    \le& \|x(\tgf) - \phi(\xinit,T_0) \|_2 + \| \phi(\xinit,T_0) - \Phi(\xinit) \|_2+ \| \Phi(\xinit) - \Phi(x(\tgf)) \|_2 
    \\\le& Ch/2 + O(\eta + \rho) \le Ch \le h/4\,.
\end{align*}

By Taylor Expansion, we conclude that 
$L(\xgf) \le \lspectraltwo \|  x(\tgf) - \Phi(x(\tgf))\|_2^2/2 \le \frac{\mu h^2}{32}$.
\end{proof}

\subsubsection{Decreasing Loss}
\label{sub:1.B}

\Cref{lem:nsam_decrease} shows that the iterates $x(t)$ converges to an $O(\rho) $ neighbor of $\Gamma$ in $O(\ln(1/\rho)/\eta)$ steps.

\begin{lemma}
\label{lem:nsam_decrease_one_step}
Under condition of~\Cref{thm:nsamphase1}, if $x(t) \in K^h$ and $\| \nabla L(x(t)) \| \ge 4 \zeta \rho$, then we have that $L(x(t+1))$ decreases with respect to $L(x(t))$, quantitatively, we have that
\begin{align*}
    L(x(t + 1)) \le  L(x(t)) (1 - \eta \mu/8)\,.
\end{align*}
Moreover the movement of the projection of the iterates on the manifold is bounded, quantitatively, we have that
\begin{align*}
    \| \Phi(x(t + 1)) - \Phi(x(t)) \| \le O(\eta^2)\,.
\end{align*}
\end{lemma}

\begin{proof}[Proof of~\Cref{lem:nsam_decrease_one_step}]
As $x(t) \in K^h$ and $L$ is $\mu$-PL in $K^h$, we have $L(x(t)) \ge 0$.

As $x(t) \in K^h$, by \Cref{lem:boundedphi} and Taylor Expansion, we have $\|\overline{x(t)x(t+1)}\| = O(\eta)$. hence for sufficiently small $\eta$, $\overline{x(t)x(t+1)} \subset K^{r}$. Using similar argument, the segment from $x(t)$ to $x(t) + \rho \ndl{x(t)}$ is in $K^r$.

Then by Taylor Expansion on $L$,
\begin{align}
\label{eq:loss_decrease_1}
    L(x(t+1)) &= L(x(t) - \eta \dl{x(t) + \rho \ndl{x(t)}})\nonumber \\
    &\le L(x(t)) - \eta \left \langle \dl{x(t)}, \dl{x(t)  +  \rho \ndl{x(t)}} \right \rangle \nonumber\\& + \frac{\lspectraltwo \eta^2 \|\dl{x(t) + \rho \ndl{x(t)}}\|^2 }{2}\,.
\end{align}
By Taylor Expansion on $\nabla L$, we have that
\begin{align*}
    \| \dl{x(t) + \rho \ndl{x(t)}} - \dl{x(t)} \| \le \lspectraltwo \rho\,.
\end{align*}
After plugging in~\Cref{eq:loss_decrease_1}, we have that
\begin{align}
\label{eq:loss_decrease_2}
    L(x(t+1)) &\le L(x(t)) - \eta \|\dl{x(t)} \|^2 + \eta \lspectraltwo \rho \| \dl{x(t)} \| + \lspectraltwo \eta^2  \|\dl{x(t)} \|^2 + \lspectraltwo^3 \eta^2 \rho^2\,.
\end{align}

As $\| \nabla L(x(t)) \| \ge 4\zeta\rho$, we have that the following term is bounded.
\begin{align*}
    \lspectraltwo \eta^2 \|\dl{x(t)} \|^2 &\le \frac12 \eta \|\dl{x(t)} \|^2 \,,\\
    \eta \lspectraltwo \rho \| \dl{x(t)} \| &\le \frac{1}{4} \eta \| \dl{x(t)}\|^2  \,,\\
    \lspectraltwo^3 \eta^2 \rho &\le \lspectraltwo^2 \eta \rho^2 \le \frac{1}{16}\eta \| \dl{x(t)}\|^2 \,.
\end{align*}
After plugging in~\Cref{eq:loss_decrease_2}, by~\Cref{lem:revertbounddl},
\begin{align*}
    L(x(t+1)) &\le L(x(t)) - \frac{1}{16}\eta \| \dl{x(t)} \|^2 \\&\le L(x(t)) (1 - \eta \mu /8)\,.
\end{align*}

As $x(t) \in K^h$, by Taylor Expansion, we have 
\begin{align*}
    \|\dl{x(t)}\| \le \lspectraltwo h \,.
\end{align*}
Hence by Lemma ~\ref{lem:boundedphi} and Taylor Expansion,
\begin{align*}
    \| \Phi(x(t+1)) - \Phi(x(t)) \| &\le \phispectraltwo \eta \rho \|\dl x \|_2 + \lspectralthree \eta \rho^2  + \phispectraltwo \eta^2 \|\dl x \|_2^2 + \phispectraltwo  \lspectraltwo^2 \eta^2 \rho^2 \le O(\eta^2),
\end{align*}
which completes the proof.
\end{proof}

\begin{lemma}
\label{lem:nsam_decrease}
Under condition of~\Cref{thm:nsamphase1}, assuming there exists $\tgf$ such that $L(x(\tgf)) \le \frac{\mu h^2}{32}$ and $x(\tgf) \in K^{h/4}$, then there exists $\tdec = \tgf + O(\ln(1/\rho)/\eta)$, such that $x(\tdec)$ is in $O(\rho)$ neighbor of $\Gamma$, quantitatively, we have that
\begin{align*}
    \| \nabla L(x(\tdec)) \|_2 &\le 4\zeta \rho\,. 
\end{align*}
Moreover the movement of the projection of $\Phi(x(\cdot))$ on the manifold is bounded,
\begin{align*}
    \| \Phi(x(\tgf)) - \Phi(x(\tdec)) \|_2 &= O(\eta \ln(1/\rho))\,.
\end{align*}
\end{lemma}

\begin{proof}[Proof of~\Cref{lem:nsam_decrease}]

Choose $\tdec$ as the minimal $t \ge \tgf$ such that $\| \nabla L(x(\tdec)) \|_2 \le 4\zeta\rho$. Define $C = \lceil \ln_{1 - \frac{\eta\mu}{8}}(64  \rho^2/h^2) \rceil = O(\ln(1/\rho)/\eta)$.

We will first perform an induction on $t \le \min\{\tdec, \tgf +  C \}=\tgf +  O(\ln(1/\rho)/\eta)$ to show that 
\begin{align*}
    L(x(t)) &\le (1 - \eta \mu/8)^{t - \tgf} L(x(\tgf)) \\
    \|\Phi(x(t)) - \Phi(x(\tgf)) \| &= O(\eta^2 (t - \tgf) )
\end{align*}

For $t = \tgf$, the result holds trivially. Suppose the induction hypothesis holds for $t$. Then by~\ref{lem:bounddl} and Taylor Expansion,
\begin{align*}
    \|\Phi(x(t)) - x(t) \| &\le \sqrt{\frac{2L(x(\tgf))}{\mu}} \le h/4 \,.
\end{align*}

Then we have that
\begin{align*}
    \dist(K,x(t)) \le& \dist(K,\xgf) + \|\xgf - \Phi(\xgf) \|_2
    \\&+ \|\Phi(\xgf) - \Phi(x(t)) \| + \|\Phi(x(t)) - x(t) \| \\
    \le& 3h/4 + O(\eta^2 (t - \tgf)) = 3h/4 + O(\eta\ln(1/\rho)) \le h\,.
\end{align*}
That is $x(t) \in K^h$. Then as $t\le \tdec$, $\| \nabla L(x(t)) \|_2 \ge 4 \zeta \rho$. Then by~\Cref{lem:nsam_decrease_one_step}, we have that 
\begin{align*}
    L(x(t+1)) &\le (1 - \eta \mu/8) L(x(t)) \le (1 - \eta \mu/8)^{t + 1- \tgf} L(x(\tgf)) \,,\\
    \|\Phi(x(t+1)) - \Phi(x(\tgf)) \| &\le \|\Phi(x(t+1)) - \Phi(x(t)) \| + \|\Phi(x(t)) - \Phi(x(\tgf)) \| \\
    &\le O(\eta^2 (t - \tgf) )\,,
\end{align*}
which completes the induction.

Now if $\tdec \ge \tgf +  C = \tgf + \Omega(\ln(1/\rho)/\eta)$, 
As the result of the induction, we have that
\begin{align*}
    L(x(\tgf + C)) &\le (1 - \frac{\eta\mu}{8})^{C} L(x(\tgf)) \le \frac{64  \rho^2}{h^2 }L(x(\tgf))  \le8 \rho^2 \mu \,.\\
\end{align*}
By~\Cref{lem:revertbounddl}, we have that 
    $\|\nabla L(x(\tgf + C))\|_2 \le \zeta \sqrt{\frac{2 L(x(\tgf + C))}{\mu}} = 4 \zeta\rho$,
which leads to a contradiction.

Hence we have that $\tdec \le \tgf +  C = \tgf + O(\ln(1/\rho)/\eta)$. By induction, we have that
\begin{align*}
    \| \Phi(\xdec) - \Phi(\xgf) \| = O(\eta^2(\tdec - \tgf)) = O(\eta\ln(1/\rho))\,.
\end{align*}
This completes the proof.
\end{proof}

\subsubsection{Entering Invariant Set}
\label{sub:nsam_phase_1_3}
We first introduce some notations that is required for the proof in this and following subsection.

Define 
\begin{align*}
    \hx &= x - \Phi(x)\,,\\
    A(x) &= \dlt{\Phi(x)}\,,\\
    \tx &= A(x) \hx\,,\\
    \tx_j &= \langle \tx, v_j(x) \rangle\,, \\
    \Pjd(x) &= \sum_{i=j}^M  v_i(x)v_i^T(x)\,.
\end{align*}
Note $\tilde x\approx \nabla L(x)$ for $x$ near the manifold $\Gamma$. We also use $\tx(t),\,A(t)$ and $\hx(t)$ to denote $\tilde{x(t)},\,A(x(t))$ and $\hat{x(t)}$.

Recall the original definition of $R_j(x)$ is 
\begin{align*}
    R_j(x) = \sqrt{\sum_{i=j}^M \lambda_i^2(x) \langle v_i(x),x - \Phi(x) \rangle^2} - \eta \rho \lambda_j^2(x)\,,
\end{align*}
Based on the above notions, we can rephrase the notion $R$ as \begin{align*}
    R_j(x) = \| \Pjd(x) \tilde x \| - \eta \rho \lambda_j^2(x)\,.
\end{align*}
We additionally define the approximate invariant set $\sI_j$ as 
\begin{align*}
    \sI_j = \{\| \Pjd(x) \tilde x \|  \le \eta \rho \lambda_j^2(x) + O(\eta \rho^2)\}\,.
\end{align*}

\begin{lemma}
\label{lem:nsam_tilde_x_bound}
Assuming $t$ satisfy that $x(t) \in K^h$, then we have that
\begin{align*}
    \frac{\mu}{2} \|x(t) - \Phi(x(t))\|\le  \|\tx(t)\| \le \lspectraltwo \|x(t) - \Phi(x(t))\|
\end{align*}
\end{lemma}

\begin{proof}[Proof of~\Cref{lem:nsam_tilde_x_bound}]
First by \Cref{lem:directiondl}, $\Phi(x(t)) \in K^r$, hence
\begin{align*}
    \| \tx(t) \| &= \| \nabla^2 L(\Phi(x(t))) (x(t) - \Phi(x(t))) \| \le \lspectraltwo \| x(t) - \Phi(x(t)) \|\,.
\end{align*}
Also
\begin{align*}
    \| \tx(t) \| &= \| \nabla^2 L(\Phi(x(t))) (x(t) - \Phi(x(t))) \| \ge \mu \|\projt[\Phi(x(t))](x(t) - \Phi(x(t))) \|\,.
\end{align*}

By \Cref{lem:directiondl} and \Cref{lem:smallerzone}, we have
\begin{align*}
    \|x(t) - \Phi(x(t))\| &\le \|\projt[\Phi(x(t))] (x(t) - \Phi(x(t))) \| + \|\projn[\Phi(x(t))](x(t) - \Phi(x(t)))\| \\
    &\le \frac{\lspectraltwo \lspectralthree}{4\mu^2} \|x(t) - \Phi(x(t))\|^2 + \frac{1}{\mu}  \|\tx(t)\| \\&\le \frac{1}{2}\|x(t) - \Phi(x(t))\| + \frac{1}{\mu}  \|\tx(t)\|.
\end{align*}
Hence $\|x(t) - \Phi(x(t))\| \le \frac{2}{\mu}\|\tx(t)\|$.
\end{proof}

\begin{lemma}
\label{lem:nsam_bounded_mov}
Assuming $t$ satisfy that $x(t) \in K^h$ and $\| \tilde x(t)\|_2 = O(\rho)$, then we have that
\begin{align*}
    \|\Phi(x(t+1)) - \Phi(x(t)) \| = O(\eta \rho^2)\,.
\end{align*}
\end{lemma}

\begin{proof}[Proof of~\Cref{lem:nsam_bounded_mov}]

By~\Cref{lem:nsam_tilde_x_bound}, we have $\|x(t) - \Phi(x(t))\| = O(\rho)$. By~\Cref{lem:boundedphi}, we have that
\begin{align*}
    \| \Phi(x(t+1)) - \Phi(x(t)) \| &\le \lspectraltwo \phispectraltwo \eta \rho \|x - \Phi(x)\|_2 + \lspectraltwo^2 \phispectraltwo \eta^2\|x - \Phi(x)\|_2^2  + \lspectralthree \eta \rho^2  + \phispectraltwo  \lspectraltwo^2 \eta^2 \rho^2 \\
    &\le O(\eta\rho^2).
\end{align*}
\end{proof}

\begin{lemma}
\label{lem:localupdate}
Assuming $t$ satisfy $x(t) \in K^{h/2}$ and $\|x(t) - \Phi(x(t))\|_2 = O(\rho)$, define $\xloc$  as $\xloc(t) = x(t)$ and for $\tau \ge t$,
\begin{align*}
    \xloc(\tau +1) &= \xloc(\tau) - \eta \nabla^2 L(\Phi(x(t))) (\xloc(\tau) - \Phi(x(t))) \\
    -& \eta \rho\nabla^2 L(\Phi(x(t)))  \frac{ \nabla^2 L(\Phi(x(t))) (\xloc(\tau) - \Phi(x(t)))}{ \| \nabla^2 L(\Phi(x(t))) (\xloc(\tau) - \Phi(x(t)))\|_2}
\end{align*}

Then 
\begin{align*}
   \| \xloc(t + 1) - x(t + 1) \|_2 = O(\eta\rho^2)
\end{align*}

and further if  $\| x(t + 1) - \Phi(x(t + 1)) \|_2 = \Omega(\eta \rho)$, then 
\begin{align*}
\| \xloc(t + 2) - x(t + 2) \|_2 = O(\eta\rho^2).
\end{align*}

\end{lemma}

\begin{proof}[Proof of~\Cref{lem:localupdate}]

By $\|x(t) - \Phi(x(t))\| = O(\rho)$, $x(t) \in K^{h/2}$, and~\Cref{lem:boundedphi}, we have that $\|x(t+1) - x(t)\| = O(\eta\rho)$ and hence $x(t + 1) \in K^{3h/4}$. This also implies $\|x(t+1) - \Phi(x(t+1))\|_2 = O(\rho)$. Similarly we have $x(t+2) \in K^{3h/4}$.

For $k \in \{1,2\}$, by Taylor Expansion,
\begin{align*}
    x(t+k+1) =& x(t+k) - \eta \nabla L(x(t+k)) - \eta \rho \nabla^2 L(x(t+k)) \ndl{x(t+k)} + O(\eta\rho^2) \\
    =& x(t+k) - \eta \nabla^2 L(\Phi(x(t+k))) (x(t+k) - \Phi(x(t+k))) + O(\eta \rho^2)\\
    &- \eta\rho \nabla^2 L(\Phi(x(t+k))) \ndl{x(t+k)} + O(\eta\rho^2)  \\
    =&  x(t+k) - \eta \nabla^2 L(\Phi(x(t+k))) (x(t+k) - \Phi(x(t+k)))\\
    &- \eta\rho \nabla^2 L(\Phi(x(t+k))) \ndl{x(t+k)} + O(\eta\rho^2).
\end{align*}

Now by~\Cref{lem:nsam_bounded_mov,lem:nsam_tilde_x_bound}, $\|\Phi(x(t+k)) - \Phi(x(t))\|_2 = O(\eta\rho^2)$,
\begin{align}
\label{eq:local_update_1}
    x(t+ k + 1) =& x(t + k) - \eta \nabla^2 L(\Phi(x)) (x(t + k) - \Phi(x(t))) \nonumber \\& - \eta\rho \nabla^2 L(\Phi(x)) \ndl{x(t + k)} + O(\eta\rho^2).
\end{align}

Now we first prove the first claim, we have for $k = 0$, $\|x(t + k) - \xloc(t + k)\|_2 = 0$, by \Cref{lem:directiondl,eq:local_update_1},
\begin{align*}
    x(t + 1) =& x(t) - \eta \nabla^2 L(\Phi(x)) (x(t) - \Phi(x))
    - \eta\rho \frac{ \nabla^2 L(\Phi(x)) (x(t) - \Phi(x))}{ \| \nabla^2 L(\Phi(x)) (x(t) - \Phi(x))\|_2} + O(\eta\rho^2) \\=& \xloc(t + 1) + O(\eta\rho^2).
\end{align*}

The second claim is slightly more complex. By the first claim and \Cref{lem:directiondl}, we have that
\begin{align}
\label{eq:local_2}
    \ndl{x(t + 1)} =& \frac{ \nabla^2 L(\Phi(x(t + 1))) (x(t + 1) - \Phi(x(t + 1)))}{ \| \nabla^2 L(\Phi(x(t + 1))) (x(t + 1) - \Phi(x(t + 1)))\|_2} \nonumber\\ +& O(\|x(t + 1) - \Phi(x(t + 1)) \|_2).
\end{align}

We first show $\| \nabla^2 L(\Phi(x(t + 1))) (x(t + 1) - \Phi(x(t + 1)))\|_2$ is of order $\| x(t + 1) - \Phi(x(t + 1)) \|_2 = \Omega(\rho^2)$ to show that the normalized gradient term is stable with respect to small perturbation,
\begin{align*}
     &\| \nabla^2 L(\Phi(x(t+1))) (x(t+1) - \Phi(x(t+1))) \|_2 \\ \ge&\| \projn[\Phi(x(t+1))]\nabla^2 L(\Phi(x(t+1))) (x(t+1) - \Phi(x(t+1)))\|_2 \\
    \ge&\| \nabla^2 L(\Phi(x(t+1)))\projn[\Phi(x(t+1))] (x(t+1) - \Phi(x(t+1)))\|_2 \\
    \ge&\mu \|\projn[\Phi(x(t+1))] (x(t+1) - \Phi(x(t+1))) \|_2 \\
    \ge&\mu (\| (x(t+1) - \Phi(x(t+1))) \|_2 -  \|\projt[\Phi(x(t+1))] (x(t+1) - \Phi(x(t+1))) \|_2) \\
    \ge&\mu (\| (x(t+1) - \Phi(x(t+1))) \|_2 - \frac{\lspectralthree\lspectraltwo}{4\mu^2} \|x(t+1) - \Phi(x(t+1)) \|_2^2) \\
    \ge&\frac{\mu}{2}\| (x(t+1) - \Phi(x(t+1))) \|_2 = \Omega(\eta\rho).
\end{align*}

Based on \Cref{lem:boundedphi}, we have
\begin{align*}
    \Phi(x(t+1)) - \Phi(x(t)) = O(\eta\rho^2).
\end{align*}

We further have by the first claim and~\Cref{lem:nsam_bounded_mov},
\begin{align*}
   &\nabla^2 L(\Phi(x(t+1))) (x(t+1) - \Phi(x(t+1))) -  \nabla^2 L(\Phi(x)) (\xloc(t + 1) - \Phi(x(t))) \\ =& 
   \nabla^2 L(\Phi(x)) (x(t+1) - \Phi(x(t+1))) -  \nabla^2 L(\Phi(x)) (\xloc(t + 1) - \Phi(x(t))\\
   &+ O(\|x(t+1) - \Phi(x(t+1)) \|_2 \|\Phi(x(t+1)) - \Phi(x) \|_2) \\
   =& \nabla^2 L(\Phi(x)) (x(t+1) - \Phi(x(t+1))) -  \nabla^2 L(\Phi(x)) (\xloc(t + 1) - \Phi(x(t)))  + O(\eta \rho^3) \\
   =&\nabla^2 L(\Phi(x)) (x(t+1) - \xloc(t + 1)) + \nabla^2 L(\Phi(x)) (\Phi(x(t+1)) - \Phi(x(t))) + O(\eta \rho^3) \\
   =& O(\eta\rho^2)
\end{align*}

This implies 
\begin{align*}
    \frac{ \nabla^2 L(\Phi(x(t+1))) (x(t+1) - \Phi(x(t+1)))}{ \| \nabla^2 L(\Phi(x(t+1))) (x(t+1) - \Phi(x(t+1)))\|_2} = \frac{ \nabla^2 L(\Phi(x)) (\xloc(t + 1) - \Phi(x))}{ \| \nabla^2 L(\Phi(x)) (\xloc(t + 1) - \Phi(x))\|_2} + O(\rho)
\end{align*}

Combining with~\Cref{eq:local_2}, we have 
\begin{align*}
    \ndl{x(t + 1)} = \frac{ \nabla^2 L(\Phi(x)) (\xloc(t + 1) - \Phi(x))}{ \| \nabla^2 L(\Phi(x)) (\xloc(t + 1) - \Phi(x))\|_2} + O(\rho)
\end{align*}

By the above approximation and~\Cref{eq:local_update_1},
\begin{align*}
    x(t+2) &= \xloc(t + 2) + O(\eta\rho^2)\,.
\end{align*}
\end{proof}

\begin{lemma}
\label{lem:tilde_x_update}
Assuming $t$ satisfy that $x(t) \in K^{3h/4}$ and $\| \tilde x(t)\|_2 = O(\rho)$, then we have that
\begin{align*}
    \| \tx(t+1) -  \tx(t) + \eta A(t)\tx(t) + \eta\rho A^2(t)\frac{\tx(t)}{\|\tx(t) \|} \|_2 = O(\eta\rho^2)\,.
\end{align*}
\end{lemma}

\begin{proof}[Proof of~\Cref{lem:tilde_x_update}]
By \Cref{lem:localupdate}, we know
\begin{align*}
    \| x(t+1) - x(t) + \eta \tx(t) + \eta\rho A(t)\frac{\tx(t)}{\|\tx(t) \|} \| \le O(\eta\rho^2)\,.
\end{align*}

This implies 
\begin{align}
\label{eq:decrease_2_1}
 &\| A(t)(x(t+1) - \Phi(x(t))) - \tx(t) + \eta A(t)\tx(t) + \eta\rho A^2(t)\frac{\tx(t)}{\|\tx(t) \|} \| \nonumber \\
=&\| A(t)(x(t+1) - x(t) + \eta \tx(t) + \eta\rho A(t)\frac{\tx(t)}{\|\tx(t) \|} )\| \nonumber \\
\le &\lspectraltwo \| x(t+1) - x(t) + \eta \tx(t) + \eta\rho A(t)\frac{\tx(t)}{\|\tx(t) \|}\| =  O(\eta\rho^2)\,.
\end{align}
We also have 
\begin{align*}
    &\tx(t+1) - A(t)(x(t+1) - \Phi(x(t)))\\=& (A(t+1) - A(t)) (x(t+1) - \Phi(x(t+1))) - A(t) (\Phi(x(t)) - \Phi(x(t+1)))\\ =& O(\eta\rho^2)\,.
\end{align*}

Plugging in~\Cref{eq:decrease_2_1}, we have that
\begin{align*}
    \| \tx(t+1) -  \tx(t) + \eta A(t)\tx(t) + \eta\rho A^2(t)\frac{\tx(t)}{\|\tx(t) \|} \|_2 = O(\eta\rho^2)\,.
\end{align*}
\end{proof}

\begin{lemma}
\label{lem:nsam_decrease_2}
Under condition of~\Cref{thm:nsamphase1}, assuming there exists $\tdec$ such that $x(\tdec) \in K^{h/2}$ and $\| \nabla L(x(\tdec))\| \le 4 \zeta \rho$, then there exists $\tdecs = \tdec + O(\ln(1/\eta)/\eta)$, such that $x(\tdecs)$ is in $\sI_1 \cap K^{3h/4}$.

Furthermore, for any $t$ satisfying $\tdecs \le t \le \tdecs + \Theta(\ln(1/\eta)/\eta)$, we have that $x(t) \in \sI_1 \cap K^{3h/4}$ and $\| \Phi(x(t)) - \Phi(x(\tdec))\| = O(\rho^2 \ln(1/\eta))$.
\end{lemma}

\begin{proof}[Proof of~\Cref{lem:nsam_decrease_2}]
For simplicity, denote $C =  \lceil \ln_{1-\eta\mu}{\frac{\eta \mu^3}{4\lspectraltwo^2}}\rceil +  \Theta(\ln(1/\rho)/\eta) = O(\ln(1/\eta)/\eta) $. Here the quantity $\Theta(\ln(1/\rho)/\eta)$ is the same quantity in the statement of the lemma.

We will prove the induction hypothesis for $ \tdec \le t \le \tdec + 2C $,
\begin{align*}
\left\{
\begin{aligned}
    \| \tx(t - 1) \| \ge \eta\rho\lambda_1^2(t), t > \tdec \Rightarrow \|\tx(t) \| &\le (1 - \eta \mu) \|\tx(t - 1) \|, \\
    \| \tx(t - 1) \| \le \eta\rho \lambda_1^2(t - 1), t > \tdec \Rightarrow  \|\tx(t) \| &\le  \eta \rho \lambda_1^2(t) + O(\eta \rho^2),\\
    \| \Phi(x(t)) - \Phi(\xdec) \| &\le O(\eta\rho^2 (t - \tdec)), \\
    x(t) &\in K^{3h/4}.
\end{aligned}
\right. 
\end{align*}

The induction hypothesis holds trivially for $t = \tdec$.

Assume the induction hypothesis holds for $t' \le t$. By \Cref{lem:nsam_tilde_x_bound,lem:bounddl}, $\| \tx(\tdec)\|_2 \le \lspectraltwo \|\xdec - \Phi(\xdec)\| \le \frac{\lspectraltwo}{\mu} \|\nabla L(\xdec) \| \le \frac{4\lspectraltwo^2}{\mu} \rho $. Combining with the induction hypothesis, we have $\| \tx(t)\| \le \frac{4\lspectraltwo^2}{\mu} \rho$.

By $x(t) \in K^{3h/4}$ and~\Cref{lem:nsam_bounded_mov}, we have that
\begin{align*}
    \| \Phi(x(t + 1)) - \Phi(x(t)) \| \le O(\eta \rho^2)\,. 
\end{align*}

Hence we have that
\begin{align}
\label{eq:nsam_inv_1}
    \| \Phi(x(t+1)) - \Phi(\xdec)\| &\le \| \Phi(x(t + 1)) - \Phi(x(t)) \| + \|\Phi(x(t)) - \Phi(\xdec) \| \nonumber \\
    &\le O(\eta\rho^2 (t + 1 - \tdec)).
\end{align}
This proves the third statement of the induction hypothesis. 

By $\| \tx(t) \| = O(\rho)$ and~\Cref{lem:tilde_x_update}, we have that 
\begin{align*}
    \| \tx(t+1) -  \tx(t) + \eta A(t)\tx(t) + \eta\rho A^2(t)\frac{\tx(t)}{\|\tx(t) \|} \|_2 = O(\eta\rho^2)\,.
\end{align*}

Analogous to the proof of~\Cref{lem:preparequad,lem:invariant_set}, we have
\begin{enumerate}
    \item  If $\|\tx(t) \| > \eta\rho \lambda_1^2(t)$, we would have \begin{align*}
    &\|\tx(t) -\eta A(t)\tx(t) - \eta \rho A^2(t) \frac{\tx(t)}{\|\tx(t) \|} \| \\
    \le& \| \tx(t)\| \|I - \eta A(t) - \eta \rho A^2(t) \frac 1 {\| \tx(t)\|} \| \\
    \le& \|\tx(t) \| \max\{\eta \lambda_1, 1 - \eta \lambda_D - \eta \rho \lambda_D^2 \frac{1}{\| x(t)\|} \}\\
    \le& \max \{ (1 - \eta \lambda_D) \|\tx(t) \| - \eta \rho \lambda_D^2, \eta \lambda_1 \|\tx(t) \|\}\\
    \le& \max \{ (1 - \eta \mu ) \|\tx(t) \| - \eta \rho \mu^2, \eta \lspectraltwo \|\tx(t) \|\} 
\end{align*}
    Hence we have
    \begin{align*}
        \| \tx(t + 1)\| \le \max \{ (1 - \eta \mu ) \|\tx(t) \| - \eta \rho \mu^2, \eta \lspectraltwo \|\tx(t) \|\}  + O(\eta\rho^2) \le (1 - \eta\mu) \|\tx(t)\|.
    \end{align*}
    \item If $\|\tx(t) \|_2 \le \eta\rho \lambda_1^2(t)$, then by~\Cref{lem:invariant_set}, we have that
    \begin{align*}
        \|\tx(t) -\eta A(t)\tx(t) - \eta \rho A^2(t) \frac{\tx(t)}{\|\tx(t) \|} \|_2 \le \eta\rho \lambda_1^2(t)\,.
    \end{align*}
    Hence by~\Cref{lem:eigendiff}
    \begin{align*}
    \|\tx(t+1) \| &\le \eta\rho\lambda_1^2(t) + O(\eta \rho^2) \le \eta\rho \lambda_1^2(t+1) + O(\eta\rho^2)\,.
\end{align*}
\end{enumerate}

Concluding the two cases, we have shown the first and second claim of the induction hypothesis holds. Hence we can show that $\| \tx(t + 1)\|  \le \frac{4\lspectraltwo^2}{\mu} \rho$. Then by~\Cref{lem:nsam_tilde_x_bound}, we have that $\|x(t+1) - \Phi(x(t+1)) \| \le \frac{8\lspectraltwo^2}{\mu^2} \rho$.

As $t \le \tdec + 2C = \tdec + O(\ln(1/\eta)/\eta)$, by~\Cref{eq:nsam_inv_1},
\begin{align*}
    \|\Phi(x(t + 1)) - \Phi(x(\tdec))\| &\le O(-\rho^2 \ln \eta)\,.
\end{align*}

This implies 
\begin{align*}
    \dist(x(t + 1), K) \le& \dist(\xdec,K) + \|\xdec - \Phi(\xdec)\| \\&+ \|\Phi(\xdec) - \Phi(x(t+1))\| + \| x(t+1) - \Phi(x(t+1))\| \\
    =& h/2 + O(\rho^2 \ln(1/\eta)) + O(\rho) \le 3h/4.
\end{align*}

This proves the fourth claim of the inductive hypothesis.

The induction is complete.

Now define $\tdecs$ the minimal $t \ge \tdec$, such that $\| \tx(t) \| \le \eta \rho \lambda^2_1(t)$.

If $\tdecs > \tdec + C$, then by the induction, \Cref{lem:nsam_tilde_x_bound,lem:bounddl},
\begin{align*}
    \| \tx(\tdec + C) \| &\le (1- \eta \mu)^{C} \|\tx(\tdec) \| \\
    &\le \frac{\eta \mu^3}{4\lspectraltwo^2} \|\tx(\tdec) \| \\
    &\le \frac{\eta \mu^3}{4\lspectraltwo^2} \lspectraltwo \|x(\tdec) - \Phi(x(\tdec) \| \\
    &\le \frac{\eta \mu^2}{4 \lspectraltwo} \| \nabla L(\tdec) \| \\
    &\le \mu^2 \eta \rho \\
    &\le \lambda_1^2(\tdec + C) \eta \rho\,.
\end{align*}
This is a contradiction. Hence we have $\tdecs \le \tdec + C$. By the induction hypothesis $x(\tdecs) \in \sI_1 \cap K^{3h/4}$.

Furthermore by induction, for any $t$ satisfying $\tdecs \le t \le \tdec + 2C$, we have that
\begin{align*}
\| \tx(t)\| \le \eta \rho \lambda_1^2(t)+O(\eta \rho^2)\,.
\end{align*}
By the induction hypothesis $x(t) \in \sI_1 \cap K^{3h/4}$ and $\| \Phi(x(t)) - \Phi(x(\tdec))\| = O(\rho^2 \ln(1/\rho))$.
\end{proof}

\begin{lemma}
\label{lem:n_sam_prepare}
Under condition of~\Cref{thm:nsamphase1}, assuming $t$ satisfy that $x(t) \in \sI_1 \cap K^{3h/4}$, then we have that
\begin{align*}
\left\{
\begin{aligned}
    R_k(x(t)) \ge 0  &\Rightarrow  R_k(x(t+1)) + \lambda_{k}^2(t + 1)\eta\rho \le (1 - \eta\mu) (R_k(x(t)) + \lambda_{k}^2(t) \eta\rho), \\
    R_k(x(t)) \le 0  &\Rightarrow  R_k(x(t + 1)) \le O(\eta \rho^2).
\end{aligned}
\right.     
\end{align*}
\end{lemma}

\begin{proof}[Proof of~\Cref{lem:n_sam_prepare}]

As $x(t) \in \sI_1$, $\| \tilde x(t)\|_2 \le \lspectraltwo \eta \rho + O(\eta\rho^2)$.

As $\| \tilde x(t)\|_2 = O(\rho)$, we have $\overline{x(t)x(t+1)} \subset K^h$ and $\overline{\Phi(x(t))\Phi(x(t+1))} \subset K^r$.

We will begin with a quantization technique separating $[M]$ into disjoint continuous subset $S_1,...,S_p$ such that $\forall i \neq j$,
\begin{align*}
    \min_{k \in S_i, l \in S_j} |\lambda_k(t) - \lambda_l(t) | \ge \rho \,.
\end{align*}

By~\Cref{lem:eigendiff,lem:nsam_bounded_mov}, we have that for any $n \in [M]$,
\begin{align*}
    |\lambda_k(t) - \lambda_k(t +1) | &= O(\| \nabla^2 L(\Phi(x(t))) - \nabla^2 L(\Phi(x(t+1))) \|) \\
    &= O(\|\Phi(x(t)) - \Phi(x(t+1)) \|) \\
    &= O(\eta\rho^2).
\end{align*}

This implies
\begin{align*}
    \min_{k \in S_i, l \in S_j} |\lambda_k(t+1) - \lambda_l(t+1) | \ge \rho - O(\eta \rho^2)  \ge 0.99 \rho\,.
\end{align*}

Define 
\begin{align*}
    P_{S^{(i)}}^{(t)} \triangleq \sum_{k \in S_i} v_n(t)v_n(t)^T\,.
\end{align*}

By~\Cref{thm:Davis-Kahan}, for any $k$,
\begin{align*}
    \| P_{S_k}^{(t)} - P_{S_k}^{(t+1)} \| \le O(\frac{\| \nabla^2 L(\Phi(x(t))) - \nabla^2 L(\Phi(x(t+1))) \|}{\rho}) = O(\eta\rho)\,.
\end{align*}

By~\Cref{lem:tilde_x_update}, we have that 
\begin{align*}
    \| \tx(t+1) -  \tx(t) + \eta A(t)\tx(t) + \eta\rho A^2(t)\frac{\tx(t)}{\|\tx(t) \|} \|_2 = O(\eta\rho^2)\,.
\end{align*}

We will write $\xloc(t+1)$ as shorthand of $\tx(t) - \eta A(t)\tx(t) - \eta\rho A^2(t)\frac{\tx(t)}{\|\tx(t) \|}$.

Now we discuss by cases,
\begin{enumerate}
    \item If $\sqrt{\sum_{i=j}^p \|P_{S^{(i)}}^{(t)} \tx(t) \|^2} > \max_{k \in S_j} \lambda_k^2(t) \eta \rho > \mu^2 \eta \rho$, by~\Cref{lem:quad_prepare},
\begin{align*}
    &\sqrt{\sum_{i=j}^p \|P_{S^{(i)}}^{(t)} \tx(t+1) \|^2} 
    \le \sqrt{\sum_{i=j}^p \|P_{S^{(i)}}^{(t)} \xloc(t+1) \|^2} + O(\eta\rho^2) \\
    \le& \max\{\bigl(1- \eta \lambda_D(t+1) \bigr)\|\sum_{i=j}^p P_{S^{(i)}}^{(t)} \tx(t) \| - \eta \rho \lambda_D(t+1)^2 \frac{\|\sum_{i=j}^p P_{S^{(i)}}^{(t)} \tx(t) \|}{\| \tx(t)\|},\\& \eta \max_{k \in S_j}\lambda_k(t+1) \|\sum_{i=j}^p P_{S^{(i)}}^{(t)} \tx(t)  \| \} + O(\eta \rho^2) \\
    \le& \max\{\bigl(1- \eta \mu \bigr)\|\sum_{i=j}^p P_{S^{(i)}}^{(t)} \tx(t) \| - \eta \rho  \frac{\mu^3}{2\zeta}, \eta \lspectraltwo \|\sum_{i=j}^p P_{S^{(i)}}^{(t)} \tx(t)  \| \} + O(\eta \rho^2) \,.
\end{align*}

This further implies
\begin{align*}
    &\sqrt{\sum_{i=j}^p \|P_{S^{(i)}}^{(t+1)} \tx(t+1) \|^2}   \le  \sqrt{\sum_{i=j}^p \|P_{S^{(i)}}^{(t)} \tx(t+1) \|^2} + O(\eta\rho \|\tx(t + 1) \|) \\
    \le& \max\{\bigl(1- \eta \mu \bigr)\|\sum_{i=j}^p P_{S^{(i)}}^{(t)} \tx(t) \| - \eta \rho  \frac{\mu^3}{2\zeta}, \eta \lspectraltwo \|\sum_{i=j}^p P_{S^{(i)}}^{(t)} \tx(t)  \| \} + O(\eta \rho^2)
    \\\le& \bigl(1- \eta \mu \bigr)\|\sum_{i=j}^p P_{S^{(i)}}^{(t)} \tx(t) \|\,.
\end{align*}

\item If $\sqrt{\sum_{i=j}^p \|P_{S^{(i)}}^{(t)} \tx(t) \|^2} \le \max_{k \in S_j} \lambda_k^2(t)  \eta \rho$, then by~\Cref{lem:invariant_set}, we have that
    \begin{align*}
        \|\sum_{i=j}^p P_{S^{(i)}}^{(t)} \xloc(t+1) \|_2 \le \eta\rho \max_{k \in S_j} \lambda_k^2(t)\,.
    \end{align*}
Hence we have that
\begin{align*}
    \sqrt{\sum_{i=j}^p \|P_{S^{(i)}}^{(t)} \tx(t+1) \|^2} 
    &\le \sqrt{\sum_{i=j}^p \|P_{S^{(i)}}^{(t)} \xloc(t+1) \|^2} + O(\eta\rho^2) \\
    &\le \max_{k \in S_j} \lambda_k^2(t)  \eta \rho + O(\eta \rho^2)\\
    &\le \max_{k \in S_j} \lambda_k^2(t+1)\eta \rho + O( \eta \rho^2)\,. \\
\end{align*}
This further implies
\begin{align*}
    \sqrt{\sum_{i=j}^p \|P_{S^{(i)}}^{(t+1)} \tx(t+1) \|^2}
    &\le  \sqrt{\sum_{i=j}^p \|P_{S^{(i)}}^{(t)} \tx(t+1) \|^2} + O(\eta\rho \|\tx(t + 1) \|)
    \\&\le \max_{k \in S_j} \lambda_k^2(t)  \eta \rho +  O(\eta \rho^2)\\
    &\le \max_{k \in S_j} \lambda_k^2(t+1)\eta \rho +  O(\eta \rho^2)\,.
\end{align*}
\end{enumerate}

Finally taking into quantization error, as all the eigenvalue in the same group at most differ $D \rho$, for any $i \in S_j$, we have that $-\lambda_i^2(t+1) + \max_{k \in S_j} \lambda_k^2(t+1) \le 2D\lspectraltwo \rho + D^2\rho^2$.

Hence the previous discussion concludes as
\begin{enumerate}
    \item If $R_k(x(t)) \ge 0$
\begin{align*}
    R_k(x(t+1)) + \lambda_{k}^2(t+1) \eta\rho   \le (1 - \eta \mu ) (R_k(x(t)) + \lambda_{k}^2(t). \eta\rho)
\end{align*}
\item If $R_k(x(t)) < 0$
\begin{align*}
    R_k(x(t+1)) &\le  O(\eta \rho^2).
\end{align*}
\end{enumerate}
\end{proof}

\begin{lemma}
\label{lem:nsam_inv}
Under condition of~\Cref{thm:nsamphase1}, assuming there exists $\tdecs$ such that for any $t$ satisfying $\tdecs \le t \le \tdecs + \Theta(\ln(1/\eta)/\eta)$, we have that $x(t) \in \sI_1 \cap K^{3h/4}$ . Then there exists $\tinv = \tdecs + O(\ln(1/\eta)/\eta))$ such that for any $t$ satisfying $\tinv \le t \le \tinv + \Theta(\ln(1/\eta)/\eta)$, we have that 
\begin{align*}
    x(t) &\in (\cap_{k \in [M]} \sI_k) \cap K^{7h/8}\,. \\
    \| \Phi(x(t)) - \Phi(x(\tdecs))\| &= O(\rho^2 \ln(1/\eta))\,.
\end{align*}
\end{lemma}
\begin{proof}[Proof of~\Cref{lem:nsam_inv}]
The proof is almost identical with~\Cref{lem:nsam_decrease_2} replacing the first two iterative hypothesis to~\Cref{lem:n_sam_prepare} and is omitted here.
\end{proof}

\subsection{Phase II (Proof of Theorem~\ref{thm:nsamphase2})}

\begin{proof}[Proof of~\Cref{thm:nsamphase2}]

Let $\talign = O(\ln(1/\rho)/\eta)$ be the quantity defined in~\Cref{lem:final_alignment}.

We will inductively prove the following induction hypothesis $\mathcal{P}(t)$ holds for $\talign \le t \le T_3/\eta\rho^2 + 1$,
\begin{align*}
x(t) \in K^{h/2},&\ \talign \le \tau \le t  \\
    |\langle x(\tau) - \Phi(x(\tau)), v_1(x(\tau)) \rangle| = \Theta(\eta\rho),&\ \talign \le \tau \le t \\
    \max_{j \in [2:M]}|\langle x(\tau) - \Phi(x(\tau)), v_j(x(\tau)) \rangle| = O(\eta\rho^2),&\ \talign \le \tau \le t \\
    \| \Phi(x(\tau)) - X(\eta \rho^2 \tau) \| = O(\eta\ln(1/\rho)),&\  \talign \le \tau \le t 
\end{align*}

$\mathcal{P}(\talign)$ holds due to~\Cref{lem:final_alignment}. Now suppose $\mathcal{P}(t)$ holds, then $x(t + 1) \in K^h$. By~\Cref{lem:final_alignment} again, $|\langle x(t+1) - \Phi(x(t+1)), v_1(x(t+1)) \rangle| = \Theta(\eta\rho)$ and $ \max_{j \in [2:M]}|\langle x(t+1) - \Phi(x(t+1)), v_j(x(t + 1)) \rangle| = O(\eta\rho^2)$ holds.

Now by~\Cref{lem:direction_n_sam},
\begin{align*}
        \| \Phi(x(\tau+1)) - \Phi(x(\tau))  + \eta \rho^2  \projt[\Phi(x(\tau))] \nabla \lambda_1(t)/2  \| = O(\eta \rho^3 + \eta^2 \rho^2)\,,\  \talign \le \tau \le t.
\end{align*}
By~\Cref{thm:deterministic_ode_approximation}, let $b(x) = - \partial \Phi(x) \nabla \lambda_1(\nabla^2 L(x))/2$, $p = \eta\rho^2$ and $\epsilon = O(\eta + \rho)$, it holds that
\begin{align*}
    &\|\Phi(x(\tau)) - X(\eta\rho^2 \tau) \| \\=& O(\|\Phi(x(\talign)) - \Phi(\xinit) \| + T_3\eta\rho^2 + (\rho + \eta)T_3) \\=& O(\eta\ln(1/\rho)), \talign \le \tau \le t + 1
\end{align*}
This implies $\|x(t+1 ) - X(\eta\rho^2(t+1)) \|_2 \le \| x(t+1) - \Phi(x(t+1)) \|_2 + \|\Phi(x(t+1)) - X(\eta\rho^2(t+1))  \|_2 = \tilde O(\eta\ln(1/\rho)) < h/2$. Hence $x(t+1) \in K^{h/2}$. Combining with $\mathcal{P}(t)$ holds, we have that $\mathcal{P}(t+1)$ holds. The induction is complete.

Now $\mathcal{P}(\lceil T_3/\eta\rho^2 \rceil)$ is equivalent to our theorem.
\end{proof}

\subsubsection{Alignment to Top Eigenvector}

We will continue to use the notations introduced in~\Cref{sub:nsam_phase_1_3}.

We further define
\begin{align*}
     S &= \{ t| \| \tx(t) \| \le \ths\rho + O(\eta\rho^2) \}\,, \\
    T &= \{t | \|\tx(t)\| \le  \frac{1}{2}\left(\frac{\eta \lambda_1^2 }{2 - \eta \lambda_1} + \frac{\eta  \lambda_2^2 }{2 - \eta \lambda_2} \right) \rho \}\,, \\
    U &= \{t | \Omega(\rho^2) \le \|\tx_1(t)\| \le  \frac{1}{2}\left(\frac{\eta \lambda_1^2 }{2 - \eta \lambda_1} + \frac{\eta  \lambda_2^2 }{2 - \eta \lambda_2} \right) \rho \}\,.
\end{align*}

Here the constant in $O$ depends on the constant in $\sI_j$ and will be made clear in~\Cref{lem:s_infinite}.

For $s \in S$, define $\nexts(s)$ as the smallest integer greater than $s$ in $S$.

\begin{lemma}
\label{lem:subtlenorminequal}
Under the condition of~\Cref{thm:nsamphase2}, there exist constants $C_1,C_2 < 1$ independent of $\eta$ and $\rho$, if $\| \tx_1(t) \|_2 \le \frac{1}{2}\left(\frac{\eta \lambda_1(t)^2 }{2 - \eta \lambda_1(t)} + \frac{\eta  \lambda_2(t)^2 }{2 - \eta \lambda_2(t)} \right)\rho$ and $x(t) \in (\cap_{j \in [M]} \sI_j) \cap K^{7h/8}$, then
\begin{align*}
    \|\tx(t) \|_2 \ge C_1\ths\rho \Rightarrow \|\tx(t+1) \|_2 \le C_2\ths\rho
\end{align*}
\end{lemma}

\begin{proof}[Proof of~\Cref{lem:subtlenorminequal}]

By~\Cref{lem:tilde_x_update}, if we write $\xloc(t+1)$ as shorthand of $\tx(t) - \eta A(t)\tx(t) - \eta\rho A^2(t)\frac{\tx(t)}{\|\tx(t) \|}$, then $\| \tx(t + 1) - \xloc(t+1) \| = O(\eta\rho^2)$.

Define $\sIq_j$ as $\{x | R_j(x) \le 0\}$. Then we can find a surrogate $\sx(t)$ such that $\sx(t) \in (\cap_{j \in [M]} \sIq_j) \cap K^{h}$ and $\| \sx(t) - \tx(t) \|_2 = O(\eta\rho^2)$. We will write $\sxq(t+1)$ as shorthand of $\sx(t) - \eta A(t)\sx(t) - \eta\rho A^2(t)\frac{\sx(t)}{\|\sx(t) \|}$. 

Let 
\begin{align*}
h(t) \triangleq
      (2 - t) \sqrt{\frac{1}{2t}\left(\frac{(\lspectraltwo - \Delta)^2}{\lspectraltwo^2} +  1 \right) + (1-\frac{1}{2t}\left(\frac{(\lspectraltwo - \Delta)^2}{\lspectraltwo^2} +  1 \right)) \max\{\frac{\lspectraltwo^2 - \mu^2}{\lspectraltwo^2},\frac{(\lspectraltwo - \Delta)^2}{\lspectraltwo^2} \}}.
\end{align*}

As $h(1) < 1$, we can choose $C_1 < 1$, such that $h(C_1) < 1$. 

We can further choose $C_2 = \max \{(h(C_1) + 1)/2, 1 - \frac{\mu^2}{3\lspectraltwo^2}\} < 1$.

We will discuss by cases

\begin{enumerate}
    \item If 
\begin{align*}
    \| \tx(t) \|_2 \ge \frac{\eta \lambda_1^4}{\lambda_1^2(1-\eta \lambda_D)+(\lambda_1^2 - \lambda_D^2)(1- \eta \lambda_1)} \rho 
\end{align*}

Then \begin{align*}
    \frac{\| \tx(t) \|_2}{\ths \rho} &= \frac{\lambda_1^2(2 - \eta \lambda_1)}{\lambda_1^2(1-\eta \lambda_D)+(\lambda_1^2 - \lambda_D^2)(1- \eta \lambda_1)} \\
    &=  \frac{\lambda_1^2(2 - \eta \lambda_1)}{\lambda_1^2(2-\eta \lambda_1 -\eta\lambda_D) - \lambda_D^2(1- \eta \lambda_1)} \\
    &\ge \frac{1}{1 - \frac{\lambda_D^2}{\lambda_1^2}\frac{1- \eta \lambda_1}{2 - \eta\lambda_1}} \ge  1 + \frac{\lambda_D^2}{\lambda_1^2}\frac{1- \eta \lambda_1}{2 - \eta\lambda_1} \ge 1 + \frac{\mu^2}{3\lspectraltwo^2}\,.
\end{align*}

In such case we have 
\begin{align*}
    \| \frac{\tx(t)}{\|\tx(t) \|} - \frac{\sx(t)}{\|\sx(t) \|} \| = O(\rho)\,.
\end{align*}
Then we have $\|\sxq(t+1) - \xloc(t+1)\| = O(\eta\rho^2)$. By \Cref{lem:norminequal}, 
we have that
\begin{align*}
\|\tx(t+1) \|_2 &\le \|\tx(t + 1) - \xloc(t+1)\|_2 + \|\tx(t+1) - \sxq(t + 1) \| + \|\sxq(t+1) \| \\
     &\le \max(\ths\rho - \eta\rho \frac{\lambda_D^4}{2\lambda_1^2},\eta \rho \lambda_1^2 -(1 - \eta \lambda_1)\| \tx(t) \|_2 ) + O(\eta\rho^2)\\
    &\le \max(1 - \frac{\lambda_D^4 (2 - \eta \lambda_1)}{2 \lambda_1^4}, (2 - \eta \lambda_1) - (1 - \eta \lambda_1)(1 + \frac{\mu^2}{3\lspectraltwo^2})) \ths \rho \\
    &\le (1 - \frac{\mu^2}{3\lspectraltwo^2})\ths \rho \le C_2 \ths\rho \,.
\end{align*}
\item 
If
\begin{align*}
    \|\tx(t) \|_2 &\le \frac{\eta \lambda_1^4}{\lambda_1^2(1-\eta \lambda_D)+(\lambda_1^2 - \lambda_D^2)(1- \eta \lambda_1)}\rho \le \frac{\eta \lambda_1^2}{1 - \eta \lambda_1} \rho.
\end{align*}

Then we have
\begin{align*}
    \frac{|-\eta\rho \lambda_D^2 + (1 - \eta \lambda_D)\| \tx(t) \|_2|}{ \eta\rho \lambda_1^2 - (1 - \eta \lambda_1)\| \tx(t) \|_2} &\le \frac{\lambda_1^2 - \lambda_D^2}{\lambda_1^2}. \\
    \frac{|\eta\rho \lambda_2^2 - (1 - \eta \lambda_2)\| \tx(t) \|_2|}{\eta \rho \lambda_1^2 - (1 - \eta \lambda_1)\| \tx(t) \|_2} &\le \frac{\lambda_2^2}{\lambda_1^2}.
\end{align*}

By~\Cref{lem:subtlecontrolonehop},
\begin{align*}
    &\|\xloc(t+1) \|_2  \\
    \le& (\eta \rho \lambda_1^2 - (1 - \eta\lambda_1) \| \tx(t)\|_2) \sqrt{\frac{\|\tx^2_1(t)\|_2}{\|\tx(t) \|_2^2} + (1-\frac{\|\tx^2_1(t)\|_2}{\|\tx(t) \|_2^2}) \max\{\frac{\lambda_1^2 - \lambda_D^2}{\lambda_1^2},\frac{\lambda_2^2}{\lambda_1^2} \}} \\
    \le& (\eta \rho \lambda_1^2 - (1 - \eta\lambda_1) \| \tx(t)\|_2) \sqrt{\frac{\|\tx^2_1(t)\|_2}{\|\tx(t) \|_2^2} + (1-\frac{\|\tx^2_1(t)\|_2}{\|\tx(t) \|_2^2}) \max\{\frac{\lspectraltwo^2 - \mu^2}{\lspectraltwo^2},\frac{(\lspectraltwo - \Delta)^2}{\lspectraltwo^2} \}}\,.
\end{align*}

As
\begin{align*}
    \| \tx_1(t) \|_2 &\le \frac{1}{2}\left(\frac{\eta \lambda_1^2 }{2 - \eta \lambda_1} + \frac{\eta  \lambda_2^2 }{2 - \eta \lambda_2} \right)\rho \,. 
\end{align*}
For $\| \tx(t) \|_2 \ge  \ths\rho C_1$,
\begin{align*}
    \frac{\|\tx_1(t)\|_2}{\|\tx(t) \|_2} \le \frac{1}{2}\left(\frac{\lambda_2^2 (2 - \eta \lambda_1)}{\lambda_1^2(2 - \eta \lambda_2)} +  1 \right)/C_1 \le \frac{1}{2}\left(\frac{\lambda_2^2}{\lambda_1^2} +  1 \right)/C_1  \le \frac{1}{2C_1}\left(\frac{(\lspectraltwo - \Delta)^2}{\lspectraltwo^2} +  1 \right).
\end{align*}

After plugging in, we have that 
\begin{align*}
   \|\tx(t+1) \|_2 \le \| \xloc(t+1)\|_2 + O(\eta\rho^2) \le h(C_1) \ths \rho + O(\eta\rho^2) \le C_2 \ths \rho.
\end{align*}
\end{enumerate}
This concludes the proof.
\end{proof}

\begin{lemma}
\label{lem:one_step}
Under the condition of~\Cref{thm:nsamphase2}, for any $t \ge 0$ satisfying that (1) $x(t) \in (\cap_{j \in [M]} \sI_j)\cap K^{h}$, (2) $t  \not \in S $, it holds that $t + 1 \in S$.

Moreover, if $| \tilde x_1 (t) | \ge \Omega(\rho^2)$ and $\|\tilde x(t)\|_2 \le  \eta\rho \lambda_1^2  - \Omega(\rho^2)$, then it holds that $\|\tilde x_1(t+1)\| \ge \Omega(\rho^2)$.
\end{lemma}

\begin{proof}[Proof of~\Cref{lem:one_step}]

As $t \not \in S$, it holds that
\begin{align*}
    \|\tx(t)\| \ge \ths \rho + \Theta(\eta\rho^2).
\end{align*}

By~\Cref{lem:tilde_x_update}, if we write $\xloc(t+1)$ as shorthand of $\tx(t) - \eta A(t+1)\tx(t) - \eta\rho A^2(t)\frac{\tx(t)}{\|\tx(t) \|}$, then $\| \tx(t + 1) - \xloc(t+1) \| = O(\eta\rho^2)$.

Define $\sIq_j$ as $\{x | R_j(x) \le 0\}$. Then we can find a surrogate $\sx(t)$ such that $\sx(t) \in (\cap_{j \in [M] } \sIq_j) \cap K^{h}$, and $\| \sx(t) - \tx(t) \|_2 = O(\eta\rho^2)$. We will write $\sxq(t+1)$ as shorthand of $\sx(t) - \eta A(t)\sx(t) - \eta\rho A^2(t)\frac{\sx(t)}{\|\sx(t) \|}$.

As $\|\tx(t)\| = \Omega(\eta\rho)$, we have
\begin{align*}
    \| \frac{\sx(t)}{\|\sx(t) \|} - \frac{\tx(t)}{\|\tx(t) \|} \|_2 = O(\rho)\,.
\end{align*}
Hence we have that $\|\tx(t+1) - \sxq(t+1)\| = \|\tx(t+1) - \xloc(t+1)\| + \|\xloc(t+1) - \sxq(t+1) \| = O(\eta\rho^2)$

Notice we have $\|\sx(t)\|_2 \ge \ths \rho$ for properly chosen function in the definition $S$, hence, by~\Cref{lem:norminequal} \begin{align*}
    \| \sxq(t+1) \|_2 \le \ths \rho.
\end{align*}
This further implies $t + 1 \in S$.

We also have 
\begin{align*}
     |\langle \sxq(t+1), v_1 \rangle|  &= |\langle \sx(t), v_1 \rangle - \eta \lambda_1 \langle \sx(t), v_1 \rangle - \eta \rho \lambda_1^2 \frac{ \langle \sx(t), v_1 \rangle}{\| \sx(t)\|} |\\
     &= |\langle \sx(t), v_1 \rangle| (\eta \lambda_1 + \frac{\eta\rho\lambda_1^2}{\| \sx(t)\|} - 1) \\
\end{align*}
We will discuss by cases. Let $C$ satisfies that $C = \sqrt{\frac{1}{2}\left( \lambda_2^4 + \lambda_1^4\right)}$.
\begin{enumerate}
    \item If $\|\sx(t)\| \le C \eta\rho$, then as we have  $\frac{\lambda_1^2}{C} \ge \frac{\sqrt{2}\lspectraltwo^2}{\sqrt{\lspectraltwo^2 + (\lspectraltwo - \Delta)^2)}}$.
    \begin{align*}
        |\langle \sxq(t+1), v_1 \rangle|  &\ge  |\langle \sx(t), v_1 \rangle| (\frac{\lambda_1^2}{C} - 1) \ge \Omega(\rho^2).
    \end{align*}
    \item  If $\|\sx(t)\| \ge C \eta\rho$, then as $x(t) \in \sI_2$, we have that $|\langle \sx(t), v_1 \rangle|  \ge \Omega(\eta\rho)$. Then as $\|\sx(t)\|\le \|\tx(t) \|_2 + O(\eta\rho^2) \le \lambda_1^2 \eta\rho - \Omega(\rho^2)$, we have that
    \begin{align*}
        |\langle \sxq(t+1), v_1 \rangle|  &\ge  |\langle \sx(t), v_1 \rangle| (\frac{\lambda_1^2\eta\rho}{\lambda_1^2 \eta\rho - \Omega(\rho^2)} - 1) \ge \Omega(\rho^2).
    \end{align*}
\end{enumerate}

By previous approximation results, we have that $\| \tilde x_1(t+1)\| \ge \Omega(\rho^2)$.
\end{proof}

\begin{lemma}
\label{lem:s_infinite}
Under the condition of~\Cref{thm:nsamphase2}, for any $t \ge 0$ satisfying that (1) $x(t) \in (\cap_{j \in [M]} \sI_j)\cap K^{15h/16}$, (2) $t \in S $, it holds that $\nexts(t)$ is well defined and $\nexts(t) \le t + 2$.
\end{lemma}

\begin{proof}[Proof of~\Cref{lem:s_infinite}]

Following similar argument in~\Cref{lem:invariant_set}, we have that $x(t+1) \in (\cap_{j \in [M]} \sI_j) \cap K^h$.

If $t + 1 \not \in S$, then we can apply~\Cref{lem:one_step} to show that $t+2 \in S$.
\end{proof}

\begin{lemma}
\label{lem:local_non_decrease}
Under the condition of~\Cref{thm:nsamphase2}, there exists constant $C > 0$ independent of $\eta$ and $\rho$, assuming that (1) $x(t) \in (\cap_{j \in [M]} \sI_j) \cap K^{7h/8}$, (2) $t \in S $, (3)  $ \Omega(\rho^2) \le \| \tx_1(t)\|$, then
\begin{align*}
 \| \tx_1(\nexts(t)) \| &\ge \| \tx_1(t) \| - O(\eta\rho^2)\,.
\end{align*}
\end{lemma}

\begin{proof}[Proof of~\Cref{lem:local_non_decrease}]

This is by standard approximation as in previous proof and~\Cref{lem:monotoneproj}.
\end{proof}

\begin{lemma}
\label{lem:local_alignment}
Under the condition of~\Cref{thm:nsamphase2}, there exists constant $C > 0$ independent of $\eta$ and $\rho$, assuming that (1) $x(t) \in (\cap_{j \in [M]} \sI_j) \cap K^{7h/8}$, (2) $t \in S $ (3)  $ \Omega(\rho^2) \le \| \tx_1(t)\| \le \frac{1}{2}\left(\frac{\eta \lambda_1^2 }{2 - \eta \lambda_1} + \frac{\eta  \lambda_2^2 }{2 - \eta \lambda_2} \right) \rho $, then
\begin{align*}
 \| \tx_1(\nexts(t)) \| &\ge \min\{(1 + C\eta)\| \tx_1(t) \|,  \frac{1}{2}\left(\frac{\eta \lambda_1^2 }{2 - \eta \lambda_1} + \frac{\eta  \lambda_2^2 }{2 - \eta \lambda_2} \right) \rho \}\,,\\
 \text{ or }  \| \tx_1(\nexts(\nexts(t))) \| &\ge \min\{ (1 + C\eta)\| \tx_1(t) \|, \frac{1}{2}\left(\frac{\eta \lambda_1^2 }{2 - \eta \lambda_1} + \frac{\eta  \lambda_2^2 }{2 - \eta \lambda_2} \right) \rho\}.
\end{align*}

\end{lemma}

\begin{proof}[Proof of~\Cref{lem:local_alignment}]

In this proof, we will sometime drop the $t$ in $\lambda_k(t)$ or $A(t)$. 
Applying~\Cref{lem:s_infinite}, we have $\nexts(t)$ and $\nexts(\nexts(t))$ are well-defined.
We can suppose $\| \tx_1(\nexts(t)) \|_2 \le \frac{1}{2}\left(\frac{\eta \lambda_1^2 }{2 - \eta \lambda_1} + \frac{\eta  \lambda_2^2 }{2 - \eta \lambda_2} \right) $, else the result holds already.

By assumption, we have $\| \tx_1(t) \| \ge \Omega(\rho^2)$.

Using \Cref{lem:tilde_x_update},
\begin{align*}
    \|\tx(t+1) - \tx(t) + \eta A\tx(t) + \eta\rho  A^2 \frac{\tx(t)}{\|\tx(t)\|} \| \le O(\eta\rho^2)\,.
\end{align*}

Denote
\begin{align*}
    \xloc(t + 1) = \tx(t) + \eta A\tx(t) + \eta \rho A^2 \frac{\tx(t)}{\|\tx(t)\|},
\end{align*}
as the one step update of SAM on the quadratic approximation of the general loss.

Now using \Cref{lem:subtlenorminequal} and the induction hypothesis, we have for some $C_1$ and $C_2$ smaller than $1$, $\|\tx(t)\| \ge C_1\ths \rho \Rightarrow \|\xloc(t+1)\| \le C_2 \ths \rho$.

We will discuss by cases,
\begin{itemize}
    \item [1] If $\| \tx(t)\| \le C_1 \ths \rho$
    
    If $\nexts(t) = t + 1$, then 
    \begin{align*}
        \frac{\| \xloc_1(t + 1) \|}{\|\tx_1(t) \|} &= \frac{\eta\rho \lambda_1^2 - (1 -\eta \lambda_1) \| \tx(t)\|} {\|\tx(t) \|} \ge \frac{(2 -C_1 ) - \eta \lambda_1 + C_1\eta\lambda_1}{C_1} \ge \frac{1}{C_1} 
    \end{align*}
    
    As we have $\tx_1(t) = \Omega(\rho^2)$, we have $\xloc_1(t + 1) = \Omega(\rho^2)$, then as $\| \tx_1(t + 1) - \xloc_1(t + 1) \| = O(\eta\rho^2)$, this implies
    \begin{align*}
        \|\tx_1(t+1) \| \ge \|\xloc_1(t + 1) \| - O(\eta\rho^2) \ge \frac{1}{C_1} \| \tx_1(t+1)\| - O(\eta\rho^2)\ge \frac12(\frac{1}{C_1}  + 1) \|\tx_1(t)\|.
    \end{align*}
    
    If $\nexts(t) = t + 2$, define $\xloc(t+2) = \xloc(t+1) - \eta A\xloc(t+1) - \eta \rho A^2 \frac{\xloc(t+1)}{\|\xloc(t+1)\|}$, as $\|\tx_1(t + 1)\| = \Omega(\eta\rho)$, by \Cref{lem:localupdate}, we have $\| \xloc(t+2) - \tx(t+2)\| = O(\eta\rho^2)$. 
    
    \begin{align*}
        \frac{\| \xloc(t+2) \|}{\|\tx_1(t) \|} &= \frac{(\eta\rho \lambda_1^2 - (1-\eta\lambda_1)\|\tx(t)\|)(\eta\rho \lambda_1^2 - (1-\eta\lambda_1)\|\xloc(t+1)\|)}{\|\tx(t)\| \| \xloc(t+1)\|}   \\
    &\ge \frac{(\eta\rho \lambda_1^2 - (1-\eta\lambda_1)\|\tx(t)\|)\left(\eta \rho \lambda_1^2 - (1-\eta\lambda_1) \left(\eta \rho\lambda_1^2 - (1-\eta \lambda_1) \| \tx(t)\|\right) \right)}{\|\tx(t)\| \left(\eta\rho \lambda_1^2 - (1-\eta \lambda_1) \| \tx(t)\|\right) } \\
    &= \frac{\eta\rho \lambda_1^2 - (1-\eta\lambda_1) \left(\eta \rho \lambda_1^2 - (1-\eta \lambda_1) \| \tx(t)\|\right)}{\|\tx(t)\| } \\
    &\ge (1-\eta\lambda_1)^2 + \frac{\eta \lambda_1}{C_1} (2 - \eta\lambda_1) \ge 1 + 4C\eta.
    \end{align*}

Combining with $|\tx_1(t)| \ge \Omega(\rho^2)$, we have that 
\begin{align*}
    \|\tx_1(\nexts(t)) \| \ge (1 + C\eta) \|\tx_1(t) \|
\end{align*}

\item[2] \textit{Case 2} $\| \tx(t)\| >  C_1 \ths \rho$, then $\|\tx(t+1) \| \le C_2\ths \rho$, $\nexts(t) = t+1$

By~\Cref{lem:local_non_decrease}, $\| \tx_1(t+1) \| \ge (1 - C\eta) \| \tx_1(t)\|$.

As $\|\tx(\nexts(t))\| \le C_2 \ths$, similar to the first case, 
\begin{align*}
    \| \tx_1(\nexts (\nexts(t)))\| \ge (1 + 4C\eta) \|\tx_1(\nexts(t)) \| \ge (1 + C\eta) \|\tx_1(t) \|.
\end{align*}
\end{itemize}

In conclusion, if $\|\tx_1(t) \| \le \frac{1}{2}\left(\frac{\eta \lambda_1^2 }{2 - \eta \lambda_1} + \frac{\eta  \lambda_2^2 }{2 - \eta \lambda_2} \right) \rho $, we would have there exists $C > 0$
\begin{align*}
    \| \tx_1(\nexts(t)) \| \ge (1 + C\eta)\| \tx_1(t) \| \text{ or }  \| \tx_1(\nexts(\nexts(t))) \| \ge (1 + C\eta)\| \tx_1(t) \|.
\end{align*}
\end{proof}

\begin{lemma}
\label{lem:final_alignment}
Under the condition of~\Cref{thm:nsamphase2},  there exists constant $T_2 > 0$ independent of $\eta$ and $\rho$, we would have that when $t = \talign= \lceil T_2 \ln(1/\rho)/\eta   \rceil$,
\begin{align*}
|\langle x(t) - \Phi(x(t)), v_1(x(t)) \rangle | &= \Theta(\eta \rho)\,, \\
\max_{j \in [2:M]} |\langle x(t) - \Phi(x(t)), v_j(x(t)) \rangle | &= O(\eta\rho^2)\,.
\end{align*}

Further if $x(t') \in K^h$ holds for $t' = 0,1,..., \tlocal$, then for $t$ satisfying $\talign \le t \le  \tlocal$
\begin{align*}
     |\langle x(t) - \Phi(x(t)), v_1(x(t)) \rangle | &= \Theta(\eta \rho)\,, \\
 \max_{j \in [2:M]} |\langle x(t) - \Phi(x(t)), v_j(x(t)) \rangle | &= O(\eta\rho^2)\,.
\end{align*}

\end{lemma}

\begin{proof}[Proof of~\Cref{lem:final_alignment}]

Let $C$ be the constant defined in~\Cref{lem:local_alignment}.

By~\Cref{lem:one_step}, we can suppose WLOG $\| \tx_1(0)\| \ge \rho^2$ and $0 \in S$. Define 
\begin{align*}
    C_1 &\triangleq \lceil \log_{1 + C\eta}(\ths/\rho) \rceil \,\\
    C_2 &\triangleq C_1 + \lceil \ln_{ \max\{1 - \frac{\mu^2}{2\lspectraltwo^2},  1 -  \frac{\Delta^2}{4\lspectraltwo^2} \}} {\frac{\rho^2}{\lspectraltwo^2}}\rceil = O(\log(1/\rho)/\eta).
\end{align*}
We will choose $\talignmid$ as the minimal $t \in S$, such that $\| \tx_1(t) \| \ge \frac{1}{2}\left(\frac{\eta \lambda_1^2 }{2 - \eta \lambda_1} + \frac{\eta  \lambda_2^2 }{2 - \eta \lambda_2} \right) \rho$.

Then by induction and~\Cref{lem:local_alignment,lem:local_non_decrease}, we easily have that for $t \le \min\{C_2 + 1,\talignmid\}$ and $t \in S$, we have that
\begin{align*}
    x(t) &\in K^{7h/8} \cap \left( \cap_j \sI_j \right)\,, \\
    \| \tx_1(t)\| &\ge \min\{(1 + C\eta)^{t/4} \|\tx_1(0)\|, \frac{1}{2}\left(\frac{\eta \lambda_1^2 }{2 - \eta \lambda_1} + \frac{\eta  \lambda_2^2 }{2 - \eta \lambda_2} \right) \rho\} \\ \mathrm{or} \quad \| \tx_1(\nexts(t))\| &\ge \min\{(1 + C\eta)^{t/4} \|\tx_1(0)\|, \frac{1}{2}\left(\frac{\eta \lambda_1^2 }{2 - \eta \lambda_1} + \frac{\eta  \lambda_2^2 }{2 - \eta \lambda_2} \right) \rho\}\,.
\end{align*}

The detailed induction is analogous to previous inductive argument and is omitted. If $\talignmid \ge C_1$, then we have for the minimal $t \ge C_1$ and $t \in S$
\begin{align*}
    \|\tx_1(t)\| \ge \ths\rho\,.
\end{align*}
This is a contradiction and we have that $\talignmid \le C_1$.

By~\Cref{lem:local_non_decrease}, $\| \tx_1 (\nexts(t)) \| \ge \| \tx_1(t)\| - O(\eta\rho^2)$ for $\| \tx_1(t) \| \ge \frac{1}{2}\left(\frac{\eta \lambda_1^2 }{2 - \eta \lambda_1} + \frac{\eta  \lambda_2^2 }{2 - \eta \lambda_2} \right)\rho$ and $t \in S$ and then by~\Cref{lem:local_alignment},
\begin{align*}
    \| \tx(t) \| \ge \|\tx_1(t)) \|  \ge \frac{1}{4}\left(\frac{\eta \lambda_1^2 }{2 - \eta \lambda_1} + 3\frac{\eta  \lambda_2^2 }{2 - \eta \lambda_2} \right) \rho.
\end{align*}
for $C_2 \ge t \ge \talignmid$.

We will then show that for  $t \ge \talignmid + C_1$ iteration, $\| \Ptd \bar x(t+1) \| \le O(\eta\rho^2)$.

For $C_1 \ge t \ge \talignmid$, 
\begin{align*}
      1 - \eta \lambda_2 - \eta\rho \frac{\lambda_2^2}{\|\tx(t)\|} &\le 1 - \eta \lambda_D - \eta\rho \frac{\lambda_D^2}{\|\tx(t)\|} \le 1 - \frac{\lambda_D^2}{2\lambda_1^2} \le 1 - \frac{\mu^2}{\lspectraltwo^2}.
\end{align*}
Notice that,
\begin{align*}
    &1 - \eta\lambda_2 - \eta\rho \frac{\lambda_2^2}{\|\tx(t) \|} \ge 1 - \eta\lambda_2 - \eta\rho \frac{\lambda_2^2}{\|\tx(t) \|} \\
    \ge& 1 - \eta\lambda_2 - \frac{4\lambda_2^2}{\lambda_1^2 + 3\lambda_2^2}(2 - \eta \lambda_2) 
    \ge -1 + \frac{2(\lambda_1^2 - \lambda_2^2)}{\lambda_1^2 + 3\lambda_2^2} \ge -1 + \frac{\Delta^2}{2\lspectraltwo^2}
\end{align*}
Hence,
\begin{align*}
     \| \Ptd (t) \xloc(t+1) \|_2  &\le \max\{1 - \frac{\mu^2}{\lspectraltwo^2},  1 -  \frac{\Delta^2}{2\lspectraltwo^2}\} \| \Ptd(t) \tx(t) \|_2
\end{align*}

Now by \Cref{lem:eigendiff} and \Cref{thm:Davis-Kahan},
\begin{align*}
    \| \Ptd(t) - \Ptd(t+1) \| \le O(\eta \rho^2) \\
    \| v_1(t) -v_1(t+1) \| \le O( \eta \rho^2) \\
    \| \lambda_1(t) -\lambda_1(t+1)\| \le O( \eta \rho^2) \\
\end{align*}

By \Cref{lem:tilde_x_update}, we have that $\|\xloc(t+1) - \tx(t+1)\| = O(\eta\rho^2)$.

Combining the above, it holds that
\begin{align*}
    \| \Ptd (t + 1) \tx(t+1) \|  &\le \max \{ 1 - \frac{\mu^2}{2\lspectraltwo^2},  1 -  \frac{\Delta^2}{4\lspectraltwo^2} \}\| \Ptd(t) \tx(t) \| + O(\eta\rho^2)
\end{align*}

Hence when $t =
\talign = \talignmid + C_2$,
\begin{align*}
    \|\tx(t)\| \ge \|\tx_1(t)\| &\ge \Omega(\eta\rho) \,, \\
    \|\Ptd(t) \tx(t)\| &\le O(\eta \rho^2)\,.
\end{align*}

By $x(t) \in \sI_1$, we easily have $\|\tx_1(t)\| = O(\eta\rho)$. Hence we conclude that 
\begin{align*}
     \|\tx_1(t)\| &= \Theta(\eta\rho)\,, \\
    \|\Ptd(t) \tx(t)\| &= O(\eta \rho^2)\,.
\end{align*}

The second claim is just another induction similar to previous steps and is omitted as well.
\end{proof}

\subsubsection{Tracking Riemannian Gradient Flow}
We are now ready to show that $\Phi(x(t))$ will track the solution of~\Cref{lambda1ode}. The main principal of this proof has been introduced in Section~\ref{sec:proof_full}.

\begin{lemma}
\label{lem:direction_n_sam}
Under the condition of~\Cref{thm:nsamphase2}, for any $t$ satisfying that 
\begin{align*}
     x(t) &\in K^h,\\
    \|\tx_1(t) \| &= \Theta(\eta\rho), \\
    \| \Ptd(t) \tx(t) \| &= O(\eta\rho^2)\,,
\end{align*}
it holds that
\begin{align*}
    \| \Phi(x(t+1)) - \Phi(x(t))  + \eta \rho^2 \projt[\Phi(x(t))] \nabla \lambda_1(t)/2 \| \le O(\eta \rho^3 + \eta^2 \rho^2)\,.
\end{align*}
\end{lemma}

\begin{proof}[Proof of~\Cref{lem:direction_n_sam}]

To begin with, we can approximate $\Phi(x(t+1)) - \Phi(x(t))$ by its first order Taylor Expansion, by~\Cref{lem:boundedphi},
\begin{align*}
    \|\Phi(x(t+1)) - \Phi(x(t)) - \partial \Phi(x(t))(x(t+1)-x(t)) \| &= O(\|x(t+1) - x(t)\|^2) = O(\eta^2\rho^2)\,.
\end{align*}

Then by plugging in the update rule and another Taylor Expansion,
\begin{align*}
    \Huge\| \partial \Phi(x(t))(x(t+1)-x(t)) - &\eta \rho \partial \Phi(x(t)) \dlt{x} \ndl{x} \\
    - &\eta \rho^2 \partial \Phi(x(t))
   \partial {\dlt {x}}[\ndl x, \ndl x]/2\Huge\|_2 = O(\eta\rho^3).
\end{align*}

Using \Cref{lem:relatelphi}, we have
\begin{align*}
    \| \eta \rho \partial \Phi(x(t)) \dlt{x} \ndl{x} \|  = \eta\rho \| \dl x\| \| \partial^2 \Phi(x(t)) \left[\ndl x, \ndl x \right] \| = O(\eta\rho \| \dl x\|)\,.
\end{align*}

Putting together, we have that
\begin{align*}
    &\| \Phi(x(t+1)) - \Phi(x(t)) - \eta \rho^2\partial \Phi(x(t))
   \partial {\dlt {\Phi(x(t))}}[\ndl x, \ndl x]/2 \| \\
   \le& O(\eta^2\rho^2 + \eta\rho^3) + O(\eta\rho \| \dl x\|)\,.
\end{align*}

As we have $\|\tx(t)\| = \Theta (\eta\rho)$, hence by~\Cref{lem:revertbounddl,lem:nsam_tilde_x_bound},
\begin{align*}
    &\| \Phi(x(t+1)) - \Phi(x(t)) - \eta \rho^2\partial \Phi(x(t))
   \partial {\dlt {\Phi(x(t))}}[\ndl{x(t)}, \ndl{x(t)}]/2 \| \\
   \le& O(\eta \rho^3 + \eta^2 \rho^2)
\end{align*}

Finally, we have that
\begin{align*}
    \|& \eta \rho^2\partial \Phi(x(t))
   \partial {\dlt {\Phi(x(t))}}[\ndl{x(t)}, \ndl{x(t)}]/2\\
   &- \eta \rho^2\partial \Phi(x(t))
   \partial {\dlt {\Phi(x(t))}}[v_1(t), v_1(t)]/2 \|  \le O(\eta \rho^3)
\end{align*}
as the angle between $\ndl{x}$ and $v_1(t)$ is $O(\rho)$. 

By \Cref{lem:relatelphi}, it holds that
\begin{align*}
    \partial \Phi(x(t))
   \partial {\dlt {\Phi(x(t))}}[v_1(t), v_1(t)] =  \proj \nabla(\lambda_1(t))
\end{align*}

Putting together we have that,
\begin{align*}
    \| \Phi(x(t+1)) - \Phi(x(t))  + \eta \rho^2 \proj \nabla \lambda_1(t)/2 \| \le O(\eta \rho^3 + \eta^2 \rho^2).
\end{align*}
It completes the proof.
\end{proof}

\subsection{Proof of Theorem~\ref{thm:nsam}}
\label{sec:nsam_proof}
\begin{proof}[Proof of Theorem~\ref{thm:nsam}]
    By~\Cref{thm:nsamphase1}, there exists constant $T_1$ independent of $\eta,\rho$, such that for any $T_1' > T_1$ independent of $\eta,\rho$, it holds that 
    \begin{align*}
    &\max\limits_{ T_1\ln(1/ \eta \rho)\le \eta t \le T_1'\ln(1/ \eta \rho) } \max_{j \in [M]} R_j(x(t)) = O(\eta\rho^2). \\
    &\max\limits_{ T_1\ln(1/ \eta \rho)\le \eta t \le T_1'\ln(1/ \eta \rho) } \|\Phi(x(t)) - \Phi(\xinit) \| = O((\eta + \rho)\ln (1/\eta\rho)).  
\end{align*} 

By~\Cref{assum:reg}, there exists step $T_1\ln(1/ \eta \rho) \le \eta \tphaseone  \le T_1'\ln(1/ \eta \rho)$, such that
\begin{align*}
    \max_{j \in [M]} R_j(x(\tphaseone)) = O(\eta\rho^2), \\
    \|\Phi(x(\tphaseone)) - \Phi(\xinit) \| = O((\eta + \rho)\ln (1/\eta\rho)), \\
    |\langle x(\tphaseone) - \Phi(x(\tphaseone)) , v_1(x(\tphaseone))\rangle | \ge \Omega(\rho^2). \\
    \| x(\tphaseone) \|_2 \le  \lambda_1(\tphaseone) \eta\rho - \Omega(\rho^2).
\end{align*}

Hence by~\Cref{thm:nsamphase2}, if we consider a translated process with $x'(t) = x(t + \tphaseone)$, we would have for any $T_3$ such that the solution $X$ of~\Cref{lambda1ode} is well defined, we have that for $t = \lceil \frac{T_3}{\eta\rho^2} \rceil$
\begin{align*}
    \|\Phi(x'(t)) - X(\eta\rho^2 t) \|_2 &= O(\eta  \ln(1/\rho))\,.
\end{align*}
This implies for $t$ satisfying $X(\eta\rho^2 (t - \tphaseone))$ is well-defined,
\begin{align*}
    \|\Phi(x(t)) - X(\eta\rho^2 (t - \tphaseone)) \|_2 &= O(\eta  \ln(1/\rho)).
\end{align*}
Finally, as 
\begin{align*}
    \| X(\eta\rho^2 (t - \tphaseone)) - X(\eta\rho^2 t) \|_2 &
    &= O(\eta\rho^2 \tphaseone) = O(\rho \ln(1/\eta\rho)) = O(\eta \ln(1/\rho)).
\end{align*}

We have that
\begin{align*}
    \|\Phi(x(t)) - X(\eta\rho^2 t)\|_2 &= O(\eta  \ln(1/\rho)).
\end{align*}

The alignment result is a direct consequence of~\Cref{thm:nsamphase2}.

\end{proof}

\subsection{Proofs of Corollaries~\ref{corr:loss_ode_n} and~\ref{corr:loss_opt_n}}
\label{app:nsamcorr}
\begin{proof}[Proof of \Cref{corr:loss_ode_n}]
We will do a Taylor expansion on $\maxloss$. By Theorem~\ref{thm:nsamphase1} and ~\ref{thm:nsamphase2}, we have $\|x(\lceil T_3/\eta\rho^2 \rceil)) - X(T_3) \| = \tilde O(\eta + \rho)$ and $\|x(\lceil T_3/\eta\rho^2 \rceil)) - \Phi(x(\lceil T_3/\eta\rho^2 \rceil))) \|_2 = O(\eta\rho)$. For convenience, we denote $x(\lceil T_3/\eta\rho^2 \rceil)$ by $x$. 
\begin{align*}
    \maxsharpness(x) &= \max_{\|v\|_2\le 1}  \rho v^T\nabla L(x) + \rho^2 v^T\nabla^2 L(x)v/2 + O(\rho^3)  
\end{align*}
Since $ \max_{\norm{v}_2\le 1}\| v^T\nabla L(x)\|_2 = O(\|x - \Phi(x) \|_2)  =O(\eta\rho)$, it holds that
\begin{align*}
    \maxsharpness(x) &= \rho^2 \max_{\|v\|_2\le 1}   v^T\nabla^2 L(x)v/2 + O(\eta^2 \rho^2 +\rho^3)\\
    &= \rho^2 \lambda_1(\nabla^2 L(x)) + O(\eta^2 \rho^2 +\rho^3)\\
    &= \rho^2 \lambda_1(\nabla^2 L(X(T_3))) + \tilde O(\eta\rho^2),
\end{align*}
which completes the proof.
\end{proof}

\begin{proof}[Proof of \Cref{corr:loss_opt_n}]

We choose $T$ such that $X(T_\eps)$ is sufficiently close to $X(\infty)$, such that $\lambda_1(X(T_\eps)) \le \lambda_1(X(\infty)) + \epsilon/2$. By \Cref{corr:loss_ode_n} (let $T_3=T_\eps$),  we have that for all $\rho,\eta$ such that $\eta \ln(1/\rho)$ and $\rho/\eta$ are sufficiently small,  $\| \maxsharpness(x(\lceil T_\eps / (\eta\rho^2)\rceil)) - \rho^2 \lambda_1(X(T_\eps))/2 \| \le \tilde o(1)$. This further implies $\| \maxsharpness(x(\lceil T_\eps / (\eta\rho^2)\rceil)) - \rho^2 \lambda_1(X(\infty))/2  \| \le \epsilon\rho^2 + o(1)$. We also have $ L(x(\lceil T_\eps / (\eta\rho^2)\rceil)) - \inf_{x\in U'} L(x) = o(1)$. Then we can leverage \Cref{thm:generalreg} and Theorem~\ref{thm:maxsharp} to get the desired bound.
\end{proof}

\section{Analysis for 1-SAM (Proof of Theorem~\ref{thm:1sam})}
\label{app:1sam}
The goal of this section is to prove the following theorem.

\thmonesam*

As mentioned in our proof setups in \Cref{appsec:proof_setup}, we will prove \Cref{thm:1sam} under a more general (and weaker) condition, namely \Cref{cond:rank1} and \Cref{assump:smoothness}. The only usage of \Cref{setting:1sam} in the proof is \Cref{thm:derive_manifold_assumption,thm:rank_1_manifold}, which are restated below.

\derivemanifoldassumption*
\derivestochasticassumption*

\generalcondition*

 Analogous to the full-batch setting, we will split the trajectory into two phases.

\begin{theorem}[Phase I]
\label{thm:1samphase1}
Let $\{x(t)\}$ be the iterates defined by SAM (\Cref{eq:1sam}) and $x(0) = \xinit \in U$, then under \Cref{assump:smoothness} and~\ref{cond:rank1}, for almost every $\xinit$, there exists a constant $T_1$, it holds for sufficiently small $(\eta + \rho)\ln1/\eta\rho$, we have with probability $1 - O(\rho)$, there exists $t \le T_1\ln(1/\eta \rho) /\eta$, such that $\|x(t) - \Phi(x(t))\|_2 = O(\eta\rho)$ and $\|\Phi(\xinit) - \Phi(x(t))\|_2 =  \tilde O(\eta^{1/2} + \rho)$.
\end{theorem}

Theorem~\ref{thm:1samphase1} shows that SAM will converges to an $\tilde O(\eta\rho)$ neighborhood of the manifold without getting far away from $\Phi(x(0))$, where we can perform a local analysis on the trajectory of $\Phi(x(t))$.

Under Assumptions~\ref{assump:smoothness} and~\ref{cond:rank1}, we have $\mathrm{Tr}(\nabla^2 L_k(x)) = \lambda_1(\nabla^2 L_k(x))$ is differentiable for $x \in \Gamma_i$. Hence $\mathrm{Tr}(\nabla^2 L(x)) =  \sum_{k=1}^M  \mathrm{Tr}(\nabla^2 L_k(x))$ is also differentiable and we have~\eqref{eq:traceode} is well defined for some finite time $T_2$.
\begin{theorem}[Phase II]
\label{thm:1samphase2}
Let $\{x(t)\}$ be the iterates defined by SAM (\Cref{eq:1sam}) under Assumptions~\ref{assump:smoothness} and~\ref{cond:rank1}, assuming (1) $\|x(0) - \Phi(x(0))\|_2 = O(\eta\rho)$ and (2) $\|\Phi(\xinit) - \Phi(x(0))\|_2 = \tilde O(\eta^{1/2} + \rho)$, then for almost every $x(0)$, for any $T_2 > 0$ till which solution of~\eqref{eq:traceode} $X$ exists, for sufficiently small $(\eta + \rho)\ln 1/(\eta\rho)$,  we have with probability $1 - O(\eta\rho)$, for all $\eta\rho^2t < T_2$, $ \|\Phi(x(t)) - X(\eta\rho^2 t) \|_2 =  \tilde O(\eta^{1/2} + \rho)$ and $\|x(t) - \Phi(x(t))\|_2 =  O(\eta\rho)$.
\end{theorem}

Combining \Cref{thm:rank_1_manifold,thm:1samphase1,thm:1samphase2}, the proof of \Cref{thm:1sam} is clear and we deferred it to~\Cref{sec:1sam_proof}.

Now we recall our notations for stochastic setting with batch size one.
\paragraph{Notations for Stochastic Setting:}
Since $L_k$ is rank-$1$ on $\Gamma_k$ for each $k\in[M]$, we can write it as $L_k(x)= \Lambda_k(x)w_k(x)w^\top_k(x)$ for any $x\in\Gamma_k$, where $w_k$ is a continuous function on $\Gamma_k$ with pointwise unit norm.
Given the loss function $L_k$, its gradient flow is denoted  by mapping $\phi_k: \R^D \times [0, \infty) \to \R^D$. Here, $\phi_k(x, \tau)$ denotes the iterate at time $\tau$ of a gradient flow starting at $x$ and is defined as the unique solution of  $\phi_k(x,\tau) = x - \int_0^{\tau} \nabla L_k(\phi_k(x,t))dt$, $\forall x\in\R^D$.
We further define the limiting map $\Phi_k$ as $\Phi_k(x) = \lim_{\tau \to \infty} \phi_k(x,\tau)$, that is, $\Phi_k(x)$ denotes the convergent point of the gradient flow starting from $x$.
Similar to \Cref{defi:attraction_set}, we define $U_k = \{x \in \R^D | \Phi(x) \text{ exists and } \Phi_k(x) \in  \Gamma_k \}$ be the attraction set of $\Gamma_i$. We have that each $U_k$ is open and $\Phi_k$ is $\overline{\mathcal{C}}^2$ on $U_k$ by Lemma B.15 in \cite{arora2022understanding}.

In this section we will define $K$ as $\{ X(t)\mid t\in[0,T_3]\}$ where $X$ is the solution of~\eqref{eq:traceode}. We will denote $h(K)$ in~\Cref{lem:smallerzone} by $h$.
Using \Cref{thm:general_well_definednesa}, we will assume the update is always well defined.

\subsection{Phase I (Proof of Theorem~\ref{thm:1samphase1})}
\label{sec:converge}

\begin{proof}[Proof of~\Cref{thm:1samphase1}]
The proof consists of two steps.
\begin{enumerate}
\item  \textit{Tracking Gradient Flow.} By~\Cref{lem:1sam_gf}, with probability $1 - \rho^2$, there exists step $\tgf=O(1/\eta)$ such that
\begin{align*}
    \|\xgf - \Phi(\xgf)\|_2 &\le h/4.\\
    \|\Phi(\xgf) - \Phi(\xinit) \|_2 &= \tilde O(\eta^{1/2} + \rho).
\end{align*}
\item \textit{Decreasing Loss.} By~\Cref{lem:1sam_decrease}, with probability $1 - O(\rho)$, there exists step $\tdec = \tgf + O(\ln(1/\rho)/\eta) = O(\ln(1/\rho)/\eta)$ such that
\begin{align*}
    \| \nabla L(\xdec) \|_2 &= O(\rho). \\
    \|\Phi(\xdec) - \Phi(\xinit) \|_2 &\le \|\Phi(\xdec) - \Phi(\xgf) \|_2 + \| \Phi(\xgf) - \Phi(\xinit)\|_2 
    \\
    &= \tilde O(\eta^{1/2} + \rho).
\end{align*}
Then by~\Cref{lem:1sam_decrease_2}, with probability $1 - O(\rho)$, there exists step $\tdecs = \tdec+O(\ln(1/\eta \rho)/\eta) = O(\ln(1/\eta \rho)/\eta)$, it holds that
\begin{align*}
    \| \xdecs - \Phi(\xdecs)\|_2 &= O(\eta\rho). \\
        \|\Phi(\xdecs) - \Phi(\xinit) \|_2 &\le \|\Phi(\xdecs) - \Phi(\xdec) \|_2 + \| \Phi(\xdec) - \Phi(\xinit)\|_2 
    \\
    &= \tilde O(\eta^{1/2} + \rho).
\end{align*}

Concluding, let $T_1$ be the constant satisfying $\tdecs \le T_1\ln(1/\eta \rho)/\eta$, then we have for $t = \tdecs \le T_1\ln(1/\eta \rho)/\eta$ such that 
\begin{align*}
    \| x(t) - \Phi(x(t)) \|_2 &= O(\eta\rho). \\
    \|\Phi(x(t)) - \Phi(\xinit) \|_2 &= \tilde O(\eta^{1/2} + \rho).
\end{align*}

\end{enumerate}

\end{proof}

\subsubsection{Tracking Gradient Flow}

\Cref{lem:1sam_gf} shows that the iterates $x(t)$ tracks gradient flow to an $O(1)$ neighbor of $\Gamma$.

\begin{lemma}
\label{lem:1sam_gf}
Under condition of~\Cref{thm:1samphase1}, with probability $1 - O(\rho^2)$, there exists $\tgf = O(1/\eta)$, such that the iterate $x(\tgf)$ is $O(1)$ close to the manifold $\Gamma$ and $\Phi(x(\tgf))$ is $\tilde O(\eta^{1/2} + \rho)$ is close to $\Phi(\xinit)$. Quantitatively, 
\begin{align*}
    L(x(\tgf)) &\le \frac{\mu h^2}{32}\\
    \|x(\tgf) - \Phi(x(\tgf)) \| &\le h/4\,,\\
    \|\Phi(x(\tgf)) - \Phi(\xinit) \| &=  \tilde O(\eta^{1/2} + \rho)\,.
\end{align*}
\end{lemma}

\begin{proof}[Proof of~\Cref{lem:1sam_gf}]
Choose $C =  \frac{1}{4} \sqrt{\frac{\mu}{\lspectraltwo}}$.

There exists $T > 0$, such that
\begin{align*}
    \| \phi(\xinit, T) - \Phi(\xinit) \|_2 &\le Ch/2\,.
\end{align*}

Consider 
\begin{align*}
    x(t+1) &= x(t) - \eta \nabla L_k(x(t) + \rho \ndli{x(t)})\\ 
    &= x(t) - \eta \nabla L_k(x(t)) + O(\eta\rho)\,.
\end{align*}

By~\Cref{thm:ode_approximation}, let $b(x) = -\nabla L(x)$,$p = \eta$ and $\epsilon = O(\rho)$, for sufficiently small $\eta$ and $\rho$, the iterates $x(t)$ tracks gradient flow $\phi(\xinit,T)$ in $O(1/\eta)$ steps in expectation, Quantitatively, with probability $1 - \rho^2$, for $\tgf = \lceil \frac{T_0}{\eta} \rceil$, we have that
\begin{align*}
    \| x(\tgf) - \phi(\xinit,T_0) \|_2  & = \tilde O(\sqrt{p} + \epsilon) \le \tilde O(\eta^{1/2} + \rho) \,.
\end{align*}

This implies $x(\tgf) \in K^h$, hence by Taylor Expansion on $\Phi$,
\begin{align*}
    \| \Phi(x(\tgf)) - \Phi(\xinit)\|_2&=\| \Phi(x(\tgf)) - \Phi(\phi(\xinit,T))\|_2 \\
    &\le O(\| x(\tgf) - \phi(\xinit,T) \|_2) \\
    &\le  \tilde O(\eta^{1/2} + \rho)\,.
\end{align*}
This implies
\begin{align*}
    \|  x(\tgf) - \Phi(x(\tgf))\|_2 \le& \|x(\tgf) - \phi(\xinit,T_0) \|_2 + \| \phi(\xinit,T_0) - \Phi(\xinit) \|_2\\&+ \| \Phi(\xinit) - \Phi(x(\tgf)) \|_2 
    \\\le& Ch/2 +  \tilde O(\eta^{1/2} + \rho) \le Ch \le h/4\,.
\end{align*}

By Taylor Expansion,
\begin{align*}
    L(\xgf) \le \lspectraltwo \|  x(\tgf) - \Phi(x(\tgf))\|_2^2/2 \le \frac{\mu h^2}{32}\,.
\end{align*}
\end{proof}

\subsubsection{Decreasing Loss}

\begin{lemma}
\label{lem:1sam_maintain_help}
Under condition of~\Cref{thm:1samphase1},
assuming $x(t_0) \in K^{h/4}$ and for any $t$ satisfying $t_0 \le t \le t_0 + O(\ln(1/\eta\rho)/\eta)$, $\max\limits_{t_0 \le \tau \le  t_0 + O(\ln(1/\eta\rho)/\eta)}L(x(\tau)) \le \frac{\mu h^2}{16}$, it holds that 
\begin{align*}
    x(\tau) \in K^h, \forall t_0 \le \tau \le t.
\end{align*}

Moreover, we have that
\begin{align*}
    \| \Phi(x(t)) - \Phi(x(t_0)) \| = O((\eta + \rho) \ln(1/\eta\rho)).
\end{align*}

\end{lemma}

\begin{proof}[Proof of~\Cref{lem:1sam_maintain_help}]

We will prove by induction. For $\tau = t_0$, the result holds trivially. Suppose the result holds for $t - 1$, then for any $\tau$ satisfying $t_0 \le \tau \le t - 1$, by~\Cref{lem:stochasticboundedphi,lem:bounddl},
\begin{align*}
    \| \Phi(x(\tau+1)) - \Phi(x(\tau)) \| &\le \phispectraltwo \eta \rho \|\dl {x(\tau)} \|_2 + \lspectralthree \eta \rho^2  + \phispectraltwo \eta^2 \|\dl {x(\tau)}\|_2^2 + \phispectraltwo  \lspectraltwo^2 \eta^2 \rho^2 \\
    &= O(\eta^2 + \eta\rho) \,.
\end{align*}

Also by~\Cref{lem:bounddl}, $\|x(t) - \Phi(x(t)) \|_2 \le h/2\sqrt{2}$, this implies,
\begin{align*}
    \dist(K,x(t)) \le& \dist(K,x(t_0)) + \|x(t_0) - \Phi(x(t_0)) \|_2
    \\&+ \|\Phi(x(t_0)) - \Phi(x(t)) \| + \|\Phi(x(t)) - x(t) \| \\
    \le& 0.99h + O(\eta^2 (t - \tgf)) = 0.99h + O(\eta\ln(1/\eta\rho)) \le h\,.
\end{align*}
\end{proof}

\begin{lemma}
\label{lem:1samdecrease_onestep}

Under condition of~\Cref{thm:1samphase1}, if $x(\tau) \in K^h$, then we have that \begin{align*}
    \E[L(x(\tau
    +1)) | x(\tau)] \le L(x(\tau)) - \frac{\eta\mu}{2} L(x(\tau)) \,.
\end{align*}

Moreover it holds that,
\begin{align*}
    \E[\ln L(x(\tau
    +1)) | x(\tau)] \le \ln \E[ L(x(\tau
    +1)) | x(\tau)] \le \ln L(x(\tau)) - \frac{\eta\mu}{2}\,.
\end{align*}

\end{lemma}
\begin{proof}[Proof of~\Cref{lem:1samdecrease_onestep}]

By~\Cref{lem:stochasticboundedphi} and Taylor Expansion,
\begin{align*}
    &\E[L(x(\tau
    +1)) | x(\tau)] \\
    =& \E\left[L\left(x(\tau) -   \eta \nabla L_k[x(\tau) + \rho \ndli{x(\tau)}]\right)|x(\tau)\right]\\
    \le& \E\left[ L(x(\tau)) -   \eta \left \langle \dl{x(\tau)}, \nabla L_k\Bigl(x(\tau) + \rho \ndli{x(\tau)}\Bigr)  \right \rangle \right ] \\
    &+ \E\left[\frac{\lspectraltwo  \eta^2}{2} \| \nabla L_k[x(\tau) + \rho \ndli{x(\tau)}] \|_2^2 \right] \\
    \le& L(x(\tau)) -  \eta \| \dl{x(\tau)}\|_2^2 +   \eta\rho\lspectraltwo \| \dl{x(\tau)}\|_2 + \lspectraltwo \eta^2 \E[\|\nabla L_k(x(\tau)) \|_2^2] + \lspectraltwo^3 \eta^2\rho^2  \\
    \le& L(x(\tau)) - \frac{\eta}{2} \|\dl{x(\tau)} \|_2^2 \\
    \le& L(x(\tau)) - \frac{\eta\mu}{2} L(x(\tau))\,.
\end{align*}
\end{proof}

\begin{lemma}
\label{lem:1sam_maintain}
Under condition of~\Cref{thm:1samphase1},
assuming $x(t_0) \in K^{h/4}$ and $L(x(t_0)) \le \frac{\mu h^2}{32}$, then with probability $1 - O(\rho)$, for any $t$ satisfying $t_0 \le t \le t_0 + O(\ln(1/\eta\rho)/\eta)$, it holds that $x(t) \in K^h$.
Moreover, we have that
\begin{align*}
    \| \Phi(x(t)) - \Phi(x(t_0)) \| = O((\eta + \rho) \ln(1/\eta\rho)).
\end{align*}
\end{lemma}
\begin{proof}[Proof of~\Cref{lem:1sam_maintain}]

By Uniform Bound and~\Cref{lem:1sam_maintain_help},
\begin{align*}
    &\prob(\exists t_0 \le t \le t_0 + O(\ln(1/\eta\rho)/\eta),  L(x(t)) \ge \frac{\mu h^2}{16} ) \\
    \le &\sum_{t = t_0}^{t_0 + O(\ln(1/\eta\rho)/\eta)} \prob(L(x(t)) \ge \frac{\mu h^2}{16} \quad \mathrm{ and } \quad L(x(\tau)) \le \frac{\mu h^2}{16}, \forall t_0 \le \tau \le t - 1) \\
    \le& \sum_{t = t_0}^{t_0 + O(\ln(1/\eta\rho)/\eta)} \prob(L(x(t)) \ge \frac{\mu h^2}{16} \quad \mathrm{ and } \quad x(\tau) \in K^h, \forall t_0 \le \tau \le t - 1)
\end{align*}

Consider each term, and applying uniform bound again,
\begin{align*}
    P(&L(x(t)) \ge \frac{\mu h^2}{16} \quad \mathrm{ and } \quad x(\tau) \in K^h, \forall t_0 \le \tau \le t - 1) 
    \\ \le \sum_{\tau = t_0}^t P(&L(x(t)) \ge \frac{\mu h^2}{16} \quad \mathrm{ and } \quad  L(x(\tau)) \le \frac{\mu h^2}{32} \\
    & \quad \mathrm{ and } \quad  \forall t - 1 \ge \tau' \ge \tau + 1, \frac{\mu h^2}{16} > L(x(\tau')) > \frac{\mu h^2}{32} \\
    &\quad \mathrm{ and } \quad \forall t - 1 \ge \tau'' \ge \tau, x(\tau'') \in K^h )\,.
\end{align*}

Then if we consider each term, we have that it is bounded by
\begin{align*}
    P(&L(x(t)) \ge \frac{\mu h^2}{16}  \quad \mathrm{ and } \quad  \forall t -1  \ge \tau' \ge \tau + 1, L(x(\tau')) > \frac{\mu h^2}{32}\\& \quad \mathrm{ and } \quad \forall t - 1 \ge \tau'' \ge \tau, x(\tau'') \in K^h 
     \mid L(x(\tau)) \le \frac{\mu h^2}{32} )\,.
\end{align*}

Define a coupled process $\tilde L(\tau + 1) = \ln L(x(\tau + 1))$ and 
\begin{align*}
    \tilde L(\tau') = \begin{cases}
     \ln L(x(\tau')), &\text{ if }\tilde L(\tau' - 1) = \ln L(x(\tau' - 1)) \ge \ln (\frac{\mu h^2}{32}), \\
    \tilde L(\tau' - 1) - \eta \mu/2, &\text{ if otherwise.}
    \end{cases}
\end{align*}

Then clearly 
\begin{align*}
    P(&L(x(t)) \ge \frac{\mu h^2}{16}  \quad \mathrm{ and } \quad  \forall t \ge \tau' \ge \tau + 1, L(x(\tau')) > \frac{\mu h^2}{32}\\& \quad \mathrm{ and } \quad \forall t \ge \tau'' \ge \tau, x(\tau'') \in K^h 
     \mid L(x(\tau)) \le \frac{\mu h^2}{32} ) \\
     \le P(&\tilde L(t) \ge \ln (\frac{\mu h^2}{16})) \,.
\end{align*}

Consider a fixed $\tau'$ satisfying $\tau + 1 \le \tau' \le t$. By~\Cref{lem:1samdecrease_onestep}, we have that
\begin{align*}
    \tilde L(x(\tau' + 1)) - \tilde L(x(\tau')) \le -\eta\mu/2.
\end{align*}
Hence $\tilde L(t) + \eta \mu t /2 $ is a super martingale.

Further it holds that if $L(x(\tau' - 1)) \ge (\frac{\mu h^2}{32})$, then 
\begin{align*}
     L(x(\tau' - 1)) - L(x(\tau')) &= O(\|x(\tau' - 1) - x(\tau') \|) = O(\eta) \,.
\end{align*}
Using the smoothness at $\log(x)$ at $\frac{\mu h^2}{32}$ which is a positive constant,
\begin{align*}
    \| \tilde L(\tau' + 1) - \tilde L(\tau')  \| \le O(\eta) \le C\eta\,.
\end{align*}
Here $C$ is a constant independent of $\eta$. This implies $\tilde L(x(\tau + 1)) \le \frac{\mu h^2}{16\sqrt{2}}$

Now by Azuma-Hoeffding bound~(\Cref{lem:azuma}), we have that 
\begin{align*}
    P(\tilde L(t) - \tilde L(\tau + 1) + (t - \tau - 1)\eta\mu/2 > a) \le 2\exp(-\frac{a^2}{8(t - \tau - 1)(C + \mu)^2\eta^2}).
\end{align*}

With $a = \ln (\frac{\mu h^2}{16\tilde L(\tau + 1)}) + (t - \tau - 1)\eta \mu/2 \ge (\ln 2  + (t - \tau - 1)\eta \mu)/2$, we have that 
\begin{align*}
    P(\tilde L(t)  > \ln (\frac{\mu h^2}{16})) &\le 2\exp(-\frac{(\ln 2  + (t - \tau - 1)\eta \mu)^2}{32(C + \mu)^2\eta^2}) \\
    &\le 2\exp(-\frac{\ln 2(t - \tau - 1) \mu}{8(C + \mu)^2\eta})
\end{align*}

Hence we have 
\begin{align*}
    &\prob(\exists t_0 \le t \le t_0 + O(\ln(1/\eta\rho)/\eta),  L(x(t)) \ge \frac{\mu h^2}{16} ) \\\le& O(2\exp(-\frac{\ln 2(t - \tau - 1) \mu}{8(C + \mu)^2\eta})\ln^2(1/\eta \rho)/ \eta^2) \le \rho.
\end{align*}

Hence with probability $1 - \rho$, $L(x(t)) \le \frac{\mu h^2}{16}, \forall t_0 \le t \le t_0 + O(\ln(1/\eta\rho)/\eta)$, combining with~\Cref{lem:1sam_maintain_help}, we have completed our proof.
\end{proof}

\begin{lemma}
\label{lem:1sam_decrease}
Under condition of~\Cref{thm:1samphase1}, assuming there exists $\tgf$ such that $L(x(\tgf)) \le \frac{\mu h^2}{32}$ and $x(\tgf) \in K^{h/4}$, then with probability $1 - O(\rho)$, there exists $\tdec = \tgf + O(\ln(1/\rho)/\eta)$, such that $x(\tdec)$ is in $O(\rho)$ neighbor of $\Gamma$, quantitatively, we have that
\begin{align*}
    \| \nabla L(x(\tdec)) \|_2 &\le 4\zeta \rho\,. 
\end{align*}
Moreover the movement of the projection of $\Phi(x(\cdot))$ on the manifold is bounded,
\begin{align*}
    \| \Phi(x(\tgf)) - \Phi(x(\tdec)) \|_2 &= O((\eta + \rho) \ln(1/\rho))\,.
\end{align*}
\end{lemma}
\begin{proof}[Proof of~\Cref{lem:1sam_decrease}]
For simplicity of writing, define $T_1 \triangleq \lceil \frac{2\ln{\frac{h^2}{256\rho^3\mu}}}{ \eta\mu} \rceil = O(\ln(1/\rho)/\eta)$.

By~\Cref{lem:1sam_maintain}, we may assume $x(t) \in K^h$ for $\tgf  \le t \le  T_1 + \tgf$.

Define indicator function as 
\begin{align*}
    \indicatordec{t} = \one[\dl{x(\tau)} \ge 4\lspectraltwo\rho, \forall t \ge \tau \ge \tgf]\,.
\end{align*}

By~\Cref{lem:1samdecrease_onestep}, we have that,
\begin{align*}
    \E[L(x(t+1)) \indicatordec{t + 1}]\le \E[L(x(t+1)) \indicatordec{t}] \le (1-\frac{  \eta\mu}{2}) \E[L(x(t)) \indicatordec{t}].
\end{align*}

We can then conclude that with $T_2 =T_1 + \tgf$, using \Cref{lem:revertbounddl},
\begin{align*}
    8 \mu \rho^2 \E\indicatordec{T_2 + 1}
    \le \E[L(x(T_2+1)) \indicatordec{T_2 + 1}] 
    \le (1-\frac{  \eta\mu}{2})^{T_1} L(x(\tgf)) 
    \le 8 \mu \rho^3.
\end{align*}

We have 
\begin{align*}
    \E\indicatordec{T_2 + 1} \le  \rho.
\end{align*}
This implies $\indicatordec{T_2 + 1}  = 0$ with probability $1 - O(\rho)$, which indicates the existence of $\tdec$. The second claim is a direct application of~\Cref{lem:1sam_maintain}.
\end{proof}

\begin{lemma}[A general version of \Cref{lem:informal_direction_stocastic_dl}]\label{lem:direction_stocastic_dl}
Under \Cref{assump:smoothness} and \Cref{cond:rank1}, for $x \in K^h$ and $p \in C, \nabla^2L_k(p) = \Lambda_k (p)w_k(p)w_k(p)^\top $, there exists $s \in \{1,-1\}$,
  \begin{align*}
      \ndli{x} &=  s w_k(p) + O(\|x - p \|_2)\,.
  \end{align*}
  
Further if $|w_k^\top (x - p)| \ge \|x - p \|_2^{3/2}$, then $s = \sign(w_k^\top (x - p))$. This implies
\begin{align*}
    \ndli{x}^\top  (x - p) &\ge s w_k^\top (x - p) - O(\|x -p \|_2^2) \\
    &\ge \|w_k^\top  (x- p) \|_2 - O(\|x - p \|_2^{3/2})\,.
\end{align*}
\end{lemma}

\begin{proof}[Proof of~\Cref{lem:direction_stocastic_dl}]
We will calculate the direction of $\ndli{x}$ using two different approximations and compare them to get our result.

\begin{itemize}
    \item[1.] According to \Cref{lem:directiondl},
    \begin{align*}
        \ndli{x} &= \frac{\nabla^2{L_k(\Phi_k(x))}(x - \Phi_k(x))}{\|\nabla^2{L_k(\Phi_k(x))}(x- \Phi_k(x)) \|_2} + O(\|x - \Phi_k(x)\|_2).
    \end{align*}
    
    Suppose $\nabla^2L_k({\Phi_k(x)}) = \Lambda_k(\Phi_k(x))w_k(\Phi_k(x)) w_k(\Phi_k(x))^\top$, then
    \begin{align*}
        \ndli{x} &= w_k(\Phi_k(x)) + O(\| x - \Phi_k(x)\|_2) 
    \end{align*}
    
    As $\nabla^2L_k(p) = \Lambda_k (p)w_k(p)w_k(p)^\top $, using Davis-Kahan Theorem~\ref{thm:Davis-Kahan}, we would have $\exists s \in \{-1,1\}$, such that $\|w_k(\Phi_k(x)) - sw_k(p)\|_2 \le \lspectraltwo \|\Phi_k(x) - p \|_2$.
    \begin{align*}
        \ndli{x} &= sw_k(p)  + O(\|\Phi_k(x) - p \|_2 + \|x - p \|_2).
    \end{align*}
    
    According to \Cref{lem:bounddl}, we have $\|x - \Phi_k(x)\|_2 \le \frac{\|\nabla L_k(x) \|_2}{\mu} \le \frac{\lspectraltwo \|x -p\|_2}{\mu} $.
    This implies,
    \begin{align}
    \label{stochasticway1}
        \ndli{x} = sw_k(p)  + O(\|x - p \|_2).
    \end{align}
    
    \Cref{stochasticway1} is our first statement.
    
    \item [2.] 
    By Taylor expansion at $p$,
    \begin{align*}
        \dli{x} = \Lambda_k(x)w_k(p)w_k(p)^\top(x - p) + O(\lspectralthree\|x - p\|_2^2).
    \end{align*}
    
    That being said, when $|w_k^\top (x - p)| \ge \|x - p\|_2^{3/2}$, we have 
    \begin{align*}
        &\|\dli{x}  - \Lambda_k w_kw_k^\top (x - p) \|_2  \le O(\|x - p\|_2^2)\,. \\
        &\| \dli{x} \| \ge \|\Lambda_k w_kw_k^\top (x - p) \|_2 - O(\|x - p\|_2^2)\ge \Omega(\|x - p\|_2^{3/2}).
    \end{align*}
    Concluding,
    \begin{align*}
    \| \ndli{x} - \frac{\Lambda_k w_kw_k^\top (x - p)}{\|\Lambda_k w_kw_k^\top (x - p)\|}\|_2 \le O(\|x - p \|_{2}^{1/2}) 
    \end{align*}
    
    Hence we have 
    \begin{align}
    \label{stochasticway2}
        \ndli{x} = \sign(w_k^\top (x - p))w_k + O(\|x - p\|_2^{1/2}) \,.
    \end{align}
    
    Comparing~\eqref{stochasticway1} and~\eqref{stochasticway2}, we have $s = \sign(w_k(p)^\top (x - p))$ when $|w_k^\top (x - p)| \ge \|x - p\|_2^{3/2}$. 
\end{itemize}
\end{proof}

\begin{lemma}
\label{lem:1sam_decrease_2_one_step}
Under condition of~\Cref{thm:1samphase1}, for any constant $C > 0$ independent of $\eta,\rho$, there exists constant $C_1 > C_2 > 0$ independent of $\eta,\rho$, if $x(t) \in K^h$ and $C_1 \eta \rho \le \|x(t) - \Phi(x(t))\| \le C\rho$, then we have that \begin{align*}
    \E_k[\|x(t + 1) - \Phi(x(t+1)) \|_2 \mid x(t)] \le \| x(t) - \Phi(x(t))\|_2 - C_2\eta\rho \,.
\end{align*}
\end{lemma}

\begin{proof}[Proof of~\Cref{lem:1sam_decrease_2_one_step}]
By~\Cref{lem:revertbounddl}, $\|x(t) - \Phi(x(t))\| = O(\rho)$. Hence we have that by Taylor Expansion,
\begin{align*}
    x(t+1) &= x(t) -   \eta \dli{x(t) + \rho\ndli{x(t)}}\\
    &= x(t) -   \eta \dli{x(t)} -   \eta \rho \nabla^2L_k(x(t)) \ndli{x(t)} + O( \eta\rho^2)\\
    &= x(t) -  \eta \dli{x(t)} -  \eta \rho\Lambda_k w_kw_k^\top  \ndli{x(t)} + O( \eta\rho^2)\,.
\end{align*}
Here $\Lambda_k, w_k$ indicates $\Lambda_k(\Phi(x(t))), w_k(\Phi(x(t)))$. 

Notice that given $\|x(t) - \Phi(x(t))\| = O(\rho)$, by~\Cref{lem:stochasticboundedphi}, we have that 
\begin{align*}
    \| \Phi(x(t + 1)) - \Phi(x(t)) \|_2 &= O(\eta\rho^2), \\
    \| x(t + 1) - x(t) \|_2 &= O(\eta\rho).
\end{align*}
This implies $x(t + 1) \in K^r$.

Further by Taylor Expansion, $\nabla L_k(x(t)) = \Lambda_k w_kw_k^\top (x(t) - \Phi(x(t))) + O(\rho^2)$.

By \Cref{lem:direction_stocastic_dl}, we have for some $s_k(t) \in \{ -1, 1 \}$.
\begin{align*}
    w_k^\top \ndli{x(t)}  &= s_k(t) w_k + O(\|x(t) - \Phi(x(t)) \|_2)\,.
\end{align*}
We also have 
\begin{align}
\label{eq:control_sign}
    s_k(t) \neq \sign(w_k^\top (x(t) - \Phi(x(t)))) \Rightarrow \| w_k^\top (x(t) - \Phi(x(t))) \|_2 \le  \|x(t) - \Phi(x(t)) \|_2^{3/2}.
\end{align}
Concluding,
\begin{align*}
    &x(t + 1) - \Phi(x(t + 1))\\
    =&  (x(t) - \Phi(x(t))) -  \eta \Lambda_k w_kw_k^\top (x(t) - \Phi(x(t))) -  \eta \rho \Lambda_k s_k(t) w_kw_k^\top  w_k + O( \eta \rho^2).
\end{align*}

After we take square and expectation,
\begin{align*}
    &\E[\| x(t + 1) - \Phi(x(t + 1)) \|_2^2 \mid x(t)]\\
    \le& \|x(t) - \Phi(x(t))\|_2^2 +   \frac{2\eta^2}{M}   \sum_{k=1}^M  \Lambda_k^2 |w_k^\top (x(t) - \Phi(x(t))) |^2 + \frac{ 2 \eta^2 \rho^2 }{M}  \sum_{k=1}^M  \Lambda_k^2
    \\&
    -  2 \frac{\eta}{M}   \sum_{k=1}^M  \Lambda_k |w_k^\top (x(t) - \Phi(x(t))) |^2
    -   2\frac{\eta\rho}{M}  \sum_{k=1}^M  \Lambda_k s_k(t) w_k^\top (x(t) - \Phi(x(t)))\\
    &+ O(\eta\rho^2 \|x(t) - \Phi(x(t)) \| + \eta^2 \rho^3)\,.
\end{align*}

We will then carefully examine each positive term,
\begin{align*}
     \frac{2\eta^2}{M}   \sum_{k=1}^M  \Lambda_k^2 |w_k^\top (x(t) - \Phi(x(t))) |^2 &= 2M\eta^2(x(t) - \Phi(x(t)))^\top \nabla^2 L(x(t))^2(x(t) - \Phi(x(t))) \\
     &\le 2M\lspectraltwo \eta^2 \|x(t) - \Phi(x(t)) \|^2 = O(\eta^2 \rho^2)\,.\\
     \frac{ 2 \eta^2 \rho^2 }{M}  \sum_{k=1}^M  \Lambda_k^2 &\le2 \lspectraltwo^2 \eta^2 \rho^2 = O(\eta^2\rho^2)\,.
\end{align*}
This implies,
\begin{align*}
    &\E[\| x(t + 1) - \Phi(x(t + 1)) \|_2^2 \mid x(t)]\\
    \le& \|x(t) - \Phi(x(t))\|_2^2 
    -   2\frac{\eta\rho}{M}  \sum_{k=1}^M  \Lambda_k s_k(t) w_k^\top (x(t) - \Phi(x(t)))\\
    &+ O(\eta\rho^2 \|x(t) - \Phi(x(t)) \| + \eta^2 \rho^2)\,. 
\end{align*}

We will now lower bound $  \sum_{k=1}^M   \Lambda_k s_k(t)w_k^\top (x(t) - \Phi(x(t)))$. By \Cref{eq:control_sign},
\begin{align*}
     \sum_{k=1}^M  \Lambda_k s_k(t)w_k^\top (x(t) - \Phi(x(t))) &\ge    \sum_{k=1}^M  \Lambda_k \|w_k^\top (x(t) - \Phi(x(t))) \|_2  - 2     \sum_{k=1}^M  \Lambda_k \|x(t) - \Phi(x(t)) \|_2^{3/2} \\
     &\ge \sum_{k=1}^M  \Lambda_k \|w_k^\top (x(t) - \Phi(x(t))) \|_2 - O( \|x(t) - \Phi(x(t)) \|_2^{3/2})\,.
\end{align*}

For $ \sum_{k=1}^M  \Lambda_k \|w_k^\top (x(t) - \Phi(x(t))) \|_2$, by Lemma~\Cref{lem:directiondl},
\begin{align*}
     \sum_{k=1}^M  \Lambda_k \|w_k^\top (x(t) - \Phi(x(t))) \|_2 &\ge \sqrt{ \sum_{k=1}^M\Lambda_k^2 \|w_k^\top (x(t) - \Phi(x(t))) \|_2^2} \\
     &= \sqrt{(x(t) - \Phi(x(t)))^\top \nabla^2 L(\Phi(x(t)))^2 (x(t) - \Phi(x(t)))}  \\
     &= \| \nabla^2 L(\Phi(x(t))) (x(t) - \Phi(x(t)))\|_2 \\
     &\ge \mu \| \partial \Phi(\Phi(x(t))) (x(t) - \Phi(x(t)))\|_2\\
     &\ge \mu \|x(t) - \Phi(x(t)) \|_2 - O(\|x(t) - \Phi(x(t)) \|_2^2)\,.
\end{align*}

Concluding, we have that
\begin{align*}
     \sum_{k=1}^M  \Lambda_k s_k(t)w_k^\top (x(t) - \Phi(x(t))) &\ge \mu\| x(t) - \Phi(x(t)) \|_2/2\,.
\end{align*}

So
\begin{align*}
    &\E[\| x(t + 1) - \Phi(x(t + 1)) \|_2^2 \mid x(t)]\\
    \le& \|x(t) - \Phi(x(t))\|_2^2 
    -   \frac{\mu\eta\rho}{M}  \| x(t) - \Phi(x(t)) \|_2\\
    &+ O(\eta\rho^2 \|x(t) - \Phi(x(t)) \| + \eta^2 \rho^2) \\
    \le& (\|x(t) - \Phi(x(t))\|_2 - C_2 \eta\rho)^2 \,.
\end{align*}
The inequality holds if $\|x(t) - \Phi(x(t))\|_2 > C_1 \eta\rho$.

Finally by Jenson's Inequality, 
\begin{align*}
    \E[\| x(t + 1) - \Phi(x(t + 1)) \|_2 | x(t)] \le \| x(t) - \Phi(x(t))\|_2 - C_2  \eta \rho.
\end{align*}
\end{proof}

\begin{lemma}
\label{lem:1sam_bounded_change}
Under condition of~\Cref{thm:1samphase1}, for any constant $C > 0$ independent of $\eta,\rho$, there exists constant $C_3 > 0$ independent of $\eta,\rho$, if $x(t) \in K^h$ and $ \|x(t) - \Phi(x(t))\| \le C\rho$, then we have that \begin{align*}
   |\|x(t + 1) - \Phi(x(t+1)) \|_2 - \| x(t) - \Phi(x(t))\|_2| \le C_3\eta\rho \,.
\end{align*}
\end{lemma}
\begin{proof}[Proof of~\Cref{lem:1sam_bounded_change}]
This is a direct application of~\Cref{lem:stochasticboundedphi}.
\end{proof}

\begin{lemma}
\label{lem:1sam_maintain_rho}
Under condition of~\Cref{thm:1samphase1},
assuming $x(t_0) \in K^{h/2}$ and $\|x(t_0) - \Phi(x(t_0))\| \le f(\eta,\rho)$ for some fixed function $f$ and $f(\eta,\rho) \in \Omega(\eta\rho \ln^2(1/\eta\rho)) \cap O(\rho)$, then with probability $1 - O(\rho)$, for any $t$ satisfying $t_0 \le t \le t_0 + O(\ln(1/\eta\rho)/\eta)$, it holds that $\|x(t) - \Phi(x(t)\| \le 2f(\eta,\rho)$.
Moreover, we have that
\begin{align*}
    \| \Phi(x(t)) - \Phi(x(t_0)) \| = O((\eta + \rho) \ln(1/\eta\rho)).
\end{align*}
\end{lemma}

\begin{proof}[Proof of~\Cref{lem:1sam_maintain_rho}]
By~\Cref{lem:1sam_maintain}, we have that $x(t) \in K^h$ for any $t$ satisfying that $t_0 \le t \le t_0 + O(\ln(1/\eta\rho)/\eta)$ and with probability $1 - O(\rho)$  we will suppose this hold for the following deduction.

By Uniform Bound,
\begin{align*}
    &\prob(\exists t_0 \le t \le t_0 + O(\ln(1/\eta\rho)/\eta), \|x(t) - \Phi(x(t)\| \ge 2f(\eta,\rho)) \\
    \le& \sum_{t = t_0}^{t_0 + O(\frac{\ln(1/\eta\rho)}{\eta})} \prob( \|x(t) - \Phi(x(t))\| \ge 2f(\eta,\rho)) \quad \mathrm{and} \quad \| x(\tau) - \Phi(x(\tau))\| \le 2f(\eta,\rho), \forall t_0 \le \tau \le t -1).
\end{align*}

Consider each term and apply Uniform bound again,
\begin{align*}
     &\prob( \|x(t) - \Phi(x(t))\| \ge 2f(\eta,\rho)) \quad \mathrm{and} \quad \| x(\tau) - \Phi(x(\tau))\| \le 2f(\eta,\rho), \forall t_0 \le \tau \le t -1)\\
    \le& \sum_{\tau = t_0}^t \prob( \|x(t) - \Phi(x(t))\| \ge 2f(\eta,\rho)) \quad \mathrm{and} \quad \| x(\tau) - \Phi(x(\tau))\| \le f(\eta,\rho),\\
    &\quad \mathrm{and} \quad 
    f(\eta,\rho) \le \| x(\tau') - \Phi(x(\tau'))\| \le 2f(\eta,\rho),
    \forall \tau + 1 \le \tau' \le t -1).\\
\end{align*}
Then if we consider each term, it is bounded by
\begin{align}
\label{eq:1sam_single}
    &\prob( \|x(t) - \Phi(x(t))\| \ge 2f(\eta,\rho)) \nonumber\\
    &\quad \mathrm{and} \quad 
    f(\eta,\rho) \le \| x(\tau') - \Phi(x(\tau'))\| \le 2f(\eta,\rho),
    \forall \tau + 1 \le \tau' \le t -1 \nonumber \\&\mid \| x(\tau) - \Phi(x(\tau))\| \le f(\eta,\rho)).
\end{align}

Now let $C$ be the positive constant satisfying $2f(\eta,\rho) \le C\rho$, suppose $C_1,C_2$ are the constants corresponds to $C$ in~\Cref{lem:1sam_decrease_2_one_step} and $C_3$ is the constant correspond to $C$ in~\Cref{lem:1sam_bounded_change}. By definition $C_3 > C_2$.

Define a coupled process $\tilde y(\tau + 1) = y(\tau + 1)$ and 
\begin{align*}
    \tilde y(\tau') = \begin{cases}
    \| x(\tau') - \Phi(x(\tau')) \|_2, &\text{ if }\tilde y(\tau' - 1) =  \| x(\tau' - 1) - \Phi(x(\tau' - 1)) \|_2 > f(\eta,\rho) \\
    \tilde y(\tau' - 1) - C_2 \eta\rho, &\text{ if otherwise}.
    \end{cases}
\end{align*}

Now clearly~\Cref{eq:1sam_single} is bounded by $\prob(\tilde y(t) \ge 2f(\eta,\rho))$.

As $\E[\tilde y(\tau')] \le \tilde y(\tau' - 1) - C_2 \eta\rho$ by~\Cref{lem:1sam_decrease_2_one_step} and $\|\tilde  y(\tau') -  \tilde y(\tau' - 1)\| \le C_3 \eta\rho$ by~\Cref{lem:1sam_bounded_change}. This implies $\|\tilde y(\tau')\| - C_2 \eta\rho \tau'$ is a super martingale. By Azuma-Hoeffding bound(\Cref{lem:azuma}), we have
\begin{align*}
    P(\tilde y(t) \ge \tilde y(\tau+1) - C_2 \eta\rho (t - \tau - 1) +  h )&\le 2\exp(-\frac{h^2}{ 4(t-\tau -1 ) (C_3 + C_2)^2 \eta^2\rho^2 }).
\end{align*}

Choosing $h = C_2\eta\rho (t - \tau - 1) - \|x(\tau + 1)  - \Phi(x(\tau + 1))\| + 2f(\eta,\rho) $
\begin{align*}
    &P(\tilde y(t+1) \ge 2f(\eta,\rho))\\
    \le&2 \exp(-\frac{(C_2\eta\rho (t - \tau) - \|x(\tau + 1)  - \Phi(x(\tau + 1))\| + 2f(\eta,\rho))^2}{8(t-\tau) (C_3 + C_2)^2 \eta^2\rho^2}) \\
    \le& 2\exp(-\frac{(C_2\eta\rho (t - \tau) + f(\eta,\rho)/2)^2}{4(t-\tau) (C_3 + C_2)^2 \eta^2\rho^2}) \\
    \le& 2\exp(-\frac{C_2 f(\eta,\rho)}{2 (C_3 + C_2)^2 \eta \rho}) \le \eta^{10} \rho^{10}.
\end{align*}

We then have 
\begin{align*}
   \prob(\exists t_0 \le t \le t_0 + O(\ln(1/\eta\rho)/\eta), \|x(t) - \Phi(x(t)\| \ge 2f(\eta,\rho)) \le \rho.
\end{align*}
\end{proof}

\begin{lemma}
\label{lem:1sam_decrease_2}
Under condition of~\Cref{thm:1samphase1}, assuming there exists $\tdec$ such that $x(\tdec) \in K^{h/2}$ and $\| \nabla L(x(\tdec))\| \le 4 \zeta \rho$, then with probability $ 1 - O(\rho)$, there exists $\tdecs = \tdec + O(\ln(1/\eta \rho)/\eta)$, such that $\| x(\tdecs) - \Phi(\tdecs) \| \le O(\eta\rho)$.

Furthermore, for any $t$ satisfying $\tdecs \le t \le \tdecs + \Theta(\ln(1/\eta \rho)/\eta)$, we have that $\| \Phi(x(t)) - \Phi(x(\tdec))\| = O(\rho^2 \ln(1/\eta \rho))$.
\end{lemma}
\begin{proof}[Proof of~\Cref{lem:1sam_decrease_2}]
We have that $x(t) \in K^h$ (\Cref{lem:1sam_maintain}) and $\| x(t) - \Phi(x(t)) \| \le C\rho$ for some constant $C$ (\Cref{lem:1sam_maintain_rho})  for any $t$ satisfying that $\tdec \le t \le \tdec + O(\ln(1/\eta\rho)/\eta)$ with probability $1 - O(\rho)$ and we will suppose this holds for the following deduction. The second statement then follows directly from~\Cref{lem:stochasticboundedphi}.

Let $C_1,C_2$ be the constant in~\Cref{lem:1sam_decrease_2_one_step} corresponding to $C$, For simplicity of writing, define $T_1 \triangleq \lceil \frac{C\ln(\frac{C}{C_1 \eta \rho^2})}{C_2\eta} \rceil = O(\ln(1/\eta \rho)/\eta)$.
Define indicator function as 
\begin{align*}
    \indicatordec{t} = \one[\|x(t) - \Phi(x(t)) \| \ge C_1\eta\rho, \forall t \ge \tau \ge \tgf]\,.
\end{align*}

By~\Cref{lem:1sam_decrease_2_one_step}, we have that,
\begin{align*}
    \E[\|x(t + 1) - \Phi(x(t + 1)) \|  \indicatordec{t + 1}] &\le \E[\|x(t + 1) - \Phi(x(t + 1)) \|\indicatordec{t}] \\
    &\le \E[\|x(t) - \Phi(x(t)) \|\indicatordec{t}]  - C_2\eta\rho \E[\indicatordec{t}] \\
    &\le \E[\|x(t) - \Phi(x(t)) \|\indicatordec{t}](1 - \frac{C_2\eta}{C})
\end{align*}

We can then conclude that with $T_2 =T_1 + \tdec$, using \Cref{lem:revertbounddl},
\begin{align*}
    C_1 \eta\rho \E\indicatordec{T_2 + 1}
    &\le \E[\|x(T_2 + 1) - \Phi(x(T_2 + 1)) \|_2 \indicatordec{T_2 + 1}]\\
    &\le (1 - \frac{C_2\eta}{C})^{T_1}   \| x(\tdec) - \Phi(x(\tdec)) \|
    \le C_1 \eta  \rho^3.
\end{align*}
This implies $\indicatordec{T_2 + 1}  = 0$ with probability $1 - O(\rho)$, which indicates the existence of $\tdecs$. 
\end{proof}

\subsection{Phase II (Proof of Theorem~\ref{thm:1samphase2})}
\label{sec:movement}

\begin{proof}[Proof of~\Cref{thm:1samphase2}]

We will inductively prove the following induction hypothesis $\mathcal{P}(t)$ holds with probability $1 - O(\eta^3\rho^3 t)$ for $t \le T_3/\eta\rho^2 + 1$,
\begin{align*}
    x(\tau) \in K^{h/2},&\ \tau \le t \\
    \|x(\tau) - \Phi(x(\tau)) \|_2 \le 2\|x(0) - \Phi(x(0)) \|_2 = O(\eta\rho),&\ \tau \le t \\
    \| \Phi(x(\tau)) - X(\eta \rho^2 \tau) \| = \tilde O(\eta^{1/2} + \rho),&\ \tau \le t
\end{align*}

$\mathcal{P}(0)$ holds trivially. Now suppose $\mathcal{P}(t)$ holds, then $x(t + 1) \in K^h$. By~\Cref{lem:1sam_maintain_rho_for_longer}, we have that with probability $1 - O(\eta^3\rho^3)$, $\|x(t + 1) - \Phi(x(t+1)) \| \le 2\|x(0) - \Phi(x(0)) \|_2 = O(\eta\rho)$. 

Now we have 
\begin{align*}
    2\|x(0) - \Phi(x(0)) \|_2 = O(\eta\rho), \tau \le t + 1.\\
    x(\tau) \in K^h, \tau \le t + 1
\end{align*}
By~\Cref{lem:direction_1_sam}, it holds that
\begin{align*}
        \| \Phi(x(\tau+1)) - \Phi(x(\tau))  + \eta \rho^2 \projt[\Phi(x(\tau))] \nabla \lambda_1\Bigl(\nabla^2 L_{k_{\tau}}\bigl(\Phi(x(\tau))\bigr)\Bigr)/2 \| \le \tilde O(\eta \rho^3 + \eta^2 \rho^2)\,.
\end{align*}
As 
\begin{align*}
    \E_{k_t} \projt[\Phi(x(t))] \nabla \lambda_1\Bigl(\nabla^2 L_{k_t}\bigl(\Phi(x(t))\bigr)\Bigr) = \projt[\Phi(x(t))] \nabla \mathrm{Tr}(\nabla^2 L(\Phi(x(t)))).
\end{align*}
By~\Cref{thm:ode_approximation}, let $b(x) = -\partial \Phi(x) \nabla \mathrm{Tr}(\nabla^2 L(x))$, $b_k(x) = -\partial \Phi(x) \mathrm{Tr}(\nabla^2 L_{k_t}(x))$, $p = \eta\rho^2$ and $\epsilon = O(\eta + \rho)$, it holds that, with probability $1 - O(\eta^3\rho^3)$, 
\begin{align*}
    &\|\Phi(x(\tau)) - X(\eta\rho^2 \tau) \| \\=& O(\|\Phi(x(0)) - \Phi(\xinit) \| + T_3\eta\rho^2 + \sqrt{\eta\rho^2 T_3 \log(2eT_3/(\eta^2\rho^4))} + (\rho + \eta)T_3) \\=& \tilde O(\eta^{1/2} + \rho), \tau \le t + 1
\end{align*}
This implies $\|x(t+1 ) - X(\eta\rho^2(t+1)) \|_2 \le \| x(t+1) - \Phi(x(t+1)) \|_2 + \|\Phi(x(t+1)) - X(\eta\rho^2(t+1))  \|_2 = \tilde O(\eta^{1/2} + \rho) < h/2$. Hence $x(t+1) \in K^{h/2}$. Combining with $\mathcal{P}(t)$ holds with probability $1- O(\eta^3\rho^3t)$, we have that $\mathcal{P}(t+1)$ holds with probability $1 - O(\eta^3\rho^3(t+1))$. The induction is complete.

Now $\mathcal{P}(\lceil T_3/\eta\rho^2 \rceil)$ is equivalent to our theorem.
\end{proof}

\subsubsection{Convergence Near Manifold}

\begin{lemma}
\label{lem:1sam_maintain_rho_for_longer}
Under condition of~\Cref{thm:1samphase2},
assuming $x(t) \in K^{h}, \forall t_0 \le t \le t_0 + O(1/\eta\rho^2)$ and $\|x(t_0) - \Phi(x(t_0))\| \le f(\eta,\rho)$ for some fixed function $f$ and $f(\eta,\rho) \in \Omega(\eta\rho \ln^2(1/\eta\rho)) \cap O(\rho)$, then with probability $1 - O(\eta^3 \rho^3)$, for any $t$ satisfying $t_0 \le t \le t_0 + O(1/\eta\rho^2)$, it holds that $\|x(t) - \Phi(x(t))\| \le 2f(\eta,\rho)$.
\end{lemma}

\begin{proof}[Proof of~\Cref{lem:1sam_maintain_rho_for_longer}]
The proof is almost identical to~\Cref{lem:1sam_maintain_rho} and is omitted.
\end{proof}

\subsubsection{Tracking Riemannian Gradient Flow}

\begin{lemma}
\label{lem:direction_1_sam}
Under the condition of~\Cref{thm:1samphase2}, for any $t$ satisfying that $x(t) \in K^h$ and 
\begin{align*}
    \|x(t) - \Phi(x(t)) \| &= O(\eta\rho\ln^2(1/\eta\rho)).
\end{align*}
It holds that
\begin{align*}
    \| \Phi(x(t+1)) - \Phi(x(t))  + \eta \rho^2 \projt[\Phi(x(t))] \nabla \lambda_1\Bigl(\nabla^2 L_{k_t}\bigl(\Phi(x(t))\bigr)\Bigr)/2 \| \le \tilde O(\eta \rho^3 + \eta^2 \rho^2)\,.
\end{align*}
\end{lemma}

\begin{proof}[Proof of~\Cref{lem:direction_1_sam}]
We will abbreviate $k_t$ by $k$ in this proof.

By Taylor Expansion,
\begin{align*}
    x(t+1) &= x(t) - \eta \dli{x(t) + \rho \ndli{x(t)}} \\
    &= x(t) - \eta \dli{x(t)} - \eta\rho \dlti{x(t)} \ndli{x(t)}\\& - \eta \rho^2 \partial^2(\nabla L_k)[\ndli{x(t)},\ndli{x(t)}]/2 + O(\eta \rho^{3}).
\end{align*}

Now as $\|x(t) - \Phi(x(t))\|_2 = \tilde O(\eta\rho)$, by \Cref{lem:stochasticboundedphi}, it implies
\begin{align*}
    \|x(t+1) - x(t) \|_2 =  O(\eta\rho)\,.
\end{align*}

Then we have 
\begin{align*}
    \| \Phi(x(t+1)) - \Phi(x(t)) - \partial\Phi(x(t)) (x(t+1) - x(t))\|_2 \le \phispectraltwo \|x(t+1) - x(t) \|_2^2 = O(\eta^2\rho^2).
\end{align*}

Using \Cref{lem:stochastic_phi_lk}, we have
\begin{align*}
    \| \eta \partial\Phi(x(t)) \dli{x(t)} \|_2 &= O(\eta \|x(t) - \Phi(x(t)) \|_2^2) = O(\eta^3\rho^2 + \eta\rho^4), \\
    \|  \eta\rho\partial\Phi(x(t)) \dlti{x(t)} \ndli{x(t)} \|_2 &= O(\eta \rho \|x(t) - \Phi(x(t)) \|_2) = \tilde O(\eta^2 \rho^2 + \eta \rho^3).
\end{align*}

Hence
\begin{align*}
    \| \Phi(x(t+1)) - \Phi(x(t)) + \eta \rho^2 \partial \Phi(x(t)) \partial^2(\nabla L_k)[\ndli{x(t)},\ndli{x(t)}]/2\|_2 = \tilde O(\eta^2 \rho^2 + \eta \rho^3).
\end{align*}

Notice finally that by \Cref{lem:direction_stocastic_dl},
\begin{align*}
    &\partial \Phi(x(t)) \partial^2(\nabla L_k)[\ndli{x(t)},\ndli{x(t)}] \\
    =& \partial \Phi(\Phi(x(t))) \partial^2(\nabla L_k)[w_k,w_k] + O(\|x(t) - \Phi(x(t)) \|_2)
    \\=& \projt[\Phi(x(t))] \nabla (\lambda_1(\nabla^2 L_k(\Phi(x(t))))) + O(\|x(t) - \Phi(x(t)) \|_2).
\end{align*}

Hence we have 
\begin{align*}
    \Phi(x(t+1)) - \Phi(x(t))  = -\eta\rho^2 \projt[\Phi(x(t))] \nabla \lambda_1\Bigl(\nabla^2 L_{k_t}\bigl(\Phi(x(t))\bigr)\Bigr)/2 + \tilde O(\eta^2 \rho^2 + \eta \rho^3)
\end{align*}
This completes the proof.
\end{proof}

\subsection{Proof of Theorem~\ref{thm:1sam}}
\label{sec:1sam_proof}

\begin{proof}[Proof of Theorem~\ref{thm:1sam}]
    By~\Cref{thm:1samphase1}, there exists constant $T_1$ independent of $\eta,\rho$, such that there exists $\tphaseone \le T_1 \ln(1/\eta\rho)/\eta$, with probability $1 - O(\rho)$, it holds that 
    \begin{align*}
    &\|x(\tphaseone) - \Phi(x(\tphaseone)) \|_2 = O(\eta\rho). \\
    &\|\Phi(x(\tphaseone)) - \Phi(\xinit) \| = \tilde O(\eta^{1/2} + \rho)
\end{align*}

Hence by~\Cref{thm:1samphase2}, if we consider a translated process with $x'(t) = x(t + \tphaseone)$, we would have for any $T_3$ such that the solution $X$ of~\Cref{eq:traceode} is well defined, we have that for $t = \lceil \frac{T_3}{\eta\rho^2} \rceil$
\begin{align*}
    \|\Phi(x'(t)) - X(\eta\rho^2 t) \|_2 &= O(\eta  \ln(1/\rho))\,.
\end{align*}
This implies for $t$ satisfying $X(\eta\rho^2 (t - \tphaseone))$ is well-defined,
\begin{align*}
    \|\Phi(x(t)) - X(\eta\rho^2 (t - \tphaseone)) \|_2 &=  \tilde O(\eta^{1/2} + \rho).
\end{align*}
Finally, as 
\begin{align*}
    \| X(\eta\rho^2 (t - \tphaseone)) - X(\eta\rho^2 t) \|_2 
    &= O(\eta\rho^2 \tphaseone) = O(\rho \ln(1/\eta\rho)) = \tilde O(\rho).
\end{align*}

We have that
\begin{align*}
    \|\Phi(x(t)) - X(\eta\rho^2 t)\|_2 &= \tilde O(\eta^{1/2} + \rho).
\end{align*}

We also have 
\begin{align*}
    \|x(t) - \Phi(x(t)) \|_2 &= O(\eta\rho).
\end{align*}
by~\Cref{thm:1samphase2}.
\end{proof}

\subsection{Proofs of Corollaries~\ref{corr:loss_ode_1} and~\ref{corr:loss_opt_1}}

\begin{proof}[Proof of \Cref{corr:loss_ode_1}]
\label{app:1samcorr}
We will do Taylor expansion on $\stmaxloss(x)$. By Theorem~\ref{thm:1samphase1} and ~\ref{thm:1samphase2}, we have $\| x(\lceil T_3/\eta\rho^2 \rceil) - X(T_3)\|_2 =  \tilde O(\eta^{1/2} + \rho)$ and $\| \Phi(x(\lceil T_3/\eta\rho^2 \rceil)) - x(\lceil T_3/\eta\rho^2 \rceil)\|_2 =  \tilde O(\eta^{1/2} + \rho)$. For convenience, we denote $x(\lceil T_3/\eta\rho^2 \rceil)$ by $x$.
\begin{align*}
    \stmaxsharpness(x) &= \max_{\|v\|\le 1} \E_k[\rho v^\top \nabla L_k(x) + \rho^2 v^\top \nabla^2 L_k(x)v/2] + O(\rho^3)  
\end{align*}
Since $\max_{\|v\|\le 1} |v^\top  \nabla L_k(x) | = O(\| x - \Phi(x) \|)=   \tilde O(\eta^{1/2} + \rho)$, it holds that,
\begin{align*}
    \stmaxsharpness(x) &= \rho^2 \E_k[ \max_{\|v\|\le 1}  v^\top \nabla^2 L(x)v/2] +O\bigl((\eta^{1/4} + \rho^{1/4})\rho^2\bigr)\\ &= \rho^2 \E_k \max_{\|v\|\le 1}  [v^\top \nabla^2 L(X(T_3))v/2] + O\bigl((\eta^{1/4} + \rho^{1/4})\rho^2\bigr) \\
    &= \rho^2 \mathrm{Tr}(X(T_3))/2 + O\bigl((\eta^{1/4} + \rho^{1/4})\rho^2\bigr)
\end{align*}
\end{proof}

\begin{proof}[Proof of \Cref{corr:loss_opt_1}]

We choose $T_{\eps}$ such that $X(T_{\eps})$ is sufficiently close to $X(\infty)$, such that $\mathrm{Tr}(X(T_{\eps})) \le \mathrm{Tr}(X(\infty)) + \epsilon/2$. By corollary~\ref{corr:loss_ode_1} (let $T_3 = T_{\eps}$), we have for all $\rho,\eta$ such that $(\eta + \rho)\ln(1/\eta\rho)$ is sufficiently small, $\| \stmaxsharpness(x(\lceil T_{\eps} / (\eta\rho^2)\rceil)) - \rho^2 \mathrm{Tr}(X(T))/2 \|_2 \le o(1)$. This further implies $\| \stmaxsharpness(x(\lceil T_{\eps} / (\eta\rho^2)\rceil)) - \rho^2 \mathrm{Tr}(X(\infty))/2  \|_2 \le \epsilon\rho^2/2 + o(1)$. We also have $ L(x(\lceil T_{\eps} / (\eta\rho^2)\rceil)) - \inf_{x \in U'} L(x) = o(1)$. Then we can leverage \Cref{thm:generalreg,thm:stmaxsharp} to get the desired bound. 
\end{proof}

\subsection{Other Omitted Proofs for 1-SAM}
We will use $\ell'(y,y_k)$ and $\ell''(y,y_k)$ to denote $\frac{d \ell(y',y_k)}{d y'}\vert_{y' = y}$ and  $\frac{d^2\ell(y',y_k)}{d y'^2}\vert_{y' = y}$.
 
\begin{lemma}
\label{lem:rank1hessian}
Under \Cref{setting:1sam}, fix $k\in [M]$, for any $p$ satisfying $\ell(f_k(p),y_k) = 0$, we have that
\begin{align*}
    \nabla^2 L_k(p) = \ell''(f_k(p),y_k)\nabla f_k(p)(\nabla f_k(p))^\top \,.
\end{align*}
\end{lemma}
\begin{proof}[Proof of~\Cref{lem:rank1hessian}]
  $\ell(f_k(p),y_k) = 0$ implies $\ell'(f_k(p),y_k) = 0$. Then by Taylor Expansion,
 \begin{align*}
     \nabla^2 L_k(p) &= \nabla^2_p \ell(f_k(p), y_k) \\
     &= \partial_p [\ell'(f_k(p),y_k) \nabla f_k(p)] \\
     &=  \ell''(f_k(p),y_k) \nabla f_k(p) \nabla f_k(p)^\top  + 
     \ell'(f_k(p),y_k) \nabla^2 f_k(p) \\
     &= \ell''(f_k(p),y_k) \nabla f_k(p) \nabla f_k(p)^\top  \,.
 \end{align*}
 This concludes the proof.
\end{proof}

 \begin{proof}[Proof of \Cref{lem:informal_direction_stocastic_dl}]
 By~\Cref{lem:rank1hessian}, as $L(p) =\frac{1}{M} \sum_{k = 1}^M L_k(p)$, we have 
 \begin{align*}
     \nabla^2 L(p) = \frac{1}{M}  \sum_{k=1}^M  \frac{\partial ^2\ell(y',y_k)}{(\partial y')^2}\vert_{y'=f_k(x)} \nabla f_k(p) \nabla f_k(p)^\top \,.
 \end{align*}
 
By definition of $\Gamma$ in \Cref{setting:1sam}, we have for any $p \in \Gamma$, $\{\nabla f_k(p)\}_{k=1}^n$ are linearly independent, which implies that $\nabla f_k(p) \neq 0$  for any $p \in \Gamma$.

 For any $p \in \Gamma$, as $\frac{\nabla f_k(p)}{\| \nabla f_k(p)\|}$ is well defined and continuous at $p$, there exists a open ball $V$ containing $p$ such that for any $x \in V$, $\norm{ \nabla f_k(x)}_2\ge C_1>0$  and $\| \nabla[\frac{\nabla f_k(x)}{\| \nabla f_k(x)\|}]\|_2\le C_2$ for some constants $C_1$ and $C_2$.
 
 Suppose $\nabla L_k(x) \neq 0$, then as by Taylor Expansion,
 \begin{align*}
     \nabla L_k(x) = \ell'(f_k(x),y_k) \nabla f_k(x) \,.
 \end{align*}
 We have $\frac{ \nabla L_k(x)}{\| \nabla L_k(x) \|} = \frac{\nabla f_k(x)}{\| \nabla f_k(x) \|} = \frac{\nabla f_k(p)}{\| \nabla f_k(p) \|} + C_2\|x -p\|$, which completes the proof. 
 \end{proof}
We note that the alignment result in \Cref{lem:informal_direction_stocastic_dl} is not directly used in our proof. Instead, we use its generalized version \Cref{lem:direction_stocastic_dl} which holds under  holds under a more general condition than \Cref{setting:1sam}, namely \Cref{cond:rank1}.

\section{Technical Lemmas}

\begin{lemma}[Corollary 4.3.15 in \cite{horn2012matrix}]
\label{lem:eigendiff}
Let $\Sigma, \hat \Sigma\in R^{D \times D}$ be symmetric and non-negative with eigenvalues $\lambda_1 \ge...\ge \lambda_D$ and $\hat \lambda_1 \ge...\ge \hat \lambda_D$, then for any $i$, \begin{align*}
    |\hat \lambda_i - \lambda_i | \le \| \Sigma - \hat \Sigma \|_2
\end{align*}
\end{lemma}

\begin{definition}[Unitary invariant norms]
A matrix norm $\|\cdot \|_{*}$ on the space of matrices in $\mathbb{R}^{p \times d}$ is unitary invariant if for any matrix $K \in \mathbb{R}^{p \times d}$, $\norm{ U K W }_{*} = \norm{ K }_{*}$ for any unitary matrices $U \in \mathbb{R}^{p \times p}, W \in \mathbb{R}^{d \times d}.$
\end{definition}
\begin{theorem}\label{thm:Davis-Kahan} [Davis-Kahan $\sin (\theta)$ theorem \citep{davis1970rotation}]
 Let $\Sigma, \hat{\Sigma} \in \mathbb{R}^{p \times p}$ be symmetric, with eigenvalues $\lambda_{1} \geq \ldots \geq \lambda_{p}$ and $\hat{\lambda}_{1} \geq \ldots \geq \hat{\lambda}_{p}$ respectively. Fix $1 \leq r \leq s \leq p$, let $d\triangleq s-r+1$ and let $V=\left(v_{r}, v_{r+1}, \ldots, v_{s}\right) \in \mathbb{R}^{p \times d}$ and $\hat{V}=\left(\hat{v}_{r}, \hat{v}_{r+1}, \ldots, \hat{v}_{s}\right) \in \mathbb{R}^{p \times d}$ have orthonormal columns satisfying $\Sigma v_{j}=\lambda_{j} v_{j}$ and $\hat{\Sigma} \hat{v}_{j}=\hat{\lambda}_{j} \hat{v}_{j}$ for $j=r, r+1, \ldots, s .$ 
Define $\Delta\triangleq \min \left\{\max\{0, \lambda_s- \hat \lambda_{s+1}\}, \max\{0,  \hat \lambda_{r-1}-\lambda_r\}\right\}$, 
 where $\hat{\lambda}_{0}\triangleq \infty$ and $\hat{\lambda}_{p+1}\triangleq -\infty$, we have for any unitary invariant norm $\| \cdot \|_{*}$,
$$
\Delta \cdot \|\sin \Theta(\hat{V}, V)\|_{*}  \leq \|\hat{\Sigma}-\Sigma\|_{*}.
$$
Here $\Theta(\hat{V}, V) \in \mathbb{R}^{d \times d},$  
with $\Theta(\hat{V}, V)_{j, j} = \arccos \sigma_j$ for any $j \in [d]$ and $\Theta(\hat{V}, V)_{i, j} = 0$ for all $i \ne j \in [d]$. $\sigma_1 \ge \sigma_2 \ge \cdots \ge \sigma_d$ denotes the singular values of $\Hat{V}^{\top} V.$ $[\sin \Theta]_{ij}$ is defined as $\sin (\Theta_{ij})$. 
\end{theorem}

\begin{lemma}[Azuma-Hoeffding Bound]
\label{lem:azuma}
Suppose $\{Z_n\}_{n\in \mathbb{N}}$ is a super-martingale, suppose $-\alpha  \le Z_{i+1} - Z_i \le \beta$, then for all $n > 0, a >0$, we have
\begin{align*}
    \prob(Z_n - Z_0 \ge a) \le 2\exp(-a^2/(2n(\alpha + \beta)^2))
\end{align*}

\end{lemma}

\begin{lemma}[Azuma-Hoeffding Bound, Vector Form, \citet{hayes2003large}]
\label{lem:azuma_vector}
Suppose $\{Z_n\}_{n\in \mathbb{N}}$ is a $\R^D$-valued martingale, suppose $ \| Z_{i+1} - Z_i\|_2 \le \sigma$, then for all $n > 0, a >0$, we have
\begin{align*}
    \prob(\norm{Z_n - Z_0}_2\ge \sigma (1+a)) \le 2\exp(1 - a^2/ 2n).
\end{align*}

In other words, for any $0<\delta<1$, with probability at least $1- \delta$, we have that
\begin{align*}
	\norm{Z_n-Z_0}_2 \le \sigma\left(1+ \sqrt{2n \log \frac{2e}{\delta}}\right)\le 2\sigma \sqrt{2n\log\frac{2e}{\delta}}.
\end{align*}
\end{lemma}

\begin{lemma}[Discrete Gronwall Inequality, \citet{borkar2009stochastic}]\label{lem:discrete_gronwall} 
Let $\{x(t)\}_{t\in\mathbb{N}}$ be a sequence of nonnegative real numbers, $\{a_n\}_{n\in\mathbb{N}}$ be a sequence of positive real numbers and $C,L>0$ scalars such that for all $n$,
\begin{align*}
	x(t)\le C+L\sum_{n=0}^{t-1} a_n x(n).
\end{align*}
Then for $T_t= \sum_{n=0}^ta_n$, it holds that $x(t+1)\le Ce^{LT_t}$.
	
\end{lemma}

\begin{lemma}[\cite{magnus1985differentiating}]\label{lem:top_eigenvector}
	Let $A:\R^D\to\R^{D\times D}$ be any $\mathcal{C}^1$ symmetric matrix function and $x^*\in\R^D$ satisfying $\lambda_1(A(x^*)) > \lambda_2(A(x^*))$ and $v_1$ be the top eigenvector of $A(x^*)$. It holds that $\nabla \lambda_1(A(x))\vert_{x=x^*} = \nabla (v_1^\top A(x)v_1)\vert_{x=x^*}$.
\end{lemma}

We then present some of the technical lemmas we required to prove \Cref{lem:norminequal}.

\begin{lemma}
\label{lem:regularlambda}
If $0 < c < \frac{b-a}{b^2}, a\sqrt{\frac{a^2 + 2b^2}{2(1-cb)}} \ge \frac{a^2+b^2}{2 - ca -cb}$, then $a > \frac{1}{2}b, cb \le \frac{1}{2}$
\end{lemma}

\begin{proof}[Proof of~\Cref{lem:regularlambda}]
Notice that 
\begin{align*}
    c a \sqrt{\frac{ a^2 + b^2}{1 - c b}} \ge c a \sqrt{\frac{ a^2 + 2b^2}{2(1 - c b)}} \ge \frac{c b^2 + c a^2}{2 - c b -c a} \ge \frac{c b^2 + c a^2}{2 - c b}.
\end{align*}

So \begin{align*}
    \sqrt{1 - c b} + \frac 1 {\sqrt{1 - c b}} \ge \sqrt{1 + \frac{b^2}{a^2}}
\end{align*}

As $c < \frac{b - a} {b^2}$, we have $1>1 - c b> \frac{a}{b}$.

So \begin{align*}
    \sqrt{\frac{a}{b}} + \sqrt  {\frac{b}{a}} \ge \sqrt{1 + \frac{b^2}{a^2}}
\end{align*}

The above inequality implies $a \ge \frac{1}{2} b$. As $c < \frac{b - a}{b^2}$,$cb \le \frac{1}{2}$.
\end{proof}

\begin{lemma}
\label{lem:technorminequal}

When $0< a < b, 0 < c < \frac{b-a}{b^2}$, we have
\begin{align*}
    cb^2  + ca^2 (2 - cb - \frac23 ca) - (1-cb) \frac{c(a^2 +b^2)}{2-ca-cb} - ca^2(\frac{1}{2}a^2 + b^2)\frac{2-ca-cb}{(a^2+b^2)} \le \frac{cb^2}{2-cb}
\end{align*}

\end{lemma}
\begin{proof}[Proof of~\Cref{lem:technorminequal}]

Equivalently, we are going to prove
\begin{align*}
(1-cb)b^2\left( \frac{1}{2-ca-cb} - \frac{1}{2-cb} \right) + a^2\frac{1-cb}{2-ca-cb}+a^2(\frac{1}{2}a^2 +b^2)\frac{2-ca-cb}{(a^2+b^2)} \ge a^2 (2 - cb - \frac23 ca) 
\end{align*}
Further simplifying, we only need to prove
\begin{align*}
\frac{(1-cb)cab^2}{(2-cb)(2-ca-cb)} + a^2 \frac{1-cb}{2-ca-cb} \ge \frac{1}{3}ca^3 + \frac{a^4}{2(a^2+b^2)}(2-ca-cb)
\end{align*}

We have the following auxiliary inequalities,
\begin{align*}
    (1 - cb)b &> a \\
    \frac{1-cb}{2-ca-cb}  \ge \frac{1}{\frac{a+b}b + \frac{b-a}{(1-cb)b}} &\ge \frac{1}{\frac{a+b}b + \frac{b-a}a} = \frac{ab}{a^2+b^2} \ge \frac{a^2}{a^2+b^2}  \\
    \frac{1-cb}{a^2} \ge \frac{2-ca-cb}{a^2 +b^2}
\end{align*}

Using the above auxiliary inequalities we have 
\begin{align*}
    &\frac{(1-cb)cab^2}{(2-cb)(2-ca-cb)} + a^2 \frac{1-cb}{2-ca-cb} \ge \frac{1}{3}ca^3 + \frac{a^4}{2(a^2+b^2)}(2-ca-cb)
    \\\Leftarrow &\frac{ca^2b}{(2-cb)(2-ca-cb)} + \left(1 - \frac{1}{2}(2 - ca -cb)\right) \frac{a^2(1-cb)}{2-ca-cb} \ge \frac{1}{3}ca^3
    \\ \Leftarrow &\frac{ca^2b}{(2-cb)(2-ca-cb)} + \frac{ca^2(a+b)(1-cb)}{2(2-ca-cb)} \ge \frac{1}3 ca^3
    \\ \Leftarrow &\frac{ca^2b}{(2-cb)(2-ca-cb)} + \frac{ca^2b(1-cb)}{2(2-ca-cb)} \ge \frac{1}3 ca^2b
    \\ \Leftarrow & \frac{1}{(2-cb)^2} + \frac{1-cb}{2(2-cb)} \ge \frac{1}{3} 
    \\ \Leftarrow &3(1-cb)(2-cb) + 6 \ge 2(2-cb)^2 
    \\ \Leftarrow &(cb)^2 - cb + 4 \ge 0
\end{align*}
\end{proof}

\begin{lemma}
\label{lem:norminequal2}

When $0 < a < b, 0 < c < \frac{b-a}{b^2},  a\sqrt{\frac{a^2 + 2b^2}{2(1-cb)}} \ge \frac{a^2+b^2}{2 - ca -cb}$, we have
\begin{align*}
    cb^2  + ca^2 (2 - cb - \frac23 ca) - (1-cb)cb^2 - ca^2(\frac{1}{2}a^2 + b^2)\frac{1}{b^2} \le \frac{cb^2}{2-cb}
\end{align*}

\end{lemma}

\begin{proof}[Proof of~\Cref{lem:norminequal2}]
Equivalently, we are going to prove,
\begin{align*}
    & cb^3 + a^2(2 - cb - \frac23 ca) \le \frac{b^2}{2-cb} + \frac{a^2(\frac{1}{2}a^2 + b^2)}{b^2} \\
    \iff & cb^3 + a^2(1 - cb - \frac23 ca) \le \frac{b^2}{2-cb} + \frac{a^4}{2b^2}
\end{align*}

We have the auxiliary inequality $\frac{1}{2-cb}> \frac{1}{2} + \frac{cb}{4}$.

Hence
\begin{align*}
    & cb^3 + a^2(1 - cb - \frac23 ca) \le \frac{b^2}{2-cb} + \frac{a^4}{2b^2} \\
    \Leftarrow &cb^3 + a^2(1 - cb - \frac23 ca) \le \frac{b^2}{2} + \frac{a^4}{2b^2} + \frac{cb^3}{4} \\
    \Leftarrow &c(\frac{3b^3}{4} - ba^2 - \frac23 a^3) \le \frac{b^2}{2} + \frac{a^4}{2b^2} - a^2
\end{align*}

\begin{itemize}
    \item [1] \textit{Case 1,} If $\frac{3b^3}{4} - ba^2 - \frac23 a^3 \le 0$, then 
    \begin{align*}
        c(\frac{3b^3}{4} - ba^2 - \frac23 a^3) \le 0 \le \frac{b^2}{2} + \frac{a^4}{2b^2} - a^2
    \end{align*}
    \item [2] \textit{Case 2,} If $\frac{3b^3}{4} - ba^2 - \frac23 a^3 > 0$, then 
    \begin{align*}
        &c(\frac{3b^3}{4} - ba^2 - \frac23 a^3) \le \frac{b^2}{2} + \frac{a^4}{2b^2} - a^2 \\
        \Leftarrow &\frac{b-a}{b^2}(\frac{3b^3}{4} - ba^2 - \frac23 a^3) \le \frac{(b^2 - a^2)^2}{2b^2} \\
        \Leftarrow & 2(\frac{3b^3}{4} - ba^2 - \frac23 a^3) \le (b-a)(b+a)^2 \\
        \Leftarrow& 2(b^3 - ba^2) - (b-a)(b+a)^2 \le \frac{b^3}{2} + \frac{4a^3}{3} \\
        \Leftarrow & (b-a) (2b(a+b) - (a+b)^2) \le \frac{b^3}{2} + \frac{4a^3}{3} \\
        \Leftarrow & (b-a)^2(b+a) \le \frac{b^3}{2} + \frac{4a^3}{3} 
    \end{align*}
    
    Using \Cref{lem:regularlambda},$a > \frac{b}{2}$,$(b-a)^2(b+a) = (b^2 -a^2) (b-a) \le b^2(b-a) \le \frac{b^3}{2}$
\end{itemize}

\end{proof}

\begin{lemma}
\label{lem:norminequal3}
When $0  \le a \le b, 0 \le c \le \frac{b-a}{b^2}, b^2 \ge a\sqrt{\frac{a^2 + 2b^2}{2(1-cb)}} \ge \frac{a^2+b^2}{2 - ca -cb} $, we have 
\begin{align*}
    cb^2 + ca^2(2-cb-\frac{2}{3}ca) - 2ca\sqrt{(b^2 + \frac{1}{2}a^2)(1-cb)} \le \frac{cb^2}{2-cb}
\end{align*}

\end{lemma}

\begin{proof}[Proof of~\Cref{lem:norminequal3}]

Define
\begin{align*}
    F(a) &\triangleq  a^2(2-cb-\frac{2}{3}ca) - 2a\sqrt{(b^2 + \frac{1}{2}a^2)(1-cb)}\\
S(c,b)&\triangleq  \{a | 0 \le a \le  b, 0 < c \le \frac{b-a}{b^2}, b^2 \ge a\sqrt{\frac{a^2 + 2b^2}{2(1-cb)}} \ge \frac{a^2+b^2}{2 - ca -cb} \} \\
a_{min}(c,b) &\triangleq  \inf S(c,b) 
\\
a_{max}(c,b) &\triangleq  \sup S(c,b) \le b -cb^2
\end{align*}

Consider 
\begin{align*}
    \frac{dF(a)}{da} &= 2a(2-cb-\frac{2}{3}ca) - \frac{2}{3}ca^2 - 2\sqrt{(b^2 + \frac{1}{2}a^2)(1-cb)} - a^2 \sqrt{\frac{1-cb}{b^2 + \frac{1}{2}a^2}} \\
    \frac{d^2F(a)}{da^2} &= 2(2-cb-\frac{2}{3}ca) - \frac{4}{3}ca - \frac{4}{3}ca - a
    \sqrt{\frac{1-cb}{b^2 + \frac{1}{2}a^2}} - 2a \sqrt{\frac{1-cb}{b^2 + \frac{1}{2}a^2}} + \frac{a^3}{2(b^2 + \frac{1}{2}a^2)^{\frac{3}{2}}} \sqrt{1-cb} \\
    &\ge 4 - 2cb - 4ca - 3a \sqrt{\frac{1-cb}{b^2 + \frac{1}{2}a^2}}
\end{align*}

Define $u \triangleq  cb, v \triangleq  \frac{a}{b}$, then $u+v \le 1$.
\begin{align*}
    \frac{d^2F(a)}{da^2} &\ge 4 - 2u - 4uv - 3\sqrt{1-u} \frac{1}{\sqrt{\frac{1}{2}+\frac{1}{v^2}}} \\
    &\ge 4 - 2u - 4u(1-u) - 3\sqrt{1-u} \frac{1}{\sqrt{\frac{1}{2}+\frac{1}{(1-u)^2}}} \\
    &\ge 4u^2 - 6u + 4 - 3\sqrt{1-u} \frac{(1-u)}{\sqrt{\frac{(1-u)^2}{2}+1}} 
\end{align*}

As $\sqrt{\frac{(1-u)^2}{2}+1} \ge \sqrt{\frac{(1-u)^2 + 1}{2}} \ge \sqrt{1-u}$,we have 

\begin{align*}
    \frac{d^2F(a)}{da^2} \ge 4u^2 - 6u + 4 - 3(1-u) = 4u^2 + 1 - 3u > 0
\end{align*}

The above inequality shows that $F(a)$ is convex w.r.t to $a$ for $a_{min}(c,b) \le a \le a_{max}(c,b)$. 
Hence $F(a) \le \max\left(F(a_{min}(c,b)),F(a_{max}(c,b))  \right)$. 
Below we use $a_{min}$, $a_{max}$ as shorthands for $a_{min}(c,b)$,$a_{max}(c,b)$.

For $F(a_{min})$,    we have $a_{min}\sqrt{\frac{a_{min}^2 + 2b^2}{2(1-cb)}} = \frac{a_{min}^2+b^2}{2 - ca_{min} -cb}$. This implies 
    
    \begin{align*}
        2a_{min}\sqrt{(b^2 + \frac{1}{2}a_{min}^2)(1-cb)} = (1-cb) \frac{(a_{min}^2 +b^2)}{2-ca_{min}-cb} + a_{min}^2(\frac{1}{2}a_{min}^2 + b^2)\frac{2-ca_{min}-cb}{(a_{min}^2+b^2)}
    \end{align*}

    Hence using \Cref{lem:technorminequal},
    \begin{align*}
        F(a_{min}) &= a_{min}^2(2-cb-\frac{2}{3}ca_{min}) - (1-cb) \frac{c(a_{min}^2 +b^2)}{2-ca_{min}-cb} - ca_{min}^2(\frac{1}{2}a_{min}^2 + b^2)\frac{2-ca_{min}-cb}{(a_{min}^2+b^2)}\\
        &\le \frac{1}{c}(\frac{cb^2}{2-cb} - cb^2)
    \end{align*}

    For $F(a_{max})$, we know that $a_{max}$ must satisfy at least of the following three equalities and we discuss three cases one by one.
    
    \begin{enumerate}

        \item  $a_{max}\sqrt{\frac{a_{max}^2 + 2b^2}{2(1-cb)}} = \frac{a_{max}^2+b^2}{2 - ca_{max} -cb}$, in this case we simply redo the calculation in Part 1.
        \item   $b^2 = a_{max}\sqrt{\frac{a_{max}^2 + 2b^2}{2(1-cb)}} $. This implies
        $$
        2a_{max}\sqrt{(b^2 + \frac{1}{2}a_{max}^2)(1-cb)} = (1-cb)b^2 + a_{max}^2(\frac{1}{2}a_{max}^2 + b^2)\frac{1}{b^2}
        $$
        Hence using \Cref{lem:norminequal2},
        \begin{align*}
        F(a_{max}) &= a_{max}^2(2-cb-\frac{2}{3}ca_{max}) - (1-cb)cb^2 - ca_{max}^2(\frac{1}{2}a_{max}^2 + b^2)\frac{1}{b^2}\\
        &\le \frac{1}{c}(\frac{cb^2}{2-cb} - cb^2)
    \end{align*}
    \item $cb^2 = b - a_{max}$. Define $v \triangleq  \frac{a_{max}}{b}, cb = 1-v$.
    Note that  $1 - cb = \frac{a_{max}}{b}$ and 
$b^2 \ge a_{max} \sqrt{\frac{b(a_{max}^2 + 2b^2)}{2a_{max}}}$. These imply $a_{max}^3 + 2a_{max}b^2 - 2b^3\le 0 \Rightarrow a_{max} < 0.9 b$. This implies $v \le 0.9$. By \Cref{lem:regularlambda}, $0.5 \le v$.

As $v \in [0.5,0.9]$, it holds that
\begin{align*}
     v(1+v) + \frac{1}{1+v} \le 2\sqrt{(1+\frac{v^2}2)v}.
\end{align*}
This implies 
\begin{align*}
    v^2(2- (1-v) - \frac23 (1-v)v ) - 2v \sqrt{(1+\frac{v^2}2)v} \le \frac{-v}{1+v}  =  \frac{1}{2-cb} - 1 
\end{align*} 

Finally,
\begin{align*}
    F(a_{max}) &=  a_{max}^2(2-cb-\frac{2}{3}ca_{max}) - 2a_{max}\sqrt{(b^2 + \frac{1}{2}a_{max}^2)(1-cb)} \\
    &= b^2 \left ( v^2(2- (1-v) - \frac23 (1-v)v ) - 2v \sqrt{(1+\frac{v^2}2)v}\right) \le b^2 (\frac{1}{2-cb} - 1) \,.
\end{align*}

\end{enumerate}

In conclusion, it holds that,
\begin{align*}
    F(a) &\le \max\left(F(a_{min}(c,b)),F(a_{max}(c,b))  \right) \\
    &\le \frac{1}{c}(\frac{cb^2}{2-cb} - cb^2).
\end{align*}
\end{proof}
\section{Omitted Proofs on Continuous Approximation}\label{sec:continuous_approximation}

In this section we give a general approximation result (\Cref{thm:ode_approximation}) between a continuous-time flow~(\Cref{eq:general_continuous}) and a discrete-time (stochastic) iterates~~(\Cref{eq:general_discrete}) in some compact subset of $\R^D$, denoted by $K$. This result is used multiple times in our analysis for full-batch SAM and 1-SAM. \footnote{Though we believe this approximation result is folklore, we cannot find a reference under the exact setting as ours. For completeness, we provide a quick proof in this section.} Let $b:K\to \R^D$ is a $C_1$-lipschitz function, that is, $\forall x,x'\in K$, it holds that $\norm{b(x)-b(x')}_2 \le C_1 \norm{x-x'}_2$. Let $ {b}_k$ be mappings from $K$ to $\R^D$ for  $k\in[M]$ satisfying that $b(x) = \frac{1}{M}\sum_{k=1}^M b_k(x)$ for all $x\in K$.

We consider the continuous-time flow $X:[0,T]\to K$, which is the unique solution of
\begin{align}\label{eq:general_continuous}
\diff X(\tau) = b(X(\tau))\diff \tau.
\end{align} 
and the discrete-time iterate $\{x(t)\}_{t\in\mathbb{N}}$ which approximately satisfy
\begin{align}\label{eq:general_discrete}
x(t+1) \approx x(t) + p {b}_{k_t}(x(t)),
\end{align}
where $k_t$ is independently sampled from uniform distribution over $[M]$ for each $t\in\mathbb{N}$ and $x(t)$ is a deterministic function of $k_0,\ldots,k_{t-1}$. We use $\mathcal{F}_t$ to denote the $\sigma$-algebra generated by $k_0,\ldots,k_{t-1}$ and $\mathcal{F}_*$ to denote the filtration $(\mathcal{F}_t)_{t\in\mathbb{N}}$. Thus $x(t)$ is adapted to filtration $\mathcal{F}_*$. Note ${b}$ is  undefined outside $K$, thus in the analysis we only consider the process stopped immediately leaving $K$, that is, $x^K(t)\triangleq x(\min(t,t_K))$, where $t_K\triangleq \{t'\in\mathbb{N}\mid x(t')\notin K\}$. If $x(t)$ is  in $K$ for all $t\ge 0$, then $t_K=\infty$. It is easy to verify that $t_K$ is a stopping time with respect to the filtration $\mathcal{F}_*$. For convenience, we denote $X_K(\tau) = X(\min(\tau, pt_K))$ as the stopped continuous counterpart of $x^K$.

\begin{theorem}\label{thm:ode_approximation}
Suppose there exist constants $C_2, \eps, \eps>0$ satisfying that 
\begin{enumerate}
	\item $\norm{b_k(x)}_2\le C_2$, for any $x\in K$ and $k\in [M]$; 
        \item $\norm{b_k(x) - b(x)}_2\le C_3$, for any $x\in K$ and $k\in [M]$;
	\item $\norm{b_{k_t}(x(t)) - \frac{x(t+1) - x(t)}{p}}_2 \le \eps$, for all $t$.
\end{enumerate}
	
	Then for any integer $0\le k \le T/p$ and $0<\delta<1$, with probability at least $1-\delta$, it holds that 
\begin{align}
	&\max_{0\le t\le T/p} \norm{x^K(t) - X^K(pt)}
	\le H_{p,\delta}  e^{C_1 T}, \notag
\end{align}
where $H_{p,\delta}\triangleq \norm{x(0) - X(0)}_2 + C_1C_2Tp + 2C_3 \sqrt{p T\log \frac{2e T}{\delta p}}  + \eps T $.
\end{theorem}

\begin{proof}[Proof of \Cref{thm:ode_approximation}]
	Summing up \Cref{eq:general_continuous} and \Cref{eq:general_discrete}, for any $t\le t_K$, we have that 
	\begin{align}\label{eq:general_continuous_integral}
		X(pt)	- X(0) = \int_{\tau = 0}^{pt} b(X(\tau))\diff \tau, 
	\end{align}
and that 
	\begin{align}\label{eq:general_discrete_integral}
		x(t)	- x(0) = \sum_{t'=0}^{t-1} x(t' + 1) - x(t')
	\end{align}

Denote $\norm{x(t)- X(pt)}_2$ by $E_t$, we have that for $t\le t_K$,
\begin{align}
&E_t- E_0 \notag \\
\le &  \norm{ \int_{\tau = 0}^{pt} b(X(\tau))\diff \tau - \sum_{t'=0}^{t-1} (x(t'+1) - x(t'))}_2 \notag \\
	\le &  \underbrace{\norm{ \int_{\tau = 0}^{pt} b(X(\tau))\diff \tau - p\sum_{t'=0}^{t-1} b(X(pt'))}_2}_{(A)} 
	 +  \underbrace{\norm{  p\sum_{t'=0}^{t-1} b(X(pt')) - p\sum_{t'=0}^{t-1} b(x(t'))}_2}_{(B)} \notag \\
	 + & \underbrace{\norm{  p\sum_{t'=0}^{t-1} b(x(t')) - p\sum_{t'=0}^{t-1} b_{k_{t'}}(x(t'))}_2}_{(C)} 
 	 +  \underbrace{\norm{  p\sum_{t'=0}^{t-1} b_{k_{t'}}(x(t')) - \sum_{t'=0}^{t-1}(x(t'+1) - x(t')) }_2}_{(D)}. \label{eq:ode_approximation_main_decomposition_rhs}
\end{align}

Below we will proceed by bounding the four terms (A), (B), (C) and (D) in \Cref{eq:ode_approximation_main_decomposition_rhs}. 
\begin{enumerate}
	\item Note that for any $0\le \tau\le \tau' \le T$, we have that 
\begin{align*}
\norm{X(\tau)-X(\tau')}_2 = \norm{\int_{s=\tau} ^{\tau'} b(X(s))\diff s}_2 \le \int_{s=\tau} ^{\tau'} \norm{ b(X(s))}_2 \diff s \le (\tau'-\tau) C_2.	
\end{align*}
Thus, by $C_1$-lipschitzness of $b$,
\begin{align*}
(A)
= & \norm{ \int_{\tau = 0}^{pt} b(X(\tau)) - b(X(\lfloor \tau/p\rfloor p))\diff \tau }_2
\le  \int_{\tau = 0}^{pt} \norm{  b(X(\tau)) - b(X(\lfloor \tau/p\rfloor p)) }_2\diff \tau\\
 \le & C_1C_2 p^2t \le C_1C_2 pT.
\end{align*}

\item By definition of $E_t$ and $C_1$-lipschitzness of $b$, we have that $(B)\le C_1p\sum_{t'=0}^{t-1} E_{t'}$.

\item We claim that for any $0< \delta<1 $, we have that for probability at least $1-\delta$, it holds that 
\begin{align}
	(C)\le 2C_3 \sqrt{2pT\log\frac{2eT}{\delta p}}.
\end{align}

Below we prove our claim. We denote $  p\sum_{t'=0}^{\min(t,t_K)-1}  b(x(t')) -  p\sum_{t'=0}^{\min(t,t_K)-1} b_{k_{t'}}(x(t'))$ by $S_t$, which is a martingale with respect to filtration $\mathcal{F}_*$, since $t_K$ is a stopping time. Note 
\begin{align*}
\norm{S_t-S_{t+1}}_2 \le \max_{k\in [M],x\in K} \norm{b(x)-b_{k}(x)}_2 \le C_3,	
\end{align*}

by Azuma-Hoeffding's inequality (vector form, \Cref{lem:azuma_vector}), it holds that for any $0\le t\le T/p$ and $0\le \delta\le 1$, with probability at least $1-\delta$,
\begin{align*}
	\norm{S_t}_2 \le 2C_3p\sqrt{2t\log\frac{2e}{\delta}}.
\end{align*}

Applying an union bound on the above inequality over $t=0,\ldots,\lfloor T/p\rfloor-1$, we conclude that with probability at least $1-\delta$, $(C)\le 2C_3 p\sqrt{2T/p\log\frac{2eT}{\delta p}} = 2C_3\sqrt{2Tp\log\frac{2eT}{\delta p}}$.

\item We have that 
\begin{align*}
	(D) \le p\sum_{t'=0}^{t-1} \norm{b_{k_{t'}}(x(t')) -  \frac{x(t'+1) - x(t')}{p}} \le pt \eps \le \eps T.
\end{align*}
\end{enumerate}

Combining the above upper bounds for (A), (B), (C) and (D), we conclude that for any $0\le t\le \min(T/p, t_K)$,
\begin{align}\label{eq:to_apply_gronwall}
	E_t \le H_{p,\delta} + C_1p\sum_{t'=0}^{t-1}E_{t'}.
\end{align}

Applying the discrete gronwall inequality (\Cref{lem:discrete_gronwall}) on \Cref{eq:to_apply_gronwall}, we have that 
\begin{align*}
	E_t \le H_{p,\delta}	e^{C_1pt} \le H_{p,\delta}	e^{C_1T},
\end{align*}
which completes the proof.
\end{proof}

\begin{corollary}\label{cor:ode_approximation}
	If  $\min_{0\le \tau\le T}\dist(X(\tau), \R^D\setminus K) > H_{p,\delta}  e^{C_1 T}$, then with probability at least $1-\delta$,  $t_K > \lfloor T/p\rfloor $ and therefore
\begin{align}
	&\max_{0\le t\le T/p} \norm{x(t) - X(pt)}
	\le H_{p,\delta}  e^{C_1 T}. \notag
\end{align}
\end{corollary}

\begin{proof}[Proof of \Cref{cor:ode_approximation}]
	By \Cref{thm:ode_approximation}, we know with probability at least $1-\delta$, we have that 
	\begin{align}
	&\max_{0\le t\le T/p} \norm{x^K(t) - X^K(pt)}
	\le H_{p,\delta}  e^{C_1 T}. \notag
\end{align}

Therefore  $\dist(x^K(t),\R^D\setminus K) \ge \dist(X^K(pt),\R^D\setminus K)  - \dist (X^K(pt), x^K(t))>0$ for any $0\le t\le T/p$, which implies $x^K(t) \notin \R^D\setminus K$, or equivalently, $x^K(t)\in K$. Thus we conclude that $t_K\ge \lfloor T/p\rfloor$.
\end{proof}

\begin{corollary}
    \label{thm:deterministic_ode_approximation}
Suppose $M=1$ and there exist constants $C_2, \eps>0$ satisfying that 
\begin{enumerate}
	\item $\norm{b(x)}_2\le C_2$ for any $x\in K$;
	\item $\norm{b(x)- \frac{x(t+1) - x(t)}{p}}\le \eps$, for all $x\in K$.
\end{enumerate}
	Then for any $k\in\mathbb{N}$ such that $kp\le T$, it holds that 
\begin{align}
	&\max_{0\le t\le T/p} \norm{x^K(t) - X^K(pt)}
	\le H_{p}  e^{C_1 T}, \notag
\end{align}
where $H_{p}\triangleq \norm{x(0) - X(0)}_2 + C_1C_2Tp + \eps T $.

Therefore, similar to \Cref{cor:ode_approximation}, if  $\min_{0\le \tau\le T}\dist(X(\tau), \R^D\setminus K) > H_{p}  e^{C_1 T}$, then it holds that  $t_K > \lfloor T/p\rfloor $ and that 
\begin{align}
	&\max_{0\le t\le T/p} \norm{x(t) - X(pt)}
	\le H_{p,\delta}  e^{C_1 T}. \notag
\end{align}
\end{corollary}

\begin{proof}[Proof of~\Cref{thm:deterministic_ode_approximation}]

For any $\delta\in (0,1]$, choosing $C_3 = 0$ and 
by~\Cref{thm:ode_approximation}, we have that
\begin{align*}
  \prob\left[  \max_{0\le t\le T/p} \norm{x^K(t) - X^K(pt)}
	\le H_{p}  e^{C_1 T}\right] \ge 1-\delta\,.
\end{align*}

Since $\delta$ can be any number in $(0,1]$, the above probability is exactly $1$. 
\end{proof}

We end this section with a summary of applications of~\Cref{thm:ode_approximation,thm:deterministic_ode_approximation} in our proofs (\Cref{tb:continuous_approximation}).

 \begin{table}[h]
 \centering
 \begin{tabular}{l|llll}
 Setting & $p$ & $b_k$  & $\eps$ \\
  \hline
 Full-batch SAM, Phase I (\Cref{lem:nsam_gf}) & $\eta$  & $-\nabla L(\cdot)$  & $\rho$\\
 Full-batch SAM, Phase II (\Cref{thm:nsamphase2}) & $\eta\rho^2$ & $-\partial \Phi(\cdot)\nabla \lambda_1(\nabla^2 L(\cdot))/2$ &   $\rho + \eta$ \\
 1-SAM, Phase I (\Cref{lem:1sam_gf}) & $\eta$ & $-\nabla L_k(\cdot)$   & $\rho$\\
 1-SAM, Phase II (\Cref{thm:1samphase2}) & $\eta\rho^2$ &  $-\partial \Phi(\cdot)\nabla \mathrm{Tr}(\nabla^2 L_k(\cdot))/2$   & $\rho + \eta$
 \end{tabular}
 \caption{Summary of applications of~\Cref{thm:ode_approximation,thm:deterministic_ode_approximation} in our analysis}\label{tb:continuous_approximation}
 \end{table}

\end{document}